\documentclass[11pt,letterpaper]{article}

\usepackage{microtype}
\usepackage[
letterpaper,
top=1.2in,
bottom=1.2in,
left=1in,
right=1in]{geometry}
\usepackage{etex}
\usepackage{nicefrac}
\usepackage{bbm}
\usepackage{multicol}
\usepackage{mathtools}  

\usepackage{amsmath,amsfonts,amssymb,amsthm} 
\usepackage{float}
\usepackage{bm}

\usepackage{epstopdf}
\usepackage{epsfig}
\usepackage{xspace}
\usepackage{multirow}
\makeatletter
\newcommand{\doublehat}[1]{%
\begingroup%
  \let\macc@kerna\z@%
  \let\macc@kernb\z@%
  \let\macc@nucleus\@empty%
  \hat{\raisebox{.35ex}{\vphantom{\ensuremath{#1}}}\smash{\hat{#1}}}%
\endgroup%
}
\makeatother

\usepackage{hhline}
\usepackage[export]{adjustbox}
\usepackage{csvsimple}
\usepackage{turnstile}

\usepackage{graphicx}
\usepackage[colorinlistoftodos]{todonotes}
\newcommand{\journalVersion}[1]{}
\usepackage{xr}
\usepackage{titlesec}

\usepackage{comment}
\usepackage{siunitx}
\usepackage{relsize}
\usepackage{ifthen}
\usepackage[colorinlistoftodos]{todonotes}

\usepackage[caption=false]{subfig}

\usepackage[vlined,ruled,linesnumbered]{algorithm2e}
\usepackage{graphics} 
\usepackage{rotating}
\usepackage{color}
\usepackage{enumerate}
\usepackage[T1]{fontenc}
\usepackage{psfrag}
\usepackage{epsfig} 
\usepackage{booktabs}
\usepackage{graphicx,url}
\usepackage{multirow}
\usepackage{array}
\usepackage{latexsym}
\usepackage{amsfonts}
\usepackage{amsmath}
\usepackage{amssymb}
\usepackage{xstring}
\usepackage[noend]{algorithmic}
\usepackage{multirow}
\usepackage{xcolor}
\usepackage{prettyref}
\usepackage{flexisym}
\usepackage{bigdelim}
\usepackage{breqn} 
\usepackage{listings}
\let\labelindent\relax
\usepackage{enumitem}
\usepackage{xspace}
\usepackage{bm}
\graphicspath{{./figures/}}
\usepackage{tikz}
\usetikzlibrary{matrix,calc}

\usepackage{mdwlist}

\let\itemize\undefined
\makecompactlist{itemize}{stditemize}

\usepackage[capitalise]{cleveref}

\newrefformat{prob}{Problem\,\ref{#1}}
\newrefformat{def}{Definition\,\ref{#1}}
\newrefformat{sec}{Section\,\ref{#1}}
\newrefformat{sub}{Section\,\ref{#1}}
\newrefformat{prop}{Proposition\,\ref{#1}}
\newrefformat{app}{Appendix\,\ref{#1}}
\newrefformat{alg}{Algorithm\,\ref{#1}}
\newrefformat{cor}{Corollary\,\ref{#1}}
\newrefformat{thm}{Theorem\,\ref{#1}}
\newrefformat{lem}{Lemma\,\ref{#1}}
\newrefformat{fig}{Fig.\,\ref{#1}}
\newrefformat{tab}{Table\,\ref{#1}}

\newtheorem{theorem}{Theorem}
\newtheorem{problem}{Problem}

\newtheorem{lemma}[theorem]{Lemma}
\newtheorem{assumption}[theorem]{Assumption}
\newtheorem{definition}[theorem]{Definition}
\newtheorem{proposition}[theorem]{Proposition}
\newtheorem{remark}[theorem]{Remark}
\newtheorem{example}[theorem]{Example}
\newtheorem{fact}[theorem]{Fact}

\newcommand{\cf}{\emph{cf.}\xspace}

\newcommand{\bdmath}{\begin{dmath}}
\newcommand{\edmath}{\end{dmath}}
\newcommand{\beq}{\begin{equation}}
\newcommand{\eeq}{\end{equation}}
\newcommand{\bdm}{\begin{displaymath}}
\newcommand{\edm}{\end{displaymath}}
\newcommand{\bea}{\begin{eqnarray}}
\newcommand{\eea}{\end{eqnarray}}
\newcommand{\beal}{\beq \begin{array}{ll}}
\newcommand{\eeal}{\end{array} \eeq}
\newcommand{\beas}{\begin{eqnarray*}}
\newcommand{\eeas}{\end{eqnarray*}}
\newcommand{\ba}{\begin{array}}
\newcommand{\ea}{\end{array}}
\newcommand{\bit}{\begin{itemize}}
\newcommand{\eit}{\end{itemize}}
\newcommand{\ben}{\begin{enumerate}}
\newcommand{\een}{\end{enumerate}}

\newcommand{\calA}{{\cal A}}
\newcommand{\calB}{{\cal B}}
\newcommand{\calC}{{\cal C}}
\newcommand{\calD}{{\cal D}}
\newcommand{\calE}{{\cal E}}

\newcommand{\calI}{{\cal I}}
\newcommand{\calJ}{{\cal J}}

\newcommand{\calL}{{\cal L}}
\newcommand{\calM}{{\cal M}}

\newcommand{\calO}{{\cal O}}

\newcommand{\calS}{{\cal S}}
\newcommand{\calT}{{\cal T}}

\newcommand{\setal}{~\emph{et~al.}\xspace}
\newcommand{\eg}{\emph{e.g.,}\xspace}
\newcommand{\ie}{\emph{i.e.,}\xspace}
\newcommand{\myParagraph}[1]{{\bf #1.}\xspace}

\newcommand{\M}[1]{{\bm #1}} 
\renewcommand{\boldsymbol}[1]{{\bm #1}}

\newcommand{\hide}[1]{}

\newcommand{\tocheck}[1]{{\color{brown} #1}}
\newcommand{\grayout}[1]{{\color{gray} #1}}

\newcommand{\hiddenText}{{\color{gray} hidden text.}}
\newcommand{\hideWithText}[1]{\hiddenText}

\newcommand{\kron}{\otimes}
\newcommand{\dist}{\mathbf{dist}}

\newcommand{\Natural}[1]{ { {\mathbb N}^{#1} } }

\newcommand{\subject}{\text{ subject to }}
\DeclareMathOperator*{\argmax}{arg\,max}
\DeclareMathOperator*{\argmin}{arg\,min}

\newcommand{\E}{{\mathbb{E}}}

\newcommand{\prob}[1]{{\mathbb P}\left(#1\right)}
\newcommand{\tran}{^{\mathsf{T}}}

\newcommand{\trace}[1]{\mathrm{tr}\left(#1\right)}

\newcommand{\rank}[1]{\mathrm{rank}\left(#1\right)}

\newcommand{\inv}{^{-1}}

\newcommand{\ones}{{\mathbf 1}}
\newcommand{\zero}{{\mathbf 0}}
\newcommand{\eye}{{\mathbf I}}
\newcommand{\vect}[1]{\left[\begin{array}{c}  #1  \end{array}\right]}
\newcommand{\matTwo}[1]{\left[\begin{array}{cc}  #1  \end{array}\right]}
\newcommand{\matThree}[1]{\left[\begin{array}{ccc}  #1  \end{array}\right]}

\newcommand{\Real}[1]{ { {\mathbb R}^{#1} } }

\newcommand{\opt}{^{\star}}

\newcommand{\at}[1]{^{(#1)}}

\newcommand{\setdef}[2]{ \{#1 \; {:} \; #2 \} }

\newcommand{\SOthree}{\ensuremath{\mathrm{SO}(3)}\xspace}

\newcommand{\MA}{\M{A}}
\newcommand{\MB}{\M{B}}
\newcommand{\MC}{\M{C}}

\newcommand{\MK}{\M{K}}

\newcommand{\MM}{\M{M}}

\newcommand{\MP}{\M{P}}

\newcommand{\MR}{\M{R}}
\newcommand{\MS}{\M{S}}

\newcommand{\MV}{\M{V}}

\newcommand{\MH}{\M{H}}

\newcommand{\MX}{\M{X}}

\newcommand{\va}{\boldsymbol{a}} 
 
\newcommand{\vb}{\boldsymbol{b}}
\newcommand{\vc}{\boldsymbol{c}}

\newcommand{\vq}{\boldsymbol{q}}

\newcommand{\vs}{\boldsymbol{s}}
\newcommand{\vu}{\boldsymbol{u}}
\newcommand{\vv}{\boldsymbol{v}}
\newcommand{\vt}{\boldsymbol{t}}
\newcommand{\vxx}{\boldsymbol{x}} 
\newcommand{\vy}{\boldsymbol{y}}

\newcommand{\vgamma}{\boldsymbol{\gamma}}

\newcommand{\valpha}{\boldsymbol{\alpha}}
\newcommand{\vbeta}{\boldsymbol{\beta}}

\newcommand{\vepsilon}{\boldsymbol{\epsilon}}

\newcommand{\scenario}[1]{{\smaller \sf#1}\xspace}

\newcommand{\cvx}{{\sf cvx}\xspace}

\newcommand{\blue}[1]{{\color{blue}#1}}

\newcommand{\linkToPdf}[1]{\href{#1}{\blue{(pdf)}}}
\newcommand{\linkToPpt}[1]{\href{#1}{\blue{(ppt)}}}
\newcommand{\linkToCode}[1]{\href{#1}{\blue{(code)}}}
\newcommand{\linkToWeb}[1]{\href{#1}{\blue{(web)}}}
\newcommand{\linkToVideo}[1]{\href{#1}{\blue{(video)}}}
\newcommand{\linkToMedia}[1]{\href{#1}{\blue{(media)}}}
\newcommand{\award}[1]{\xspace} 

\newcommand{\vz}{\boldsymbol{z}}

\newcommand{\ransac}{RANSAC\xspace}

\newcommand{\SLIDESlong}{Sparse LIst-Decodable Estimation\xspace} 
\newcommand{\SLIDES}{SLIDE\xspace} 

\newcommand{\polyring}[1]{\mathbb{R}[#1]}
\newcommand{\binomial}[2]{\underline{#1}_{#2}}

\newcommand{\paper}{monograph\xspace}
\newcommand{\Paper}{Monograph\xspace}

\newcommand{\MRout}{\MR^\prime}
\newcommand{\mth}{th\xspace} 

\newcommand{\fstar}{f^\star}
\newcommand{\vxxstar}{\vxx^{\star}}
\newcommand{\pstar}{p^\star}
\newcommand{\nchoosek}[2]{\left(\substack{#1 \\ #2} \right)}

\newcommand{\meq}{\mathrm{eq}}
\newcommand{\mathmom}{\mathrm{mom}}
\newcommand{\mathreq}{\mathrm{req}}
\newcommand{\mathloc}{\mathrm{loc}}

\newcommand{\sym}[1]{\mathbb{S}^{#1}}
\newcommand{\pd}[1]{\sym{#1}_{++}}
\newcommand{\psd}[1]{\sym{#1}_{+}}

\renewcommand{\int}{\mathbb{Z}}

\renewcommand{\deg}[1]{\mathrm{deg}\parentheses{#1}}
\newcommand{\bbX}{\mathbb{X}_\scenario{sdp}}

\newcommand{\ceil}[1]{\left\lceil #1 \right\rceil}
\newcommand{\floor}[1]{\left\lfloor #1 \right\rfloor}

\newcommand{\barc}{\bar{c}}
\newcommand{\barcsq}{\barc^2}
\newcommand{\vectorize}[1]{{\mathrm{vec}}\left(#1\right)}

\newcommand{\mosek}{\scenario{MOSEK}}

\newcommand{\bmat}{\left[ \begin{array}}
\newcommand{\emat}{\end{array} \right]}

\newcommand{\bal}{\begin{align}}
\newcommand{\eal}{\end{align}}

\newcommand{\trinum}{\mathfrak{t}}

\newcommand{\MXstar}{\MX^\star}

\newcommand{\parentheses}[1]{\left(#1\right)}

\newcommand{\mmom}{m_{\mathrm{mom}}}

\newcommand{\revise}[1]{#1}
\newcommand{\vcat}{\,;\,} 
\newcommand{\hcat}{\,,\,} 

\newcommand{\maybeOmit}[1]{} 

\newcommand{\inprod}[2]{\left\langle #1, #2 \right\rangle}

\newcommand{\cbrace}[1]{\left\{ #1\right\}}
\newcommand{\sbracket}[1]{\left[ #1\right]}

\newcommand{\geometricperception}{geometric perception\xspace}
\newcommand{\nrMeasurements}{n}

\newcommand{\SOd}{\ensuremath{\mathrm{SO}(d)}\xspace}
\newcommand{\SEd}{\ensuremath{\mathrm{SE}(d)}\xspace}
\renewcommand{\vectorize}[1]{\ensuremath{vec}(#1)}

\newcommand{\Domain}{\mathbb{\MX}}
\newcommand{\sumOverMeas}{\sum_{i=1}^{\nrMeasurements}}
\newcommand{\sumOverMeasj}{\sum_{j=1}^{\nrMeasurements}}
\newcommand{\vomega}{\M{\omega}}

\newcommand{\normOne}[1]{\left\| #1 \right\|_1}
\newcommand{\normTwo}[1]{\left\| #1 \right\|_2}
\newcommand{\normFrob}[1]{\left\| #1 \right\|_\scenario{F}}

\newcommand{\normInf}[1]{\left\| #1 \right\|_\infty}

\newcommand{\nondegenerate}{nondegenerate\xspace}
\newcommand{\Nondegenerate}{Nondegenerate\xspace}

\newcommand{\nondegeneracy}{nondegeneracy\xspace}

\newcommand{\minDim}{\bar{d}}
\newcommand{\dimy}{d_y}
\newcommand{\dimx}{d_x}
\newcommand{\dimd}{\dimx}
\newcommand{\dimJ}{d_\calJ}
\newcommand{\dimAmbient}{d_x}
\newcommand{\dimA}{\dimx \times \dimy}
\newcommand{\maxCon}{_\scenario{MC}}
\newcommand{\maxConi}[1]{_{\scenario{MC},#1}}
\newcommand{\ls}{_\scenario{LS}}
\newcommand{\tls}{_\scenario{TLS}}
\newcommand{\lts}{_\scenario{LTS}}
\newcommand{\ltssdp}{_\scenario{lts-sdp1}}
\newcommand{\ltssdpT}{_\scenario{lts-sdp2}}
\newcommand{\mcsdp}{_\scenario{mc-sdp}}
\newcommand{\tlssdp}{_\scenario{tls-sdp}}
\newcommand{\gt}{^\circ}
\newcommand{\uncorr}{\opt}
\newcommand{\InlierSet}{\calI}

\newcommand{\setInliers}{\calI}
\newcommand{\err}{\scenario{err}}
\newcommand{\errFeasible}[3]{\err(#2,#3)}

\newcommand{\optt}{\scenario{opt}}
\newcommand{\opttsos}{\widehat{\optt}\ltssdp}
\newcommand{\DistInliers}{{\widetilde \calI}}
\newcommand{\DistSampleInliers}{\dddot{\calI}}

\newcommand{\outlierRate}{\beta}
\newcommand{\inlierRate}{\alpha}

\newcommand{\Expect}[2]{\E_{#1}\left[#2\right]}
\renewcommand{\dist}{\mu}
\newcommand{\pdist}{\tilde{\mu}}
\newcommand{\levelpd}{\ell}
\newcommand{\relaxOrder}{r}
\newcommand{\newSize}{t}
\newcommand{\satisfyOrder}{k}
\newcommand{\satisfyOrderTwo}{k'}
\newcommand{\pExpect}[2]{\tilde{\E}_{#1}\left[#2\right]}
\newcommand{\support}{\text{support}}

\newcommand{\sosimply}[2]{\sststile{#1}{#2}} 
\newcommand{\psatisfy}[2]{\sdtstile{#1}{#2}} 

\newcommand{\oneOverNrMeas}{\frac{1}{\nrMeasurements}}
\newcommand{\aveOverMeas}{\oneOverNrMeas\sumOverMeas}

\newcommand{\residuali}[1]{\normTwo{\vy_i - \MA_i\tran #1}}
\newcommand{\residualUncorri}[1]{\normTwo{\vy_i\uncorr - (\MA_i\uncorr)\tran #1}}

\newcommand{\HoldersInequality}{H\"older's inequality\xspace}
\newcommand{\CS}{Cauchy-Schwarz\xspace}
\newcommand{\sos}{sos\xspace}
\newcommand{\Sos}{Sos\xspace}

\newcommand{\explanation}[1]{ \grayout{#1} \nonumber \\}
\newcommand{\textExplanation}[1]{ \grayout{\text{(#1)}} }

\newcommand{\degree}[1]{\text{deg}\left(#1\right)}

\newcommand{\List}{\calL}
\newcommand{\Longlist}{\doublehat{\calS}}

\newcommand{\axiomsLTS}{\calL_{\vomega,\vxx}}
\newcommand{\axiomsLTST}{\calT_{\vomega,\vxx}}
\newcommand{\axiomsLDR}{\calT_{\vomega,\vxx}}
\newcommand{\axiomsMC}{\calM_{\vomega,\vxx}}
\newcommand{\axiomsTLS}{\calM_{\vomega,\vxx}}

\newcommand{\sumOverInliers}{ \sum_{i\in \setInliers}}
\newcommand{\aveOverInliers}{\frac{1}{|\setInliers|} \sumOverInliers}

\newcommand{\sumOverOutliers}{ \sum_{i\in \setOutliers}}

\newcommand{\wt}{\text{wt}}
\newcommand{\an}{\inlierRate \, \nrMeasurements}
\newcommand{\bn}{\outlierRate \, \nrMeasurements}
\newcommand{\setOutliers}{\calO}

\newcommand{\positiv}{Positivstellens\"atze\xspace}
\newcommand{\momentMatrix}{moment matrix\xspace}

\newcommand{\naiveBoundMC}{
\frac{ 2 \sqrt{\dimJ} \;\barc }{ 
\min_{ \calJ \subset \InlierSet\maxCon, |\calJ|=\dimJ } \sigma_{\min(\MA_\calJ)}
 } }

 \newcommand{\naiveBoundTLS}{
\frac{ 2 \sqrt{\dimJ} \;\barc }{ 
\min_{ \calJ \subset \InlierSet\tls, |\calJ|=\dimJ } \sigma_{\min(\MA_\calJ)}
 } }

\newcommand{\indicator}[1]{{\mathbbm{1}}\!\left(#1\right)}
\newcommand{\boundOne}{\barc}

\newcommand{\relaxLevel}{k}

\newcommand{\boundx}{M_x}

\newcommand{\Crelax}{C\!\left(\nicefrac{\relaxLevel}{2}\right)}

\newcommand{\coeffCk}{\outlierRate^{\frac{\relaxLevel}{2}-1} 
 \Crelax^{\frac{\relaxLevel}{2}} 2^{ 3\relaxLevel-1} }

\newcommand{\getEntries}[2]{#1_{[#2]}}

\newcommand{\myskip}{\vspace{0mm}}
\renewcommand{\tocheck}[1]{#1}

\newcommand{\mathcom}{,}
\newcommand{\mathper}{.}

\newcommand{\antiConReq}{( \frac{\inlierRate^2 \eta^2 (1-2\barc)^2}{32 \barc} ,2\barc,2\boundx)}

\newcommand{\probBound}{1 - \left(1 - \frac{\inlierRate}{2} \right)^{ \frac{N}{\inlierRate} }}
\newcommand{\probBoundReal}{0.99}

\newcommand{\blke}[2]{_{[#1 \,,\, #2]}}
\newcommand{\blku}[1]{_{[#1]}}
\newcommand{\myspa}{\,&}
\newcommand{\vm}{\bm{m}}

\newcommand{\sosProofDef}[1]{Given a system of polynomial constraint $\calA$ and a polynomial $g$,  
a sum-of-square (\sos) proof that the system $\calA$ implies $g \geq 0$ 
consists of sum-of-squares polynomials $\{p_\calS\}_{ \calS \subseteq [m]}$ such that:

\vspace{-6mm}

\begin{align}
\label{#1}
g= \sum_{\calS\subseteq [m]} p_\calS \cdot \prod_{i\in\calS}f_i.
\end{align}
We say that the proof has degree $\satisfyOrder$ if for every $\calS \subseteq [m]$, $\degree{p_\calS \cdot \prod_{i\in\calS}f_i} \leq \satisfyOrder$ where $\degree{\cdot}$ denotes the degree of a polynomial.
We use the notation:
\begin{align}
\calA(\vxx) \sosimply{\vxx}{\satisfyOrder} \{ g(\vxx) \geq 0 \} 
\qquad \text{or}\qquad
\{f_i(\vxx) \geq 0,\ldots,f_m(\vxx) \geq 0\} \sosimply{\vxx}{\satisfyOrder} \{ g(\vxx) \geq 0 \} 
\end{align}
to denote that there is a proof of degree at most $\satisfyOrder$ of the fact that $\calA = \{f_i(\vxx) \geq 0,\ldots,f_m(\vxx) \geq 0\}$ implies $g \geq 0$ (\ie any $\vxx$ that satisfies $\calA(\vxx)$ 
is such that $g(\vxx)\geq 0$). 
We omit the variables and write $\calA(\vxx) \sosimply{}{\satisfyOrder} \{ g(\vxx) \geq 0 \} $, when they are clear from the context.
Moreover, we write 
\begin{align}
\sosimply{\vxx}{\satisfyOrder} \{ g(\vxx) \geq 0 \} 
\end{align} 
if there is a sum-of-squares proof that $g(\vxx)\geq 0$ for any $\vxx \in \Real{\dimd}$ (\ie $g(\vxx)$ is sum-of-squares).}

\newcommand{\ltsObjectiveThm}[1]{
Consider~\cref{prob:outlierRobustEstimation} with 
measurements ($\vy_i, \MA_i$), $i \in [\nrMeasurements]$, and known outlier rate $\beta$. 
Call $\DistSampleInliers$ the set of measurements ($\vy_i\uncorr, \MA_i\uncorr$), $i \in [\nrMeasurements]$, where the outliers are replaced by inliers
and assume that the set of matrices $\MA_i\uncorr$, $i \in \DistSampleInliers$, is $k$-certifiably $C$-hypercontractive with $k\geq 4$. Moreover, assume $\outlierRate < \beta_{\max} = \sqrt[\frac{k}{2}-1]{
1 / (\Crelax^{\frac{\relaxLevel}{2}} 2^{ 3\relaxLevel-1} )}$.
Then, \cref{algo:lts1} with relaxation order $\relaxOrder \geq \relaxLevel$ outputs an estimate 
 $\vxx\ltssdp$ (not necessarily in $\Domain$) such that:

 \vspace{-8mm}

\begin{align}
\hspace{-5mm}
\label{#1}
\err_{\DistSampleInliers} ( \vxx\ltssdp ) 
\leq
(1 + C_1(\relaxLevel,\outlierRate)^{\frac{2}{\relaxLevel}} )
\;  \optt_{\DistSampleInliers}+
C_2(\relaxLevel,\outlierRate)^{\frac{2}{\relaxLevel}}
\left(\frac{1}{\nrMeasurements} \sumOverMeas 
\normTwo{\vy_i\uncorr - (\MA_i\uncorr)\tran \vxx\opt}^{\relaxLevel} \right)^{\frac{2}{\relaxLevel}}
\mathcom
\end{align}
where $C_1(\relaxLevel,\outlierRate)$ and $C_2(\relaxLevel,\outlierRate)$ are given functions, 
 $\err_\DistSampleInliers(\vxx) \!\triangleq\!\aveOverMeas \residualUncorri{\vxx}^2$ 
is the residual error of an estimate $\vxx$ with respect to the inliers $\DistSampleInliers$, 
$\vxx\opt \triangleq  \argmin_{\vxx \in \Domain} \aveOverMeas \residualUncorri{\vxx}^2$ 
is the best estimate from an oracle estimator that has access to all the inliers,
and $\optt_{\DistSampleInliers} \triangleq \err_\DistSampleInliers(\vxx\opt)$ is the corresponding residual error with respect to the inliers $\DistSampleInliers$.
}

\newcommand{\ltsThm}[1]{
Consider~\cref{prob:outlierRobustEstimation} with 
measurements ($\vy_i, \MA_i$), $i \in [\nrMeasurements]$, and outlier rate $\outlierRate < 0.5$ (or, equivalently,  inlier rate $\inlierRate = 1-\outlierRate > 0.5$). 
Call $\setInliers$ the set of inliers 
and assume that the set of matrices $\MA_i$, $i \in \setInliers$, is $k$-certifiably $\antiConReq$-anti-concentrated for some $\eta >0$. 
Then, \cref{algo:lts2} with relaxation order $\relaxOrder \geq \relaxLevel/2$ outputs an estimate 
 $\vxx\ltssdpT$ (not necessarily in $\Domain$) such that:

\vspace{-3mm}

\beq
\label{#1}
\normTwo{\vxx\ltssdpT - \vxx\gt} \leq \boundx \left( \frac{\inlierRate \; \eta \;  }{2} +  2\frac{1-\inlierRate}{\inlierRate} \right).
\eeq
}

\newcommand{\mcThm}[1]{
Consider~\cref{prob:outlierRobustEstimation} with 
measurements ($\vy_i, \MA_i$), $i \in [\nrMeasurements]$, and outlier rate $\outlierRate < 0.5$ (or, equivalently,  inlier rate $\inlierRate = 1-\outlierRate > 0.5$). 
Call $\setInliers$ the set of inliers 
and assume that the set of matrices $\MA_i$, $i \in \setInliers$, is $k$-certifiably $\antiConReq$-anti-concentrated  for some $\eta >0$. 
Then, \cref{algo:mc} with relaxation order $\relaxOrder \geq \relaxLevel/2$ outputs an estimate 
 $\vxx\mcsdp$ (not necessarily in $\Domain$) such that:

\vspace{-3mm}

\beq
\label{#1}
\normTwo{\vxx\mcsdp - \vxx\gt} \leq \boundx \left( \frac{\inlierRate \; \eta \;  }{2} +  2\frac{1-\inlierRate}{\inlierRate} \right).
\eeq
}

\newcommand{\tlsThm}[1]{
Consider~\cref{prob:outlierRobustEstimation} with 
measurements ($\vy_i, \MA_i$), $i \in [\nrMeasurements]$, and outlier rate $\outlierRate < 0.5$ (or, equivalently,  inlier rate $\inlierRate = 1-\outlierRate > 0.5$). 
Call $\setInliers$ the set of inliers 
and assume that the set of matrices $\MA_i$, $i \in \setInliers$, is $k$-certifiably $\antiConReq$-anti-concentrated for some $\eta >0$. 
Then, \cref{algo:tls} with relaxation order $\relaxOrder \geq \relaxLevel/2$ outputs an estimate 
 $\vxx\tlssdp$ (not necessarily in $\Domain$) such that:

\vspace{-2mm}

\beq
\label{#1}
\normTwo{\vxx\tlssdp - \vxx\gt} \leq  \frac{ \inlierRate  \boundx \nrMeasurements }{\an - \frac{\gamma\gt}{\barcsq} } \left( 
\frac{\inlierRate \; \eta}{2} + 2 \frac{1-\inlierRate}{\inlierRate}
 \right),
\eeq

\vspace{-2mm}

where $\gamma\gt \triangleq  \sumOverInliers \residuali{\vxx\gt}^2$ is the 
squared\,residual\,error\,of\,the ground truth $\vxx\gt$ over the inliers $\setInliers$.
\vspace{-2mm}
}

\newcommand{\ldrThm}[1]{
Consider~\cref{prob:outlierRobustEstimation} with 
measurements ($\vy_i, \MA_i$), $i \in [\nrMeasurements]$, and known outlier rate $\outlierRate > 0.5$ (or, equivalently, known inlier rate $\inlierRate = 1-\outlierRate < 0.5$). 
Call $\setInliers$ the set of inliers 
and assume that the set of matrices $\MA_i$, $i \in \setInliers$, is $k$-certifiably $\antiConReq$-anti-concentrated for some $\eta >0$. 
Then, with probability at least $\probBound$ (over the draw of the samples in the algorithms), where $N \geq 1$ is a user-defined parameter, \cref{algo:ldr} with relaxation order $\relaxOrder \geq \relaxLevel/2$  outputs a list 
$\List$ of size $N/\inlierRate$ such that there is an estimate $\vxx \in \List$ (with $\vxx$ not necessarily in $\Domain$) such that 

\vspace{-3mm}

\beq
\label{#1}
\normTwo{\vxx - \vxx\gt} \leq  \eta \boundx. 
\eeq
Moreover, when $\alpha \geq 0.01$ (\ie at least 1\% of the measurements are inliers) and $N=10$, 
the relation $\normTwo{\vxx - \vxx\gt} \leq  \eta \boundx$ holds with probability at least $\probBoundReal$ 
over the draw of the samples.
}

\titlespacing*{\section}{0pt}{4mm}{2mm}
\titlespacing*{\subsection}{0pt}{2mm}{2mm}

\title{\huge{Estimation Contracts for \\ Outlier-Robust  Geometric Perception}}

\author{Luca Carlone\thanks{
The author is with the Laboratory for 
Information \& Decision Systems (LIDS) and the Department of Aeronautics and Astronautics at the 
Massachusetts Institute of Technology, Cambridge, USA 
(email: {\tt lcarlone@mit.edu}). \hspace{1cm} This work was partially funded by the NSF
CAREER award ``Certifiable Perception for Autonomous
Cyber-Physical Systems'' and by ARL DCIST CRA W911NF-17-
2-0181.}
  \vspace{-1cm}
}
 
\begin{document}

\maketitle


\begin{abstract}
Outlier-robust estimation is a fundamental problem and has been extensively investigated 
by statisticians and practitioners. The last few years have seen a convergence across research fields towards ``algorithmic robust statistics'', which focuses on developing tractable outlier-robust techniques for high-dimensional estimation problems.  
Despite this convergence, research efforts across fields 
have been mostly disconnected from one another.
This \paper bridges recent work on certifiable outlier-robust estimation for geometric perception in robotics and computer vision with  
parallel work in robust statistics.
In particular, 
we adapt and extend recent results on  \emph{robust linear regression}  (applicable to the low-outlier regime with $\ll 50\%$ outliers) and 
\emph{list-decodable regression} (applicable to the high-outlier regime with $\gg 50\%$ outliers) 
to the setup commonly found in robotics and vision, where 
(i)~variables (\eg rotations, poses) belong to a non-convex domain, 
(ii)~measurements are vector-valued, and 
(iii)~the number of outliers is not known a priori.
The emphasis here is on performance guarantees: rather than proposing radically new algorithms, we provide
 conditions on the input measurements under which 
 modern estimation algorithms (possibly after small modifications) are guaranteed to recover an estimate close to the ground truth in the presence of outliers. These conditions are what we call an ``estimation contract''. 
 The \paper also provides numerical experiments to shed light on the applicability of the theoretical results 
and to showcase the potential of list-decodable regression algorithms in geometric perception.  
Besides the proposed extensions of existing results, 
we believe the main contributions of this \paper are (i)~to unify parallel research lines by pointing out 
commonalities and differences, (ii)~to introduce advanced material (\eg sum-of-squares proofs) in an accessible and  self-contained presentation for the practitioner, and (iii)~to point out a few immediate opportunities and open questions in outlier-robust geometric perception.
\end{abstract}


\clearpage
  \microtypesetup{protrusion=false}
  \tableofcontents{}
  \microtypesetup{protrusion=true}
\clearpage


\section{Introduction}
\label{sec:introduction}

Geometric perception 
 is the problem of estimating unknown geometric models (\eg poses, rotations, 3D structure) from sensor data (\eg camera images, lidar scans, inertial data, wheel odometry). 
 Geometric perception has been at the center stage of robotics and computer vision research since their inception, 
 and includes problems such as object pose (and possibly shape) estimation~\cite{Yang20tro-teaser,Zhou15cvpr}, 
 robot or camera motion estimation~\cite{Scaramuzza11ram}, sensor calibration~\cite{Hartley13ijcv},
 Simultaneous Localization And Mapping (SLAM)~\cite{Cadena16tro-SLAMsurvey}, and Structure from Motion (SfM)~\cite{Triggs99-modern}, to mention a few.

 At its core, geometric perception solves an estimation problem, where, given measurements $\vy_i, i=1,\ldots,\nrMeasurements$, one has to compute a variable of interest $\vxx\gt$ (the ``ground truth''). For instance, in an object pose estimation problem, $\vxx\gt$ is the to-be-computed 3D pose of the object (say, a car), while the $\vy_i$'s might be observations of relevant points on the object (\eg the wheels and the headlights of the car). 
 The unknown $\vxx\gt$ and the measurements $\vy_i$ are related by a \emph{measurement (or generative) model}. 
In this \paper, we focus our attention on the common case where the measurements are vector-valued, \ie $\vy_i \in \Real{\dimy}$, and the noise is additive, leading to measurement models in the form:
 \begin{align}
\label{eq:genModel}
\vy_i = f_i(\vxx\gt)+ \vepsilon, \quad  \text{ with } \quad  \vy_i \in \Real{\dimy} 
\quad \text{ and } \quad \vxx\gt \in \Domain \subseteq \Real{\dimx},
\end{align}
where $f_i(\cdot)$ is a known function, $\vepsilon$ is the measurement noise, and $\Domain$ is the domain of $\vxx\gt$ (\eg the set of 3D poses in an object pose estimation problem). 
As we will see in~\cref{sec:motivatingProblems}, many geometric perception problems 
have measurement models in the form of eq.~\eqref{eq:genModel}.\footnote{These assumptions imply a small loss of generality, \eg in SLAM and rotation averaging the measurements belong to a smooth manifold rather than a vector space and the noise is multiplicative. However, even in these cases, the resulting outlier-free formulations ---under suitable noise assumptions--- assume the form of standard least squares~\cite{Rosen18ijrr-sesync,Hartley13ijcv}, hence 
we believe adapting the results in this \paper to those setups is indeed possible, see~\cref{sec:openProblems}.} 
 

When the noise in~\eqref{eq:genModel} is zero-mean and Gaussian, the maximum likelihood estimate of $\vxx\gt$ can be computed via 
standard least squares:\footnote{Without loss of generality, we assume $\vepsilon$ to have an isotropic Gaussian distribution with identity covariance, but arbitrary covariances can be easily accommodated by rescaling $\vy_i$ and $f_i(\cdot)$ by the square root of the inverse covariance.} 
\begin{align}
\label{eq:LS}
\tag{LS}
\vxx_\ls = \argmin_{ \substack{ \vxx \in \Domain } } 
\sumOverMeas \normTwo{ \vy_i - f_i(\vxx) }^2 \,.
\end{align}
While problem~\eqref{eq:LS} can be still hard to solve (\eg due to potential non-convexity of $f_i(\cdot)$ or $\Domain$), its structure ---at least for common geometric perception problems--- has been extensively studied in robotics and vision, and the literature offers a broad range of solvers, including closed-form solutions~\cite{Horn87josa},
 iterative local solvers~\cite{Dellaert17fnt-factorGraph}, 
minimal solvers~\cite{Larsson18cvpr-minimalSolvers}, 
and convex relaxations~\cite{Rosen18ijrr-sesync,Shi21rss-pace,Zhou15cvpr,Kahl07IJCV-GlobalOptGeometricReconstruction,Carlone15icra-verification,Carlone16tro-duality2D}.

\myParagraph{Outlier-robust estimation}
In practice, many of the measurements fed to the estimation process are \emph{outliers}, \ie they largely deviate from the measurement model~\eqref{eq:genModel} and possibly do not carry any information about $\vxx\gt$. In robotics and vision, 
the measurements $\vy_i$ are the result of a pre-processing of the raw sensor data; such preprocessing is often referred to as the \emph{perception front-end}, while the estimation algorithms that compute $\vxx\gt$ from the $\vy_i$'s are referred to as the \emph{perception back-end}. For instance, in an object pose estimation problem, the perception front-end extracts the position of relevant features $\vy_i$ on the object from raw image pixels (typically using a neural network), while the back-end computes the object pose given the $\vy_i$'s. The perception front-end is prone to errors (\eg the network may mis-detect the wheels of the car in the image), resulting in measurements $\vy_i$ with large errors. 
In the presence of outliers, the least squares estimator~\eqref{eq:LS} is known to produce grossly incorrect results, hence it is desirable to adopt an outlier-robust estimator that can correctly estimate $\vxx\gt$ in the presence of many outliers. In this \paper, we do not make assumption on the nature of the outliers and consider the worst case 
where a fraction $\outlierRate$ of measurements is arbitrarily corrupted, a setup commonly referred to as the \emph{strong adversary model} in statistics and learning~\cite{Klivans18arxiv-robustRegression}. 

\myParagraph{The robust statistics lens}
 Classical robust statistics~\cite{Tukey75picm,Huber81,Rousseeuw81,Rousseeuw87book} provides many alternative formulations to~\eqref{eq:LS} that allow regaining robustness to outliers. 
For instance, if the number of outliers is known, say a fraction $\outlierRate$ of the $\nrMeasurements$ measurements is corrupted, we can use the \emph{Least Trimmed Squares}~\eqref{eq:LTS} estimator~\cite{Rousseeuw81} to compute an outlier-robust estimate: 
\begin{align}
\label{eq:LTS}
\tag{LTS}
\vxx_\lts = \argmin_{ \substack{ \vomega \in \{0;1\}^\nrMeasurements \\ \vxx \in \Domain } } 
\;&\; \sumOverMeas \omega_i \cdot \normTwo{ \vy_i - f_i(\vxx) }^2 
\;, \quad
\subject 
\quad \sumOverMeas{\omega_i} = \an \;, \nonumber
\end{align}
where we defined the \emph{inlier rate} $\inlierRate \triangleq 1- \outlierRate$, and introduced binary variables
$\vomega \in \{0;1\}^\nrMeasurements$ which are in charge of selecting the best $\an$ measurements (when $\omega_i = 1$, the $i$-th measurement is selected as an inlier by~\eqref{eq:LTS}, while $\omega_i=0$ otherwise); 
in words,~\eqref{eq:LTS} selects the $\an$ measurements that induce the smallest error 
for some estimate $\vxx$ and disregards the remaining measurements as outliers.   
Unfortunately, the optimization problem~\eqref{eq:LTS}, as well as many other popular outlier-robust formulations, are NP-hard~\cite{Bernholt06tr-hardnessRobustEstimation} and for a long while no tractable algorithm was available for high-dimensional outlier-robust estimation problems (\eg in the problems we discuss in~\cref{sec:motivatingProblems} and~\cref{sec:openProblems}, $\vxx$'s dimension ranges from $9$ to potentially more than a thousand). 
In recent years, \emph{algorithmic robust statistics} came to the rescue, by proposing polynomial-time algorithms for outlier-robust estimation with strong performance guarantees, including~\cite{Klivans18arxiv-robustRegression,Diakonikolas19icml-robustRegression,Prasad20jrss-robustEstimation,Diakonikolas19soda-robustRegression,Bhatia17neurips-robustRegression}. For instance, while not explicitly recognized in the paper,
the algorithm by Klivans\setal~\cite{Klivans18arxiv-robustRegression} can be understood as a convex relaxation for problem~\eqref{eq:LTS} for the case where $f_i(\cdot)$ is a real-valued linear function. 
 Many of these works use Lasserre's moment relaxation~\cite{Lasserre01siopt-LasserreHierarchy} as an algorithmic workhorse, and adopt the dual view of sum-of-squares relaxations~\cite{Parrilo00thesis} to prove bounds on the quality of the estimates. 

\myParagraph{The computer vision lens}
In typical robotics and vision applications, the number of outliers is unknown, therefore 
outlier-robust estimators have to simultaneously look for a suitable estimate of $\vxx\gt$ while searching for a large set of inliers. In computer vision, a common formulation for outlier-robust estimation with unknown number of outliers is 
\emph{consensus maximization}~\cite{Chin17slcv-maximumConsensusAdvances}, which searches for the largest set of inliers such that the measurements selected as inliers have a low error with respect to some estimate:
\begin{align}
\label{eq:MC}
\tag{MC}
\vxx_\maxCon = \argmax_{ \substack{ \vomega \in \{0;1\}^\nrMeasurements \\ \vxx \in \Domain }} 
\;&\; \sumOverMeas \omega_i \;, \quad
\subject 
\quad \omega_i \cdot \normTwo{ \vy_i - f_i (\vxx) }^2 \leq \barcsq \;, \nonumber
\end{align}
where the given constant $\barc \geq 0$ is the maximum error for a measurement to be considered an inlier.
Problem~\eqref{eq:MC} has been shown to be inapproximable~\cite{Antonante21tro-outlierRobustEstimation,Chin18eccv-robustFitting}, and the literature has been traditionally split between 
fast heuristics (which do not provide performance guarantees) and globally optimal solvers (which 
can compute optimal solutions but run in worst-case exponential time). 
The recent work~\cite{Yang22pami-certifiablePerception} shows that for common geometric perception problems,~\eqref{eq:MC} can be written as a polynomial optimization problem (POP) and 
relaxed via Lasserre's moment relaxation. The key insight behind~\cite{Yang22pami-certifiablePerception}, reviewed in~\cref{sec:motivatingProblems}, is that ---for common perception problems--- the domain $\Domain$ is a basic semi-algebraic set (\ie it can be written as a set of polynomial inequalities), while with a suitable parametrization, the function $f_i(\cdot)$ becomes a (vector-valued) linear function.  

\myParagraph{The robotics lens}
In robotics, the go-to approach for outlier-robust estimation has been the use of M-estimators~\cite{Huber81}, which 
replace the least squares cost in~\eqref{eq:LS} with a robust loss function. 
In this \paper we focus on a particular choice of robust loss function, the \emph{truncated least squares (or truncated quadratic)} cost:
%
\begin{align}
\label{eq:TLS}
\vxx_\tls 
&= 
\argmin_{ \substack{  \vxx \in \Domain } } 
\sumOverMeas \min\left( \normTwo{ \vy_i - f_i (\vxx) }^2 \;,\; \barcsq \right) \nonumber
\\
\tag{TLS}
&= \argmin_{ \substack{ \vomega \in \{0;1\}^\nrMeasurements \\ \vxx \in \Domain } } 
\sumOverMeas \omega_i \cdot \normTwo{ \vy_i - f_i (\vxx) }^2 + (1-\omega_i) \cdot \barcsq \;,
\end{align}
where the objective is the pointwise minimum of a quadratic and a constant function, \ie it is quadratic 
for small \emph{residuals} $\normTwo{ \vy_i -f_i (\vxx) } \leq \barc$, and becomes constant for large residuals.
In the second row of~\eqref{eq:TLS} we noticed that the truncated least squares cost can be equivalently rewritten using auxiliary binary variables $\vomega$, by observing that for two numbers $a,b$, $\min(a,b) = \min_{\omega\in\{0;1\}} \omega \cdot a + (1-\omega) \cdot b$. 
Also problem~\eqref{eq:TLS} has been shown to be inapproximable in the worst case~\cite{Antonante21tro-outlierRobustEstimation}. While traditionally problem~\eqref{eq:TLS} has been attacked using local solvers~\cite{MacTavish15crv-robustEstimation} or continuation schemes~\cite{Yang20ral-GNC}, 
recent work~\cite{Yang22pami-certifiablePerception,Yang20neurips-certifiablePerception,Yang20tro-teaser,Yang19iccv-quasar,Lajoie19ral-DCGM} has shown that for common perception problems,~\eqref{eq:TLS} can be written as a POP and relaxed via Lasserre's moment relaxation.
More surprisingly, many works have empirically observed the relaxation to be tight~\cite{Yang22pami-certifiablePerception,Yang20neurips-certifiablePerception,Yang19iccv-quasar}, at least for reasonable levels of noise and outliers, with very recent work~\cite{Peng22arxiv-robustRotationSearch} providing initial 
theoretical results to support such empirical evidence, at least for the specific problem of rotation search.
 However, the performance of these estimators is commonly demonstrated via empirical evaluation, and the literature is still lacking more general
 theoretical guarantees on the quality of the resulting estimates.

\myParagraph{Catalyst, convergence, and contribution}
Despite the heterogeneity of the formulations reviewed above, 
we observe that
recent years have witnessed a convergence across fields 
 towards designing tractable algorithms for high-dimensional outlier-robust estimation using 
 moment relaxations. 
 A few examples include~\cite{Klivans18arxiv-robustRegression,Karmalkar19neurips-ListDecodableRegression,Yang22pami-certifiablePerception,Yang20neurips-certifiablePerception,Yang20tro-teaser,Yang19iccv-quasar,Lajoie19ral-DCGM,Peng22arxiv-robustRotationSearch,Carlone18ral-robustPGO2D}. 
Such a convergence has been triggered by the progress in polynomial optimization via moment and sum-of-squares relaxations, starting from the seminal works~\cite{Lasserre01siopt-LasserreHierarchy,Parrilo00thesis,Shor98noa,Nesterov00ao,Parrilo03mp-sos}. 
 At the same time, research across fields has remained disconnected, with researchers being mostly unaware of the parallel work in other areas. 

The goal of this \paper is to bridge this gap and connect geometric perception problems in robotics and vision to novel tools 
in outlier-robust statistics. Towards this goal 
we adapt and extend recent results from robust statistics
to the setup and formulations commonly found in robotics and vision.  
For the case with low outlier rates (\ie $\beta \ll 0.5$), we adapt results 
from~\cite{Klivans18arxiv-robustRegression}, which considers outlier-robust regression using least trimmed squares~\eqref{eq:LTS} with scalar linear measurements, 
to the robotics setup where the measurements are vector valued and 
the variables belong to a non-convex domain; we also develop a simple bound 
on the distance of the estimate from the ground truth (while~\cite{Klivans18arxiv-robustRegression} focuses on bounding the 
residual errors for the inliers).  
Then, we extend these results to the case where the number of outliers is unknown. In particular, we compute bounds on the 
estimation error (\ie the distance between the estimate and $\vxx\gt$) for~\eqref{eq:MC} and~\eqref{eq:TLS}.
These results constitute the first general performance guarantees for the convex relaxations~\cite{Yang22pami-certifiablePerception,Yang20neurips-certifiablePerception,Yang19iccv-quasar}, going beyond the empirical observations in~\cite{Yang22pami-certifiablePerception} and the problem-specific optimality guarantees 
in~\cite{Yang19iccv-quasar,Peng22arxiv-robustRotationSearch}.

Then we consider the case with high outlier rates (\ie $\beta \gg 0.5$), where a majority of the measurements are outliers. While in robotics and vision it has been observed that with random (\ie non-adversarial) outliers, 
the point estimators~\eqref{eq:MC} and~\eqref{eq:TLS} are still able to retrieve good estimates for $\vxx\gt$~\cite{Yang22pami-certifiablePerception,Yang20neurips-certifiablePerception,Yang19iccv-quasar,Yang20tro-teaser}, in the presence of adversarial outliers, the estimate resulting 
from~\eqref{eq:LTS},\,\eqref{eq:MC},\,and\,\eqref{eq:TLS}  can be arbitrarily far from the ground truth: intuitively, since the outliers constitute the majority of the measurements, they can agree on an arbitrary $\vxx$ and form a large set of mutually consistent measurements that are picked as solution to~\eqref{eq:LTS},~\eqref{eq:MC}, and~\eqref{eq:TLS}.\footnote{\label{foot:motionEstimation}Note that the case with a high number of adversarial outliers is often the one encountered in practice in robotics and vision: think about a motion estimation problem where the robot has to estimate its motion from point features detected by the camera~\cite{Scaramuzza11ram}: if there is a large moving object in front of the camera, most features may fall on the moving object (rather that on the static portion of the scene), leading to incorrect motion estimates.} 
In robotics and related fields, this setup has been recognized to require computing multiple estimates, in order to find one that is close to the ground truth, ranging from early work on multi-hypothesis target tracking~\cite{BarShalom93} and particle filters~\cite{Dellaert99}, to recent work on multi-hypothesis smoothing~\cite{Hsiao19icra-mhISAM2,Fourie16iros-isam3}. However, none of these works simultaneously provide tractable algorithms and performance guarantees for the resulting estimates. 
In this \paper, we connect to the recent literature on \emph{list-decodable regression}~\cite{Karmalkar19neurips-ListDecodableRegression}, which proposes polynomial-time estimators that return a small list of estimates such that 
with high probability at least one of the estimates is close to the ground truth.
In particular, we provide a minor adaptation of the results in~\cite{Karmalkar19neurips-ListDecodableRegression} to 
account for vector-valued measurements.

Finally, we present numerical experiments on a canonical geometric perception problem to shed light on the theoretical results. The experiments provide encouraging evidence that many of the assumptions supporting the theoretical analysis (\eg certifiable hypercontractivity or certifiable anti-concentration) are often satisfied by real data. 
At the same time, they reveal a large gap between theory (which mostly guarantees performance for high-order, computationally expensive moment relaxations)  and practice (where low-order relaxations already exhibit impressive performance). Our numerical evaluation also provides the first empirical evidence that a sparse low-order moment relaxation for list-decodable regression (based on a modified version of the algorithm proposed in~\cite{Karmalkar19neurips-ListDecodableRegression}) is able to accurately recover estimates in geometric perception problems with high outlier rates, where~\eqref{eq:LTS},~\eqref{eq:MC}, and~\eqref{eq:TLS} are doomed to fail. Moreover, the experiments show that if the measurements are generated by multiple estimates (\eg different subsets of measurements are generated by different variables $\vxx\gt$), then our sparse moment relaxation for list-decodable regression is able to simultaneously recover \emph{all} the estimates generating the data.\footnote{With reference to the motion estimation example in footnote~\ref{foot:motionEstimation}, such an algorithm would simultaneously recover the motion of all the objects in the scene, rather than just the motion with respect to the object capturing most point features, which would be quite useful in practical applications.} 
We release open-source code to reproduce our numerical experiments, including an implementation of key algorithms covered in this \paper at~\url{https://github.com/MIT-SPARK/estimation-contracts}.

We remark that the emphasis in this \paper is on performance guarantees. We do not present new algorithms (we mostly propose small modifications to existing algorithms) but rather try to address the question: under which conditions on the input measurements
 can we guarantee that modern outlier-robust estimation algorithms based on moment relaxations recover an estimate close to the ground truth in the presence of outliers? These conditions are what we call an ``estimation contract''. 
Besides the proposed extensions of existing results, 
we believe the main contributions of this \paper are (i) to unify parallel research lines by pointing out 
commonalities and differences, (ii) to introduce advanced material (\eg sum-of-squares proofs) in an accessible and  self-contained presentation for the practitioner, and (iii) to point out a few immediate opportunities and open questions in outlier-robust geometric perception.
%
%
 %
 This ``unification'' is expected to benefit both practitioners and researchers in robust statistics.
 On the robotics and computer vision side, this \paper provides new and fairly general performance guarantees for 
 robust estimation algorithms based on moment relaxations, applied to geometric perception problems. Moreover, the \paper reviews a new proof system 
 (based on sum-of-squares proofs) that provides a richer language to discuss properties of  moment relaxations beyond the typical analysis based on a manual design of dual certificates~\cite{Yang19iccv-quasar,Peng22arxiv-robustRotationSearch,Eriksson18cvpr-strongDuality,Rosen18ijrr-sesync}. 
 Furthermore, it positions list-decodable regression based on moment relaxations as a useful and computationally tractable tool for 
 multi-hypotheses estimation.
On the robust statistics side, we hope the reader will be intrigued by the 
remarks about the practical performance and the empirical tightness of the moment relaxation of~\eqref{eq:TLS} 
 (discussed in greater detail in~\cite{Yang22pami-certifiablePerception}) and the practical performance of a low-order relaxation for list-decodable regression, which we believe deserve 
 further investigation. We also hope to attract further attention towards the case where the number of
 outlier is unknown and the variables are confined to semi-algebraic sets, which is the setup commonly encountered in robotics and vision problems. 

\myParagraph{\Paper structure}
\cref{sec:relatedWork} starts by reviewing related works across fields.
\cref{sec:motivatingProblems} showcases the fact that many estimation problems in robotics and vision can be 
formulated using a linear measurement model with variables belonging to a basic semi-algebraic set.
\cref{sec:notationAndPreliminaries} introduces notation and preliminaries (while postponing as many details as possible to the appendix).
\cref{sec:problemStatement} succinctly states the problem of outlier-robust estimation and our quest for estimation contracts.
\cref{sec:lowOutliers} studies the case with low outlier rates and provides error bounds for~\eqref{eq:LTS}, \eqref{eq:MC}, and~\eqref{eq:TLS}. 
\cref{sec:highOutliers} studies the case with high outlier rates and adapts results from list-decodable regression.
\cref{sec:experiments} presents numerical experiments on a rotation search problem.
\cref{sec:openProblems} discusses opportunities and open problems and~\cref{sec:conclusions} concludes the \paper.


\section{Related Work}
\label{sec:relatedWork}

\myParagraph{Outlier-robust estimation in robotics and computer vision} 
Traditional algorithms for outlier-robust estimation for geometric perception can be divided into \emph{fast heuristics} and \emph{globally optimal solvers}. Two general frameworks for designing fast heuristics are \ransac~\cite{Fischler81} and \emph{graduated non-convexity} (GNC)~\cite{Black96ijcv-unification,Yang20ral-GNC,Antonante21tro-outlierRobustEstimation}. {\ransac} solves problem~\eqref{eq:MC} by repeatedly sampling 
a minimal set of measurements, computing an estimate $\vxx$ from the minimal set, and searching for a large set of measurements that agrees with the estimate $\vxx$~\cite{Chin17slcv-maximumConsensusAdvances}; while being a well-established algorithm, \ransac mostly applies to low-dimensional problems, since the expected
number of  iterations required by \ransac to find an outlier-free set of measurements grows exponentially in the size of the minimal measurement set; moreover, the number of \ransac iterations   also grows exponentially with the outlier rate $\outlierRate$~\cite{Bustos2015iccv-gore3D}.
{GNC} solves M-estimation problems, including~\eqref{eq:TLS}, via a continuation scheme that starts from a convex approximation of the cost function and then gradually recovers the robust cost. GNC can scale to high-dimensional problems and applies to a broad class of loss functions ---including adaptive loss functions~\cite{Barron19cvpr-adaptRobustLoss,Chebrolu20arxiv-adaptiveCost}. However, the effectiveness of the continuation scheme is problem-dependent.
Iterative local optimization 
is also a popular fast heuristics for the case where an initial guess is available~\cite{Schonberger16cvpr-SfMRevisited,Agarwal13icra,Sunderhauf12icra,Hitchcox22ral-robustEstimation}. Approximate but deterministic algorithms have also been designed to solve consensus maximization~\cite{Le19pami-deterministicApproximateMC}. 
On the other hand, globally optimal solvers are guaranteed to retrieve optimal solutions for outlier-robust estimation but run in worst-case exponential time. These solvers are typically designed using Branch and Bound 
or mixed-integer programming~\cite{Bazin12accv-globalRotSearch,Bustos18pami-GORE,Izatt17isrr-MIPregistration,Yang2014ECCV-optimalEssentialEstimationBnBConsensusMax,Paudel15iccv-robustSOS,Li09cvpr-robustFitting,Li07iccv-3DRegistration,Enqvist08eccv-robustOptimalPose}. 
Enqvist\setal~\cite{Enqvist15ijcv-tractableRobustEstimation} present an outlier-robust estimator that 
runs in polynomial time with respect to the number of measurements, but still has exponential complexity
 in the dimension of the variable to be estimated (see also~\cite{Enqvist12eccv-robustFitting,Olsson08cvpr-polyRegOutlier}). 
Chin\setal~\cite{Chin15cvpr-CMTreeAstar} frame~\eqref{eq:MC} as a tree search problem and propose 
an approach based on $A^*$ search.

\emph{Certifiably optimal outlier-robust algorithms}~\cite{Yang22pami-certifiablePerception,Yang20neurips-certifiablePerception,Yang20tro-teaser,Yang19iccv-quasar,Lajoie19ral-DCGM} have recently emerged as 
a way to obtain optimal solutions to certain outlier-robust problems in polynomial time (both with respect to the number of measurements and the dimension of the variable to be estimated).
These approaches relax non-convex outlier-robust formulations, including~\eqref{eq:MC} and~\eqref{eq:TLS}, into a  convex optimization. 
The key insight behind these algorithms is twofold: 
(i) for several estimation problems arising in robotics and vision, we can rewrite 
the optimization problems arising in outlier-robust estimation as 
\emph{polynomial optimization problems} (POP), and 
(ii) for these POPs, Lasserre's moment relaxation~\cite{Lasserre01siopt-LasserreHierarchy}  is empirically seen to be tight at the lowest relaxation order (\ie order 2), hence providing a tractable way to compute robust estimates. 
These works mostly provide a posteriori optimality certificates~\cite{Yang22pami-certifiablePerception,Yang20neurips-certifiablePerception} (\ie they solve the relaxation, compute a rounded estimate, and then compute a suboptimality gap for that specific estimate, possibly certifying its optimality), while the papers~\cite{Yang19iccv-quasar,Peng22arxiv-robustRotationSearch} 
provide problem-specific a priori conditions under which the relaxation is tight.

\myParagraph{Outlier-robust estimation in robust statistics} 
Outlier-robust estimation has been studied in robust statistics, starting from the seminal work of 
Huber, Tukey, and Rousseeuw~\cite{Tukey75picm,Huber81,Rousseeuw81,Rousseeuw87book}, among many others. 
However, these classical frameworks do not immediately lead to tractable algorithms and 
 are often provably hard to 
 solve~\cite{Bernholt06tr-hardnessRobustEstimation,Antonante21tro-outlierRobustEstimation,Chin17slcv-maximumConsensusAdvances}, leading to algorithms whose complexity increases exponentially in the dimension of $\vxx$.
The growing interest towards high-dimensional outlier-robust estimation has triggered a large number of 
recent works on ``algorithmic robust statistics'', which focus on the design of tractable algorithms for high-dimensional estimation with outliers.
 Early algorithms along this line have focused on clustering and moment estimation (\eg how to robustly estimate mean and covariance of a distribution given samples)~\cite{Lai16focs-momentEstimation,Diakonikolas16focs-robustMomentEstimation,Diakonikolas19siam-robustMomentEstimation,Charikar17stoc-robustEstimationTheory,Kothari17arxiv-clustering,Kothari18stoc-robustMomentEstimation,Diakonikolas18soda-robustLearningGaussians}
and subspace learning for classification in the presence of malicious noise~\cite{Klivans09alp-subspaceLearning,Diakonikolas18stoc-subspaceLearning,Awasthi17acm-robustLinearSeparators}. 

More relevant to this \paper is the recent work on \emph{robust linear regression}~\cite{Klivans18arxiv-robustRegression, 
Diakonikolas19icml-robustRegression, 
Prasad20jrss-robustEstimation, 
Diakonikolas19soda-robustRegression, 
Bhatia17neurips-robustRegression, 
Karmalkar18arxiv-robustL1Regression},  
where one has to compute an estimate of $\vxx\gt$ given an outlier-corrupted set of linear measurements:
 \begin{align} 
 \label{eq:linModelStats}
 y_i = \va_i\tran \vxx\gt + \epsilon, \quad 
 \text{ with } \quad  y_i \in \Real{} 
\quad \text{ and } \quad \vxx\gt \in \Real{\dimx},
 \end{align}
where $y_i$ are given scalar measurements, $\va_i$ are given vectors of suitable dimension, and $\epsilon$ is the measurement noise. Typical contamination models studied in the literature include the \emph{Huber contamination model} (\eg~\cite{Du17arxiv-sparseRobustEstimation}), which assumes that a fraction of measurements are randomly generated by an unknown outlier distribution, and the \emph{strong adversary model}, in which the outlier-generation mechanism has access to the inliers
and can replace a given fraction of them with arbitrary outliers.
The literature on outlier-robust linear regression in the low-outlier regime ($\outlierRate \ll 0.5$) includes approaches based on 
iterative outlier filtering~\cite{Diakonikolas19icml-robustRegression,Diakonikolas19soda-robustRegression}, 
robust gradient estimation~\cite{Prasad20jrss-robustEstimation}, 
hard thresholding~\cite{Bhatia17neurips-robustRegression,Bhatia15neurips-hardThresholding,Chen13icml-robustSparseRegression}, 
$\ell_1$-regression~\cite{Nguyen13tit-l1robustEstimation,Karmalkar18arxiv-robustL1Regression,Wright10tit-l1robustEstimation},
and moment/sum-of-squares relaxations~\cite{Klivans18arxiv-robustRegression}.
Our interest towards moment/sum-of-squares relaxations is motivated by that fact that 
 the framework in~\cite{Klivans18arxiv-robustRegression} can accommodate 
a large class of linear measurement models ---more precisely, fairly general distributions for the vectors $\va_i$ in~\eqref{eq:linModelStats}--- while related work often takes stronger assumptions on $\va_i$ (\eg~\cite{Diakonikolas19soda-robustRegression} assumes $\va_i$'s are sampled from a Gaussian distribution);
moreover, the corresponding algorithms operate in the strong adversary model, which is useful to derive worst-case 
guarantees for geometric perception. 
%
In the high-outlier regime ($\outlierRate \gg 0.5$), it is still possible to recover a good estimate when the outliers are not adversarial. The works~\cite{Nguyen13tit-l1robustEstimation,Wright10tit-l1robustEstimation} provide approaches based on  $\ell_1$-minimization to recover good estimates assuming oblivious outliers (rather than adversarial), 
and under relatively strong assumptions on the measurements 
(\cite{Wright10tit-l1robustEstimation} assumes sub-orthogonality, 
\cite{Nguyen13tit-l1robustEstimation} assumes the vectors $\va_i$ to be drawn from an isotropic Gaussian). The works~\cite{DOrsi21icml-outliersOverwhelm,Tsakonas13spl-robustEstimation,Norman22arxiv-robustRegression} also assume oblivious outliers and analyze the Huber loss estimator (\cite{Tsakonas13spl-robustEstimation} assumes Gaussian $\va_i$,~\cite{DOrsi21icml-outliersOverwhelm} requires the column span of the matrix staking all the vectors  $\va_i$ not to contain approximately sparse vectors,~\cite{Norman22arxiv-robustRegression} assumes a bounded or sub-Gaussian distribution for the $\va_i$'s).
The works~\cite{Karmalkar19neurips-ListDecodableRegression,Raghavendra20soda-ListDecodableRegressions} are the first to provide polynomial-time algorithms for \emph{list-decodable} linear regression with adversarial outliers, where the estimator returns multiple hypotheses such that at least one of the hypotheses is close to the ground truth $\vxx\gt$. List-decodable learning was originally introduced in~\cite{Balcan08stoc-clustering} and was also studied in the context of moment estimation in~\cite{Charikar17stoc-robustEstimationTheory,Diakonikolas18stoc-listDecodable,Kothari17arxiv-clustering}.


\section{Motivating Problems}
\label{sec:motivatingProblems}


This section shows that many foundational problems in \geometricperception can be formulated as 
\emph{linear} estimation problems with variables belonging to a \emph{basic semi-algebraic set} (\ie a set that can be described by a finite number of polynomial inequality constraints).\footnote{This observation constitutes the basis for many certifiable solvers for outlier-free and outlier-robust estimation that have been developed in robotics and vision, see~\cite{Yang22pami-certifiablePerception,Rosen18ijrr-sesync} and the references therein.} 
Mathematically, we will formulate the measurement model for many problems of interest as:
 \begin{align} 
 \label{eq:linModel}
 \vy_i = \MA_i\tran \vxx\gt + \vepsilon, \quad  \text{ with } \quad  \vy_i \in \Real{\dimy} 
\quad \text{ and } \quad \vxx\gt \in \Domain \subseteq \Real{\dimx}.
 \end{align}
 This observation will be crucial towards adapting existing results in robust statistics to the geometric perception setup.
Indeed, the measurement model~\eqref{eq:linModel} is essentially the same of the one used in robust regression, see~eq.~\eqref{eq:linModelStats}, with the 
exception that measurements $\vy_i \in \Real{\dimy}$ are vector-valued, and the variable $\vxx\gt$ belongs to 
a specific domain $\Domain$ (typically, a smooth manifold) rather than $\Real{\dimAmbient}$. 
Towards recasting many estimation problems in \geometricperception as in eq.~\eqref{eq:linModel},
we start by restating the well-known fact that 
---for common variables of interest in robotics and vision--- the domain $\Domain$ is a basic semi-algebraic set (see \eg~\cite{Yang22pami-certifiablePerception,Tron15rssws3D-dualityPGO3D,Carlone15icra-verification,Briales18cvpr-global2view,Yang20cvpr-shapeStar}).

\begin{fact}[Variables in geometric perception]
\label{fact:variablesInGeoPerception}
The $d$-dimensional \emph{Special Orthogonal group} 
$\SOd \triangleq \{ \MR \in \revise{\Real{d\times d}} \mid \MR\tran \MR = \eye_d, \det(\MR) = +1 \}$
(\ie the group of rotations), 
and the \emph{Special Euclidean group} 
$\SEd \triangleq \left\{ {\tiny \matTwo{\MR & \vt \\ \zero & 1} } \in \revise{\Real{(d+1)\times (d+1)}} \mid  \MR \in \SOd, \vt \in\Real{d} \right\}$
(\ie the group of poses and rigid transformations) 
are basic semi-algebraic 
sets.\footnote{ More precisely, $\SOd$ and $\SEd$ are \emph{algebraic varieties}, \ie sets that can be described 
by a finite set of polynomial \emph{equality} constraints 
$\{ f_1(\vxx) = 0, f_2(\vxx) = 0, \ldots,f_m(\vxx)=0 \}$. Note that we can rewrite a variety  
as a basic semi-algebraic set by replacing each equality constraint $f_i(\vxx) = 0$ with two inequality constraint $f_i(\vxx) \geq 0$ and $f_i(\vxx) \leq 0$, hence a variety can be understood as a special case of a basic semi-algebraic set.
 } 
Moreover, several geometric constraints (\eg field-of-view or maximum distance constraints) can be 
written as basic semi-algebraic sets.
\end{fact}

 Now we observe that several \geometricperception problems can be written as \emph{linear} models, 
akin to eq.~\eqref{eq:linModel} (we will discuss few more examples, including SLAM and rotation averaging in~\cref{sec:openProblems}). 
 The reader familiar with \geometricperception can safely skip this section. 

\setcounter{theorem}{0}

\begin{example}[Rotation search (a.k.a. the Wahba problem)]
\label{ex:wahba}
Estimate the rotation $\MR \in \SOthree$ that aligns pairs of 3D points $(\va_i,\vb_i)$, $i=1,\ldots,\nrMeasurements$. 
The measurement model for the (inlier) measurements is given by: 

\vspace{-6mm}

\begin{align}
\label{eq:wahba}
\vb_ i = \MR \va_i + \vepsilon 
\quad 
\xRightarrow{
 \vxx\gt \triangleq \vectorize{\MR}
 }
\quad
\overbrace{\vb_ i = (\va_i\tran \kron \eye_3) \vxx\gt + \vepsilon}^{\text{same form as~\eqref{eq:linModel}}},
\end{align}
where $\vepsilon$ is the measurement noise.  In~\eqref{eq:wahba}, 
we used the vectorization operator $\vectorize{\cdot}$ to transform a 3D matrix into a vector
 and manipulated the expression using standard vectorization properties.
Rotation search arises, for instance, in satellite attitude estimation~\cite{wahba1965siam-wahbaProblem} and image stitching~\cite{Yang19iccv-quasar}.
\end{example}

\begin{example}[3D point cloud registration]
\label{ex:pointCloudRegistration}
Estimate the rigid transformation $(\MR, \vt)$, with $\MR \in \SOthree$ and $\vt \in\Real{3}$,  
that aligns pairs of 3D points $(\va_i,\vb_i)$, $i=1,\ldots,\nrMeasurements$. 
The measurement model for the (inlier) measurements is: 
\begin{align}
\label{eq:3Dregistration}
\vb_ i = \MR \va_i + \vt + \vepsilon 
\quad 
\xRightarrow{
{\vxx\gt \triangleq {\tiny \matTwo{ \vectorize{\MR} \\ \vt} } } 
}
\quad
\overbrace{
\vb_ i = \matTwo{\va_i\tran \kron \eye_3 & \eye_3} \vxx\gt + \vepsilon 
}^{\text{same form as~\eqref{eq:linModel}}},
\end{align}
Point-to-plane 3D registration can be similarly formulated using a linear model involving a rigid transformation~\cite{Yang22pami-certifiablePerception}. Registration problems are commonly encountered in instance-level object pose estimation, scan-matching for 3D reconstruction, and (stereo or RGB-D) visual odometry~\cite{Yang20tro-teaser}.
\end{example}


\begin{example}[3D-3D category-level object pose and shape estimation]
\label{ex:category}
Estimate the rigid transformation $(\MR, \vt)$, with $\MR \in \SOthree$ and $\vt \in\Real{3}$, 
and the shape parameters $\vc \in \Real{K}$ (describing the shape of a 3D object) from 3D point measurements $\vb_i$, $i=1,\ldots,\nrMeasurements$.  
The generative model for the (inlier) measurements is: 
\begin{align}
\label{eq:3D3DCategoryLevel}
\vb_ i = \MR \MS_i \vc + \vt + \vepsilon 
\quad
\xRightarrow{
	\substack{ \vt' \triangleq \MR\tran\vt \;,\; \vepsilon' \triangleq \MR\tran\vepsilon \;,\; 
\vxx\gt \triangleq {\tiny \vect{\vectorize{\MR\tran} \\ \vt' \\ \vc} }}
} 
\quad
\zero = \matThree{ -\vb_i\tran \kron \eye_3 & \eye_3 & \MS_i } \vxx\gt + \vepsilon',
\end{align}
where $\MS_i \in \Real{3 \times K}$ is a matrix of given basis shapes (such that the final shape $\MS_i \vc$, $i=1,\ldots,\nrMeasurements$, is written as a linear combination of the basis shapes). 
Note that for~\eqref{eq:3D3DCategoryLevel} to fall in the class of problems~\eqref{eq:linModel}, we have to assume that $\vepsilon$ has an isotropic distribution, such that the distribution of $\vepsilon' \triangleq \MR\tran\vepsilon$ does not depend on the unknown $\MR$. 
The model in~\eqref{eq:3D3DCategoryLevel} is known as the 
\emph{active shape model}~\cite{Cootes95cviu}, 
and finds application in face detection, human pose estimation, and object pose and shape estimation, see~\cite{Zhou15cvpr,Yang20cvpr-shapeStar,Shi21rss-pace,Shi22arxiv-PACE} among others.
In these problems, it is not uncommon for the number of shapes $K$ to be large, \eg $K \gg 100$.  
\end{example}

\begin{example}[Absolute pose estimation]
\label{ex:absolutepose}
Estimate the camera pose $(\MR, \vt)$, with $\MR \in \SOthree$ and $\vt \in\Real{3}$, 
from (calibrated) pixel observations, written as unit-norm vectors $\vu_i$, picturing known 3D points $\va_i \in \Real{3}$, $i=1,\ldots,\nrMeasurements$.  
The generative model for the (inlier) measurements is: 
\begin{align}
\label{eq:absolutePose}
\lambda_i \vu_i = \MR \va_i + \vt + \vepsilon 
\quad 
\xRightarrow{
\substack{ \vepsilon_i' \triangleq [\vu_i]_\times \vepsilon \;,\; \vxx\gt \triangleq {\tiny \vect{\vectorize{\MR} \\ \vt } }} 
}
\quad
\overbrace{
\zero = \matThree{ [\vu_i]_\times (\va_i\tran \kron \eye_3)  & [\vu_i]_\times} \vxx + \vepsilon_i' 
}^{\text{same form as~\eqref{eq:linModel}}},
\end{align}
Intuitively, the measurement model on the left describes the fact that ---up to an unknown scale $\lambda_i$---
 the pixel measurement $\vu_i$ pictures the 3D point $\va_i$ after it is transformed to the camera frame according to the camera pose $(\MR,\vt)$.
While the model on the left is already linear, on the right-hand side we algebraically eliminated the scale factors $\lambda_i$ and 
obtained a lower-dimensional linear measurement model by 
multiplying both sides by the orthogonal projector $[\vu_i]_\times\triangleq \eye_3 - \vu_i \vu_i\tran$, which is such that $[\vu_i]_\times \vu_i = \zero$.
 The absolute pose estimation problem arises in camera localization in known scenes and object pose estimation from camera images, see \eg~\cite{Kneip2014ECCV-UPnP,Schweighofer2008bmvc-SOSforPnP,Yang21iccv-damp}.
\end{example}

Two remarks are in order.
First, we observe that while the adoption of the linear model in~\eqref{eq:linModel} with variables in $\Domain$ might seem  an obvious choice, 
 this has not necessarily been the go-to approach in robotics and vision.
 In many cases, researchers still prefer non-linear measurement models over Euclidean space,\footnote{For instance, one can parametrize a rotation using Euler angles (or a tangent-space representation), which makes the corresponding measurement models nonlinear, but allows treating the variables as 
 unconstrained quantities.}  rather than
 a linear model over the semi-algebraic domain $\Domain$, since the former leads to \emph{unconstrained} nonlinear least squares problems that can be quickly solved by local solvers when an initial guess is available~\cite{Dellaert17fnt-factorGraph}. 
 The second remark is that ---in the outlier-free case--- all the \geometricperception examples in this section can be considered solved; in particular, rotation search and 3D registration admit a closed-form solution~\cite{Horn87josa}; 
  absolute pose estimation can be solved globally using Gr\"obner basis~\cite{Kneip2014ECCV-UPnP}; 
  3D-3D category-level perception can be solved via a tight convex relaxation~\cite{Shi22arxiv-PACE}.
  On the other hand, some of these problems remain challenging in the presence of outliers and are still 
  the subject of active research, see~\cite{Bustos2015iccv-gore3D,Yang20tro-teaser,Shi22arxiv-PACE} and the references therein.

We conclude this section with an assumption that is required for the theoretical analysis and practical performance of the machinery used in this \paper, \ie moment/sum-of-squares relaxations.

\setcounter{theorem}{1}
\begin{assumption}[Explicitly bounded domain]
\label{ass:explicitlyBoundedxs}
In this \paper we assume that the domain $\Domain$ is \emph{explicitly bounded} (or \emph{Archimedian}), meaning that it contains a constraint in the form $\normTwo{\vxx}^2 \leq M_x^2$ for some finite constant $\boundx > 0$. 
\end{assumption} 

This assumption is typically not restrictive since many geometric variables already belong to bounded sets. 
For instance, rotations $\MR\in\SOthree$ satisfy $\normTwo{\vectorize{\MR}}^2 = 3$; 
the shape vector in~\cref{ex:category} is typically assumed to belong to the probability simplex, hence it
satisfies $\normTwo{\vc}^2 \leq 1$; 
finally, translations can be assumed to be bounded since the sensors producing the measurements in the examples above have finite range.


\section{Preliminaries on Moment Relaxations and Sum-of-Squares Proofs}
\label{sec:notationAndPreliminaries}

This section reviews the three main ingredients of modern techniques for outlier-robust estimation: 
semidefinite programming, moment relaxations, and sum-of-squares proofs. 
In the next sections, we will use these concepts to state outlier-robust estimation algorithms 
(using moment relaxations, which lead to relaxing our estimation problems to tractable semidefinite programs) 
and to analyze their performance (using sum-of-squares proofs).
We keep this presentation short and pragmatic, and refer the interested reader to the appendix and 
to specialized references~\cite{Lasserre10book-momentsOpt,Parrilo03mp-sos,Barak16notes-proofsBeliefsSos} for details.

\myParagraph{Notation} We use lowercase characters (\eg $y$) to denote real scalars, bold lowercase characters (\eg $\vy$) for real (column) vectors, and bold uppercase characters (\eg $\MA$) for real matrices. $\eye_d$ denotes the identity matrix of size $d \times d$, $\ones_d$ denotes the vector of ones of size $d$, $\zero$ denotes the all-zero vector or matrix of appropriate size.
The symbol $\kron$ denotes the Kronecker product.
For a square matrix $\MA \in \Real{d \times d}$,
$\trace{\MA} \triangleq \sum_{i=1}^d a_{ii}$ denotes the matrix trace.
For $\MA,\MB \in \Real{m \times d}$, $\inprod{\MA}{\MB}\triangleq \trace{\MA\tran \MB} =  \sum_{i=1}^m \sum_{j=1}^d A_{ij} B_{ij}$ denotes the usual inner product between real matrices.   
$[\MA,\MB]$ and $[\MA\;\MB]$ denote the \emph{horizontal} concatenation, while $[\MA \vcat \MB]$ denotes the \emph{vertical} concatenation, for proper $\MA,\MB$. 
For a vector $\vv$, we use $\normOne{\vv}$, $\normTwo{\vv}$, and $\normInf{\vv}$ to denote the $\ell_1$, $\ell_2$, and $\ell_\infty$ norm of $\vv$, respectively. 
For $a \in \Real{}$, the symbol $\ceil{a}$ returns the smallest integer $\geq a$.
%
We use $\sym{d}$ to denote the space of $d \times d$ real symmetric matrices, and $\psd{d}$ (resp. $\pd{d}$) to denote the set of matrices in $\sym{n}$ that are \emph{positive semidefinite} (resp. definite). We also write $\MX \succeq 0$ (resp. $\MX \succ 0$) to indicate $\MX$ is positive semidefinite (resp. definite).   
We denote with $\Natural{}$ the set of natural numbers (nonnegative integers), 
and for a given $m \in \Natural{}$ with $m \geq 1$, we use the notation $[m]\triangleq \{1,\ldots,m\}$ to denote the set of indices from $1$ to $m$. 
 For a finite set $\calA$, $|\calA|$ denotes the cardinality of $\calA$. 
Finally, $\ones_\calA \in \{0;1\}^m$ is the indicator vector of the set $\calA\subseteq[m]$, whose $i$-th entry is 1 if $i\in\calA$ or zero otherwise.

\subsection{Semidefinite Programming}
\label{sec:pre-sdp}

The algorithms discussed in this \paper require solving large semidefinite programs. 
A semidefinite program is a convex optimization problem and can be solved in polynomial time 
by off-the-shelf optimization solvers, \eg~\cite{mosek,CVXwebsite}.
More formally, a \emph{multi-block} semidefinite programming (SDP) problem
is an optimization problem in the following {primal}  form \cite{tutuncu03MP-SDPT3}:
\begin{equation}\label{eq:primalSDP}
\min_{\MX \in \bbX} \cbrace{\inprod{\MC}{\MX} \mid \calA (\MX) = \vb,\ \MX \succeq 0} \tag{SDP},
\end{equation}
where the variable $\MX = (\MX_1,\dots,\MX_l)$ is a collection of $l$ square matrices (the ``blocks'') with $\MX_i \in \Real{d_i \times d_i}$ for $i=1,\ldots,l$ (conveniently ordered such that $d_1\geq \dots \geq d_l$); 
the domain $\bbX \triangleq \sym{d_1} \times \dots \times \sym{d_l}$ and the constraint $\MX \succeq 0$ restrict the matrices to be symmetric positive semidefinite.
 The objective is a linear combination of the matrices in $\MX$, \ie 
$\inprod{\MC}{\MX} \triangleq \sum_{i=1}^l \inprod{\MC_i}{\MX_i}$ (for given matrices $\MC_i\in \sym{d_i}, i=1,\ldots,l$).
The problem includes independent linear constraints $\calA (\MX) = \vb$, where:
\beq
\calA (\MX) \triangleq
\sbracket{
\sum_{i=1}^l  \inprod{\MA_{i1}}{\MX_i} \vcat
\dots \vcat
\sum_{i=1}^l  \inprod{\MA_{im}}{\MX_i}
} \in \Real{m},
\eeq
for given matrices $\MA_{ij}\!\in\!\sym{d_i}, i\!=\!1,\ldots,l$, and $j\!=\!1,\dots,m$,
and a given vector $\vb\!\in\!\Real{m}$.  

\subsection{Polynomial Optimization and Lasserre's Hierarchy of Moment Relaxations}
\label{sec:pre-pop}

Our interest towards polynomial optimization stems from the fact 
that the outlier-robust formulations discussed in~\cref{sec:introduction}, \ie~\eqref{eq:LTS},~\eqref{eq:MC}, and~\eqref{eq:TLS}, can be written as 
polynomial optimization problems when the measurement model is linear (or, more generally, polynomial) and $\vxx$ belongs to a basic semi-algebraic set. Moreover, as we will see in this section, Lasserre's hierarchy of moment relaxations provides a systematic approach to relax polynomial optimization problems (which in general are hard to solve) into semidefinite programs (which can be solved in polynomial time). 
Indeed, the tools covered in this section already provide a 
 fairly general template to design outlier-robust estimators for geometric perception and 
 have been used in recent works~\cite{Yang22pami-certifiablePerception,Yang20neurips-certifiablePerception,Yang20tro-teaser,Yang19iccv-quasar,Lajoie19ral-DCGM}.

{\bf Polynomial optimization}. Given a vector $\vxx = [x_1 \vcat x_2 \vcat \ldots \vcat x_{\dimx}] \in \Real{\dimx}$,
a \emph{monomial} in $\vxx$ is a product of $x_i$'s with \emph{nonnegative} integer exponents 
(for instance $x_1^2 x_5 x_6^3$ is a monomial). The sum of the exponents 
is called the \emph{degree} of the monomial (\eg the monomial $x_1^2 x_5 x_6^3$ has degree $6$). 
A real \emph{polynomial} $p(\vxx)$ is a finite sum of monomials with real coefficients. 
\revise{We shorthand $p$ in place of $p(\vxx)$ when the variable $\vxx$ is clear from the context.}
The degree of a polynomial $p$, denoted by $\deg{p}$, is the \emph{maximum} degree of its monomials. 
The ring of real polynomials is denoted by $\polyring{\vxx}$. 

A polynomial optimization problem (POP) is an optimization problem in the form: 
\begin{equation}\label{eq:pop}
\pstar  \triangleq \min_{\vxx \in \Real{\dimx}} \cbrace{p(\vxx) \ \middle\vert\ \substack{ \displaystyle h_i(\vxx) = 0, i=1,\dots,l_h \\ \displaystyle g_j(\vxx) \geq 0, j = 1,\dots,l_g } }, \tag{POP}
\end{equation} 
where $p, h_i, g_j \in \polyring{\vxx}$. Problem~\eqref{eq:pop} is hard to solve in general~\cite{Lasserre10book-momentsOpt} (\eg hard combinatorial problems with binary constraints $x_i \in \{0,1\}$ can be written in the form~\eqref{eq:pop} by imposing $x_i^2 = x_i, \; i=1,\dots,\dimx$), but it admits a well-studied convex relaxation that we review below. 

{\bf Lasserre's hierarchy of moment relaxations~\cite{Lasserre10book-momentsOpt,Lasserre01siopt-LasserreHierarchy,Lasserre18icm-momentRelaxation}}. 
We denote with $[\vxx]_{\relaxOrder}$ 
  the \revise{vector} of monomials of degree up to $\relaxOrder$. For example, if $\vxx = [x_1 \vcat x_2]$ and $\relaxOrder=2$, then $[\vxx]_2 = [1\vcat x_1 \vcat x_2 \vcat x_1^2 \vcat x_1 x_2 \vcat x_2^2]$. The dimension of $[\vxx]_\relaxOrder$ is $\binomial{d}{\relaxOrder} \triangleq \nchoosek{\dimx+\relaxOrder}{\relaxOrder}$. With $[\vxx]_\relaxOrder$, we form the so-called \emph{\momentMatrix} $\MX_{2\relaxOrder} \triangleq [\vxx]_\relaxOrder [\vxx]_{\relaxOrder}\tran$. 
For instance, for $\vxx = [x_1 \vcat x_2]$ and $\relaxOrder=2$ (\cf $[\vxx]_2$ above)\footnote{Contrary to~\cite{Yang22pami-certifiablePerception}, we use the subscript $2\relaxOrder$ instead of $\relaxOrder$ for the moment matrix computed as $[\vxx]_\relaxOrder [\vxx]_{\relaxOrder}\tran$; we believe this notation is more intuitive (the moment matrix contains monomials up to degree $2\relaxOrder$) and is more consistent with the standard definition of  pseudo-distributions, as we will see in~\cref{app:pseudo-distributions}.}:
\beq\label{eq:momentMatrix}
\MX_{4} \triangleq  [\vxx]_2 [\vxx]_2\tran \!=\!
\small{  
\left[
\begin{array}{cccccc}
	1 &  x_1  &  x_2  & x_1^2 & x_1 x_2 & x_2^2 \\
	x_1 &  x_1^2  &  x_1 x_2  & x_1^3 & x_1^2 x_2 & x_1 x_2^2 \\
	x_2 &  x_1 x_2  &  x_2^2  & x_1^2 x_2 & x_1 x_2^2 & x_2^3 \\
	x_1^2 &  x_1^3  &  x_1^2 x_2  & x_1^4 & x_1^3 x_2 & x_1^2 x_2^2 \\
	 x_1 x_2 &   x_1^2 x_2  &   x_1 x_2^2  & x_1^3 x_2 & x_1^2 x_2^2 & x_1 x_2^3 \\
	 x_2^2 &  x_1 x_2^2 &  x_2^3  & x_1^2 x_2^2 & x_1 x_2^3 & x_2^4 
\end{array}
\right]
}.
\eeq

By construction, $\MX_{2\relaxOrder}$ is positive semidefinite and has $\rank{\MX_{2\relaxOrder}} = 1$. 
Moreover, the set of \emph{unique} entries in $\MX_{2\relaxOrder}$ is simply $[\vxx]_{2\relaxOrder}$, \ie the set of monomials of degree up to $2\relaxOrder$ (these monomials appear multiple times in $\MX_{2\relaxOrder}$, \eg see $x_1 x_2$ in eq.~\eqref{eq:momentMatrix}).
Therefore, a key fact is that \emph{---for a suitable matrix $\MA$--- 
the linear function $\inprod{\MA}{\MX_{2\relaxOrder}}$
can express any polynomial in $\vxx$ of degree up to $2\relaxOrder$.}\footnote{
For instance, we can write the polynomial $5 x_1^3 + 0.3 x_1^4 + 4 x_1^2 x_2^2 + x_2^4$ as $\inprod{\MA}{\MX_{2\relaxOrder}}$ with:
\begin{align}
\MA = \smaller{  
\left[
\begin{array}{cccccc}
	0 &  0  &  0  & 0 & 0 & 0 \\
	0 &  0  &  0  & \frac{5}{2} & 0 & 0 \\
	0 &  0  &  0  & 0 & 0 & 0 \\
	0 &  \frac{5}{2}  &  0  & 0.3 & 0 & 2 \\
	0 &  0  &  0  & 0 & 0 & 0 \\
	0 &  0  &  0  & 2 & 0 & 1 
\end{array}
\right].
}
\end{align}
}

The key idea of Lasserre's hierarchy of moment relaxations is to (i) rewrite~\eqref{eq:pop} using the moment matrix $\MX_{2\relaxOrder}$, 
(ii) relax the (non-convex) rank-1 constraint on $\MX_{2\relaxOrder}$ (and only enforce $\MX_{2\relaxOrder}\succeq 0$, which is a convex constraint),\footnote{As\,shown\,in~\cref{app:pseudo-distributions}, relaxing\,the\,rank constraint converts the {\momentMatrix into a \emph{pseudo-moment matrix}.}} and
(iii) add redundant constraints that are trivially satisfied in~\eqref{eq:pop} but still contribute to improving the quality of the relaxation.
This leads to a \emph{convex} semidefinite program in the form~\eqref{eq:primalSDP} (the interested reader can find a more detailed derivation in~\cref{app:momentRelaxation}), which can be conveniently solved in polynomial time:
\begin{equation}\label{eq:momentRelaxation}
m\opt \triangleq \min_{\MX \in \bbX} \cbrace{\inprod{\MC}{\MX} \mid \calA (\MX) = \vb,\ \MX \succeq 0} 
\tag{LAS$_{\relaxOrder}$}.
\end{equation}
One can solve the relaxation for different choices of $\relaxOrder$, leading to a \emph{hierarchy} of convex relaxations; $\relaxOrder$~is typically referred to as the \emph{order} of the relaxation.
The importance of Lasserre's hierarchy lies in its stunning theoretical properties: when the set defined by the constraints is explicitly bounded (\cref{ass:explicitlyBoundedxs}), then 
Lasserre's relaxation~\eqref{eq:momentRelaxation} can be proven to compute a solution to the original~\eqref{eq:pop} when $\relaxOrder \rightarrow \infty$, and ---under further technical conditions--- the relaxation computes a solution to~\eqref{eq:pop} 
for some finite $\relaxOrder$~\cite{Lasserre18icm-momentRelaxation,Nie14mp-finiteConvergenceLassere}. 
In practice, the dimension of $\MX_{2\relaxOrder}$ quickly grows for increasing $\relaxOrder$, hence solving the SDP~\eqref{eq:momentRelaxation} is only practical for small relaxation orders $\relaxOrder$. 
Therefore, it would be interesting to see if we can compute a good approximation of the solution $\vxx\opt$ of~\eqref{eq:pop} from the solution of the moment relaxation at some small $\relaxOrder$.
We expand on this point below.

\myParagraph{Rounding and a posteriori guarantees}
How can we estimate a ``good'' solution for~\eqref{eq:pop} and assess its quality using~\eqref{eq:momentRelaxation}? 
Given a solution $\MX\opt$ to~\eqref{eq:momentRelaxation} (again, computable in polynomial time) for some small relaxation order $\relaxOrder$, one can extract a feasible (but possibly suboptimal) estimate $\hat{\vxx}$ for~\eqref{eq:pop}, using a  \emph{rounding} procedure; for instance, we can extract the entries of $\MX\opt$ corresponding to $\vxx$, which we denote as $\getEntries{\MX\opt}{\vxx}$ (\cf~\eqref{eq:momentMatrix}, where we can extract 
$\getEntries{\MX\opt}{\vxx} = [x_1 \vcat x_2
]$ from the first column of the matrix), and then project the corresponding vector to the set $\Domain$. 
The rounding procedure is often problem dependent, but it is straightforward to implement in the geometric perception problems considered in this \paper, where projecting to the feasible set~\eqref{eq:pop} is easy (\eg given a generic $3\times 3$ matrix, it is easy to project the matrix onto the set of 3D rotations, see~\cite{Hartley13ijcv}). 
Interestingly, checking the quality of the estimate $\hat{\vxx}$ \emph{a posteriori} (\ie after solving the moment relaxation) is also easy: if we call $\hat{p} \triangleq p(\hat{\vxx})$ the objective attained by $\hat{\vxx}$ in~\eqref{eq:pop}, 
it holds:
\begin{align}
\label{eq:suboptimality}
m\opt 
\leq 
p\opt 
\leq 
 \hat{p},
\end{align}
where the first inequality follows from the fact that~\eqref{eq:momentRelaxation} is a relaxation, while 
the second follows from the fact that $p\opt$ is the optimal (lowest) cost over the feasible set of~\eqref{eq:pop} and $\hat{\vxx}$ is feasible for~\eqref{eq:pop}. 
Since we can compute $m\opt$ after solving the relaxation~\eqref{eq:momentRelaxation}
 and we can also compute $\hat{p}$ after rounding, we can use~\eqref{eq:suboptimality} to bound the suboptimality of $\hat{\vxx}$ (\ie from~\eqref{eq:suboptimality}, we can bound the suboptimality gap of $\hat{\vxx}$ as $\hat{p} - p\opt \leq \hat{p} - m\opt$) and understand how far is $\hat{\vxx}$ from an optimal solution. 
Moreover, if $\hat{p} = m\opt$, the inequalities~\eqref{eq:suboptimality} become tight, and we can conclude that $\hat{p} = p\opt$ and $\hat{\vxx}$ is indeed an optimal solution.
This a posteriori checks are at the basis of the certifiable algorithms proposed in 
robotics and vision~\cite{Yang22pami-certifiablePerception,Yang20neurips-certifiablePerception,Yang20tro-teaser,Yang19iccv-quasar,Lajoie19ral-DCGM}, which first reformulate robust estimation as a~\eqref{eq:pop} and then apply the following algorithmic workflow:
\begin{align}
\label{eq:workflow}
\text{\eqref{eq:pop}}
\xRightarrow{ \text{moment relaxation} }
\text{\eqref{eq:momentRelaxation}}
\xRightarrow{ \text{SDP solver} }
(\MX\opt, m\opt)
\xRightarrow{ \text{rounding} }
(\hat{\vxx}, \hat{p})
\xRightarrow{ \text{certification} }
\hat{p} \overset{?}{=} m\opt,
\end{align}
where the last step is to certify optimality of $\hat{\vxx}$ whenever $\hat{p} = m\opt$.
In practice, what is making the works~\cite{Yang22pami-certifiablePerception,Yang20neurips-certifiablePerception,Yang20tro-teaser,Yang19iccv-quasar,Lajoie19ral-DCGM} (as well as the previous work on outlier-free estimation, \eg~\cite{Rosen18ijrr-sesync,Carlone15iros-duality3D,Carlone15icra-verification,Briales17cvpr-registration,Eriksson18cvpr-strongDuality}) compelling is the empirical observation that the moment relaxation is \emph{tight} (meaning $\hat{p} = m\opt$) in many practical problems for $\relaxOrder=1$ or $\relaxOrder=2$, hence the workflow above allows computing optimal solutions to~\eqref{eq:pop} efficiently in practice. 
Moreover, the certification step can be directly used in practice, \eg a robot can trust an estimate 
if it is certified as optimal or discard it (or at least handle it more carefully) if no optimality certificate is obtained.

{\bf What about a priori guarantees?} 
 The a posteriori guarantees above can only be obtained for a given problem and input data
 after solving the corresponding relaxation, and for a specific choice of rounding. 
 Moreover, the empirical observation that the relaxation is often tight seems to 
 be a \emph{deus ex machina} and unexpectedly solve a provably hard problem. 
Therefore in this \paper we are interested in \emph{a priori guarantees}. How can we theoretically justify the 
empirical performance of the moment relaxation for geometric perception? can we characterize the 
input data for which we expect to obtain good relaxations and rounded estimates close to the ground-truth variable we are trying to estimate? These questions are of practical relevance, since answering these questions (i)~would enable a better understanding of the conditions under which the perception system of a robot is expected to work well, and 
(ii)~would allow the design of novel perception front-ends that can produce measurements that are more likely to lead to good estimates.

Providing a priori guarantees would be easy for very large relaxation orders, \ie $\relaxOrder \rightarrow \infty$, since in this case the solution of~\eqref{eq:momentRelaxation} can be proven to retrieve the solution of~\eqref{eq:pop} (see~\cref{app:momentRelaxation}). However, here we want to provide a priori guarantees for very small $\relaxOrder$. In robotics and vision, such guarantees have only appeared 
for specific problems~\cite{Rosen18ijrr-sesync,Eriksson18cvpr-strongDuality,Yang19iccv-quasar,Peng22arxiv-robustRotationSearch}. 
However, it would be desirable to have a more general language to discuss properties of moment relaxations for small relaxation orders $\relaxOrder$. 
Luckily, the sum-of-squares proofs system, described below, provides such a language.
%
%
%
%
%

\subsection{Sum-of-Squares Proofs}
\label{sec:pre-sos}

Sum-of-squares (\sos) proofs provide an advanced way to reason about polynomial constraints and 
to infer properties of the moment relaxation introduced above. 
 Here we want to give some intuition about the \sos proof system, 
 and we postpone a more formal introduction to~\cref{app:sosProofs}.

 {\bf Why do we need to reason about polynomial constraints?} Let us start with some intuition and motivation, before formalizing the concept of \sos proof.
 Assume that we rephrased one of the problems presented in the introduction, \eg~\eqref{eq:TLS}, as a~\eqref{eq:pop} and obtained the corresponding moment relaxation~\eqref{eq:momentRelaxation}. 
We then solved~\eqref{eq:momentRelaxation} to obtain a solution matrix $\MX\opt$ (later, we are going to call this object a ``pseudo-moment matrix''). Now, we would like to infer that $\MX\opt$ satisfies some property of interest; for instance, in our estimation problems we may want to ensure that for some suitable linear function $\calL(\cdot)$, the following holds:
\begin{align}
\label{eq:exampleRelation}
\normTwo{ \calL(\MX\opt) - \vxx\gt }^2 \leq \eta^2,
\end{align}  
which states that the estimate $\calL(\MX\opt)$ computed from $\MX\opt$ is within a distance $\eta$ from the ground-truth  $\vxx\gt$. 
The \sos proofs system does exactly that: it provides a systematic way to conclude that the solution of the relaxation of a system of polynomial constraints, such as the constraints in~\eqref{eq:TLS} (let us call these polynomial constraints $\calA$) also satisfies a desired polynomial relation, \eg~\eqref{eq:exampleRelation} (let us call this polynomial constraint $\calB$). 
 The key idea is that if we can provide an \emph{\sos proof} that $\calA$ ``implies'' $\calB$, a novel type of proof that we introduce below, 
 then a moment relaxation of $\calA$ will also satisfy $\calB$.
 Note that~\eqref{eq:exampleRelation} is only an example of implication we might be attempting to prove, while in general, we might try to prove other (polynomial-expressible) properties of the 
 moment relaxation. 

{\bf What is an \sos proof?}
First of all, we recall that a polynomial $p(\vxx)$ is \emph{sum-of-squares} (\sos) if there exist polynomials
 $q_1(\vxx),\ldots,q_t(\vxx)$ such that $p = q^2_1 + \ldots + q_t^2$.
Now consider a system of polynomial constraints $\calA(\vxx) = \{f_1(\vxx)\geq 0, f_2(\vxx)\geq 0,\ldots,f_m(\vxx)\geq 0\}$ for some given polynomials $f_i(\vxx)$, $i \in [m]$, and 
the inequality $g(\vxx)\geq 0$ (for some polynomial $g(\vxx)$).
We are then interested in defining a proof that $\calA(\vxx)$ implies $g(\vxx)\geq 0$, \ie any $\vxx$ that satisfies $\calA(\vxx)$ is such that $g(\vxx)\geq 0$.

\begin{definition}[Sum-of-squares proof]\label{def:sosProof}
\sosProofDef{eq:sosProof-main}
\end{definition}

From eq.~\eqref{eq:sosProof-main}, it is clear why the polynomials $p_\calS$ are a ``proof'' of $g \geq 0$ for any $\vxx$ satisfying $\calA$: for any $\vxx \in \calA$, $\prod_{i\in\calS}f_i \geq 0$ by definition, hence if we can write $g$ as the product of a sum-of-squares (hence non-negative) polynomial and $\prod_{i\in\calS}f_i$, we automatically prove that $g \geq 0$ whenever $\vxx \in \calA$. 
We note that the existence of an \sos proof is a \emph{sufficient} condition for the fact that 
$g(\vxx) \geq 0$ whenever $\calA(\vxx)$ is satisfied. However, it is not a necessary condition, in that 
there might be valid relations that cannot be proved using \sos proofs. 
For instance, while the polynomial $p(\vxx) = x_1^4 x_2^2 + x_1^2 x_2^4 + 1 - 3 x_1^2 x_2^2$ 
(the Motzkin polynomial, see~\cite[p. 59]{Blekherman12Book-sdpandConvexAlgebraicGeometry}) 
is such that $p(\vxx) \geq 0$ for all $\vxx \in \Real{2}$, it is not possible to develop an sos proof
for such a fact, since $p(\vxx)$ is 
not \sos~\cite[p. 59]{Blekherman12Book-sdpandConvexAlgebraicGeometry}.
In other words, the \sos proof system is more stringent than the traditional proofs we might be used to, 
and properties that hold with traditional proofs might not hold in the \sos sense. At the same time, if we are able to derive an \sos proof we can obtain strong guarantees for our moment relaxations, and existing results reassure us that all relevant properties of moment relaxations can be proven via \sos proofs (see~\cref{app:sosProofs}). 

{\bf How to derive an \sos proof?}
Sum-of-squares provide a proof system to reasons about polynomial
constraints. 
For instance, imagine that our goal is to derive a proof 
that $\calA$ implies $g \geq 0$. 
In a traditional mathematical proof system, we might first prove that 
$\calA$ implies $g' \geq 0$ (for some other polynomial $g'$) and that $g' \geq g$, to finally conclude that $\calA$ implies $g \geq 0$.
Similarly, the \sos proof system provides a systematic mechanism to derive this chain of implications, 
but \emph{with more stringent rules} compared to the ones we are typically used to in robotics and vision.
For instance, we have already observed the fact that $p(\vxx) \geq 0$ for some degree $\satisfyOrder$ polynomial does not necessarily imply that there is a sum-of-squares proof $\sosimply{\vxx}{\satisfyOrder} \{ p(\vxx) \geq 0 \}$.
Similarly, for some polynomial constraints $\calA$ and polynomials $g'(\vxx)$ and $g(\vxx)$
 such that $g(\vxx) \geq g'(\vxx)$ for every $\vxx \in \Real{\dimd}$, 
the fact that $\calA \sosimply{\vxx}{\satisfyOrder} \{ g'(\vxx) \geq 0 \}$ 
does not necessarily imply that $\calA \sosimply{\vxx}{\satisfyOrder} \{ g(\vxx) \geq 0 \}$, since the latter fact might not admit a sum-of-squares proof. In this sense, the \sos proof system is more restrictive than the typical algebraic manipulation we are used to. 
Fortunately, previous work provides a toolkit of inference rules that can be used to correctly reason in the \sos proof system. We collect a set of ``\sos rules'' in~\cref{app:sosProofs}, mostly drawing 
from~\cite{Karmalkar19neurips-ListDecodableRegression,Klivans18arxiv-robustRegression,
Hopkins18stoc-mixtureModelAndSoS,Ma16focs-tensorDecompositionViaSoS,Diakonikolas22pmlr-robustMeanViaSoS}. 

{\bf How to use \sos proofs?} 
Assume we were able to derive a proof that 
$\calA(\vxx) \sosimply{\vxx}{\satisfyOrder} \{ g(\vxx) \geq 0 \}$. 
Then, the key fact that we mention here informally, and formalize in~\cref{app:sosProofs}, 
is that any pseudo-moment matrix that satisfies a moment relaxation of 
  $\calA$ also satisfies a moment relaxation of $g \geq 0$.\footnote{We will need some extra notation 
  to make this claim more precise, and postpone those details to~\cref{app:sosProofs}. 
  } 
  This fact will be instrumental in proving that the solution of the moment relaxations 
  of~\eqref{eq:LTS},~\eqref{eq:TLS}, and~\eqref{eq:MC} have to satisfy some desirable properties, 
  and will be key to deriving the error bounds we present below.
 The \sos proof system has found extensive applications in algorithmic statistics (tracing back to the seminal paper~\cite{Barak12stoc-sos}), leading to the so called ``proof to algorithms'' paradigm, 
  where the proof system directly suggests tractable algorithms to solve a problem.

 The reader should be able to follow the rest of this \paper without reading the appendices, whose material 
 is mostly useful to support the technical proofs and for a more rigorous introduction to the material 
 informally reviewed in this section. At the same time, we invite the reader to use \cref{app:momentRelaxation} to \cref{app:sosProofs} as an 
accessible introduction to the world of moment relaxations and \sos proofs: in those appendices, we attempt to
 bridge the advanced presentation that is typically found in statistics papers and books with the more familiar optimization lens used in robotics and vision.


\section{Estimation Contracts: Problem Statement}
\label{sec:problemStatement}

This \paper is concerned with the following problems.

 \begin{problem}[Outlier-robust estimation in geometric perception]\label{prob:outlierRobustEstimation}
 Estimate $\vxx\gt \in \Domain$ (where $\Domain$ is a basic semi-algebraic set) given $\nrMeasurements$ measurements 
 $(\vy_i,\MA_i)$, $i \in [\nrMeasurements]$, such that a subset of $\an$ measurements
  (the \emph{inliers}) follows the measurement model:
   \begin{align} 
   \label{eq:linModelProbStatement}
 \vy_i = \MA_i\tran \vxx\gt + \vepsilon, \quad  \text{ with } \quad  \vy_i \in \Real{\dimy} 
\quad \text{ and } \quad \vxx\gt \in \Domain \subseteq \Real{\dimx},
 \end{align}
  with $\normTwo{\vepsilon}\leq\barc$ (for a given noise bound $\barc$) and the remaining $\bn$ measurements (the \emph{outliers}, with $\beta = 1-\alpha$) are arbitrary (and possibly adversarially chosen).  
 \end{problem}

 \begin{problem}[Estimation contracts]\label{prob:estimationContracts}
 For each algorithm developed to solve~\cref{prob:outlierRobustEstimation}, 
 provide conditions on the inliers 
 such that the resulting estimate \mbox{is guaranteed to be close to the ground truth~$\vxx\gt$.} 
 \end{problem}

 As discussed in~\cref{sec:motivatingProblems} (and further stressed in~\cref{sec:openProblems} below), Problem~\ref{prob:outlierRobustEstimation} arises in many geometric perception applications.
 Our main focus in this \paper will be on designing estimation contracts (\cref{prob:estimationContracts}): 
rather than proposing radically new algorithms for~\cref{prob:outlierRobustEstimation}, we review existing algorithms  (possibly with small modifications), and then derive suitable conditions under which those algorithms are guaranteed to return good estimates.

\cref{sec:lowOutliers} below studies the low-outlier case, where $\beta \ll 0.5$. 
 There, we review algorithms to attack~\cref{prob:outlierRobustEstimation} 
  based on moment relaxations of~\eqref{eq:LTS},~\eqref{eq:MC}, and~\eqref{eq:TLS}.
Then, for each algorithm, we provide estimation contracts.

\cref{sec:highOutliers} studies the high-outlier case, where $\beta \gg 0.5$. 
In such a case, a sound algorithm must return multiple estimates and our estimation contracts 
derive conditions under which at least one of the returned estimates is close to the ground truth.
The results we present are an adaptation to the geometric perception setup of the recent work on \emph{list-decodable regression} by
Karmalkar\setal~\cite{Karmalkar19neurips-ListDecodableRegression}.


\section{Estimation Contracts for Low Outlier Rates}
\label{sec:lowOutliers}

This section provides estimation contracts for~\eqref{eq:LTS},~\eqref{eq:MC}, and~\eqref{eq:TLS}
 for problems with low-outlier rates (\ie $\beta \ll 0.5$). 
 We start by deriving some naive contracts for~\eqref{eq:MC} and~\eqref{eq:TLS} in~\cref{sec:lowOut-aPosterioriResults}: these naive results 
 only use basic manipulations and do not rely on any of the machinery presented above; at the same time,
 their derivation shares several insights with the more advanced contracts we 
 provide in the subsequent sections and motivates 
 the need for the machinery in~\cref{sec:notationAndPreliminaries}. 
 Then,~\cref{sec:lowOut-aPrioriResults-LTS},~\cref{sec:lowOut-aPrioriResults-MC}, and~\cref{sec:lowOut-aPrioriResults-TLS} 
 present more advanced estimation contracts based on \sos proofs for~\eqref{eq:LTS},~\eqref{eq:MC}, and~\eqref{eq:TLS}, respectively.


\subsection{A Posteriori (Naive) Estimation Contracts for~\eqref{eq:MC} and~\eqref{eq:TLS}}
\label{sec:lowOut-aPosterioriResults}

Here we develop two simple results for~\eqref{eq:MC} and~\eqref{eq:TLS} that quantify the 
distance of an optimal solution of the two problems from the ground-truth parameter $\vxx\gt$.
 These results are straightforward to prove, but have several shortcomings that we discuss in~\cref{rmk:issuesLowOut-aPosterioriResults} at the end of the section. 

Before presenting the results in this section, 
we need to remark that in all the estimation problems considered in~\cref{sec:motivatingProblems}, we cannot reconstruct the unknown $\vxx\gt$ from a single measurement, but we rather need a sufficiently large subset of measurements.

\begin{definition}[\Nondegenerate and minimal measurement set]
A set $\calJ$ of measurements is \emph{\nondegenerate} if the following optimization problem admits a unique solution:
\begin{align}
\label{eq:nondegeneracy}
\min_{\vxx \in \Domain} \sum_{i \in \calJ} \normTwo{ \vy_i - \MA_i\tran \vxx }^2.
\end{align}
A \nondegenerate set of minimal size $\minDim$ is called a \emph{minimal} measurement set. 
\end{definition}

Characterizations of the minimal sets for common geometric perception problems
 are well known in the literature.
Indeed, a subfield of computer vision is specifically concerned with the design of \emph{minimal solvers}, which 
compute an estimate
$\vxx$ from a minimal set of measurements.\footnote{The popularity of minimal solvers stems from their extensive use in outlier rejection schemes, such as RANSAC~\cite{Fischler81}.} 
For instance, the rotation search problem and the 3D registration problem in~\cref{sec:motivatingProblems}
require at least $\minDim = 3$ non-collinear 3D point measurements for the resulting estimate in~\eqref{eq:nondegeneracy} to be unique.

The following proposition provides our first estimation contract, 
which establishes when an optimal solution of~\eqref{eq:MC} is close to the ground truth 
 $\vxx\gt$. 

\noindent\fbox{
\parbox{\textwidth}{
\begin{proposition}[Low-outlier case: a posteriori estimation contract for~\eqref{eq:MC}]\label{thm:lowOut-aposterioriNoisyMC}
Consider~\cref{prob:outlierRobustEstimation} with 
measurements ($\vy_i, \MA_i$), $i \in [\nrMeasurements]$, and
assume the measurement set contains $\an \geq \frac{\nrMeasurements + \minDim}{2}$ inliers, where $\minDim$ is the 
size of a minimal set. Moreover, assume that every subset of $\minDim$ inliers is \nondegenerate. 
Then, for any integer $\dimJ$ such that $\minDim \leq \dimJ \leq (2\inlierRate-1)n$,
an optimal solution $\vxx\maxCon$ of~\eqref{eq:MC} satisfies
\beq
\label{eq:noisyBoundMC-statement}
\normTwo{\vxx\maxCon - \vxx\gt} \leq \naiveBoundMC \mathcom
\eeq
where $\InlierSet\maxCon$ is the set of inliers selected by~\eqref{eq:MC}, $\MA_\calJ$ is the matrix obtained by horizontally stacking all submatrices $\MA_i$ for all $i \in \calJ$, and $\sigma_\min(\cdot)$ denotes the smallest singular value of a matrix. 
Moreover, if the inliers are noiseless, \ie $\vepsilon = \zero$ in eq.~\eqref{eq:linModelProbStatement}, and $\barc = 0$,  then $\vxx\maxCon = \vxx\gt$.
\end{proposition}
}}

We report the proof here (while all other technical proofs are postponed to the appendix) since 
the structure of the proof is quite enlightening in its simplicity. Indeed, we will see that this simple
proof shares many insights with the proofs of more advanced results presented later in this \paper.

\begin{proof}
Given an optimal solution $(\vxx\maxCon,\vomega\maxCon)$ of problem~\eqref{eq:MC}, 
let us call $\InlierSet$ and $\InlierSet\maxCon$ the true inliers and the set of measurements selected by~\eqref{eq:MC} (\ie $\InlierSet\maxCon \triangleq  \setdef{i\in[n]}{ \omega\maxConi{i} = 1 }$), respectively. 
Recall that the proposition assumes $|\InlierSet| = \an \geq \frac{\nrMeasurements + \minDim}{2}$. 
We prove the result in two steps.

\myParagraph{(i) The solution of \eqref{eq:MC} captures enough inliers}
We denote with $\ones_{\InlierSet}$ the indicator vector of the set $\InlierSet$, \ie the $\nrMeasurements$-vector that has the $i$-th entry equal to $1$ if $i \in \InlierSet$ or zero otherwise.
We then note that $(\vxx\gt, \ones_{\InlierSet})$ is a feasible solution for~\eqref{eq:MC}  and attains a cost $\an$. Therefore, by optimality of $(\vxx\maxCon,\vomega\maxCon)$,  
it follows $|\InlierSet\maxCon| \geq \an$.
Since $|\InlierSet| = \an$ and $|\InlierSet\maxCon| \geq \an$ then:
\beq 
|\InlierSet \cap \InlierSet\maxCon| \overbrace{\geq}^{\text{sets overlap in $[\nrMeasurements]$}} (2\inlierRate - 1) \nrMeasurements \overbrace{\geq}^{\text{using } \an \geq \frac{\nrMeasurements + \minDim}{2}} \!\minDim,
\eeq
 \ie the measurements selected by~\eqref{eq:MC} must include at least $(2\inlierRate-1)n \geq  \minDim$ inliers.

\myParagraph{(ii) The inliers captured by~\eqref{eq:MC}  bound the estimation error}
Let us pick any subset $\calJ$ of the set $\InlierSet \cap \InlierSet\maxCon$, such that $|\calJ| = \dimJ$, for any  choice of integer $\dimJ $ such that $\minDim \leq \dimJ \leq (2\inlierRate-1)n$. 
 Note that both $\vxx\gt$ and $\vxx\maxCon$ are feasible for~\eqref{eq:MC} over $\calJ$, hence:
 \begin{align}
 \label{eq:boundsGT1}
 \max_{ i \in  \calJ} \normTwo{ \vy_i - \MA\tran_i \vxx\gt }^2 \leq \barc^2 
 \Longrightarrow
 \sum_{ i \in  \calJ} \normTwo{ \vy_i - \MA\tran_i \vxx\gt }^2 \leq \dimJ \; \barc^2 
 \Longrightarrow 
 \normTwo{ \vy_\calJ - \MA\tran_\calJ \vxx\gt } \leq \sqrt{\dimJ} \;\barc 
  \\
   \label{eq:boundsMC1}
   \max_{ i \in  \calJ} \normTwo{ \vy_i - \MA\tran_i \vxx\maxCon }^2 \leq \barc^2  
   \Longrightarrow
   \sum_{ i \in  \calJ} \normTwo{ \vy_i - \MA\tran_i \vxx\maxCon }^2 \leq \dimJ \; \barc^2  
   \Longrightarrow
  \normTwo{ \vy_\calJ - \MA\tran_\calJ \vxx\maxCon } \leq \sqrt{\dimJ} \;\barc, 
 \end{align} 
 where $\vy_\calJ$ and $\MA\tran_\calJ$ vertically stack all the measurements $\vy_i$ and matrices $\MA\tran_i$ 
 for $i \in \calJ$.

Now we observe that the mismatch between $\vxx\maxCon$ and $\vxx\gt$ after multiplying by $\MA\tran_\calJ$ must be small:
\begin{align}
\normTwo{  \MA\tran_\calJ (\vxx\maxCon - \vxx\gt) } 
\overbrace{=}^{\text{adding and subtracting $\vy_\calJ$}} 
\normTwo{ (\vy_\calJ - \MA\tran_\calJ \vxx\gt) - (\vy_\calJ - \MA\tran_\calJ \vxx\maxCon) }
\\ 
\overbrace{\leq}^{\text{triangle inequality}} 
\normTwo{ \vy_\calJ - \MA\tran_\calJ \vxx\gt } + \normTwo{ \vy_\calJ - \MA\tran_\calJ \vxx\maxCon } 
\overbrace{\leq}^{\text{using~\eqref{eq:boundsGT1}-\eqref{eq:boundsMC1}}} 2 \sqrt{\dimJ} \;\barc.
\label{eq:proof3}
\end{align}
Recalling 
that for a matrix $\MM$ and vector $\vv$, $\normTwo{\MM\vv} \geq \sigma_\min(\MM) \normTwo{\vv}$, where 
$\sigma_\min(\MM)$ is the smallest singular value of $\MM$:
\begin{align}
 \label{eq:proof4}
 \normTwo{ \MA\tran_\calJ (\vxx\maxCon - \vxx\gt) }
\geq \sigma_\min(\MA\tran_\calJ) \normTwo{ \vxx\maxCon - \vxx\gt}.
\end{align}
Combining~\eqref{eq:proof4} and~\eqref{eq:proof3}:
\vspace{-10mm}

 \begin{align}
 \label{eq:proof5}
\normTwo{\vxx\maxCon - \vxx\gt} \leq \frac{  2 \sqrt{\dimJ} \;\barc  }{  \sigma_\min(\MA\tran_\calJ)  }.
\end{align}
Since we do not know which subset $\calJ$ was selected by~\eqref{eq:MC}, we choose the  
set $\calJ$ of cardinality $\dimJ$ attaining the smallest singular value across the set of inliers selected by~\eqref{eq:MC}, yielding 
the desired result for the case of noisy inliers. 
In the case of noiseless inliers and $\barc = 0$,~\eqref{eq:boundsGT1}-\eqref{eq:boundsMC1}  
hold exactly, \ie $\vy_\calJ - \MA\tran_\calJ \vxx\gt = \zero$ and $\vy_\calJ - \MA\tran_\calJ \vxx\maxCon = \zero$ and both $\vxx\gt$ and $\vxx\maxCon$ attain the minimum (with zero cost in this case) of~\eqref{eq:nondegeneracy}. However, from the non-degeneracy assumption, problem~\eqref{eq:nondegeneracy} admits a unique minimizer, hence $\vxx\maxCon = \vxx\gt$, which concludes the proof.
\end{proof}

The proof of~\cref{thm:lowOut-aposterioriNoisyMC} involves two steps: 
(i)~we proved that the solution of~\eqref{eq:MC} must capture a sufficient number of true inliers (\ie there must be enough overlap between the set of measurements selected by~\eqref{eq:MC} and the true inliers),
(ii)~the fact that the solution $\vxx\maxCon$ has to be consistent with the true inliers forces the estimation error to be small. In the noiseless case ($\barc = 0$), the proposition predicts that $\vxx\maxCon = \vxx\gt$ as long as 
there are at least  $\frac{\nrMeasurements + \minDim}{2}$ (\nondegenerate) inliers.\footnote{While still easy to prove, in the case of noiseless inliers and $\barc = 0$, the result $\vxx\maxCon = \vxx\gt$ does not directly follow as a consequence of the noisy bound
$\normTwo{\vxx\maxCon - \vxx\gt} \leq \naiveBoundMC$: in the noiseless case the singular value $\sigma_{\min(\MA_\calJ)}$ might become zero for any subset $\calJ$ (since the ground truth might be in the null space of $\MA_i\tran$, \cf~\eqref{eq:3D3DCategoryLevel} and~\eqref{eq:absolutePose}), therefore both the numerator (in particular, $\barc$) and the denominator of the noisy upper bound go to zero.} 
 The noisy bound in~\eqref{eq:noisyBoundMC-statement} holds for any choice of integer $\dimJ$ between $\minDim$ and $(2\inlierRate-1)n$. As we will see in~\cref{sec:experiments}, larger values of $\dimJ$ lead to better bounds but are harder to compute (due to the need to search over all possible subsets of measurements of size $\dimJ$).

\cref{fig:maxConNoiseless} stresses the key role of \nondegeneracy in the estimation contract in~\cref{thm:lowOut-aposterioriNoisyMC}. The figure shows a simple linear regression problem in 3D where we need to fit a plane 
given 3D points belonging to the plane. In particular, we are given a set of $15$ measurements with $9$ inliers and $6$ outliers. While the set of inliers has size $\frac{\nrMeasurements + \minDim}{2} = 9$ (which satisfies one requirement in the proposition), there are subsets of degenerate points (in particular, the 4 collinear points at the intersection between the two planes). 
In this case,~\eqref{eq:MC} will produce the estimate in red, which is far from the ground truth (black plane).
Intuitively, degenerate sets of measurements can be easily ``stolen'' by an adversary since they are compatible with multiple estimates of $\vxx$ (including estimates far from $\vxx\gt$).

\begin{figure}[t!]
\centering
	\includegraphics[width=0.6\columnwidth, trim= 20mm 80mm 30mm 80mm, clip]{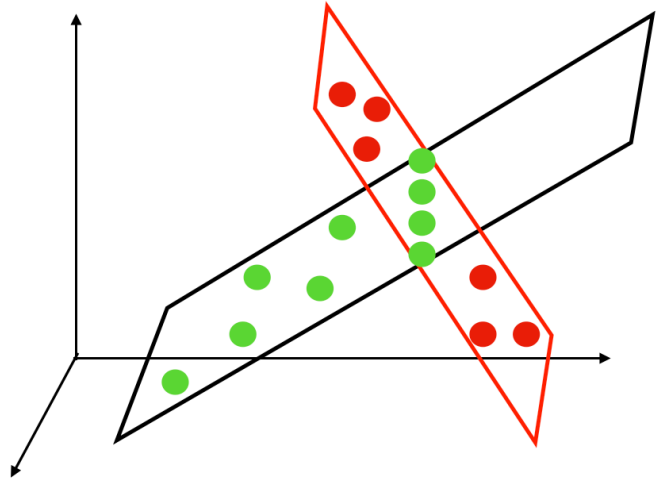} 
	\caption{Plane fitting with 9 inliers (green 3D points) and 6 outliers (red 3D points). 
	True plane is shown in black, while the plane computed by~\eqref{eq:MC} is shown in red. 
	When the estimation contract in~\cref{thm:lowOut-aposterioriNoisyMC} is violated,~\eqref{eq:MC} 
	may fail to recover the variable $\vxx\gt$ even when given $\frac{\nrMeasurements + \minDim}
	{2}$ inliers.  \label{fig:maxConNoiseless}}
\end{figure}

A result similar to~\cref{thm:lowOut-aposterioriNoisyMC} can be proven for~\eqref{eq:TLS}.

\noindent\fbox{
\parbox{\textwidth}{
\begin{proposition}[Low-outlier case: a posteriori estimation contract for~\eqref{eq:TLS}]
\label{thm:lowOut-aposterioriNoisyTLS}
Consider~\cref{prob:outlierRobustEstimation} with 
measurements ($\vy_i, \MA_i$), $i \in [\nrMeasurements]$, and
denote with $\gamma\gt$ the squared residual error of the ground truth $\vxx\gt$ over the set of inliers $\InlierSet$, 
\ie $\gamma\gt \triangleq\sum_{ i \in \InlierSet } \normTwo{ \vy_i - \MA_i \tran \vxx\gt }^2$.
Moreover, assume the measurement set contains at least $\frac{\nrMeasurements + \minDim}{2} + \frac{\gamma\gt}{\barcsq}$ 
inliers, where $\minDim$ is the 
size of a minimal set, and that every subset of $\minDim$ inliers is \nondegenerate. 
Then, for any integer $\dimJ$ such that $\minDim \leq \dimJ \leq (2\inlierRate-1)n - \frac{\gamma\gt}{\barcsq}$,
an optimal solution $\vxx\tls$ of~\eqref{eq:TLS} satisfies
\beq
\label{eq:noisyBoundTLS-statement}
\normTwo{\vxx\tls - \vxx\gt} \leq \naiveBoundTLS \mathcom
\eeq
where $\InlierSet\tls$ is the set of inliers selected by~\eqref{eq:TLS}, $\MA_\calJ$ is the matrix obtained by horizontally stacking all submatrices $\MA_i$ for all $i \in \calJ$, and $\sigma_\min(\cdot)$ denotes the smallest singular value of a matrix. 
Moreover, if the inliers are noiseless, \ie $\vepsilon = \zero$ 
in eq.~\eqref{eq:linModelProbStatement}, and for a sufficiently small $\barc > 0$, 
 $\vxx\tls = \vxx\gt$.
\end{proposition}
}}


While~\cref{thm:lowOut-aposterioriNoisyMC} 
and~\cref{thm:lowOut-aposterioriNoisyTLS} allow us to gain intuition about the problem, 
the applicability of these estimation contracts is limited, as we discuss below.

\begin{remark}[A posteriori estimation contracts]
\label{rmk:issuesLowOut-aPosterioriResults} 
The estimation contracts in \cref{thm:lowOut-aposterioriNoisyMC} and~\cref{thm:lowOut-aposterioriNoisyTLS} 
 refer to the optimal solutions of~\eqref{eq:MC} and~\eqref{eq:TLS}, respectively. 
 However, both problems are NP-hard~\cite{Antonante21tro-outlierRobustEstimation,Chin18eccv-robustFitting}, hence computing optimal solutions is difficult in general.
 We can still apply these contracts as follows. 
 Assume that we solve~\eqref{eq:MC} and~\eqref{eq:TLS} using the moment relaxation described in~\cref{sec:notationAndPreliminaries} (see also~\cite{Yang22pami-certifiablePerception}) and that the relaxation produces a certifiably optimal result (\cf the workflow~\eqref{eq:workflow}); then, such a certifiably optimal solution would still enjoy the guarantees in~\cref{thm:lowOut-aposterioriNoisyMC} and~\cref{thm:lowOut-aposterioriNoisyTLS}. 
 However, since we can only check optimality a posteriori (and only apply these contracts when the relaxation is tight), we call these contracts ``a posteriori'' estimation contracts. 
\end{remark}

In practice, we would like to have more general performance guarantees regardless of the tightness of the 
moment relaxation. Luckily, the \sos proof machinery we introduced in~\cref{sec:notationAndPreliminaries} 
does exactly that: it allows inferring properties of the solution of moment relaxations, regardless of its tightness.
Using such a machinery, we can derive ``a priori'' estimation contracts as shown below. 	

\subsection{A Priori Estimation Contracts for~\eqref{eq:LTS},~\eqref{eq:MC}, and~\eqref{eq:TLS}}
\label{sec:lowOut-aPrioriResults}

In the rest of this \paper we develop ``a priori'' estimation contracts, which establish conditions 
on the input data such that certain algorithms based on moment relaxations of~\eqref{eq:LTS},~\eqref{eq:MC}, and~\eqref{eq:TLS} are able to compute an estimate for~\cref{prob:outlierRobustEstimation} 
which is provably close to the ground truth $\vxx\gt$. 

The philosophy is quite different from the typical work done in robotics and vision. 
While in those fields, one is typically concerned with the distribution 
of the measurement noise $\vepsilon$ in~\eqref{eq:linModel}, here we are concerned with both 
the measurement noise and the distribution producing the matrices $\MA_i$ in~\eqref{eq:linModel}.\footnote{While these matrices are considered given in robotics and computer vision applications, we can imagine they are generated by sampling from a distribution.} 
One for instance might assume the entries of the measurement matrices to be sampled from a zero-mean Gaussian and 
study the behavior of an outlier-robust estimator. 
The importance of these matrices for outlier-robust estimation should already be apparent from the statements of~\cref{thm:lowOut-aposterioriNoisyMC} and~\cref{thm:lowOut-aposterioriNoisyTLS}.  
In the following, we write $\MA$ to denote a matrix random variable, while $\MA_i$, $i\in[\nrMeasurements]$, are the realizations of such a random variable.
Moreover, we denote with $\DistInliers$ the distribution producing the matrices $\MA_i$, $i\in[\nrMeasurements]$ for the inlier measurements (note the tilde, which we use to distinguish the 
distribution of the inliers from the inlier set $\InlierSet \subseteq [\nrMeasurements]$, \cf~\cref{thm:lowOut-aposterioriNoisyMC}).

In this context, we hope to make the least restrictive assumptions on $\DistInliers$; in other words, we would like 
our estimators to be guaranteed to work well for a broad class of distributions that generate the input data. 
Moreover, we want to restrict any assumptions to the inliers, while the mechanism generating outliers can be arbitrary.
In particular we will consider two large families of distributions: 
certifiable hypercontractive and certifiable anti-concentrated distributions, which we introduce below.
These definitions are based on~\cite{Klivans18arxiv-robustRegression,Karmalkar19neurips-ListDecodableRegression}, 
but we extend both to be defined over matrices.

\begin{definition}[Certifiable hypercontractivity for matrices, adapted from~\cite{Klivans18arxiv-robustRegression}]
For a function $C : [k] \mapsto \Real{}_+$, we say that a distribution $\DistInliers$ over matrices $\MA$ 
is $k$-certifiably C-hypercontractive if for every $t \leq k/2$, there is a degree $k$ 
sum-of-squares proof of the following inequality in variable $\vv$ 
\begin{align}
\label{eq:hyper-def}
\Expect{\DistInliers}{\normTwo{ \MA\tran \vv }^{2t} } \leq C(t)^t 
\left( \Expect{\DistInliers}{\normTwo{ \MA\tran \vv }^2 } \right)^{t}.
\end{align}
We say that a set of matrices $\MA_i$, $i \in \InlierSet$, is $k$-certifiably C-hypercontractive if the uniform distribution over the 
set is $k$-certifiably C-hypercontractive, \ie
 \begin{align}
 \label{eq:hyper-uniform}
\left( \aveOverInliers \normTwo{ \MA_i\tran \vv }^{2t} \right) \leq C(t)^t 
\left( \aveOverInliers \normTwo{ \MA_i\tran \vv }^2 \right)^{t}.
\end{align}
\end{definition}

We remark that certifiable hypercontractivity not only requires~\eqref{eq:hyper-def} (or~\eqref{eq:hyper-uniform}) to be satisfied, but also requires this fact to have a sum-of-squares proof (see~\cref{sec:notationAndPreliminaries}). 
Hypercontractivity essentially requires controlling the moments of $\normTwo{ \MA\tran \vv }$, 
or, in the case of~\eqref{eq:hyper-uniform}, it bounds the $2t$-norm of the vector with entries 
$\normTwo{ \MA_i\tran \vv }$ with respect to the $2$-norm. In the proofs, this property will be used to bound 
terms such as $\normTwo{ \MA\tran (\vxx - \vxx\gt) }^2$ for some $\vxx$, an aspect that was also key to proving~\cref{thm:lowOut-aposterioriNoisyMC} and~\cref{thm:lowOut-aposterioriNoisyTLS} (\cf with eq.~\eqref{eq:proof4}). 
What makes certifiable hypercontractivity interesting is the fact that 
a large class of distributions satisfies this property, including Gaussians, 
product of distributions with sub-gaussian marginals, and uniform distributions over the boolean 
hypercube~\cite{Bakshi21soda-listDecodableSubspaceRecovery}.
Moreover, in~\cref{sec:experiments} we show that this property is empirically satisfied in many real problems.
Finally, we note that, given a set of matrices $\MA_i$, $i \in \InlierSet$, we can easily check if~\eqref{eq:hyper-uniform} is satisfied (for a given $t$), since one can check if the polynomial expression~\eqref{eq:hyper-uniform} (in variable $\vv$) is sum-of-squares in polynomial time using semidefinite programming (see~\cref{sec:experiments},~\cref{app:sosProofs}, and~\cite[chapter 3]{Blekherman12Book-sdpandConvexAlgebraicGeometry}). 

While we use hypercontractivity in~\cref{thm:lowOut-apriori-LTS-objective}, 
in most of the \paper, we will use a stronger assumption on the inlier distribution, known as \emph{certifiable anti-concentration}.
Since this property is more involved, we first introduce the concept of anti-concentration, 
and then provide its \sos counterpart, \ie certifiable anti-concentration.
Both~\cref{def:Anticon} and~\cref{def:certAnticon} are based on~\cite{Karmalkar19neurips-ListDecodableRegression} and extend the corresponding definitions to 
work with matrices (rather than vectors).

\begin{definition}[Anti-concentration for matrices, adapted from~\cite{Karmalkar19neurips-ListDecodableRegression}]
\label{def:Anticon}
A zero-mean random matrix $\MA \in \Real{\dimA}$ has a \emph{$\delta$-anti-concentrated} distribution if 
$\prob{ \normTwo{\MA\tran \vv} = 0  } \leq \delta$ for all non-zero $\vv$. 
\end{definition}

In the technical proofs, we are going to apply these properties to terms such as 
$\normTwo{\MA_i\tran (\vxx - \vxx\gt)}$ and anti-concentration is essentially asking that 
any mismatch vector $(\vxx - \vxx\gt)$ remains ``visible'' (with some probability) after mapping it through $\MA_i\tran$; similarly to the proof of~\cref{thm:lowOut-aposterioriNoisyMC}, this will allow us to bound $\normTwo{\vxx - \vxx\gt}$ given a bound on $\normTwo{\MA_i\tran (\vxx - \vxx\gt)}$.
Anti-concentration 
is also related to the probability of sampling a set of matrices 
$\MA_i$, $i \in \InlierSet$, such that $[\MA_1\tran\vcat \MA_2\tran \vcat\ldots\vcat \MA\tran_{|\InlierSet|}]$ is not full rank. In this sense, it is connected to the notion of non-degeneracy 
and to the results in~\cref{sec:lowOut-aPosterioriResults} (\eg if a matrix is not full rank, its smallest singular value is zero, \cf~\cref{thm:lowOut-aposterioriNoisyMC}).

The \sos version of~\cref{def:Anticon} reads as follows.

\begin{definition}[Certifiable anti-concentration for matrices, adapted from~\cite{Karmalkar19neurips-ListDecodableRegression}]
\label{def:certAnticon}
A random matrix $\MA$ has a \emph{$k$-certifiably $(C,\delta,M)$-anti-concentrated} distribution if there is an even univariate polynomial $p$ satisfying $p(0)=1$ such that there is a degree $k$ sum-of-squares proof of the following two inequalities:\footnote{We remark that $p\left(\normTwo{\MA\tran \vv}\right)$ in eq.~\eqref{eq:certAntiCon1} 
is also a polynomial function of the vector $\MA\tran \vv$  since $p$ is an even polynomial, hence it only includes monomials of even degree which can be written as a function of the polynomial $\normTwo{\MA\tran \vv}^2$; the same observation holds for $p\left(\normTwo{\MA_i\tran \vv}\right)$ in eq.~\eqref{eq:certAntiConSet1} with respect to $\MA_i\tran \vv$. }
\begin{align}
&\forall \vv \;, \; \normTwo{\MA\tran \vv}^2 
\leq \delta^2 
\quad \text{implies} \quad
p^2\left(\normTwo{\MA\tran \vv} \right)  \geq (1-\delta)^2, \label{eq:certAntiCon1}
\\
&\forall \vv \;, \; \normTwo{\vv}^2 \leq M^2 \quad \text{implies} \quad
\|\vv\|^2 \cdot \Expect{}{  \; p^2\left(\normTwo{\MA\tran \vv} \right) } \leq C \, \delta \,M^2. \label{eq:certAntiCon2}
\end{align}
A set of matrices $\MA_i$, $i \in \InlierSet$, is \emph{$k$-certifiably $(C,\delta,M)$-anti-concentrated} 
if the uniform distribution over the set is \emph{$k$-certifiably $(C,\delta,M)$-anti-concentrated}, \ie
\begin{align}
&\forall \vv \;, \; \normTwo{\MA_i\tran \vv}^2 
\leq \delta^2  
\quad \text{implies} \quad
 p^2\left(\normTwo{\MA_i\tran \vv} \right) \geq (1-\delta)^2, \label{eq:certAntiConSet1}
\\
&\forall \vv \;, \; \normTwo{\vv}^2 \leq M^2 \quad \text{implies} \quad
\|\vv\|^2 \cdot
\aveOverInliers \; p^2\left(\normTwo{\MA_i\tran \vv}\right)  \leq C \, \delta \,M^2. \label{eq:certAntiConSet2}
\end{align}\end{definition}

The connection between~\cref{def:Anticon} and its \sos counterpart in~\cref{def:certAnticon} 
is not immediate, so a few comments are in order. First, let us consider the probability in~\cref{def:Anticon} 
and observe that by definition, for any given $\boundOne \geq 0$,
$\prob{ \normTwo{\MA\tran \vv}^2 \leq \barcsq } = \Expect{}{ \indicator{ \normTwo{\MA\tran \vv}^2 \leq \barcsq } }$, 
where $\Expect{}{\cdot}$ is the standard expectation and $\indicator{\cdot}$ denotes the indicator function, which is such that, for a boolean condition $a$, $\indicator{a}=1$ if the condition is satisfied or zero otherwise.
Ideally, we would like to just define certifiable anti-concentration to be 
satisfied when there is an \sos proof for:
\begin{align}
\label{eq:anticonInsight2}
\prob{ \normTwo{\MA\tran \vv}^2 \leq \barcsq } = \Expect{}{ \indicator{ \normTwo{\MA\tran \vv}^2 \leq \barcsq } } \leq \delta.
\end{align} 
Unfortunately, 
the indicator function is not a polynomial so such a requirement would not make sense.
Therefore, the insight behind~\cref{def:certAnticon} is twofold.
 First, we require the existence of a polynomial $p$ that ``behaves'' like the indicator function (\cf conditions~\eqref{eq:certAntiCon1} and~\eqref{eq:certAntiConSet1}): this polynomial is required to be close to $1$ (\ie close to the indicator function) whenever the input is smaller than a threshold. Second, we impose such a polynomial to satisfy an anti-concentration bound similar to~\eqref{eq:anticonInsight2} (\cf conditions~\eqref{eq:certAntiCon2} and~\eqref{eq:certAntiConSet2}). Intuitively, the right-hand-side inequality in~\eqref{eq:certAntiConSet2} 
 prevents that vectors with large norm $\| \vv \|$ produce small $\normTwo{\MA_i\tran \vv}$, since that would in turn make $p^2\left(\normTwo{\MA_i\tran \vv}\right)$ large, hence violating the inequality.  
Certifiable anti-concentration will allow us to control the norm $\normTwo{\vxx-\vxx\gt}$ from the norm of $\MA_i\tran(\vxx-\vxx\gt)$, similarly to what we did in the proofs in~\cref{sec:lowOut-aPosterioriResults}. 

Certifiable anti-concentration, in its original definition in~\cite{Karmalkar19neurips-ListDecodableRegression}, is a stronger requirement compared to certifiable hypercontractivity, 
but it still encompasses several relevant distributions, including the standard Gaussian distribution 
and any anti-concentrated spherically symmetric distributions with strictly sub-exponential tails~\cite{Karmalkar19neurips-ListDecodableRegression}. In addition, this requirement has been shown to be \emph{necessary} for list-decodable linear regression in~\cite{Karmalkar19neurips-ListDecodableRegression}. 
We remark that our definition is slightly different from the one in~\cite{Karmalkar19neurips-ListDecodableRegression}, \eg we do not normalize the right-hand-side of the inequality $\normTwo{\MA\tran \vv}^2 
\leq \delta^2$ by the variance of $\normTwo{\MA\tran \vv}$ and assume a slightly different implication in~\eqref{eq:certAntiCon1} and~\eqref{eq:certAntiConSet1}; these choices were made for the sake of simplicity, at the cost of some desirable properties of the original definition (\eg scale invariance) and the need for an additional parameter $M$ in the definition. In~\cref{sec:experiments}, we empirically show that certifiable anti-concentration is much harder to satisfy in practice (in particular for low-degree polynomials $p$), but it still applies to some practical problems.
We are now ready the present the main results of this \paper.

%
%
%
%
%

\subsubsection{Estimation Contracts for~\eqref{eq:LTS}}
\label{sec:lowOut-aPrioriResults-LTS}

This section presents estimation contracts for two slightly different estimators based on moment relaxations of the~\eqref{eq:LTS} problem. The first contract is from~\cite{Klivans18arxiv-robustRegression} (which we adapt to vector-valued measurements) and  bounds the residual error of the estimate with respect to the inliers; the second is novel and directly bounds the distance of the estimate with respect to the ground truth.

\begin{algorithm}[h!]
	\caption{Moment relaxation for~\eqref{eq:LTS}, version 1~\cite{Klivans18arxiv-robustRegression}.\label{algo:lts1}}
	\SetAlgoLined
	\KwIn{ input data $(\vy_i,\MA_i)$, $i \in [\nrMeasurements]$, inlier rate $\inlierRate$, exponent $\relaxLevel$,
	relaxation order $\relaxOrder \geq \relaxLevel$.}
	\KwOut{ estimate of $\vxx\gt$. }

	\tcc{Algorithm solves a relaxation of the following~\eqref{eq:LTS} problem:\\
	\vspace{-5mm}
	\beq
	\hspace{-5mm}
	\label{eq:LTS-k}\tag{LTS1}
	\min_{ \vomega, \vxx, \bar{\vy}_i, \bar{\MA}_i, i\in[\nrMeasurements] } 
	\!\! \left( \aveOverMeas \normTwo{ \bar{\vy}_i - \bar{\MA}_i\tran \vxx }^2 \right)^{\frac{\relaxLevel}{2}} 
	\text{s.t.} 
	\;\;
	\axiomsLTS \triangleq \left\{
	\begin{array}{cl}
	\omega_i^2 = \omega_i, \;\; i\in[\nrMeasurements] \\ 
	\sumOverMeas{\omega_i} = \alpha \nrMeasurements \\
	\omega_i \cdot(\bar{\vy}_i - {\vy}_i) = \zero \;\; i\in[\nrMeasurements]\\
	\omega_i \cdot(\bar{\MA}_i - {\MA}_i) = \zero \;\; i\in[\nrMeasurements]\\
	\vxx \in \Domain
	\end{array}
	\right\}
	\hspace{-18mm}
	\eeq
	\vspace{-5mm}
	}

	\tcc{Compute matrix $\MX\opt$ by solving SDP resulting from moment relaxation}
	$\MX\opt = \scenario{solve\_moment\_relaxation\_at\_order\_}r\,\eqref{eq:LTS-k}$ 
	\myskip

    \tcc{Pick entries of $\MX\opt$ corresponding to $\vxx$}
    $\vxx\ltssdp \triangleq \getEntries{\MX\opt}{\vxx}$
   \myskip

	\Return{$\vxx\ltssdp$.}
\end{algorithm}

Let us first review the estimator proposed in~\cite{Klivans18arxiv-robustRegression}, that we report in~\cref{algo:lts1}.
The algorithm corresponds to Algorithm 5.2 in~\cite{Klivans18arxiv-robustRegression}, with the 
exception that we consider vector-valued measurements and the unknown $\vxx$ in our problem 
belongs to a basic semi-algebraic set $\Domain$.
The algorithm is based on a moment relaxation of a slightly different (but equivalent) reformulation of~\eqref{eq:LTS}, given in~\eqref{eq:LTS-k} within~\cref{algo:lts1}. In particular, 
problem~\eqref{eq:LTS-k} enforces the constraint $\omega_i^2 = \omega_i$, which is equivalent to the constraint  
$\omega_i \in \{0;1\}$ in~\eqref{eq:LTS}; moreover, similar to~\eqref{eq:LTS}, it includes the constraints $\sumOverMeas \omega_i = \an$ and $\vxx \in \Domain$.
However, problem~\eqref{eq:LTS-k} also includes extra variables $\bar{\vy}_i, \bar{\MA}_i, i\in[\nrMeasurements]$,  that are such that $\bar{\vy}_i = \vy_i$ and $\bar{\MA}_i = {\MA}_i$ whenever $\omega_i=1$ (which is enforced by the constraints $\omega_i \cdot(\bar{\vy}_i - {\vy}_i) = \zero, \omega_i \cdot(\bar{\MA}_i - {\MA}_i) = \zero$), or are zero otherwise (which is enforced by the objective, since whenever $\omega_i=0$ these variables are no longer constrained).
 Clearly, this is equivalent to using $\omega_i$ in the objective as we did in~\eqref{eq:LTS}.
Finally, problem~\eqref{eq:LTS-k} also elevates the objective to the power $k/2$ (where $k$ is an input parameter), which again does not change the optimal solution $\vxx\lts$ of the non-relaxed problem as compared to the formulation~\eqref{eq:LTS}. In~\eqref{eq:LTS-k}, we denoted with $\axiomsLTS$ the set of constraints of the problem, 
since in the technical proofs we will establish \sos proofs that connect the set of polynomial constraints $\axiomsLTS$ to relevant properties of the estimate $\vxx\ltssdp$.
After solving the moment relaxation of~\eqref{eq:LTS-k} at order $\relaxOrder$,
\cref{algo:lts1} returns the entries of the (pseudo-moment) matrix $\MX\opt$ corresponding to $\vxx$, \ie $\getEntries{\MX\opt}{\vxx}$.

We are now ready to present the estimation contract for~\cref{algo:lts1}, which is 
an adaptation of Theorem 5.1 in~\cite{Klivans18arxiv-robustRegression} 
to vector-valued measurements.\footnote{
The reader might notice that, at first glance, the statement~\cref{thm:lowOut-apriori-LTS-objective} is quite different from the statement of Theorem 5.1 in~\cite{Klivans18arxiv-robustRegression}; however, the substance is identical: we work directly on the set of matrices $\MA_i$, $i\in[\nrMeasurements]$, rather than the generative distribution creating such samples (in the notation of~\cite{Klivans18arxiv-robustRegression}, we work on $\widehat{\calD}$ 
instead of $\calD$), to avoid the complexity of discussing generalization bounds and requiring bounds on the bit complexity of $\vxx\opt$, which do not add much to the discussion in this \paper.
}

\noindent\fbox{
\parbox{\textwidth}{
\begin{theorem}[Low-outlier case: a priori estimation contract for~\cref{algo:lts1}, adapted from  Theorem 5.1 in~\cite{Klivans18arxiv-robustRegression}]
\label{thm:lowOut-apriori-LTS-objective}
\ltsObjectiveThm{eq:lowOut-apriori-LTS-objective}
\end{theorem}
}}

The theorem states that ---under certifiable hypercontractivity and as long as the outlier rate is below $\outlierRate_\max$---
 the estimate $\vxx\ltssdp$ fits well the inlier measurements.
In order words, the residual error of $\vxx\ltssdp$ over the inliers 
is not much larger than 
the best residual $\optt_{\DistSampleInliers}$ that an ``oracle'' estimator that 
has access to all the inliers would achieve, plus some some higher-order 
terms that depend on the optimal estimate $\vxx\opt$ returned by the oracle.

The proof of~\cref{thm:lowOut-apriori-LTS-objective} given in~\cref{app:proof-lowOut-apriori-LTS-objective} is quite involved, but a key tool in it is to bound 
terms such as $\MA_i\tran (\vxx - \vxx\opt)$ using certifiable hypercontractivity.
In particular, the proof derives an \sos proof that the constraint set $\axiomsLTS$
of~\eqref{eq:LTS-k} implies a desired error bound, and then uses this proof to infer the bound 
in~\eqref{eq:lowOut-apriori-LTS-objective}. On the downside, 
\cref{thm:lowOut-apriori-LTS-objective} requires a moment relaxation of order $\relaxOrder \geq \relaxLevel \geq 4$: this 
requirement stems from the fact that the objective function  in~\eqref{eq:LTS-k} 
 is a polynomial of degree $2\relaxLevel$ (which requires a moment relaxation of order at least $\relaxOrder \geq \relaxLevel$) 
  and the fact that the theorem demands $k\geq 4$.
Unfortunately, 
 solving moment relaxations of order $4$ is currently impractical.\footnote{Recall from~\cref{sec:notationAndPreliminaries} that the moment matrix has size 
 $\nchoosek{d + \relaxOrder}{\relaxOrder} \times \nchoosek{d +\relaxOrder}{\relaxOrder}$ 
  where $d$ is the dimension of the variables in the polynomial optimization problem~\eqref{eq:LTS-k} and $\relaxOrder$ is the relaxation order; note that the dimension $d$ grows with the number of measurements $\nrMeasurements$, and in practical problems $\nrMeasurements \gg 100$. 
  Therefore, the matrix quickly becomes too large to handle for current SDP solvers. As an example, for $d = 100$ and 
  $\relaxOrder = 4$, $\nchoosek{d+\relaxOrder}{\relaxOrder} = \num{4598126}$.} 
 For instance, related work~\cite{Yang22pami-certifiablePerception} 
 relies on relaxations of order $2$ and still has to develop sparse approximations and ad-hoc solvers to make them run in a reasonable time. 
 Moreover, the bound~\eqref{eq:lowOut-apriori-LTS-objective} does not immediately inform us about how far the estimate is from the ground truth.

In the rest of this section we present a second algorithm and the corresponding estimation contract that 
directly quantifies the distance between the estimate and the ground truth. This result is novel, but relies on two lemmas proposed for a different goal in~\cite{Karmalkar19neurips-ListDecodableRegression}.

Let us start by stating a slightly different~\eqref{eq:LTS}-like estimator in~\cref{algo:lts2}.
\begin{algorithm}[h!]
	\caption{Moment relaxation for~\eqref{eq:LTS}, version 2.\label{algo:lts2}}
	\SetAlgoLined
	\KwIn{ input data $(\vy_i,\MA_i)$, $i \in [\nrMeasurements]$, inlier rate $\inlierRate$, relaxation order $\relaxOrder \geq 2$.}
	\KwOut{ estimate of $\vxx\gt$. }

	\tcc{Algorithm solves a relaxation of the following~\eqref{eq:LTS}-like problem:\\
	\vspace{-5mm}
	\beq
	\label{eq:LTS-T}\tag{LTS2}
	 \min_{ \vomega, \vxx } 
	\;\;  \aveOverMeas \omega_i \cdot \normTwo{ \vy_i - {\MA}_i\tran \vxx }^2 
	\text{s.t.} 
	\;\;
	\axiomsLTST \triangleq \left\{
	\begin{array}{cl}
	\omega_i^2 = \omega_i, \;\; i\in[\nrMeasurements] \\ 
	\sumOverMeas{\omega_i} = \alpha \nrMeasurements \\
	\omega_i \cdot \residuali{\vxx}^2 \leq \barc^2 \;\; i\in[\nrMeasurements]\\
	\vxx \in \Domain
	\end{array}
	\right\}
	\eeq
	\vspace{-5mm}
	}

	\tcc{Compute matrix $\MX\opt$ by solving SDP resulting from moment relaxation}
	$\MX\opt = \scenario{solve\_moment\_relaxation\_at\_order\_}r\,\eqref{eq:LTS-T}$ 
	\myskip

    \tcc{Compute estimate}
    for each $i \in [\nrMeasurements]$ set:  $\vv_i = 
    \left\{
    \begin{array}{ll}
    \frac{\getEntries{\MX\opt}{\omega_i\vxx}}{\getEntries{\MX\opt}{\omega_i}} & \text{ if $\getEntries{\MX\opt}{\omega_i}>0$}
    \\
    \zero & \text{ otherwise}
    \end{array}\right.$ \label{line:vi}

    $\vxx\ltssdpT = \sumOverMeas \frac{ \getEntries{\MX\opt}{\omega_i} }{ \sumOverMeasj \getEntries{\MX\opt}{\omega_j} } \vv_i$ \label{line:ltssdpT}
    \myskip

	\Return{$\vxx\ltssdpT$.}
\end{algorithm}

\cref{algo:lts2} is fairly different from~\cref{algo:lts1}.
First of all, it is based on a relaxation of an ``\eqref{eq:LTS}-like'' problem: 
the non-relaxed problem~\eqref{eq:LTS-T} has additional constraints $\omega_i \cdot \residuali{\vxx}^2 \leq \barc^2$ that do not appear in~\eqref{eq:LTS}, but it can be seen to be equivalent to~\eqref{eq:LTS} with $f_i(\vxx) = \MA_i\tran \vxx$ otherwise.\footnote{Note that the pair $(\vxx\gt,\vomega\gt)$, where $\omega_i\gt=1$ if $i\in \InlierSet$ and zero otherwise, satisfies all the constraints in~\eqref{eq:LTS-T}, hence the problem is feasible. } 
Moreover, while~\cref{algo:lts2} also uses a moment relaxation, it computes an estimate $\vxx\ltssdpT$ by averaging multiple vectors extracted from the solution of the moment relaxation (lines~\ref{line:vi}-\ref{line:ltssdpT}). 

We provide the following estimation contract for~\cref{algo:lts2}.

\noindent\fbox{
\parbox{\textwidth}{
\begin{proposition}[Low-outlier case: a priori estimation contract for~\cref{algo:lts2}]
\label{thm:lowOut-apriori-LTS}
\ltsThm{eq:lowOut-apriori-LTS}
\end{proposition}
}}
While~\cref{thm:lowOut-apriori-LTS} takes a stronger assumption on the input data 
(see the discussion on certifiable hypercontractivity vs. certifiable anti-concentration in~\cref{sec:lowOut-aPrioriResults}), 
it provides a direct bound on the distance of the estimate from the ground truth $\vxx\gt$. 
\cref{fig:boundLTST} plots the bound on the right-hand-side of eq.~\eqref{eq:lowOut-apriori-LTS} as a solid blue line, for $\boundx = 1$, $\eta = 0.01$, and inlier rates $\inlierRate$ between $0$ and $1$. 
For comparison, we also report the trivial bound $\normTwo{\vxx\ltssdpT - \vxx\gt} \leq 2 \boundx$ as a dashed red line.\footnote{The bound follows from the triangle inequality $\normTwo{\vxx\ltssdpT - \vxx\gt} \leq \normTwo{\vxx\ltssdpT} + \normTwo{\vxx\gt} \leq 2 \boundx$.} 
We observe that, as expected, the proposed bound is only informative in the 
low-outlier case (\ie when $\inlierRate > 0.5$). Moreover, the bound predicts decreasing estimation errors when the number of inliers increases, consistently with our expectations, and approaches $\frac{\boundx \eta}{2}$ as $\inlierRate$ approaches 1.

\begin{figure}[h!]
\centering
	\includegraphics[width=0.6\columnwidth, trim= 0mm 0mm 0mm 0mm, clip]{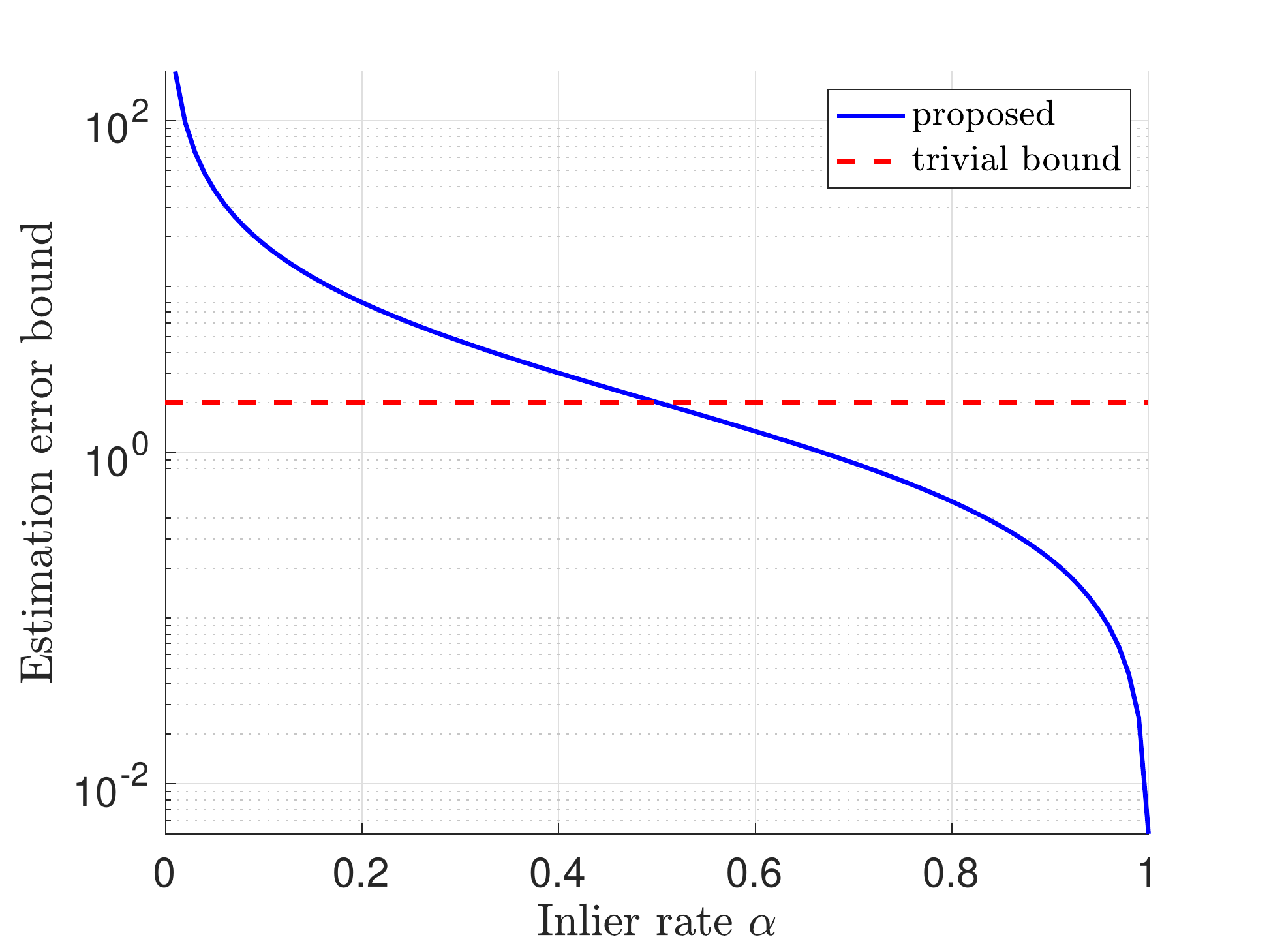}\vspace{-3mm} 
	\caption{Comparison between the proposed bound in eq.~\eqref{eq:lowOut-apriori-LTS} 
	with $\boundx = 1$, $\eta = 0.01$ (solid blue line) and the trivial bound $\normTwo{\vxx\ltssdpT - \vxx\gt} \leq 2 \boundx$ (dashed red line) for inlier rates $\inlierRate \in [0,1]$.   \label{fig:boundLTST}}
\end{figure}

We conclude this section by noting that ---while in principle~\cref{algo:lts2} might require a relaxation of order as low as $2$, similar to the ones commonly used in practice (see~\cite{Yang22pami-certifiablePerception})--- 
the choice of the order is still dictated by the anti-concentration properties of the input data, due to the constraint $\relaxOrder \geq \relaxLevel/2$, where $\relaxOrder$ is the order of the moment relaxation, while $\relaxLevel$ depends on the anti-concentration properties of the inliers. As we will see in~\cref{sec:experiments}, $\relaxLevel$ will need to be  $\geq6$ (due to the degree of the polynomials involved in~\eqref{eq:certAntiConSet2}), thus again restricting the practicality of the result in~\cref{thm:lowOut-apriori-LTS}, which in practice would only be applicable to relaxations of order $\geq 3$. 


\subsubsection{Estimation Contract for~\eqref{eq:MC}}
\label{sec:lowOut-aPrioriResults-MC}

This section presents an estimation contract for~\cref{algo:mc} below, 
which is based on a moment relaxation of the~\eqref{eq:MC} problem. 
Notice that~\eqref{eq:MC-lin} in~\cref{algo:mc}  is equivalent to~\eqref{eq:MC} for $f_i(\vxx) = \MA_i\tran \vxx$.
Let us start by presenting the algorithm the estimation contract refers to.
\begin{algorithm}[h!]
	\caption{Moment relaxation for~\eqref{eq:MC}.\label{algo:mc}}
	\SetAlgoLined
	\KwIn{ input data $(\vy_i,\MA_i)$, $i \in [\nrMeasurements]$, relaxation order $\relaxOrder \geq 2$.}
	\KwOut{ estimate of $\vxx\gt$. }

	\tcc{Algorithm solves a relaxation of the following~\eqref{eq:MC} problem:\\
	\vspace{-1mm}
	\beq
	\label{eq:MC-lin}\tag{MC1}
	\max_{ \vomega, \vxx} 
	\;\;  \sumOverMeas \omega_i, \qquad
	\text{s.t.} 
	\;\;
	\axiomsMC \triangleq \left\{
	\begin{array}{cl}
	\omega_i^2 = \omega_i, \;\; i\in[\nrMeasurements] \\ 
	\omega_i \cdot \residuali{\vxx}^2 \leq \barc^2 \;\; i\in[\nrMeasurements]\\
	\vxx \in \Domain
	\end{array}
	\right\}
	\eeq
	\vspace{-5mm}
	}

	\tcc{Compute matrix $\MX\opt$ by solving SDP resulting from moment relaxation}
	$\MX\opt = \scenario{solve\_moment\_relaxation\_at\_order\_}r\,\eqref{eq:MC-lin}$ 
	\myskip

    \tcc{Compute estimate}
    for each $i \in [\nrMeasurements]$ set: $\vv_i = 
    \left\{
    \begin{array}{ll}
    \frac{ \getEntries{\MX\opt}{\omega_i\vxx} }{ \getEntries{\MX\opt}{\omega_i} } & \text{ if $\getEntries{\MX\opt}{\omega_i}>0$}
    \\
    \zero & \text{ otherwise}
    \end{array}\right.$

    $\vxx\mcsdp = \sumOverMeas \frac{ \getEntries{\MX\opt}{\omega_i} }{ \sumOverMeasj \getEntries{\MX\opt}{\omega_j} } \vv_i$
    \myskip

	\Return{$\vxx\mcsdp$.}
\end{algorithm}

Contrary to the algorithms in~\cref{sec:lowOut-aPrioriResults-LTS},~\cref{algo:mc} does not require prior knowledge about the number of inliers $\an$. The moment relaxation in~\cref{algo:mc} is similar to the one proposed in~\cite{Yang22pami-certifiablePerception};\footnote{The presentation in~\cite{Yang22pami-certifiablePerception} uses the slightly different but equivalent parametrization where  
the binary variables are restricted to $\omega_i \in \{-1;+1\}$ and their binary nature is enforced via $\omega_i^2 = 1$, while here we use a more straightforward parametrization with 
$\omega_i \in \{0;1\}$ and enforce it via $\omega_i^2 = \omega_i$.} 
the computation of the estimate $\vxx\mcsdp$ is different from the one in~\cite{Yang22pami-certifiablePerception}, but the two ``rounding'' schemes are equivalent when the relaxation is tight (in which case $\getEntries{\MX\opt}{\omega_i} \in \{0,1\}$ and $\getEntries{\MX\opt}{\omega_i\vxx} = \getEntries{\MX\opt}{\vxx}$ $\forall i$ such that $\getEntries{\MX\opt}{\omega_i} > 0$). 
We provide the following estimation contract for~\cref{algo:mc}.

\noindent\fbox{
\parbox{\textwidth}{
\begin{proposition}[Low-outlier case: a priori estimation contract for~\cref{algo:mc}]
\label{thm:lowOut-apriori-MC}
\mcThm{eq:lowOut-apriori-MC}
\end{proposition}
}}

Note that the performance guarantees are essentially the same as~\cref{thm:lowOut-apriori-LTS} and the 
proof proceeds along the same line, but in this case the objective forces the solution to pick enough inliers, while in~\cref{thm:lowOut-apriori-LTS} the same role was played by the constraint 
$\sumOverMeas{\omega_i} = \an$. 


\subsubsection{Estimation Contract for~\eqref{eq:TLS}}
\label{sec:lowOut-aPrioriResults-TLS}

This section presents an estimation contract for~\cref{algo:tls}, which is based on a moment relaxation of the~\eqref{eq:TLS} problem. 
Notice that~\eqref{eq:TLS-lin} in~\cref{algo:tls}  is equivalent to~\eqref{eq:TLS} for $f_i(\vxx) = \MA_i\tran \vxx$.\footnote{The attentive reader might notice that $\axiomsMC$ in~\eqref{eq:TLS-lin} includes redundant constraints $\omega_i \cdot \residuali{\vxx}^2 \leq \barc^2$ in $\axiomsMC$: these constraints do not change the optimal solution of~\eqref{eq:TLS} ---since the objective is already forcing $\omega_i = 0$ whenever $\residuali{\vxx}^2 > \barcsq$--- but will make our proofs easier.}
\begin{algorithm}[h!]
	\caption{Moment relaxation for~\eqref{eq:TLS}.\label{algo:tls}}
	\SetAlgoLined
	\KwIn{ input data $(\vy_i,\MA_i)$, $i \in [\nrMeasurements]$, relaxation order $\relaxOrder$.}
	\KwOut{ estimate of $\vxx\gt$. }

	\tcc{Algorithm solves a relaxation of the following~\eqref{eq:TLS} problem:\\
	\vspace{-5mm}
	\begin{align}
	\label{eq:TLS-lin}\tag{TLS1}
	\min_{ \vxx,\vomega } 
		& \hspace{3mm} \sumOverMeas \omega_i \cdot \residuali{\vxx}^2 + (1-\omega_i) \cdot \barcsq
	\\
	\text{s.t.}
	& \hspace{3mm}
	\axiomsMC \triangleq \left\{
	\begin{array}{cl}
	\omega_i^2 = \omega_i, \;\; i\in[\nrMeasurements] \\ 
	\omega_i \cdot \residuali{\vxx}^2 \leq \barc^2 \;\; i\in[\nrMeasurements] \nonumber\\
	\vxx \in \Domain
	\end{array}
	\right\}
	\end{align}
	\vspace{-5mm}
	}

	\tcc{Compute matrix $\MX\opt$ by solving SDP resulting from moment relaxation}
	$\MX\opt = \scenario{solve\_moment\_relaxation\_at\_order\_}r\,\eqref{eq:TLS-lin}$ 
	\myskip

    \tcc{Compute estimate}
    for each $i \in [\nrMeasurements]$ set: $\vv_i = 
    \left\{
    \begin{array}{ll}
    \frac{ \getEntries{\MX\opt}{\omega_i\vxx} }{ \getEntries{\MX\opt}{\omega_i} } & \text{ if $\getEntries{\MX\opt}{\omega_i}>0$}
    \\
    \zero & \text{ otherwise}
    \end{array}\right.$

    $\vxx\tlssdp = \sumOverMeas \frac{ \getEntries{\MX\opt}{\omega_i} }{ \sumOverMeasj \getEntries{\MX\opt}{\omega_j} } \vv_i$
    \myskip

	\Return{$\vxx\tlssdp$.}
\end{algorithm}

Similar to~\cref{algo:mc} in~\cref{sec:lowOut-aPrioriResults-MC},~\cref{algo:tls} does not require prior knowledge about the number of inliers $\an$. Moreover, the moment relaxation is similar to the one proposed in~\cite{Yang22pami-certifiablePerception} (except from the redundant constraints 
$\omega_i \cdot \residuali{\vxx}^2 \leq \barc^2$ in $\axiomsMC$); the computation of the estimate $\vxx\tlssdp$ is different from the one in~\cite{Yang22pami-certifiablePerception}, but the two ``rounding'' schemes are equivalent when the relaxation is tight. 
We provide the following estimation contract for~\cref{algo:tls}.

\noindent\fbox{
\parbox{\textwidth}{
\begin{proposition}[Low-outlier case: a priori estimation contract for~\cref{algo:tls}]
\label{thm:lowOut-apriori-TLS}
\tlsThm{eq:lowOut-apriori-TLS}
\end{proposition}
}}
Contrarily to~\cref{thm:lowOut-apriori-LTS} and~\cref{thm:lowOut-apriori-MC}, the 
error bound in~\cref{thm:lowOut-apriori-TLS} depends on the ground truth residual error $\gamma\gt$. In particular, if $\gamma\gt = 0$ (\ie noiseless inliers), the bound becomes the same as the ones in~\cref{thm:lowOut-apriori-LTS} and~\cref{thm:lowOut-apriori-MC}, but as the residual error  approaches $\gamma\gt = \an \; \barcsq$ (\ie each inlier has the maximum allowed error $\barcsq$), the bound become vacuous.


\section{Estimation Contracts for High Outlier Rates}
\label{sec:highOutliers}

This section focuses on the high-outlier case where $\beta \gg 0.5$.
Note that in the high-outlier case any point estimator ---including~\eqref{eq:LTS},~\eqref{eq:MC}, and~\eqref{eq:TLS}--- can be tricked into returning an arbitrarily wrong estimate: intuitively, 
since the majority of the measurements are outliers, the outliers can agree on an $\vxx$ that 
optimizes~\eqref{eq:LTS},~\eqref{eq:MC}, or~\eqref{eq:TLS}, while being far from $\vxx\gt$.
 At the same time, an algorithm can still compute a provably accurate estimate 
 if it is allowed to return a \emph{list} of potential hypotheses, in the hope that at least one of them is correct. 
 This setup is typically referred to as \emph{list-decodable regression}~\cite{Karmalkar19neurips-ListDecodableRegression}. 
In the following, we review the algorithm and theoretical guarantees from~\cite{Karmalkar19neurips-ListDecodableRegression}, 
which we adapt to the case of vector-valued measurements.

\subsection{A Priori Estimation Contract for List-Decodable Estimation}
\label{sec:highOut-aPrioriResults}

In the high-outlier case, Karmalkar\setal~\cite{Karmalkar19neurips-ListDecodableRegression} 
proposed using~\cref{algo:ldr} (given below) to perform list-decodable outlier-robust regression. 
The algorithm returns a (small) list of potential estimates, and, as we will see below, 
under suitable conditions on the inliers, it guarantees that at least one of such estimates is close to the ground truth. Both the algorithm and estimation contract in this section are adaptations of the results in~\cite{Karmalkar19neurips-ListDecodableRegression} to the case of vector-valued measurements. 

\begin{algorithm}[h!]
  \caption{List-Decodable Outlier-Robust Estimation~\cite{Karmalkar19neurips-ListDecodableRegression}.\label{algo:ldr}}
  \SetAlgoLined
  \KwIn{ input data $(\vy_i,\MA_i)$, $i \in [\nrMeasurements]$, inlier rate $\inlierRate$, relaxation order $\relaxOrder \geq 2$, integer $N \geq 1$.}
  \KwOut{ list of estimates of $\vxx\gt$. }

  \tcc{Algorithm solves a relaxation of the following problem:\\
  \vspace{-5mm}
  \begin{align}
  \label{eq:LDR}\tag{LDR}
  \min_{\vomega,\vxx} & \normTwo{\vomega}^2, 
  \;\;
  \text{s.t.} \;\; \axiomsLDR \triangleq \left\{
  \begin{array}{cl}
  \omega_i^2 = \omega_i, \;\; i=[\nrMeasurements] \\ 
  \sumOverMeas{\omega_i} = \inlierRate \nrMeasurements  \\
  \omega_i \cdot \residuali{\vxx}^2 \leq \barcsq, \;\; i=[\nrMeasurements]\\
  \vxx \in \Domain 
  \end{array}
  \right\}
  \end{align}
  \vspace{-5mm}
  }

  \tcc{Compute matrix $\MX\opt$ by solving SDP resulting from moment relaxation}
  $\MX\opt = \scenario{solve\_moment\_relaxation\_at\_order\_}r\,\eqref{eq:LDR}$ 

  \tcc{Compute list of estimates}
  
  for each $i \in [\nrMeasurements]$ set: $\vv_i = 
  \left\{
  \begin{array}{ll}
  \frac{ \getEntries{\MX\opt}{\omega_i\vxx} }{ \getEntries{\MX\opt}{\omega_i} } & \text{ if $\getEntries{\MX\opt}{\omega_i}>0$}
  \\
  \zero & \text{ otherwise}
  \end{array}\right.$ \label{line:ldr2}

  create empty list $\List = \emptyset$.

  sample $N/\inlierRate$ times from $[\nrMeasurements]$ with probability $\frac{1}{\an} \getEntries{\MX\opt}{\omega_i}$, and for each extracted $i$ add $\vv_i$ to $\List$. \label{line:ldr4}

  \Return{$\List$.}
\end{algorithm}

The rationale behind problem~\eqref{eq:LDR} is not as immediate as the problems presented in the previous sections. For instance, in~\eqref{eq:MC-lin}, it was clear we tried to select the largest number of measurements, \ie the largest $\sumOverMeas \omega_i$, such that the corresponding estimate $\vxx$ had residual below $\barc$. 
However,~\eqref{eq:LDR} seems to do just the opposite: it is looking for the \emph{smallest} $\sumOverMeas \omega_i^2$. 

\myParagraph{Intuitions behind~\cref{algo:ldr}}
Imagine we could build a list $\Longlist$ containing every possible subset $\calS'$ of $\an$ measurements that satisfy $\residuali{\vxx'}^2 \leq \barcsq, \forall i\in\calS'$ and for some  estimate $\vxx'$. Clearly, the set of inliers $\setInliers$ would be one of such subsets and hence belong to $\Longlist$.
For each such subset $\calS' \in \Longlist$, 
we could define the indicator vector $\vomega' = \ones_{\calS'}$, and by inspection,
 $(\vomega',\vxx')$ would be feasible for~\eqref{eq:LDR}. 
Therefore, the first intuition behind~\cref{algo:ldr} is that the feasible set of $\axiomsLDR$ of~\eqref{eq:LDR} includes every such $(\vomega',\vxx')$, including $(\vomega_\setInliers,\vxx\gt)$.
 Unfortunately, the list $\Longlist$ has exponential size in general~\cite{Karmalkar19neurips-ListDecodableRegression}. 
 Therefore, the second intuition behind~\cref{algo:ldr} is that the moment relaxation~\eqref{eq:LDR} 
 will compute an indicator vector $\vomega$ that will ``spread'' across many subsets in $\Longlist$.
 While this is formalized in the proofs (and further discussed in~\cite{Karmalkar19neurips-ListDecodableRegression}), it is relatively easy to visualize why the relaxation would ``spread'' across the subsets by considering the following simpler relaxation of~\eqref{eq:LDR}: 
  \begin{align}
  \label{eq:LDRrelaxlin}
  \min_{\vomega,\vxx} & \normTwo{\vomega}^2, 
  \;\;
  \text{s.t.} \;\; \left\{
  \begin{array}{cl}
  \omega_i \in [0,1], \;\; i=[\nrMeasurements] \\ 
  \sumOverMeas{\omega_i} = \an  \\
  \omega_i \cdot \residuali{\vxx}^2 \leq \barcsq, \;\; i=[\nrMeasurements]\\
  \vxx \in \Domain 
  \end{array}
  \right\}
  \end{align}
  where we only relaxed $\omega_i \in \{0;1\}$ to $\omega_i \in [0,1]$.
  Imagine that the list $\Longlist$ contains two overlapping sets 
 $\calS'$ and $\calS"$ and hence $\vomega' = \ones_{\calS'}$ and $\vomega" = \ones_{\calS"}$ would be feasible for~\eqref{eq:LDRrelaxlin} for some $\vxx$. In such a case the vector $\vomega = \frac{\ones_{\calS'} + \ones_{\calS"}}{2}$
 ---that spreads across both subsets---  would still satisfy the constraint $\sumOverMeas{\omega_i} = \an$ and achieve a lower objective in~\eqref{eq:LDRrelaxlin} compared to 
  $\ones_{\calS'}$ and $\ones_{\calS"}$.\footnote{ 
  Calling $t$ the number of common elements between $\calS'$ and $\calS"$:
  $\normTwo{\vomega}^2 = \frac{1}{4} (\normTwo{\ones_{\calS'} + \ones_{\calS"}}^2) = 
  \frac{1}{4} (2 \an  + 2 t) = \frac{\an + t}{2}$ (which is $< \an$ as long as the sets do not fully overlap), while each set would achieve an objective
  $\normTwo{\vomega'}^2 = \normTwo{\vomega"}^2 = \an$.}
   Similarly, after applying the moment relaxation to~\eqref{eq:LDR}, the objective of~\eqref{eq:LDR} will favor indicator vectors that span multiple sets. The last intuition behind~\cref{algo:ldr} is that we can use these indicator vectors to sample good estimates of $\vxx\gt$, as done in lines~\ref{line:ldr2}-\ref{line:ldr4}  of the algorithm.

\myParagraph{Estimation contract}
The estimation contract for~\cref{algo:ldr} is given as follows. 

\noindent\fbox{
\parbox{\textwidth}{
\begin{theorem}[High-outlier case: a priori estimation contract for~\cref{algo:ldr}, adapted from  Theorem 1.5 in~\cite{Karmalkar19neurips-ListDecodableRegression}]
\label{thm:highOut-apriori}
\ldrThm{eq:highOut-apriori}
\end{theorem}
}}

The estimation contract in~\cref{thm:highOut-apriori} provides probabilistic guarantees on the outcome of the estimator, as opposed to the deterministic results in the previous sections.
 The sampling stage is needed to keep the list small, while also ensuring that it captures at least an estimate close to the ground truth.   
 Note that the final statement in the theorem is just a particularization of the claim that eq.~\eqref{eq:highOut-apriori} holds with probability at least $\probBound$; this is 
 given to reassure the reader that the approach can provide good estimates with high probability already for small $N$ (\ie $N=10$) and for very challenging problems where $\inlierRate$ is very low (\ie $\inlierRate = 0.01$).
  Clearly, the probability will increase with the user-specified parameter $N$, which controls the number of samples to draw.

\cref{thm:highOut-apriori} and~\cref{algo:ldr} are of interest for geometric perception for several reasons.
First of all, contrary to the literature on multi-hypothesis estimation  (\eg~\cite{BarShalom93,Hsiao19icra-mhISAM2,Fourie16iros-isam3}) and particle filtering (\eg~\cite{Dellaert99}) in robotics and related fields, 
\cref{algo:ldr} provides performance guarantees while retaining a small number of hypotheses; 
on the other hand, in multi-hypothesis tracking, one either has to retain an exponential number of hypotheses or 
typically loses performance guarantees, while in particle filtering one typically only has asymptotic guarantees for  growing number of samples.
Furthermore, while the sampling scheme might carry some resemblance of \ransac~\cite{Fischler81}, 
the number of samples (\ie iterations) required by \ransac grows exponentially with the size of the minimal set (which, in turn, depends on the dimension of $\vxx\gt$) and the outlier rate~\cite{Bustos18pami-GORE}; this is in stark contrast with \cref{algo:ldr}, where the number of samples 
is independent on the dimension of $\vxx\gt$ and only grows as $N/\inlierRate$ (also note that the number of samples in~\cref{algo:ldr} is upper bounded by $\nrMeasurements$).
%
%
%

In the following section we show that while the guarantees in~\cref{thm:highOut-apriori} fall short in practical problems, a simple modification of~\cref{algo:ldr} with $\relaxOrder=2$ has impressive performance in practice.	


\section{Numerical Experiments}
\label{sec:experiments}

This section elucidates on the theoretical results, by  
(i) assessing the tightness and the applicability of the error bounds discussed in this \paper,
(ii) evaluating how strict are the assumptions made on the input data (\ie certifiable hypercontractivity or certifiable anti-concentration), and 
(iii) assessing the practical performance of (a variant of)~\cref{algo:ldr} for list-decodable estimation.
The analysis is performed on a canonical perception problem, namely \emph{rotation search} (see~\cref{ex:wahba}), which finds application in satellite attitude estimation~\cite{wahba1965siam-wahbaProblem} and image stitching~\cite{Yang19iccv-quasar} problems, among others.
After introducing the experimental setup in~\cref{sec:experiments-setup}, 
\cref{sec:experiments-aposteriori} focuses on the a posteriori bound introduced in~\cref{thm:lowOut-aposterioriNoisyTLS}, \cref{sec:experiments-hyper} analyzes the certifiable hypercontractivity assumption and the contract in~\cref{thm:lowOut-apriori-LTS-objective}, 
\cref{sec:experiments-anticon} analyzes the certifiable anti-concentration assumption and the contract in~\cref{thm:lowOut-aposterioriNoisyTLS}. 
\cref{sec:experiments-ldr} concludes the experiments with an empirical evaluation of list-decodable estimation.


\subsection{Experimental Setup}
\label{sec:experiments-setup}

We use \emph{rotation search} (also known as the \emph{Wahba problem})~\cite{wahba1965siam-wahbaProblem,Yang19iccv-quasar}
 as a prototypical outlier-robust estimation problem, see Example~\ref{ex:wahba}. The goal of the problem is to estimate a rotation $\MR$ that aligns two sets of 3D vectors (typically normalized to have unit norm). The inputs to the problem are the vector pairs $\va_i, \vb_i \in \Real{3}$, $i=1,\ldots,\nrMeasurements$. The vectors $\va_i$ are used to build the measurement matrix $\MA_i$ as in Example~\ref{ex:wahba}; 
these vectors are either known or measured: for instance, in satellite pose estimation these vectors are read from a star database, while in image stitching they correspond to detected image features. 
$\vb_i$ are measured bearing vectors. The association between pairs of vectors $(\va_i,\vb_i)$ is obtained using image processing or learning-based methods and is prone to errors, making a large number of measurements outliers.
 We test on both synthetic 
 and real data, as described below.

\myParagraph{Synthetic data}
We use the same data generation protocol of~\cite{Yang19iccv-quasar}.
In particular, we create each test instance by sampling $\nrMeasurements$ 
  vectors $\{\va_i\}_{i=1}^\nrMeasurements$ uniformly at random on the unit sphere. 
 Then, we pick a random rotation $\MR$, and apply it to each $\va_i$ according to eq.~\eqref{eq:wahba} and
  add zero-mean Gaussian noise with covariance $10^{-4} \cdot \eye_3$ to get $\{\vb_i\}_{i=1}^\nrMeasurements$. To generate outliers, we replace a fraction $\outlierRate$ of $\vb_i$'s with random unit-norm vectors. We set the maximum error bound $\barcsq$ from the quantile of the $\chi^2$
 distribution with three degrees of freedom and lower tail probability equal to $0.9999$ (see Remark~1 in~\cite{Yang19iccv-quasar}), thus obtaining $\barcsq = 0.0021$ (which also accounts for the measurement variance).
Results on synthetic data are averaged over 100 Monte Carlo runs unless specified otherwise. 

\begin{figure}[ht]
\hspace{0.5cm}
\begin{minipage}[b]{0.50\linewidth}
\centering
\includegraphics[height=6.65cm, trim= 0mm 0mm 0mm 0mm, clip]{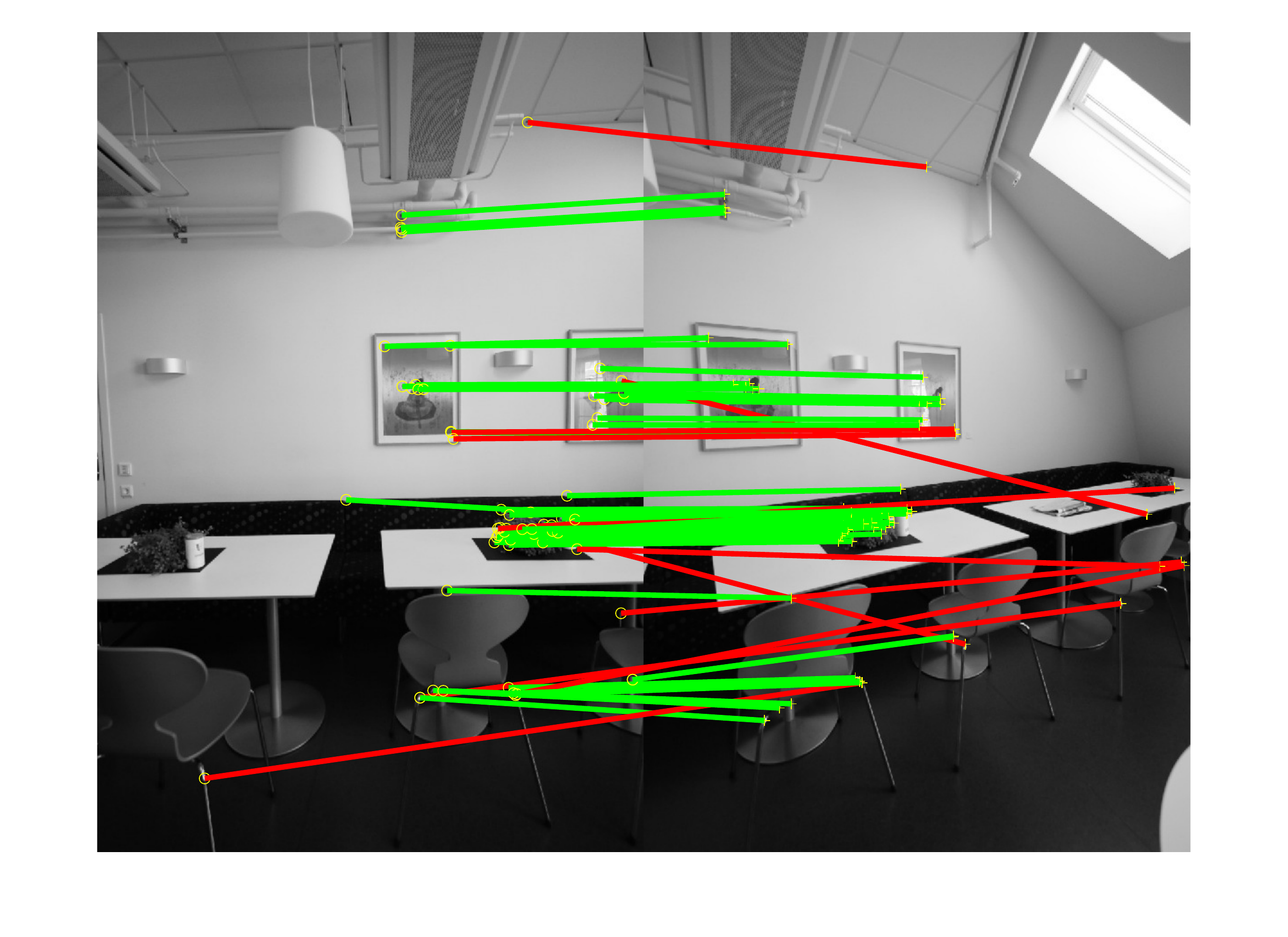}\vspace{-7mm}
\\
(a) 
\vspace{-0.1cm}
\end{minipage} 
\begin{minipage}[b]{0.50\linewidth}
\centering
\includegraphics[height=6.5cm, trim= 0mm 0mm 0mm 0mm, clip]{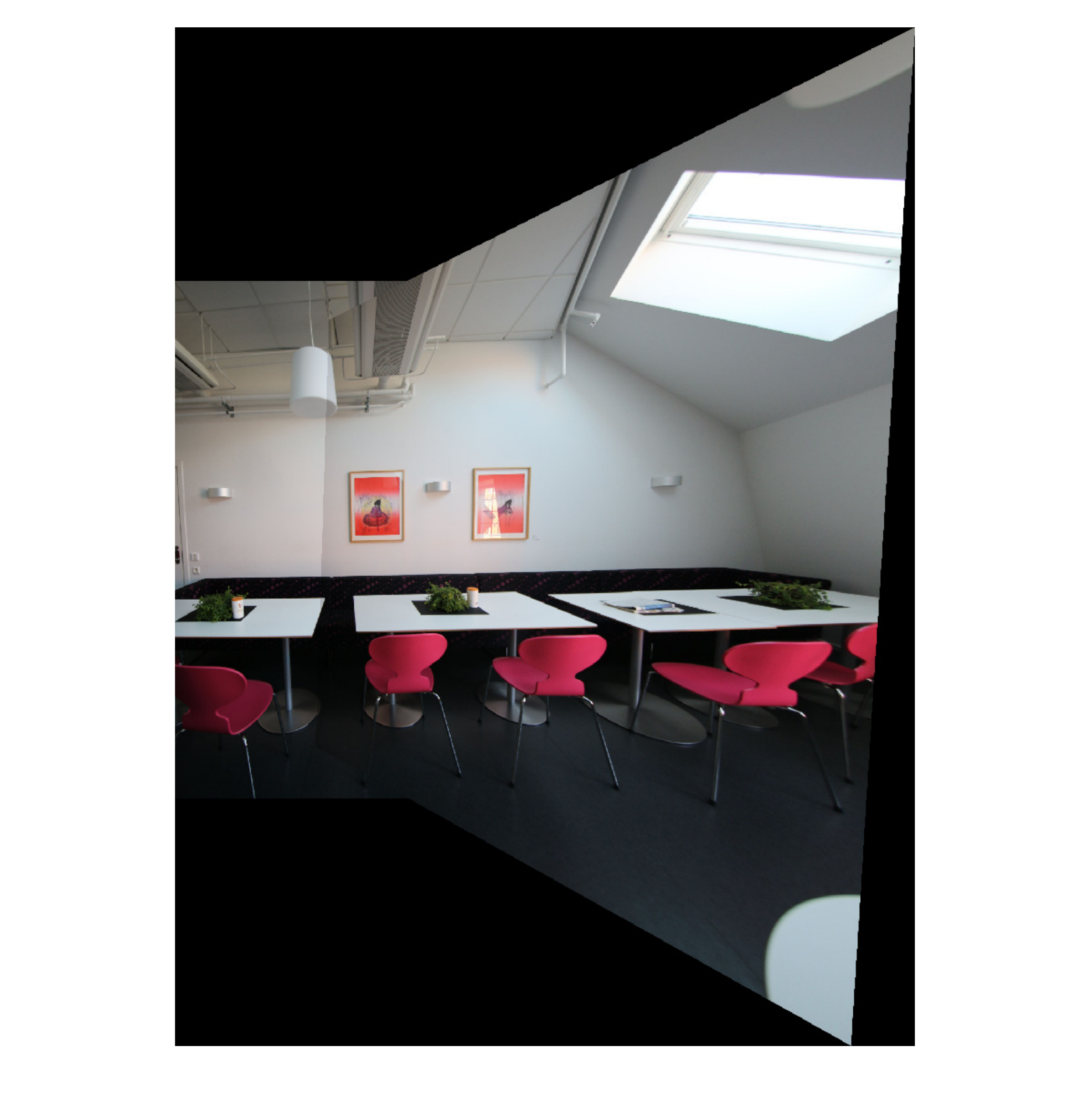}\vspace{-5mm}
\\
(b) 
\end{minipage}
\vspace{-8mm}
\caption{Image stitching example: (a) Feature matches  between two images from which the vectors $(\va_i,\vb_i)$, $i=1,\ldots,\nrMeasurements$, are generated; the matches include many outliers, shown in red (inliers are shown in green).
(b) Image stitching results using the moment relaxation approach in~\cite{Yang19iccv-quasar}. 
\label{fig:imageStitching}}
\end{figure}

\myParagraph{Real data}
We use the PASSTA dataset to test rotation search problems arising in image stitching applications~\cite{Meneghetti15scia-PASSATImageStitchingData} and use the same data generation protocol of~\cite{Yang19iccv-quasar}. 
In particular, to generate the vector pairs $\va_i, \vb_i \in \Real{3}$, $i=1,\ldots,\nrMeasurements$, we first use SURF~\cite{Bay06eccv} to detect and match point features between the two images.
\cref{fig:imageStitching}(a) shows the matches for one of the image pairs in the dataset.
From the SURF feature points, we apply the inverse of the known camera intrinsic matrix $\MK$ to obtain unit-norm bearing vectors $\{\va_i,\vb_i\}_{i=1}^{70}$ observed in each camera frame. 
Then, we solve the rotation search problem to estimate the rotation $\MR$ between the two camera frames, using the outlier-corrupted pairs $\{\va_i,\vb_i\}_{i=1}^{70}$. 
Finally, using the estimated $\MR$, we can compute the homography matrix as $\MH=\MK\MR\MK\inv$ to stitch the pair of images together; see the example in
\cref{fig:imageStitching}(b). 

\myParagraph{Algorithms and implementation details}
We use the sparse moment relaxation proposed in~\cite{Yang19iccv-quasar} 
to solve the~\eqref{eq:TLS} formulation of the rotation search problem, and use STRIDE~\cite{Yang22pami-certifiablePerception,Yang22mapr-stride} as an SDP solver for the resulting semidefinite relaxation. 
Note that the SDP proposed in~\cite{Yang19iccv-quasar} is a sparse version of the relaxation in~\cref{algo:tls} (with order $\relaxOrder=2$), and the output of the two relaxations match whenever the sparse relaxation is tight;  in all results below 
we observe the sparse moment relaxation of~\eqref{eq:TLS} to be tight (\ie relaxation gap below $10^{-7}$) hence we do not differentiate between the two algorithms.
We use \cvx~\cite{CVXwebsite} and \mosek~\cite{mosek} as a parser/solver for the SDPs arising in list-decodable estimation. 
The numerical results are obtained on a MacBook Pro with 2.8 GHz Quad-Core Intel Core i7 processor. 
The Matlab code to reproduce the experiments below can be found at~\url{https://github.com/MIT-SPARK/estimation-contracts}. 

\subsection{Low Outlier Rates: A Posteriori Bounds}
\label{sec:experiments-aposteriori}

\begin{figure}[t]
\centering
\includegraphics[width=0.6\columnwidth, trim= 0mm 0mm 0mm 0mm, clip]{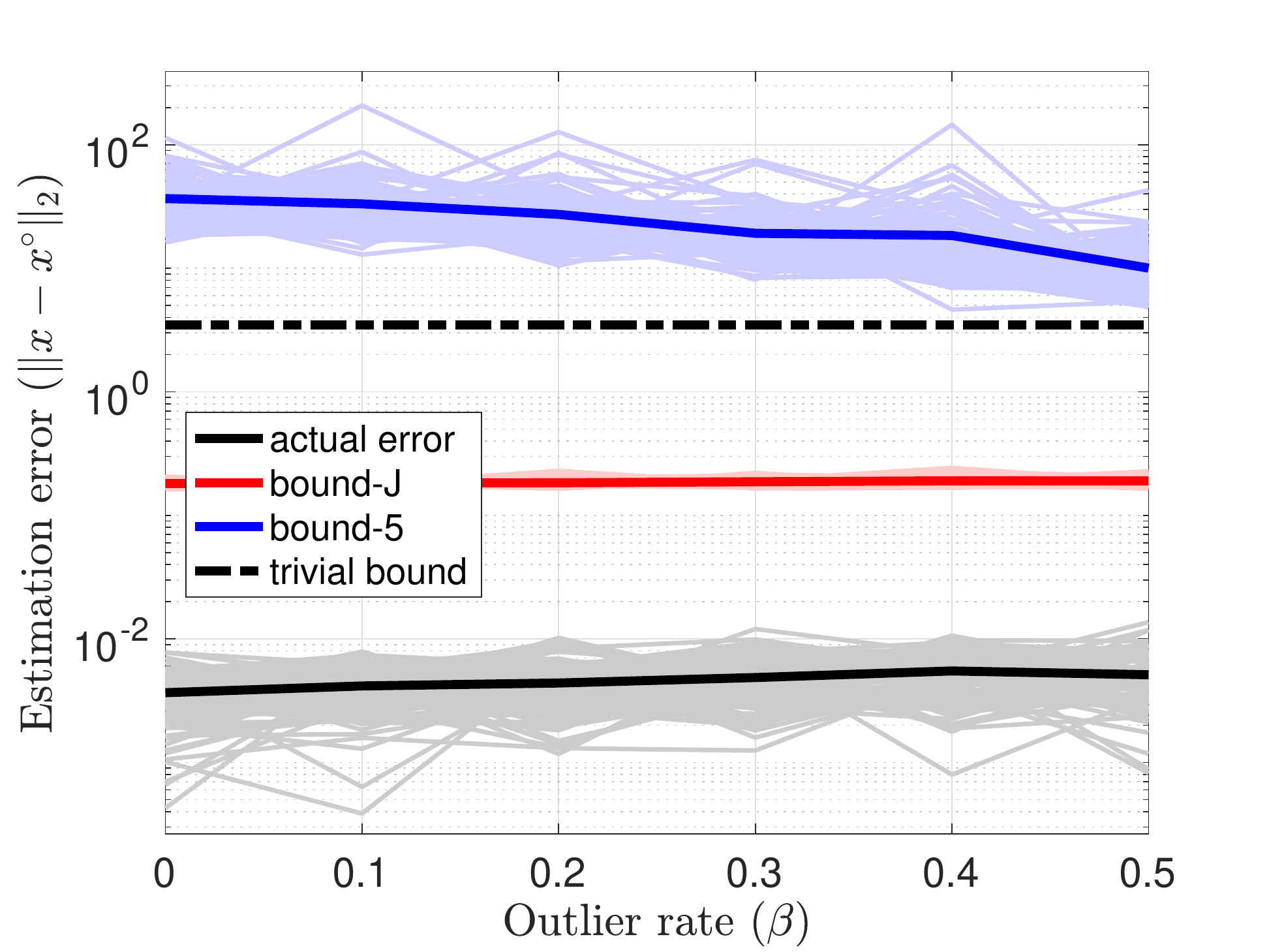} 
\vspace{-3mm}
\caption{Actual error of the~\eqref{eq:TLS} estimator compared to the a posteriori bounds from~\cref{thm:lowOut-aposterioriNoisyTLS}  and the trivial bound $2 \boundx = 2\sqrt{3}$, with $\nrMeasurements = 50$ and increasing outlier rates $\beta$. We report a posteriori bounds for $\dimJ=5$ (``bound-5'') and for the set $\calJ$ chosen as the correctly selected inliers (``bound-J''). \label{fig:aposterioriBounds}} 
\end{figure}

This section evaluates the quality of the a posteriori bound in~\cref{thm:lowOut-aposterioriNoisyTLS}.
 The bound applies since in all the tests in this section the relaxation from~\cite{Yang19iccv-quasar} is tight and hence it obtains an optimal solution to~\eqref{eq:TLS}.
 Moreover, with our data generation method, every subset of $\minDim=3$ inliers is nondegenerate by construction.\footnote{In rotation search, the unknown rotation can be uniquely estimated from $\minDim=3$ non-collinear 3D points, and our data generation method for the 3D vectors $\va_i$ produces triplets of collinear points with probability zero.}  
We note that the bound holds for any choice of integer {$\dimJ = \minDim,\ldots, (2\inlierRate-1)n - \frac{\gamma\gt}{\barcsq}$}.
For a small $\dimJ$, we can compute the bound by brute-force testing every subset $\calJ$ of the set of inliers computed by~\eqref{eq:TLS} and exhaustively searching the smallest value of $\sigma_\min(\MA_\calJ)$; in our tests, we are able to compute the bound for $\dimJ=5$ in a matter of seconds.
However, we expect large values of $\dimJ$ to produce more informative bounds (intuitively, in the proof of~\cref{thm:lowOut-aposterioriNoisyTLS} those bounds leverage a larger number of measurements); since the bound~\cref{thm:lowOut-aposterioriNoisyTLS} is expensive to compute for large $\dimJ$, we manually compute the set $\calJ$ corresponding to the intersection of the measurements selected by~\eqref{eq:TLS} as inliers and the true inliers, and compute the resulting bound for that choice of $\calJ$: this is the set used in our proofs to derive the estimation error bounds; note that such a set $\calJ$ can only be computed in simulation, since in practice the true inliers are unknown. 

\Cref{fig:aposterioriBounds} compares the actual estimation error with the bounds from~\cref{thm:lowOut-aposterioriNoisyTLS}. The estimation error is computed as $\normTwo{\vxx\tls - \vxx\gt}$, where $\vxx$ and $\vxx\gt$ are the vectorized representations of the estimated and the ground-truth rotation, respectively. For the bounds, we plot the bound we compute for $\dimJ=5$ (label: ``bound-5'') and the bound obtained from the choice of $\calJ$ described above (label: ``bound-J''). 
We also plot the trivial bound $2\boundx = 2\sqrt{3}$ computed by using the triangle inequality 
on $\normTwo{\vxx\tls - \vxx\gt}$ and by noting that every (vectorized) rotation has norm $\sqrt{3}$. 
Several comments are in order. First, the actual error increases with the outlier rate (the trend is slightly more difficult to see due to the log scale): this is expected since with increasing outlier rates the number of ``useful'' measurements decreases; however, the error remains very small (\ie less than 1 degree) for all outlier rates. This confirms the empirical robustness of moment relaxations of~\eqref{eq:TLS} already observed in related work~\cite{Yang19iccv-quasar,Yang22pami-certifiablePerception}, which have been reported to achieve impressive performance even for much higher rates of random outliers. 
Second, bound-5 is unfortunately too loose and uninformative, since ---while it improves for larger outlier rates\footnote{The bound improves for larger $\outlierRate$ since the set of inliers selected by~\eqref{eq:TLS} $\InlierSet_\tls$ typically shrinks for increasing $\beta$ and hence it becomes less likely to sample data producing small $\sigma_\min(\MA_\calJ)$ }--- it remains larger than the trivial bound.
On the other hand, bound-J is significantly better than the trivial bound. 
While there is still a large gap between the actual error and bound-J, we remark that the bound considers the worst-case scenario where the outliers are possibly adversarial;
on the other hand, the outlier generation mechanism we borrow from~\cite{Yang19iccv-quasar} is relatively benign, since the outliers are sampled independently, at random, and without knowledge of the inliers. In other words, there might well be instances where the actual error could be much closer to bound-J.

\subsection{Low Outlier Rates: A Priori Bounds and Hypercontractivity}
\label{sec:experiments-hyper}
This section shows that $k$-certifiable hypercontractivity is a relatively mild assumption for $k=4$ and for choices of $C(t)^t$ as small as $5$ or $6$. This observation is validated on both synthetic and real data. On the other hand, the experiments offer a closer look at the bound in~\cref{thm:lowOut-apriori-LTS-objective}, which in hindsight only applies to relatively ``easy'' problems with very small amounts of outliers. 

\myParagraph{Numerically checking hypercontractivity}
Certifiable hypercontractivity can be easily assessed numerically using the definition in eq.~\eqref{eq:hyper-def}. The definition of $k$-certifiable hypercontractivity essentially requires that the polynomial 
\beq 
h(\vv) = C(t)^t 
\left( \aveOverMeas \normTwo{ \MA_i\tran \vv }^2 \right)^{t} - \left( \aveOverMeas \normTwo{ \MA_i\tran \vv }^{2t} \right)
\eeq
 is a sum-of-squares polynomial for every $t \leq k/2$ and for a given choice of $C(t)^t$.
 While larger values of $k$ lead to slightly better bounds here we are interested in the smallest choice of $\relaxLevel$ (which is $\relaxLevel = 4$ according to~\cref{thm:lowOut-apriori-LTS-objective}) since the larger the $\relaxLevel$ the larger the size of the resulting moment relaxation.
 Choosing $\relaxLevel = 4$, we are only left to test that $h(\vv)$ is \sos for $t=2$ (the condition is trivially satisfied for $t=1$ as long as $C(t)^t \geq 1$). For $t=2$, we test multiple values of $C(t)^t$ and for each we check if $h(\vv)$ is \sos using the {\tt findsos} function in the SOSTOOLS library~\cite{Papachristodoulou13-sostools}. The function either returns a decomposition of the polynomial into a sum of squares or reports that the polynomial is not \sos. 

We remark that~\cref{thm:lowOut-apriori-LTS-objective} requires that the set of all $\nrMeasurements$ matrices $\MA_i$  ---where potential outliers are replaced with inliers--- satisfies $k$-certifiable hypercontractivity. Therefore, if we assume that the outliers only contaminate the vectors $\vb_i$ (hence do not alter the $\MA_i$'s, which are only influenced by the vectors $\va_i$'s),\footnote{This assumption is mild in many rotation search applications, for instance when $\va_i$ are read from a database or we have control on how to sample them.} we can actually check hypercontractivity in practice. 
 Note that for $t=2$, $h(\vv)$ is a degree-4 polynomial in the variable $\vv$ (which is 9-dimensional in our rotation search example): therefore, as we will see, it is quite fast to check if $h(\vv)$ 
 is \sos using off-the-shelf solvers.

\myParagraph{Experiments on synthetic data}
We show that 4-certifiable hypercontractivity is satisfied for small $C(t)^t$ in our synthetic examples and observe that
we can numerically check if a given set of measurements is certifiably hypercontractive in a fraction of a second. 

\begin{figure}[t]
\begin{minipage}[b]{0.45\linewidth}
\centering
\includegraphics[width=1\columnwidth, trim= 0mm 0mm 0mm 0mm, clip]{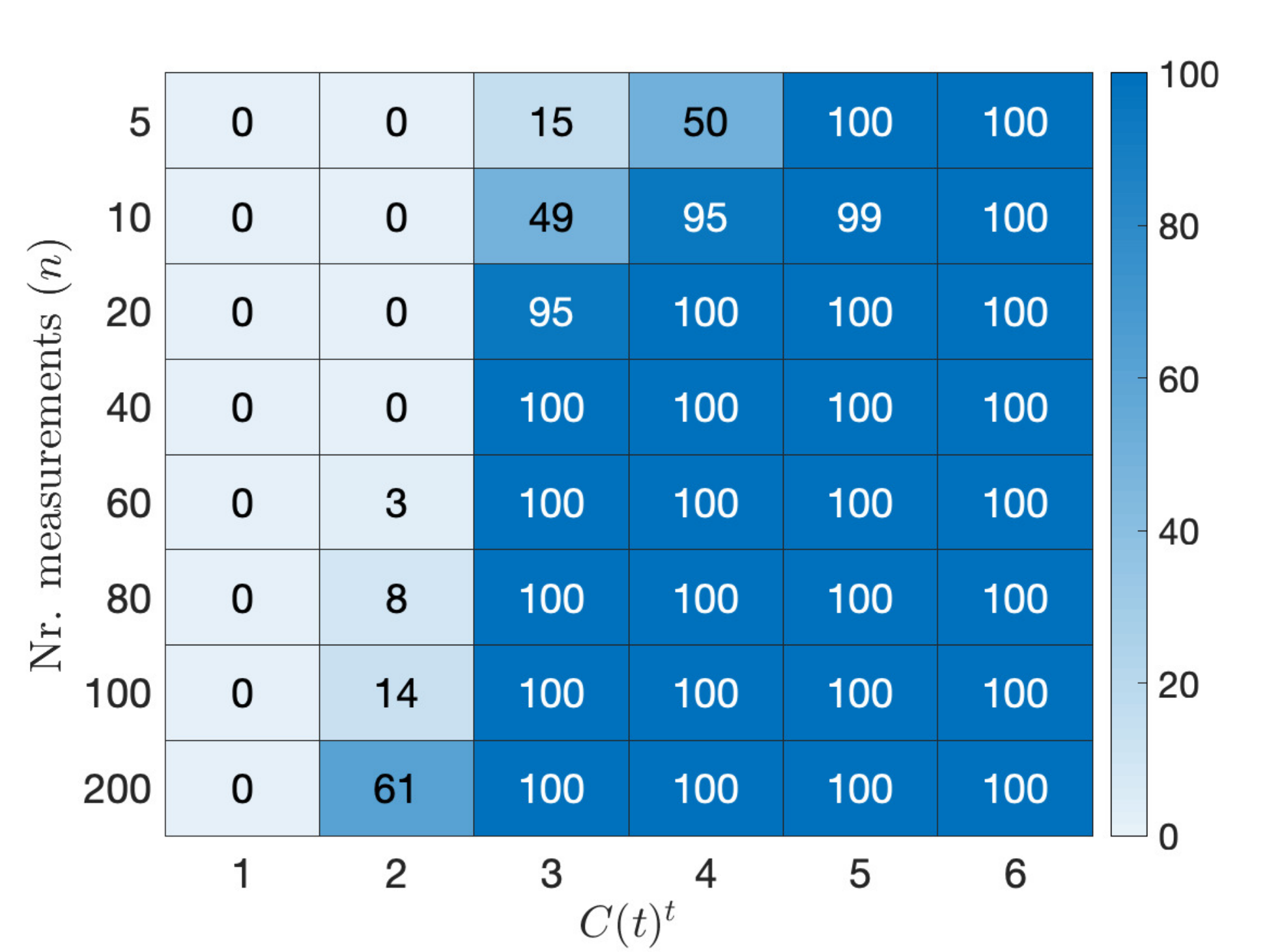}\vspace{0mm}\\
{(a)}
\end{minipage}
\hspace{0.5cm}
\begin{minipage}[b]{0.45\linewidth} 
\centering
\includegraphics[width=1\columnwidth, trim= 0mm 0mm 0mm 0mm, clip]{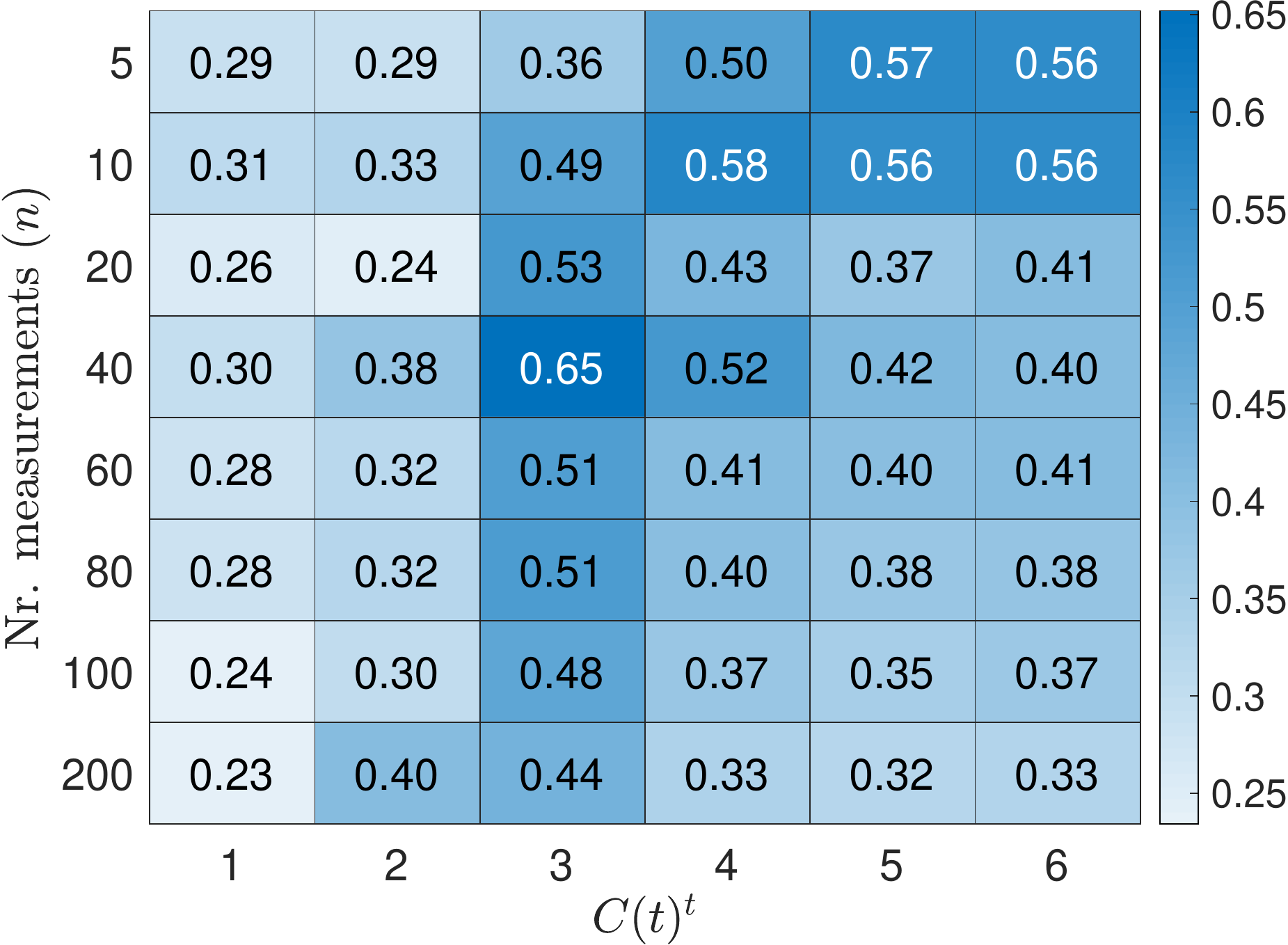}\vspace{0mm}
\\
{(b)}
\end{minipage}\vspace{-3mm}
\caption{Certifiable hypercontractivity results on synthetic data: (a) Percentage of tests where $4$-certifiable hypercontractivity holds 
for various choices of $C(t)^t$ and number of measurements $\nrMeasurements$.
(b) Average runtime (in seconds, averaged over 100 tests) for the same test configurations in figure (a).  \label{fig:stats_synthetic_hyper}}
\end{figure}

\Cref{fig:stats_synthetic_hyper}(a) reports the number of tests (over 100 runs) where $4$-certifiable hypercontractivity holds for various choices of $C(t)^t$ (chosen as a constant ranging from $1$ to $6$) and for increasing number of measurements $\nrMeasurements$ (we test $\nrMeasurements = \{5, 10, 20, 40, 60, 80, 100, 200\}$). All instances were found to satisfy $4$-certifiably hypercontractivity for $C(t)^t = 6$ (independently of the choice of $\nrMeasurements$), while no instance satisfied the property for $C(t)^t = 1$. From \cref{fig:stats_synthetic_hyper}(a), we observe a sharp phase transition between the two regimes, and we note that large measurement sets tend to satisfy $4$-certifiable hypercontractivity for even smaller values of $C(t)^t$.
This is desirable, since the smaller the $C(t)^t$, the better are the error bounds we obtain from~\cref{thm:lowOut-apriori-LTS-objective}.

\Cref{fig:stats_synthetic_hyper}(b) reports the average runtime (in seconds, averaged over 100 tests) for the same test configurations in~\cref{fig:stats_synthetic_hyper}(a). 
The average runtime remains below 1 second in all cases and is fairly insensitive to the number of measurements. It is interesting to notice that the solver is faster in resolving instances that are clearly certifiably hypercontractive or not, while it has to work harder to decide the instances at the ``boundary'' of the phase transition observed in~\cref{fig:stats_synthetic_hyper}(a).

\begin{table*}[h]
\centering
\caption{Certifiable hypercontractivity results on real image stitching data: percentage of tests where $4$-certifiable hypercontractivity holds 
for increasing $C(2)^2$.}
\label{tab:realResultsHyper}
\begin{tabular}{|r|c|c|c|c|c|c|c|c|c|c|c|c|}
        \toprule
                  $C(2)^2 = $      & $1$   & $2$  & $3$  & $4$  & $5$ & $6$& $7$& $8$& $9$& $10$& $15$& $20$  \\  
        \midrule
        \% cert. hypercon.               
                        & $0.0$  & $0.0$  & $0.0$  & $8.3$  & $50.0$  & $58.3$  & $75.0$  & $83.3$  & $83.3$  & $91.7$  & $100.0$  & $100.0$ \\ 
        \bottomrule
\end{tabular}
\end{table*}

\myParagraph{Experiments on real data}
We show that $4$-certifiable hypercontractivity also holds for real data for a small value of $C(t)^t$ (with $t=2$), 
with a similar trend as the one observed in the synthetic experiments above. 
Towards this goal, we consider the image stitching data described in~\cref{sec:experiments-setup}.
\Cref{tab:realResultsHyper} reports the percentage of tests 
where $4$-certifiable hypercontractivity holds for various choices of $C(t)^t$ (chosen as a constant between 1 and 20).
 Most instances satisfy $4$-certifiably hypercontractivity already for $C(t)^t = 6$ and all instances satisfy the property for $C(t)^t = 15$, confirming that certifiable hypercontractivity is a mild assumption for the rotation search problem.

\begin{figure}[h]
\begin{minipage}[b]{0.48\linewidth}
\centering
\includegraphics[width=1\columnwidth, trim= 0mm 0mm 0mm 0mm, clip]{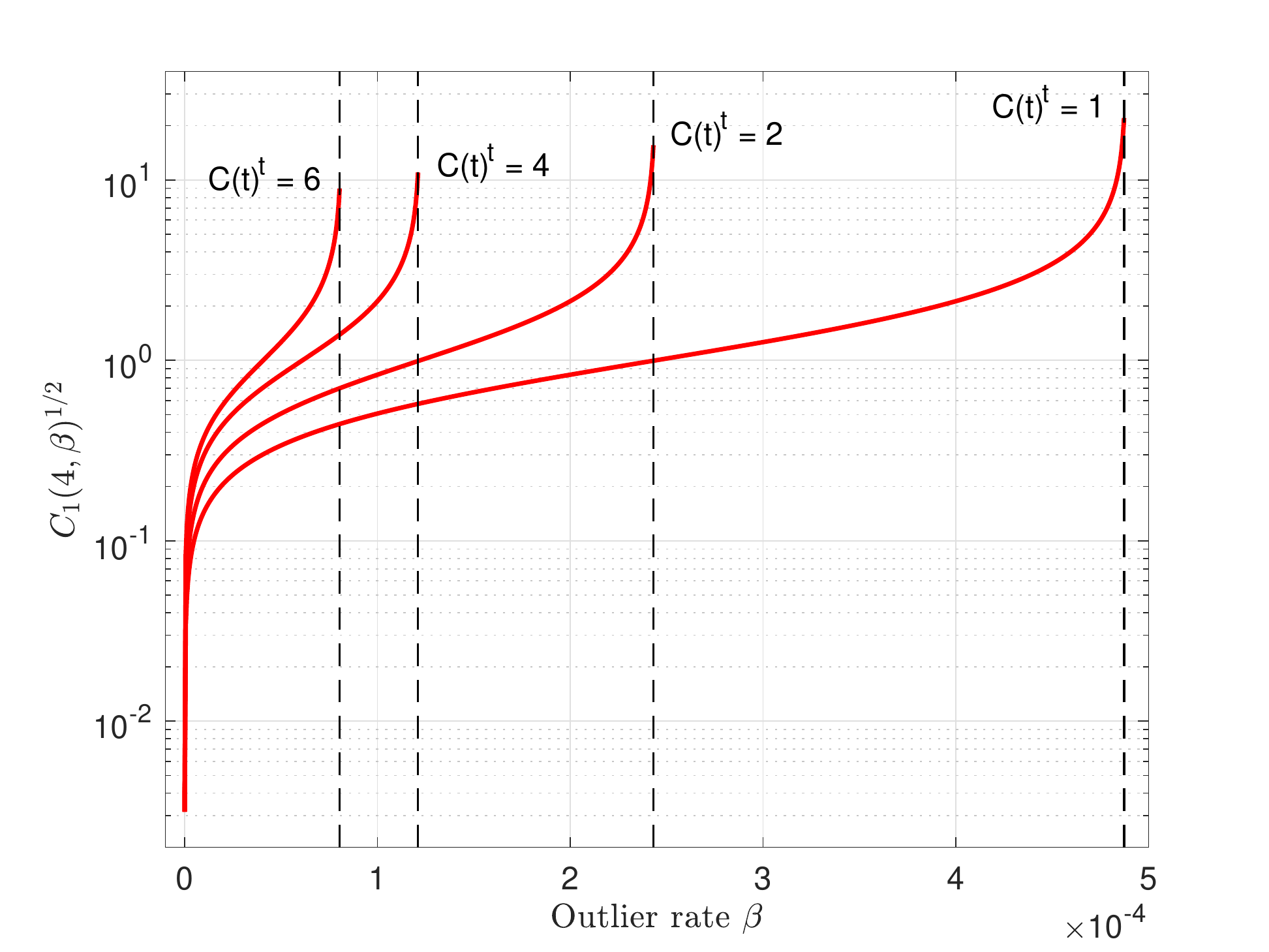}
\vspace{-7mm}\\ (a)
\end{minipage}
\begin{minipage}[b]{0.48\linewidth} 
\centering
\includegraphics[width=1\columnwidth, trim= 0mm 0mm 0mm 0mm, clip]{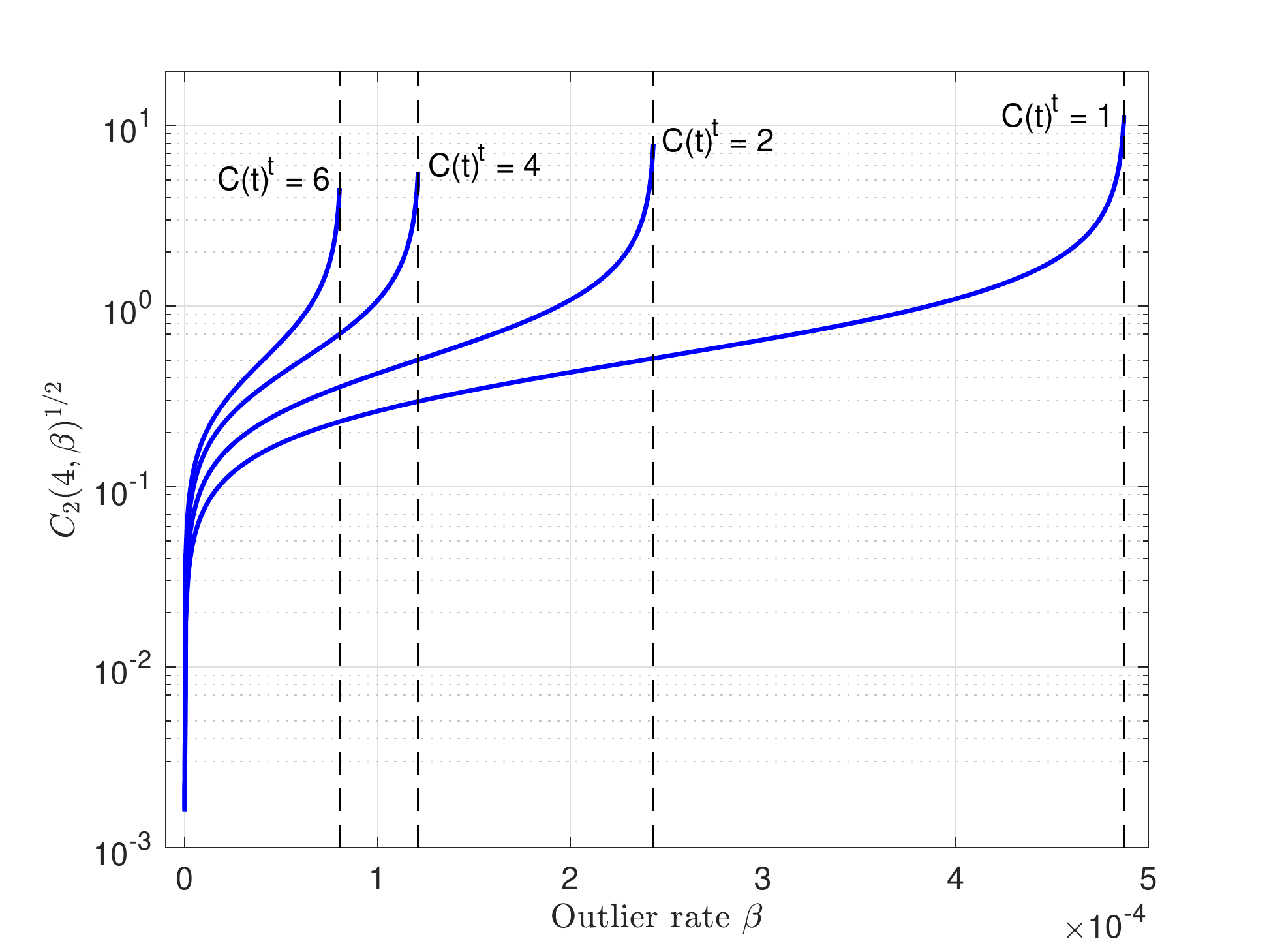}
\vspace{-7mm}\\ (b)
\end{minipage}
\\
\begin{minipage}[b]{0.48\linewidth} 
\centering
\includegraphics[width=1\columnwidth, trim= 0mm 0mm 0mm 0mm, clip]{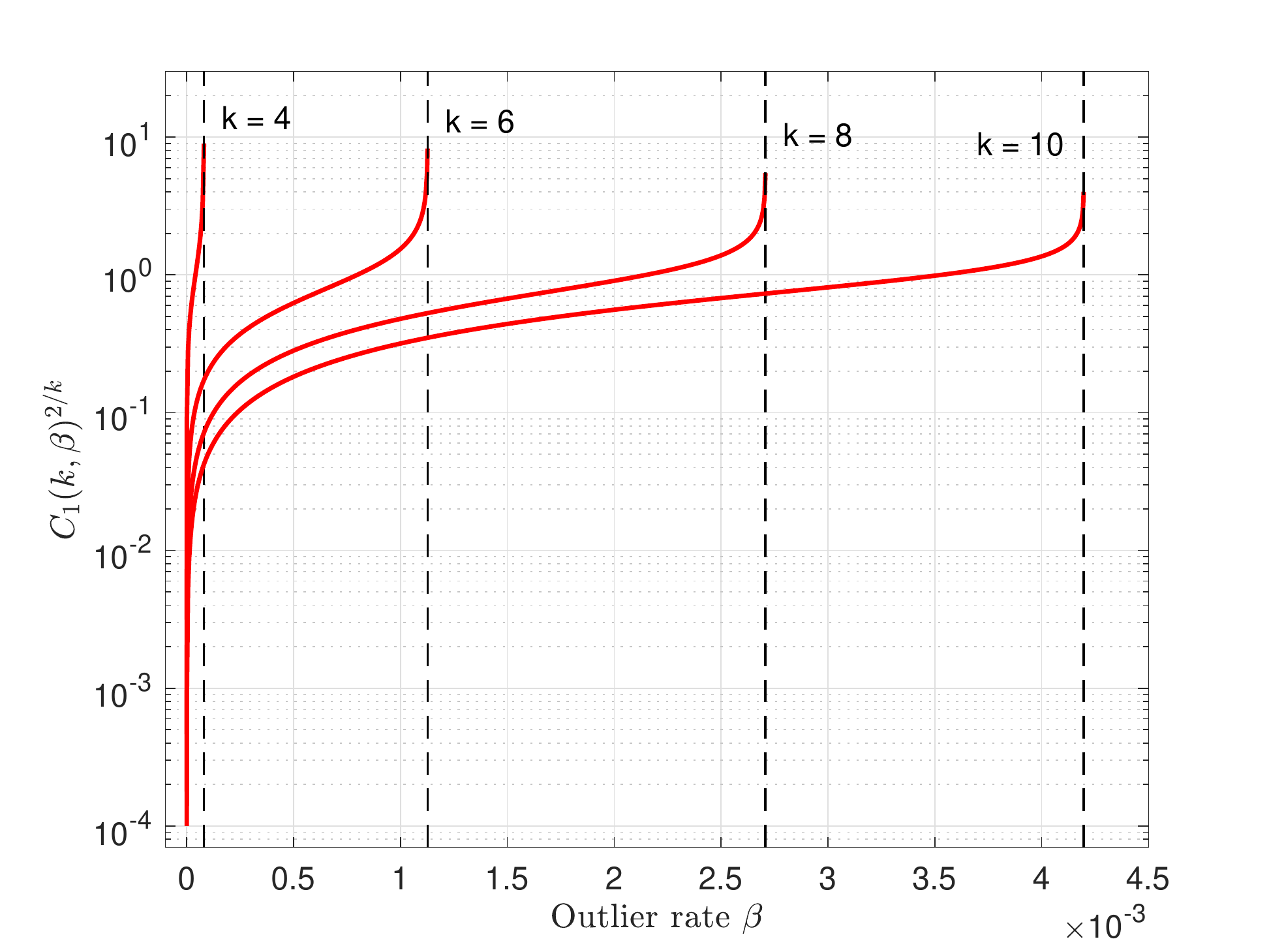}
\vspace{-7mm}\\ (c)
\end{minipage}
\begin{minipage}[b]{0.48\linewidth} 
\centering
\includegraphics[width=1\columnwidth, trim= 0mm 0mm 0mm 0mm, clip]{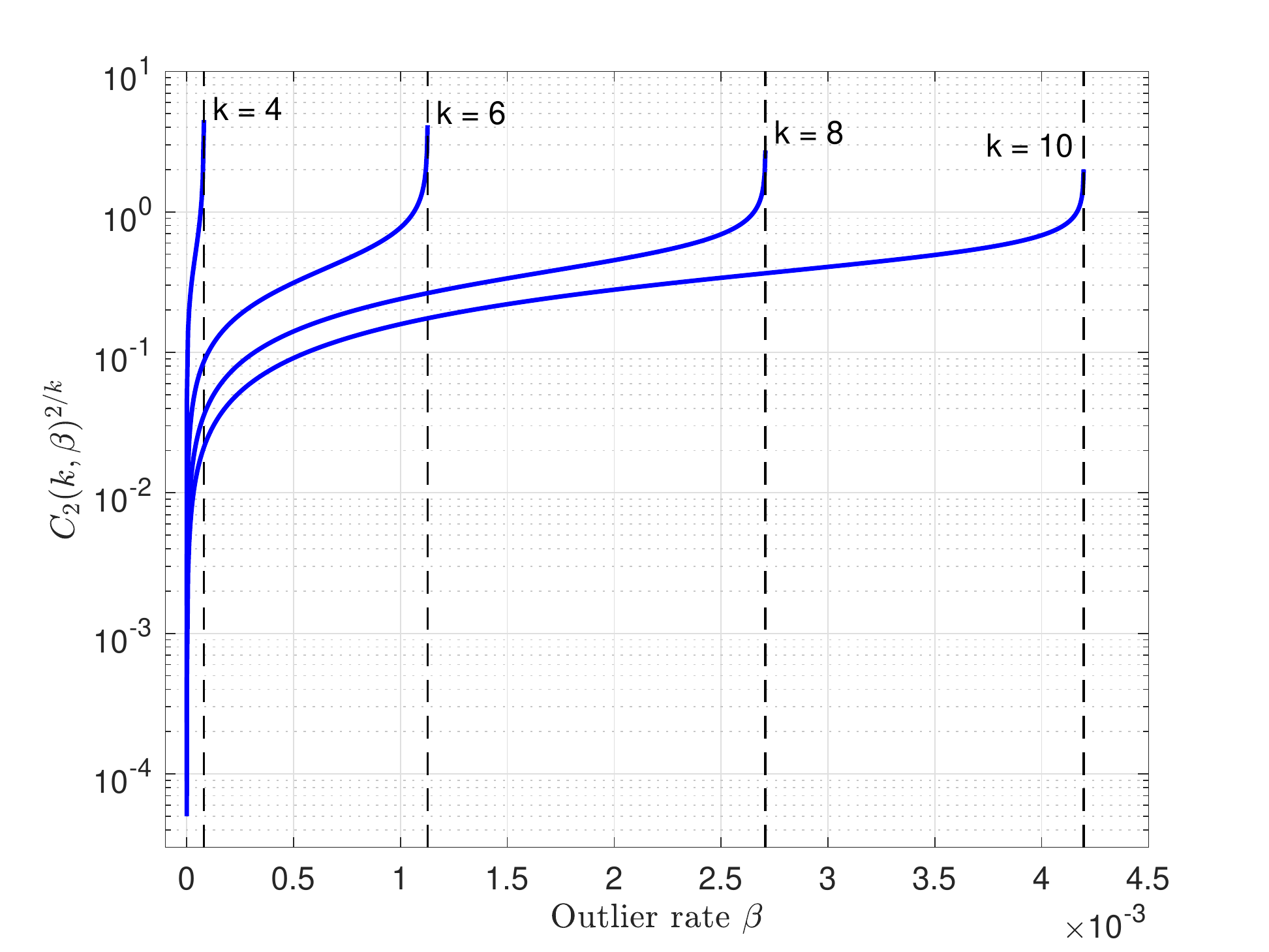}
\vspace{-7mm}\\ (d)
\end{minipage}
\caption{Coefficients $C_1(\relaxLevel,\outlierRate)^{\frac{2}{k}}$ (left column), $C_2(\relaxLevel,\outlierRate)^{\frac{2}{k}}$ (right column), and $\beta_\max$ (visualized as dashed vertical lines in each figure) in~\cref{thm:lowOut-apriori-LTS-objective} for different choices of $C(t)^t$ (sub-figures (a) and (b)) and $k$ (sub-figures (c) and (d)), and for increasing outlier rates. 
 \label{fig:boundObjective}}
\end{figure}

\myParagraph{Bound in~\cref{thm:lowOut-apriori-LTS-objective}}
We evaluate the bound in~\cref{thm:lowOut-apriori-LTS-objective} by plotting the coefficients $C_1(\relaxLevel,\outlierRate)^{\frac{2}{k}}$ and $C_2(\relaxLevel,\outlierRate)^{\frac{2}{k}}$ in the theorem, and the maximum outlier rate $\beta_\max$ for which the theorem holds (\ie $\beta_\max = \sqrt[\frac{k}{2}-1]{
1 / (\Crelax^{\frac{\relaxLevel}{2}} 2^{ 3\relaxLevel-1} )}$).
 Ideally, we would like for both $C_1(\relaxLevel,\outlierRate)^{\frac{2}{k}}$ and $C_2(\relaxLevel,\outlierRate)^{\frac{2}{k}}$ to be as small as possible, which would make the residual error of our estimate $\vxx\ltssdp$ with respect to the inliers $\err_{\DistSampleInliers} ( \vxx\ltssdp )$ as close as possible to the residual error $\optt_{\DistSampleInliers}$ an oracle estimator that has access to all the inliers would commit. Moreover, we would like $\beta_\max$ to be as large as possible (but still below $0.5$), in order for the theorem to be broadly applicable.

\Cref{fig:boundObjective}(a-b) show the coefficients $C_1(\relaxLevel,\outlierRate)^{\frac{2}{k}}$ and $C_2(\relaxLevel,\outlierRate)^{\frac{2}{k}}$ for $\relaxLevel =4$ and for increasing outlier rates $\outlierRate$; we plot multiple lines corresponding to different choices of the coefficient $C(t)^t \in \{1,2,4,6\}$.
For each choice of $C(t)^t$ we also plot the maximum outlier rate $\beta_\max$ (visualized as a  dashed vertical line) for which that error bound holds (indeed it can be seen that the bounds diverge as they approach their admissible upper bound).
To interpret the coefficients, the reader might observe that when $C_1(\relaxLevel,\outlierRate)^{\frac{2}{k}} = 1$ it means that the estimator has twice the error of the oracle, while when  $C_1(\relaxLevel,\outlierRate)^{\frac{2}{k}} = 10$ it has an error that is more than 10 times larger than the oracle.
By observing the $x$-axis of the plot, we quickly realize that for $k=4$, even for small values of $C(t)^t$, the maximum outlier rate captured by the theorem is very small (in the order of $10^{-4}$); this makes the theorem overly conservative for robotics applications where even in the low-outlier case, the outliers can approach  $50\%$, and where the number of measurements typically ranges from tens to thousands (which, with an outlier rate of $10^{-4}$, would only allow for a handful of outliers). 

\Cref{fig:boundObjective}(c-d) show the coefficients $C_1(\relaxLevel,\outlierRate)^{\frac{2}{k}}$ and $C_2(\relaxLevel,\outlierRate)^{\frac{2}{k}}$ for $C(t)^t =6$ and for increasing outlier rates $\outlierRate$;
in this case, we plot multiple lines corresponding to different choices of $k \in \{4,6,8,10\}$.
For each choice of $k$ we also plot the maximum outlier rate $\beta_\max$ (visualized as a  dashed vertical line). 
As we mentioned, the bound in the theorem becomes stronger for higher $k$ (at the cost of much increased computation in solving the moment relaxation). At the same time, even with a $k$ as large at $10$ the bound remains applicable only to outlier rates below $4.5 \cdot 10^{-3}$, which is unfortunately still tiny to be of practical interest for robotics and vision. 

\subsection{Low Outlier Rates: A Priori Bounds and Anti-Concentration}
\label{sec:experiments-anticon}

This section shows that $k$-certifiable anti-concentration is a more stringent condition in practice, 
but it can  still be satisfied for $k=6$ and for suitable choices of $\eta$ that make the error bounds in~\cref{thm:lowOut-apriori-LTS,thm:lowOut-apriori-MC} nontrivial. While in the rotation search problem the property is only satisfied for outlier rates close to zero, we show that on a variant of the problem, certifiable anti-concentration is satisfied for outlier rates up to $50\%$. To the best of our knowledge, this is the first study evaluating certifiable anti-concentration numerically and showing its applicability with low-degree polynomials  $p$.

\myParagraph{Numerically checking anti-concentration}
Certifiable anti-concentration can be assessed numerically (given the set of inliers $\InlierSet$) using the definition in eqs.~\eqref{eq:certAntiConSet1}-\eqref{eq:certAntiConSet2}. 
Let us rewrite the anti-concentration conditions as two polynomial optimization problems:
\begin{align}
\label{eq:pop-anticon1}
\min_\vv \;\;&\;\; p^2\left(\normTwo{\MA_i\tran \vv} \right) - (1 - \delta)^2, 
\;\;\;\;
\subject 
\normTwo{\MA_i\tran \vv}^2  \leq \delta^2
\\
\label{eq:pop-anticon2}
\min_\vv \;\;&\;\; C \, \delta \,M^2 - \|\vv\|^2 \cdot
\aveOverInliers \; p^2\left(\normTwo{\MA_i\tran \vv}\right), 
\;\;\;\;
\subject \normTwo{\vv}^2 \leq M^2  \mathper
\end{align}
Anti-concentration requires that the optimal cost of both problems remains greater than zero 
such that $p^2\left(\normTwo{\MA_i\tran \vv} \right) \geq (1 - \delta)^2$ (for all $i \in \InlierSet$) and 
$\|\vv\|^2 \cdot
\aveOverInliers \; p^2\left(\normTwo{\MA_i\tran \vv}\right) \leq C \, \delta \,M^2$
as requested by the definition.
\emph{Certifiable} anti-concentration requires an \sos proof of anti-concentration; 
because of the duality between moment relaxations and \sos relaxations, 
certifiable anti-concentration essentially requires
that even if we apply a moment relaxation (of suitable order) to the polynomial optimization problems~\eqref{eq:pop-anticon1}-\eqref{eq:pop-anticon2}, then the optimal objective of the relaxation remains positive. Clearly, asking the moment relaxation to produce a positive objective is a sufficient condition for the objective of~\eqref{eq:pop-anticon1}-\eqref{eq:pop-anticon2} to remain positive, hence certifiable anti-concentration is a more stringent condition. 

Given the matrices $\MA_i$ (for all the inliers), the polynomial $p$, and the coefficients $C, \delta, M$, and $\eta$, we can numerically solve a moment relaxation of~\eqref{eq:pop-anticon1}-\eqref{eq:pop-anticon2} using the
{\tt findbound} function in the SOSTOOLS library~\cite{Papachristodoulou13-sostools}.
We note that~\eqref{eq:pop-anticon1} involves polynomials of degree up to $2 \cdot \degree{p}$ 
while~\eqref{eq:pop-anticon2} involves polynomials of degree up to $2 \cdot \degree{p} + 2$, both in the variable $\vv$ (which is 9-dimensional in our rotation search example).
In practice, the {\tt findbound} function runs out of memory for degrees larger than $8$ in our problem, which imposes to choose
a degree-2 polynomial $p$. For such choice of $p$,~\eqref{eq:pop-anticon1} involves degree-4 polynomials, while~\eqref{eq:pop-anticon2} involves degree-6 polynomials, which has implications on the minimum order of the moment relaxation we compute via {\tt findbound}.\footnote{
{\tt findbound} takes a ``{\tt degree}'' input argument, that can be understood as the level of the pseudo-distribution computed by the function. In this case, we use the minimum {\tt degree}=4 for~\eqref{eq:pop-anticon1}
and {\tt degree}=6 for~\eqref{eq:pop-anticon2}.}  

So far we concluded that we can numerically check certifiable anti-concentration given the 
matrices $\MA_i$, the polynomial $p$, and the coefficients $C, \delta, M$, and $\eta$.
The matrices $\MA_i$ are given as an input to the problem, while our theorems (\cref{thm:lowOut-apriori-LTS,thm:lowOut-apriori-MC,thm:lowOut-apriori-TLS}) specify how to set $\delta$, $C$, and $M$.
Therefore, we are only left to figure our how to set the constant $\eta$ and the polynomial $p$.
As we will see, these parameters control the trade-off between tightness of our error bound and computational cost (more precisely, the order of the moment relaxation) of the corresponding estimators. 

\begin{figure}[t]
\begin{minipage}[b]{0.5\linewidth}
\centering
\includegraphics[width=0.9\columnwidth, trim= 0mm 0mm 0mm 0mm, clip]{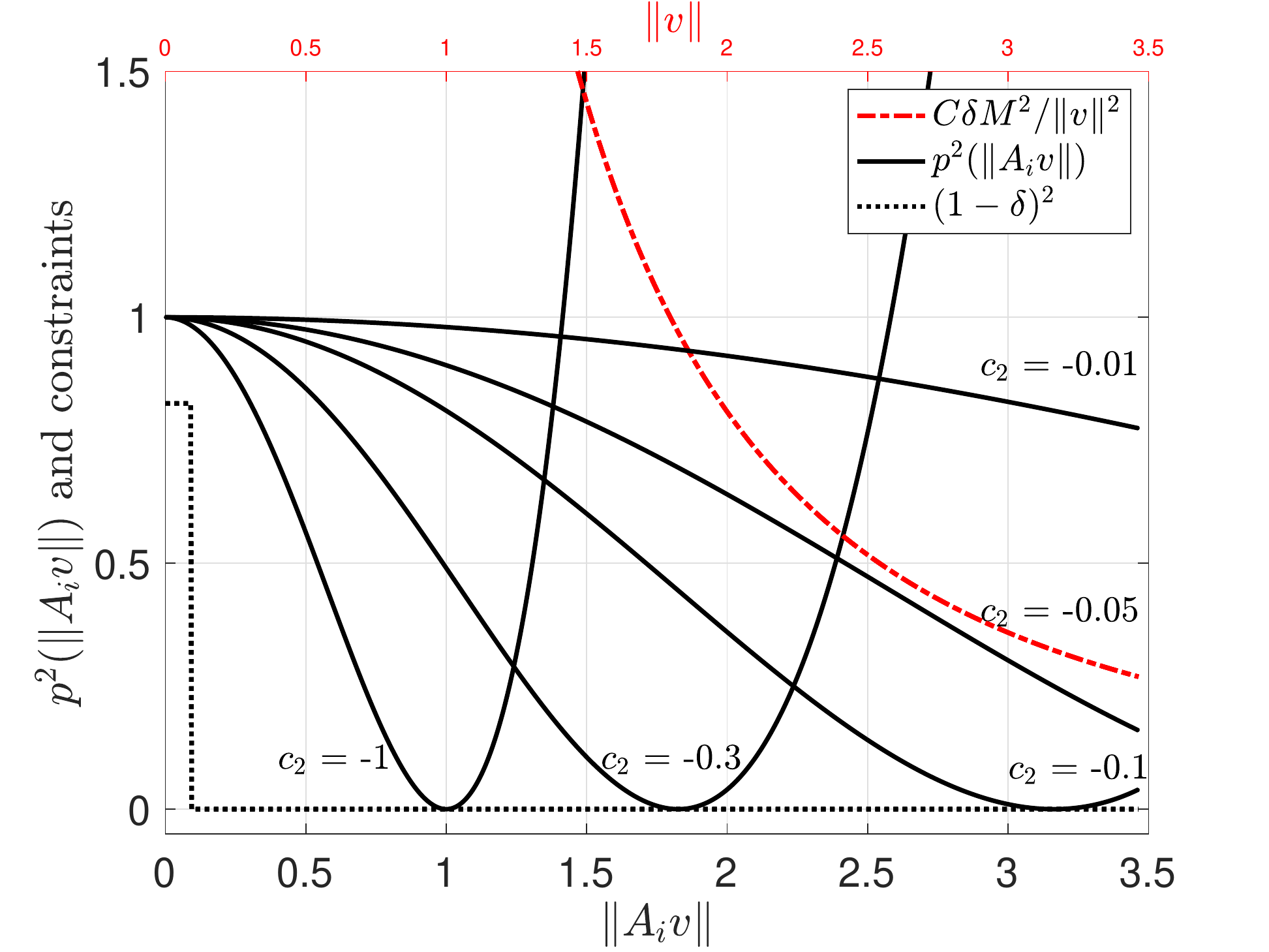}
\vspace{-1mm}\\ (a)
\end{minipage}
\begin{minipage}[b]{0.5\linewidth} 
\centering
\includegraphics[width=0.9\columnwidth, trim= 0mm 0mm 0mm 0mm, clip]{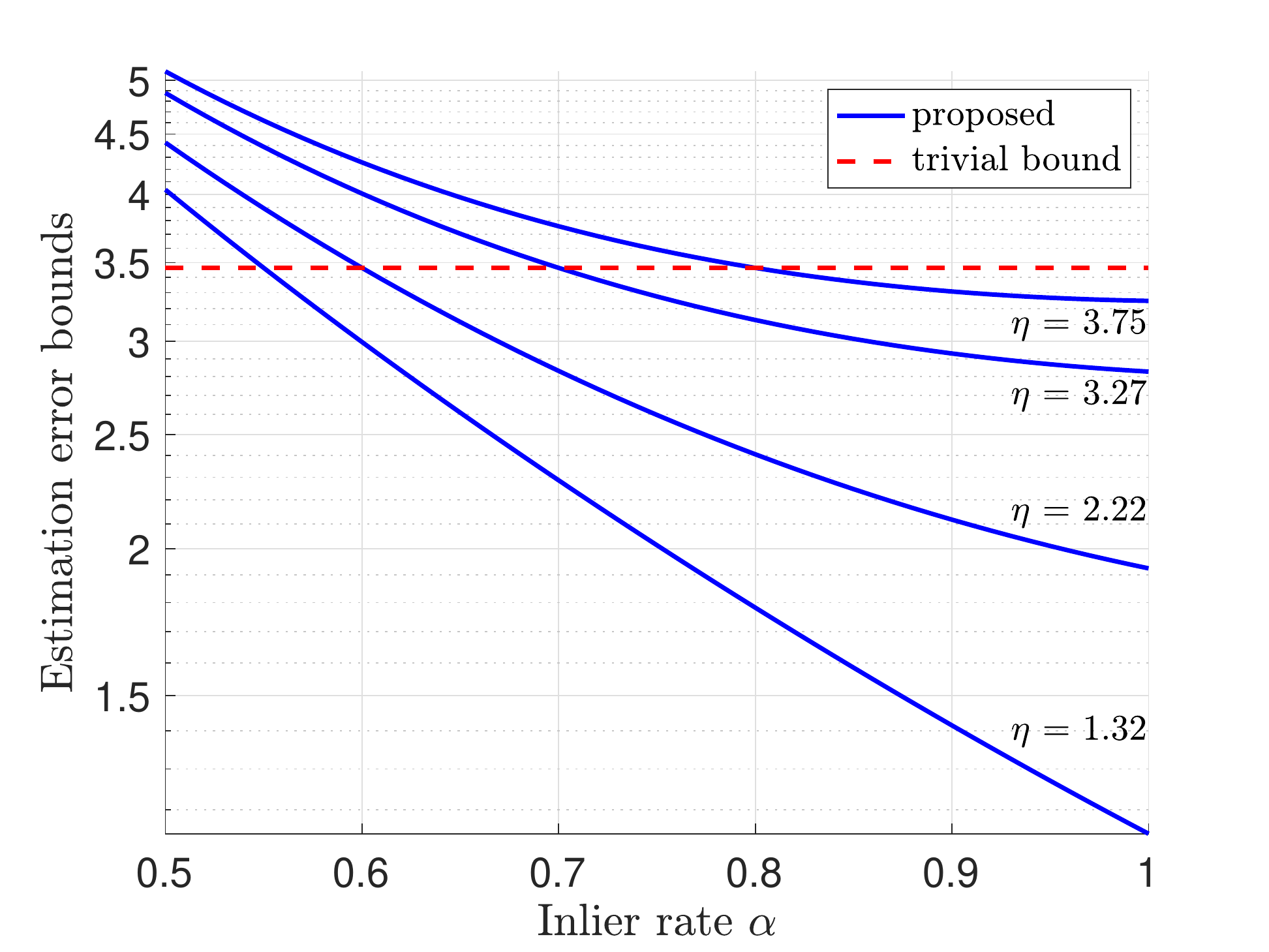}
\vspace{-1mm}\\ (b)
\end{minipage}
\\
\begin{minipage}[b]{0.5\linewidth} 
\centering
\includegraphics[width=0.9\columnwidth, trim= 0mm 0mm 0mm 0mm, clip]{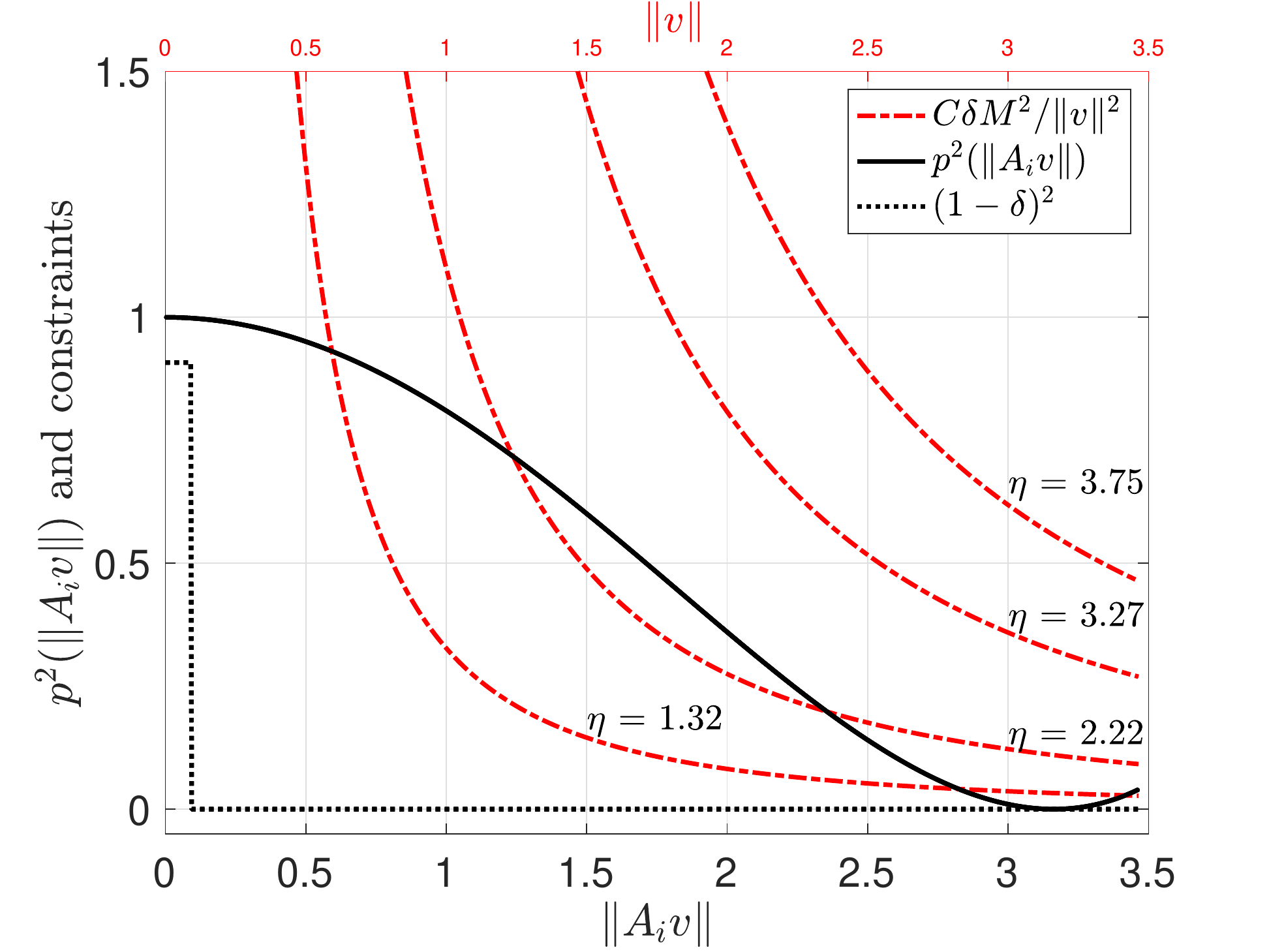}
\vspace{-1mm}\\ (c)
\end{minipage}
\begin{minipage}[b]{0.5\linewidth} 
\centering
\includegraphics[width=0.9\columnwidth, trim= 0mm 0mm 0mm 0mm, clip]{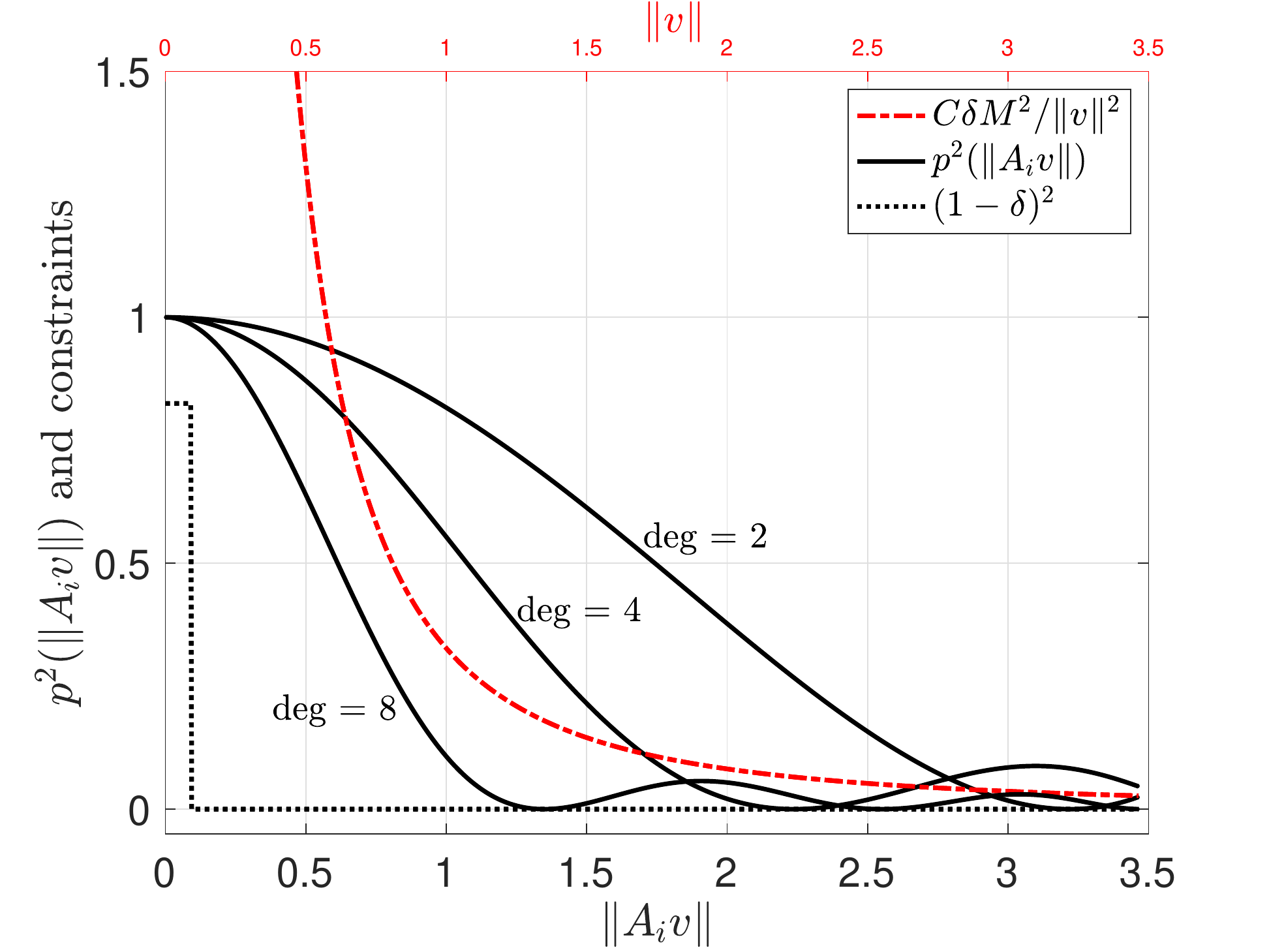}
\vspace{-1mm}\\ (d)
\end{minipage}
\vspace{-6mm}
\caption{Visualization of the functions involved in the definition of anti-concentration in eq.~\eqref{eq:functionsAntiCon1}-\eqref{eq:functionsAntiCon2}.
(a) Polynomials $p^2$ parametrized as in~\eqref{eq:p} for different choices of coefficient $c_2$ 
compared to the bounds in~\eqref{eq:functionsAntiCon1}-\eqref{eq:functionsAntiCon2} for $\eta = 3.26$.
(b) Estimation error bounds from~\cref{thm:lowOut-apriori-MC,thm:lowOut-apriori-LTS} for different choices of $\eta$.
(c) Polynomials $p^2$ parametrized as in~\eqref{eq:p} with $c_2=-0.1$, and bounds in~\eqref{eq:functionsAntiCon1}-\eqref{eq:functionsAntiCon2} for different choices of $\eta$.
(d) Polynomials $p^2$ of degree 2, 4, 8, and suitable choices of coefficients 
compared to the bounds in~\eqref{eq:functionsAntiCon1}-\eqref{eq:functionsAntiCon2} for $\eta = 1.32$.
 \label{fig:anticonVisualizations}}
\end{figure}

\myParagraph{Setting $p$ and $\eta$}
The goal of this subsection is to provide insights on the selection of $p$ and $\eta$.
As we have already observed, for computational reasons we are going to restrict ourselves to degree-2 polynomials $p$; however later in this subsection we are also going to comment on the impact of higher-degree polynomials $p$. 
\cref{def:certAnticon} imposes $p(0)=1$ and requires $p$ to be an even polynomial, hence a degree-2 
polynomial $p$ is restricted to the form:
\beq
\label{eq:p}
p(a) = 1 + c_2 a^2 \mathcom
\eeq
for some scalar variable $a$ and coefficient $c_2$. 
In order to evaluate potential choices of $c_2$ we consider the case of a single inlier ($|\InlierSet|=1$) and visualize the functions involved in the definition of certifiable anti-concentration.
In particular, we observe that a necessary condition for certifiable anti-concentration to be satisfied is that the following inequalities hold in the traditional sense:\footnote{Recall again that asking for an \sos proof imposes a stronger condition compared to just asking for the inequality to be satisfied in the traditional sense.} 
\begin{align}
   p^2\left(\normTwo{\MA_i\tran \vv} \right) \geq (1-\delta)^2, 
 & \quad \forall \vv \text{ such that } \normTwo{\MA_i\tran \vv}^2  \leq \delta^2
\label{eq:functionsAntiCon1}\\
p^2\left(\normTwo{\MA_i\tran \vv}\right) \leq \frac{C \, \delta \,M^2}{\|\vv\|^2}, 
& \quad \forall \vv \text{ such that }
 \normTwo{\vv}^2 \leq M^2 \mathcom
 \label{eq:functionsAntiCon2}
\end{align}
where we assumed $|\InlierSet|=1$.
Now we see that the conditions~\eqref{eq:functionsAntiCon1}-\eqref{eq:functionsAntiCon2} can be understood as constraints on the behavior of the polynomial 
$p^2\left(\normTwo{\MA_i\tran \vv} \right)$.~\Cref{fig:anticonVisualizations}(a) shows the behavior of the functions involved in the inequalities~\eqref{eq:functionsAntiCon1}-\eqref{eq:functionsAntiCon2}. 
We visualize 4 choices of $p^2\left(\normTwo{\MA_i\tran \vv} \right)$, corresponding to 4 different choices of coefficients $c_2 \in \{-1, -0.3, -0.1, -0.05, -0.01\}$ in~\eqref{eq:p}. We also 
 visualize the bounds $(1 - \delta)^2$ in~\eqref{eq:functionsAntiCon1}, which only has to hold for $\normTwo{\MA_i\tran \vv}^2  \leq \delta^2$. Moreover, we visualize the upper bound $\frac{C \, \delta \,M^2}{\|\vv\|^2}$ for $C,\delta,M$ chosen as in~\cref{thm:lowOut-apriori-LTS} and
 for $\eta = 3.26$ (we discuss how to choose $\eta$ below). Note that the latter function has a different x-axis ($\|\vv\|$, in red), but, noting that the matrices $\MA_i$ in the rotation search problem have maximum  singular value equal to 1, the axes are still useful to visualize together.\footnote{Since the matrices $\MA_i$ are built from unit vectors, it can be easily shown that the matrices have the largest singular value equal to $1$. 
Therefore, for each choice of $\|\vv\|$ one can choose the direction of $\vv$ to make $\|\MA_i \vv\| = \|\vv\|$, making the top and bottom x-axes of~\cref{fig:anticonVisualizations}(a) commensurable; however note that $\|\MA_i \vv\| \leq \|\vv\|$ in general, and $\|\MA_i \vv\|$ can become as small as zero for rank-deficient $\MA_i$.}
In this example with $|\InlierSet|=1$, a necessary condition for certifiable anti-concentration to hold is that the squared polynomial $p^2$ remains below the red upper bound in~\cref{fig:anticonVisualizations}(a) and above the lower bounds $(1-\delta)^2$ visualized as a black dotted line in the figure.
From the figure it is clear that the latter condition is relatively easy to satisfy (at least for degree-2 polynomials), while remaining below the red upper bound is a more stringent constraint.
In particular, if we choose large negative values of the coefficient $c_2$ (\eg $c_2 = -1$ or $c_2 = -0.3$), the squared polynomial $p^2$ ``bounces back'' and quickly reaches the upper bound.  If the coefficient is too small (\eg $c_2 = -0.01$ in the figure), the polynomial does not decrease sufficiently fast and still violates the upper bound. However, choosing $c_2 = -0.1$ and $c_2 = -0.05$ satisfies both constraints, with $c_2 = -0.1$ being the farthest of the two from the upper bound. Later on, we will show that $c_2 = -0.1$ indeed leads to certifying anti-concentration in our rotation search problem for a suitably chosen $\eta$.
We remark that while~\cref{fig:anticonVisualizations}(a) is somewhat insightful in visualizing the constraints that the polynomial $p$ has to satisfy, when $|\InlierSet| > 1$ the upper bound has to be only satisfied in average  (\ie $\aveOverInliers p^2\left(\normTwo{\MA_i\tran \vv}\right) \leq \frac{C \, \delta \,M^2}{\|\vv\|^2}$ rather than $p^2\left(\normTwo{\MA_i\tran \vv}\right) \leq \frac{C \, \delta \,M^2}{\|\vv\|^2}$ for all $i$), giving some extra slack in satisfying anti-concentration.

Now let us turn our attention to the parameter $\eta$.
The parameter $\eta$ can be roughly understood as a measure of the ``quality'' of the inlier set: 
good inlier sets will be certifiably anti-concentrated for a small $\eta$, which induces 
better error bounds (all bounds in~\cref{thm:lowOut-apriori-LTS,thm:lowOut-apriori-MC,thm:lowOut-apriori-TLS} increase monotonically with $\eta$). 
On the other hand, while we can arbitrarily increase the $\eta$ to make a given batch of inliers certifiably anti-concentrated, a large $\eta$ might lead to a trivial bound in the corresponding theorem. 
In particular, for the bounds in~\cref{thm:lowOut-apriori-LTS,thm:lowOut-apriori-MC} to be nontrivial (\ie smaller than the trivial bound $2\boundx$), we need:
\begin{align}
\label{eq:choiceEta}
\left( \frac{\inlierRate \; \eta \;  }{2} +  2\frac{1-\inlierRate}{\inlierRate} \right) < 2 
\quad \Longrightarrow \quad
\eta < \frac{2}{\inlierRate} \left(2-2\frac{1-\inlierRate}{\inlierRate} \right) \mathper
\end{align} 
Therefore, for a certain inlier rate $\inlierRate$, we can choose the $\eta$ such that our theorems do not produce a trivial bound and then test is such $\eta$ ensures certifiable anti-concentration. Ideally, we would like to have a nontrivial bound already around $\inlierRate=0.5$, but as we will see in a second that might lead to a small $\eta$, which makes the certifiable anti-concentration condition too strict. Below, we choose potential values of $\eta$ according to~\eqref{eq:choiceEta},  
such that they produce nontrivial bounds for $\inlierRate \in [0.55, 0.6, 0.7, 0.8]$, respectively leading to 
the $\eta \in [1.32, 2.22, 3.26, 3.75]$.~\Cref{fig:anticonVisualizations}(b) shows the trend of the bounds in~\cref{thm:lowOut-apriori-LTS,thm:lowOut-apriori-MC} for those choices of $\eta$, compared with the trivial bound $2 \boundx = 2 \sqrt{3}$. As expected, our choices of $\eta$ are designed such that the proposed bounds  move below the trivial bound for $\inlierRate \in [0.55, 0.6, 0.7, 0.8]$, respectively. 
 Ideally, we would like to choose the smallest possible $\eta$.~\Cref{fig:anticonVisualizations}(c) visualizes the price we need to pay to decrease the $\eta$:  the figure shows how smaller values of $\eta$ lead to increasingly tight upper bounds $\frac{C \, \delta \,M^2}{\|\vv\|^2}$ that the polynomial $p$ needs to satisfy.
 The figure already shows that in this case ($|\InlierSet|=1$) it is unlikely we can find a degree-2 polynomial that satisfies the upper bound for $\eta = 1.32$.
%
\Cref{fig:anticonVisualizations}(d) shows that we can still satisfy the upper bound $\frac{C \, \delta \,M^2}{\|\vv\|^2}$ for a small $\eta$ if we are willing to increase the degree of the polynomial $p$: the figure plots 
even polynomials of degree 2, 4, and 8 (with suitably chosen coefficients) contrasting them with the upper bound obtained for $\eta=1.32$. It can be observed that the degree-8 polynomial remains below the upper bound in the figure. This shows that we can afford stronger guarantees (\ie smaller $\eta$) if we are willing to pay the computational cost of using a higher-degree polynomial, which in turns implies a larger $k$ in the definition of certifiable anti-concentration and a higher order of the moment relaxations involved in our robust estimators.
In the following we stick to the more realistic case of degree-2 polynomial $p$ and we numerically investigate which $\eta$ we can ``afford'' in synthetic problems for both the rotation search problem and for a variation of the problem.

\begin{figure}[t]
\begin{minipage}[b]{0.33\linewidth}
\centering
\includegraphics[width=1\columnwidth, trim= 0mm 0mm 0mm 0mm, clip]{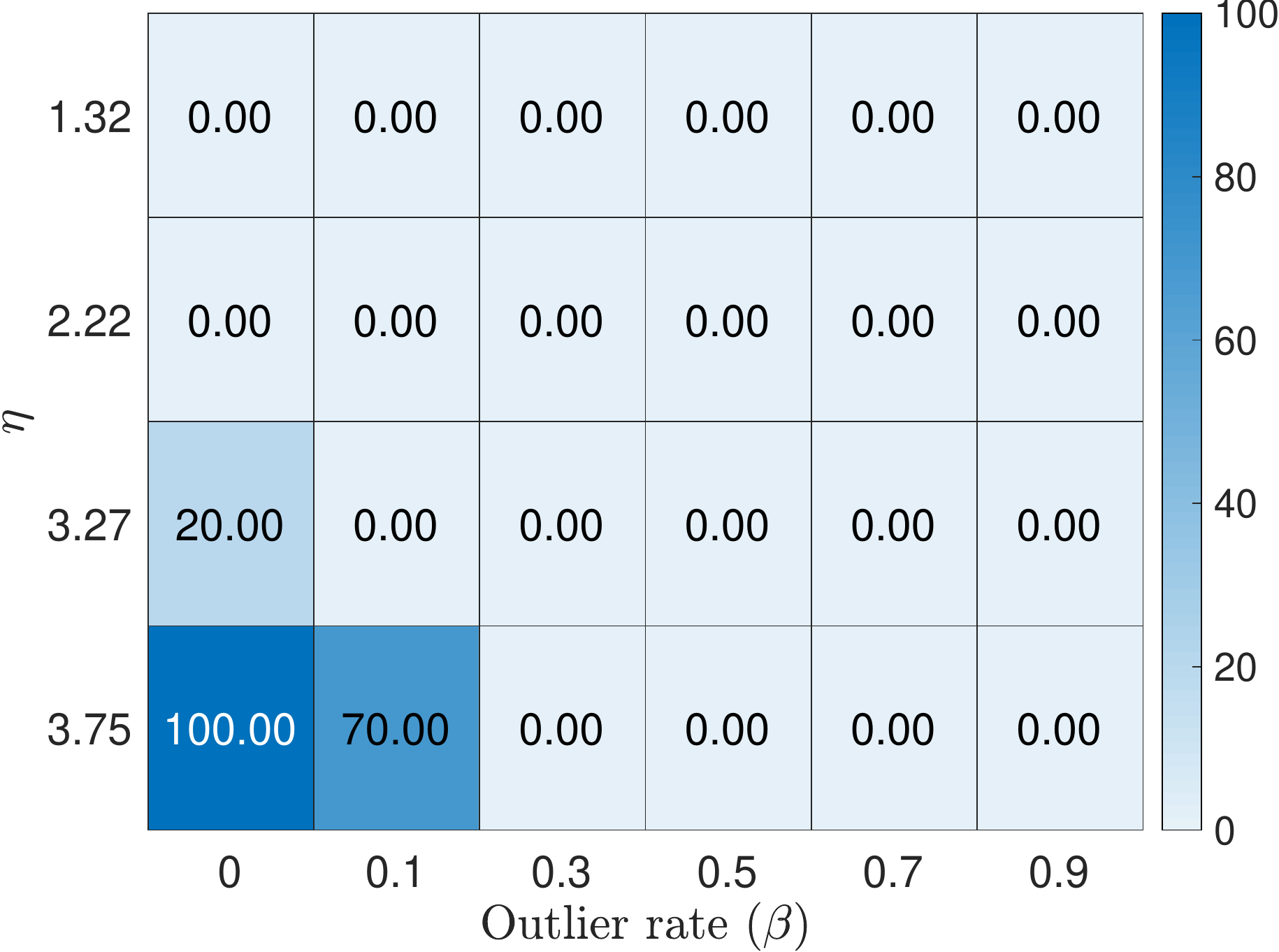}
\vspace{0mm}\vspace{-3mm}\\
{(a)}
\end{minipage}
\begin{minipage}[b]{0.33\linewidth} 
\centering
\includegraphics[width=1\columnwidth, trim= 0mm 0mm 0mm 0mm, clip]{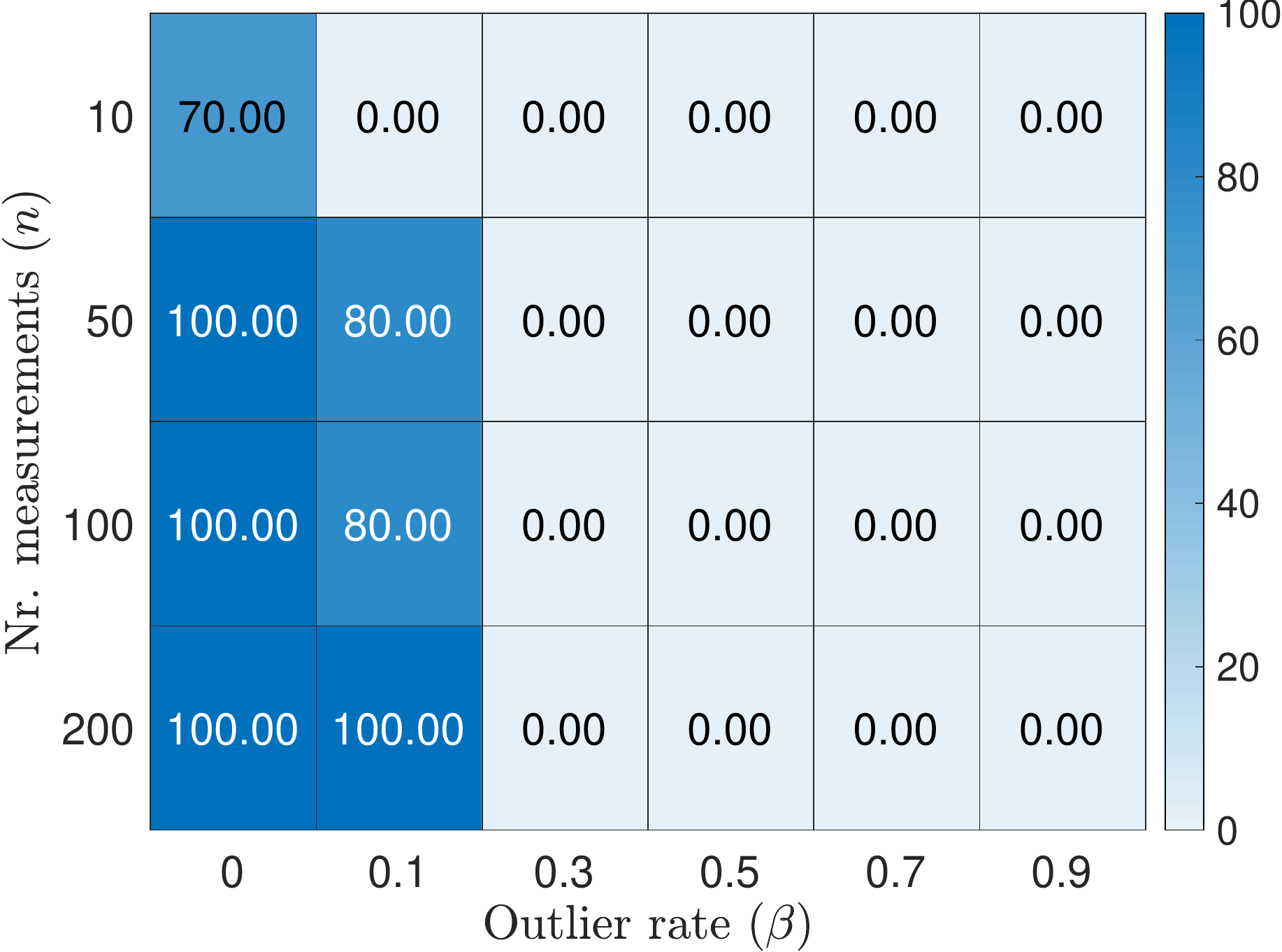} 
\vspace{0mm}\vspace{-3mm}\\
(b)
\end{minipage}\vspace{-3mm}
\begin{minipage}[b]{0.33\linewidth} 
\centering
\includegraphics[width=1\columnwidth, trim= 0mm 0mm 0mm 0mm, clip]{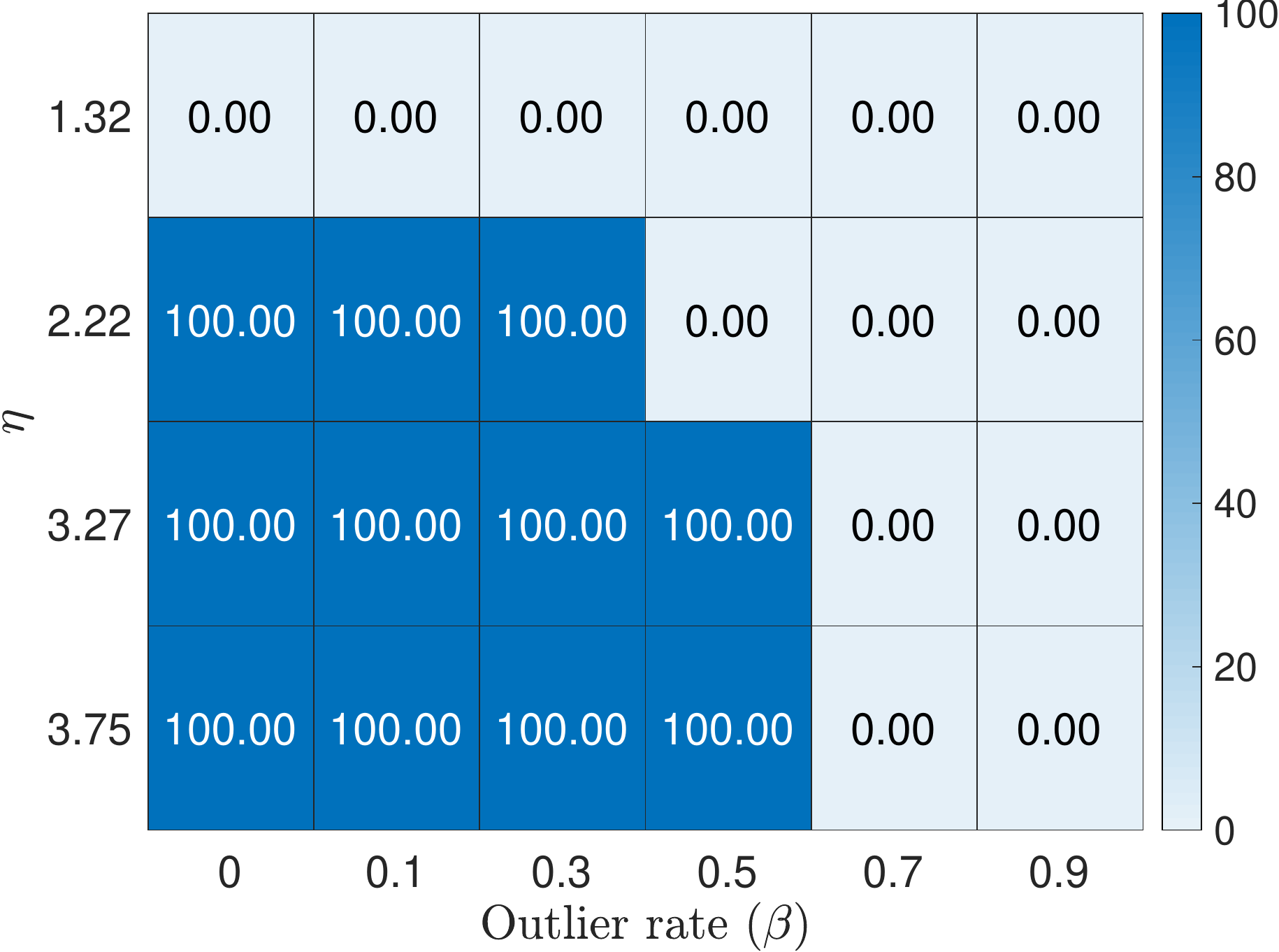}
\vspace{0mm}\vspace{-3mm}\\
{(c)}
\end{minipage}\vspace{-3mm}
\caption{Certifiable anti-concentration results on synthetic data:
(a) Rotation search problem: percentage of tests where $6$-certifiable anti-concentration holds 
for various choices of $\eta$ and outlier rates $\beta$ (with $\nrMeasurements = 50$). 
(b) Rotation search problem: percentage of tests where $6$-certifiable anti-concentration holds 
for increasing number of measurements $\nrMeasurements$ and outlier rates $\beta$ (with $\eta = 3.75$). 
(b) Modified version of the rotation search problem: 
percentage of tests where $6$-certifiable anti-concentration holds 
for various choices of $\eta$ and $\beta$ (with $\nrMeasurements = 50$).
Statistics computed over 10 runs. 
\label{fig:stats_synthetic_anticon_modified}}
\end{figure}

\myParagraph{Experiments on synthetic data} 
We choose the univariate polynomial to be $p(a) = 1 - 0.1 a^2$ and test certifiable anti-concentration in a rotation search problem with $\nrMeasurements = 50$ measurements, generated as described in~\cref{sec:experiments-setup}. Testing certifiable anti-concentration is much slower than testing certifiable hypercontractivity, hence we only compute statistics over 10 runs. We test whether 
certifiable anti-concentration holds for increasing outlier rates $\outlierRate \in \{0;0.1;0.3;0.5;0.7;0.9\}$ and for different choices of $\eta \in \{1.32; 2.22; 3.26; 3.75\}$. 
\cref{fig:stats_synthetic_anticon_modified}(a) reports the percentage of tests where $6$-certifiable anti-concentration holds for various choices of $\eta$  and outlier rates $\outlierRate$.
We observe that certifiable anti-concentration is only satisfied for a relatively large $\eta$ and for small outlier rates. In particular, the parameter $C$ in~\eqref{eq:functionsAntiCon2} depends on the square of $\inlierRate = 1-\outlierRate$ according to~\cref{thm:lowOut-apriori-LTS,thm:lowOut-apriori-MC}, hence quickly making the condition too strict for increasing outlier rates $\outlierRate$. 

\Cref{fig:stats_synthetic_anticon_modified}(b) reports the percentage of tests where $6$-certifiable anti-concentration holds for increasing number of measurements $\nrMeasurements$  and outlier rates $\outlierRate$, and for $\eta = 3.75$. The figure shows that certifiable anti-concentration is more likely to hold for larger sets of measurements, while still only applying to problems with relatively low outlier rates (\ie 10\%).
While \Cref{fig:stats_synthetic_anticon_modified}(a-b)  seem to suggest that certifiable anti-concentration with degree-2 polynomials $p$ is too strict of a condition to be of practical interest, below we report a variant of the rotation search problem, where the condition is satisfied for much larger outlier rates $\outlierRate$.

We consider a modified version of the rotation search problem, where each measurement, instead of being a single 3D vector $\vb_i$, is a triplet of orthogonal vectors $\vb_i\at{1}$, $\vb_i\at{2}$, $\vb_i\at{3}$. In this variant of the rotation search problem, the measurement model becomes:
\beq
\label{eq:modifiedWahba}
\vy_i = \matTwo{\vb_i\at{1} \\ \vb_i\at{2} \\ \vb_i\at{3}}
 = \matTwo{\MR \va_i\at{1} + \vepsilon\at{1} \\ \MR \va_i\at{2}+ \vepsilon\at{2} \\ \MR \va_i\at{3} + \vepsilon\at{3}}
  = \matTwo{ (\va_i\at{1})\tran \kron \eye_3 \\ (\va_i\at{2})\tran \kron \eye_3 \\ (\va_i\at{3})\tran \kron \eye_3 } \cdot \vectorize{\MR} = \MA_i\tran \vxx\gt + \vepsilon \mathper
\eeq
The expert reader might notice that this is now similar to a \emph{rotation averaging} problem (see~\cref{ex:singleRotationAveraging} in~\cref{sec:openProblems}), since each measurement allows reconstructing the full rotation $\MR$. We numerically test certifiable anti-concentration in this problem variant. Towards this goal, we use the same data generation protocol described in~\cref{sec:experiments-setup} to sample $\va_i$, and 
set $\va_i\at{1} = \va_i$; then we create an orthogonal vector $\va_i\at{2}$ by sampling a second random vector, projecting it onto the null space of $\va_i\at{1}$, and normalizing it to have unit norm; 
finally, we set $\va_i\at{3} = \va_i\at{1} \times \va_i\at{2}$, where $\times$ is the cross product.
Note that the vectors $\vb_i\at{1}, \vb_i\at{2}, \vb_i\at{3}$ do not appear in the definition of anti-concentration, hence we do not have to generate them.
\Cref{fig:stats_synthetic_anticon_modified}(c) reports the percentage of tests where $6$-certifiable anti-concentration holds for various choices of $\eta$ and outlier rates $\outlierRate$ in this variant of the rotation search problem.
In this case, certifiable anti-concentration is more broadly satisfied for several choices of $\eta$ and for outlier rates as high as $50\%$.
To the best of our knowledge, this is the first numerical evidence supporting the use of certifiable anti-concentration with degree-2 polynomials $p$. Moreover, since certifiable anti-concentration holds for relatively small $\eta$, this variant of the problem enjoys the tighter error bounds shown in~\cref{fig:anticonVisualizations}(b).
 Besides the intuition that this variant of the rotation search problem is ``easier'' (an aspect that is correctly captured by $\eta$ for which certifiable anti-concentration holds), there is a precise technical reason that makes the property easier to satisfy. 
 By inspection, we note that the matrix $\MA_i$ in eq.~\eqref{eq:modifiedWahba} is an orthogonal matrix,\footnote{$\MA_i\tran \MA_i \!=\!\matTwo{ (\va_i\at{1})\tran \kron \eye_3 \\ (\va_i\at{2})\tran \kron \eye_3 \\ (\va_i\at{3})\tran \kron \eye_3 } \!\cdot\! 
 \matTwo{ \va_i\at{1} \kron \eye_3 \; \va_i\at{2} \kron \eye_3 \; \va_i\at{3} \kron \eye_3 } 
 \!=\!
 \matThree{
(\va_i\at{1})\tran \va_i\at{1} \eye_3 
\!\!&\!\! (\va_i\at{1})\tran \va_i\at{2}  \eye_3 
\!\!&\!\!(\va_i\at{1})\tran \va_i\at{3}  \eye_3 
\\
\star
\!\!&\!\! (\va_i\at{2})\tran \va_i\at{2}  \eye_3 
\!\!&\!\!(\va_i\at{2})\tran \va_i\at{3}  \eye_3 
\\
\star
\!\!&\!\! \star
\!\!&\!\!(\va_i\at{3})\tran \va_i\at{3}  \eye_3
 } \!=\! \eye_9$} 
 hence $\| \MA_i\tran \vv \| = \|\vv\|$. This makes the top and bottom x-axes in~\cref{fig:anticonVisualizations}(a) the same, and the anti-concentration condition becomes a point-wise comparison between the polynomial $p^2$ and the bounds.\footnote{As we observed, in the standard rotation search problem, $\| \MA_i\tran \vv \|$ and $\|\vv\|$ are different in general, but it still holds $\| \MA_i\tran \vv \| \leq \|\vv\|$ since the largest singular value of $\MA_i$ is equal to $1$. 
 Since, however, for certain choices of $\vv$, $\| \MA_i\tran \vv \|$ can be much smaller than $\|\vv\|$ (and be even zero if the matrix $\MA_i$ is rank-deficient), the condition~\eqref{eq:functionsAntiCon2} may need to be satisfied for any $0 \leq \| \MA_i\tran \vv \| \leq \|\vv\|$. On the other hand, when the matrix $\MA_i$ is orthogonal, $\| \MA_i\tran \vv \| = \|\vv\|$ and the condition in~\eqref{eq:functionsAntiCon2} simplifies to $p^2\left(\normTwo{\vv}\right) \leq \frac{C \, \delta \,M^2}{\|\vv\|^2}$, which is less stringent to satisfy. 
 } 

The runtime to check certifiable anti-concentration is much higher than the one required for certifiable hypercontractivity, since now we have to solve constrained optimization problems, also involving higher-degree polynomials.
Solving a moment relaxation of~\eqref{eq:pop-anticon1} using {\tt findbound}~\cite{Papachristodoulou13-sostools} takes around 2.9 seconds on a standard laptop. However, solving a relaxation of~\eqref{eq:pop-anticon2} requires around 7 minutes, due to the need to use a higher-order relaxation (recall that~\eqref{eq:pop-anticon2} involves degree-6 polynomials).
%


\subsection{High Outlier Rates: List-Decodable Estimation}
\label{sec:experiments-ldr}

We conclude this experimental section by discussing list-decodable estimation and the high-outlier  regime. In particular, this section shows that despite the fact that the theoretical guarantees of~\cref{thm:highOut-apriori} fall short in practice, a small variation of~\cref{algo:ldr} has impressive  empirical performance. To the best of our knowledge, the algorithm we describe is the 
 first practical algorithm for list-decodable estimation based on the framework proposed in~\cite{Karmalkar19neurips-ListDecodableRegression}.

We start by observing that the theoretical results in~\cref{thm:highOut-apriori} fall short of providing usable performance guarantees in the high-outlier-rate case: first of all, certifiable anti-concentration, even in the simpler variant of the rotation search problem in~\cref{fig:stats_synthetic_anticon_modified}(c), does not empirically hold for outlier rates $\outlierRate > 0.5$ (at least when using a degree-2 $p$), hence the results only cover the low-outlier regime in practice. Second, the guarantees in~\cref{thm:highOut-apriori} (\ie the bound $\normTwo{\vxx - \vxx\gt} \leq  \eta \boundx$) are non-trivial only for $\eta < 2$; however, \cref{fig:stats_synthetic_anticon_modified}(c) shows that 
{anti-concentration starts holding for $\eta > 2$ in most cases,} making the result in~\cref{thm:highOut-apriori} of pure theoretical interest.
 Finally, the smallest $k$ allowed by the theorem is $k=6$, which leads to large order-3 moment relaxations that we cannot solve with current SDP solvers. 
Despite these limitations, one might still wonder if~\cref{algo:ldr} would produce good estimates even for lower relaxation orders, beyond what's covered by the theory.

 This section shows that a sparse order-2 relaxation of problem~\eqref{eq:LDR} in~\cref{algo:ldr} already leads to accurate list-decodable estimation. 
In particular, we show that a small variant of~\cref{algo:ldr}, that we call \emph{\SLIDESlong} (\SLIDES), 
(i) produces good estimates for problems with high rates of random outliers,
(ii) produces good estimates for problems with high rates of {``adversarial''} (more precisely, mutually consistent) outliers,
and 
(iii) if the measurements are generated by multiple rotations (\ie different subsets of measurements are generated by different rotations), then \SLIDES simultaneously recovers \emph{all} the rotations generating the data.

\myParagraph{\SLIDESlong (\SLIDES)} 
The proposed algorithm entails three main changes to~\cref{algo:ldr}, for $\relaxOrder=2$. First of all, instead of solving a moment relaxation of order $\relaxOrder=2$, which is still expensive for large $\nrMeasurements$ (we test for $\nrMeasurements = 50$), we develop a sparse relaxation, whose details are given in~\cref{app:sparse-relaxations-LDR}. The relaxation essentially considers a sparse monomial basis, which neglects certain degree-2 monomials (\eg $\omega_i \cdot \omega_j$) that would make the dimension of the resulting SDP much larger; the monomial basis is designed to still give us access to the entries of the pseudo-moment matrix
used in~\cref{algo:ldr}.
The second modification is to round the list of solutions of~\cref{algo:ldr} to the domain $\Domain$.
The latter is a consequential change: we empirically noticed that the original approach in~\cref{algo:ldr} (but with our sparse relaxation) produces estimates with norm close to zero. As we will see, projecting the estimates to $\Domain$ has the effect of re-normalizing the result and correcting scaling problems. Finally, we always return $\nrMeasurements$ hypotheses, rather than sampling; 
this makes the results deterministic and independent from an arbitrary choice of number of hypotheses (which can no longer be guided by the guarantees in~\cref{thm:highOut-apriori}).
Below, we will refer to this modified algorithm as ``\SLIDES'' (\SLIDESlong).

\begin{figure}[h!]
\begin{minipage}[b]{0.5\linewidth} 
\centering
\includegraphics[width=1\columnwidth, trim= 0mm 0mm 0mm 0mm, clip]{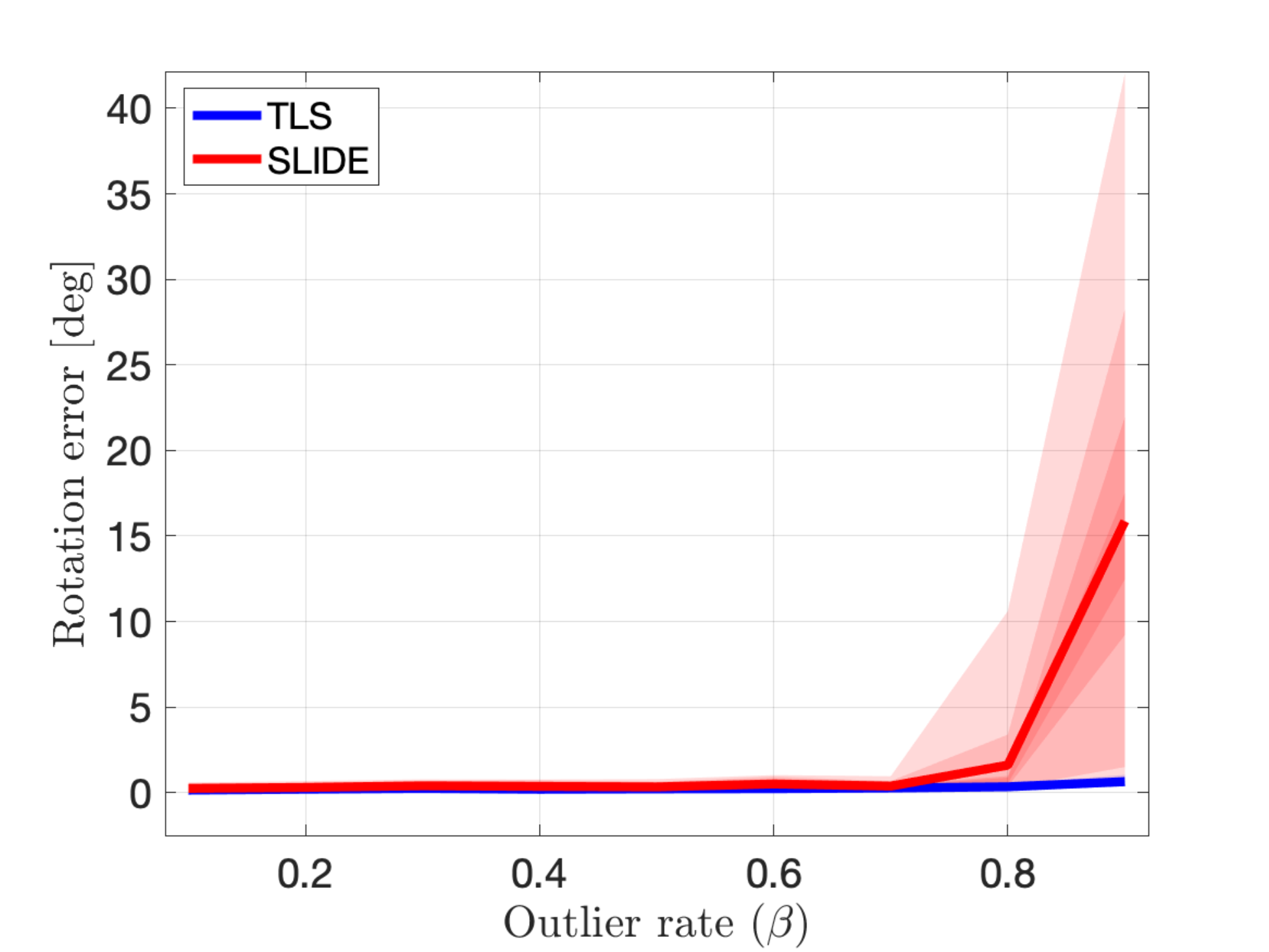}
\vspace{-5mm}\\
(a)
\end{minipage}
\begin{minipage}[b]{0.5\linewidth} 
\centering
\includegraphics[width=1\columnwidth, trim= 0mm 0mm 0mm 0mm, clip]{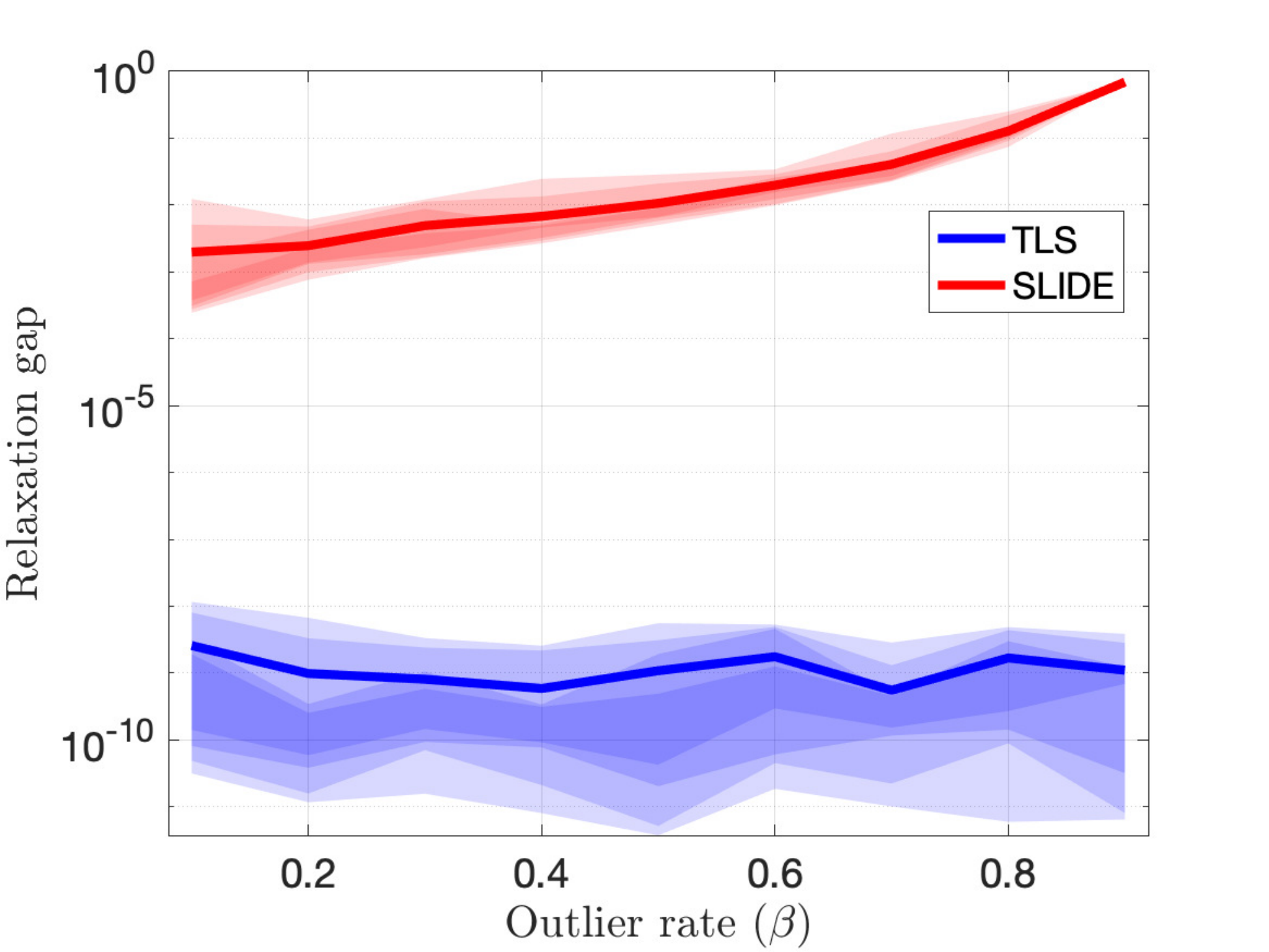}
\vspace{-5mm}\\
(b)
\end{minipage}\vspace{-7mm}
\caption{Rotation search with random outliers: 
(a) Rotation error in degrees for the TLS approach from~\cite{Yang19iccv-quasar} and \SLIDES for increasing outlier rates $\outlierRate$. 
(b) Relaxation gap for the same approaches ($y$-axis is on log scale). 
Shaded areas correspond to the 25\mth, 50\mth, 75\mth, and 90\mth percentiles.
Results are averaged over 10 runs.\label{fig:stats_LDR_1object}}
\end{figure}

\myParagraph{Experiments on synthetic data}
\Cref{fig:stats_LDR_1object} reports statistics comparing the estimation errors of the sparse relaxation of the~\eqref{eq:TLS} problem proposed in~\cite{Yang19iccv-quasar} (label: ``TLS'') against \SLIDES, using the data generation protocol of~\cref{sec:experiments-setup} with $\nrMeasurements = 50$ and increasing outlier rates $\outlierRate$ between $0.1$ and $0.9$ at increments of $0.1$.
\Cref{fig:stats_LDR_1object}(a) reports the rotation error in degrees between the estimate from each approach and the ground truth. For \SLIDES, we report the smallest error across the estimates in the list. 
 The TLS results confirm the observation from~\cite{Yang19iccv-quasar} that TLS is extremely robust to extreme rates of non-adversarial outliers (as described in~\cref{sec:experiments-setup}, in this case we sample the outliers at random). However, the interesting observation is that \SLIDES is also able to retrieve good estimates in practice as long as the amount of inliers is sufficiently large; 
 note that for $\outlierRate = 0.9$ we only have $5$ inliers, which is comparable to the size of the minimum set of measurements we need to uniquely solve the problem ($\minDim = 3$ in rotation search).
\Cref{fig:stats_LDR_1object}(b) reports the relative relaxation gap for both approaches, computed as in eq.~(24) in~\cite{Yang22pami-certifiablePerception}:
\beq
\gamma = \frac{| f_{\text{sdp}} - \hat{f} |}{1 + |f_{\text{sdp}} | + |\hat{f}|}
\eeq 
where $f_{\text{sdp}}$ is the optimal objective of the SDP relaxation and $\hat{f}$ is the objective attained by a rounded solution.\footnote{For \SLIDES, the relaxation gap can be easily computed by observing that the optimal objective of the non-relaxed polynomial optimization problem~\eqref{eq:LDR} is always $\inlierRate \cdot \nrMeasurements$.}
As already reported in~\cite{Yang19iccv-quasar}, the TLS relaxation remains tight across the spectrum. 
On the other hand, \SLIDES's relaxation is typically loose, while still returning accurate solutions.

\begin{figure}[h!]
\begin{minipage}[b]{0.5\linewidth} 
\centering
\includegraphics[width=1\columnwidth, trim= 0mm 0mm 0mm 0mm, clip]{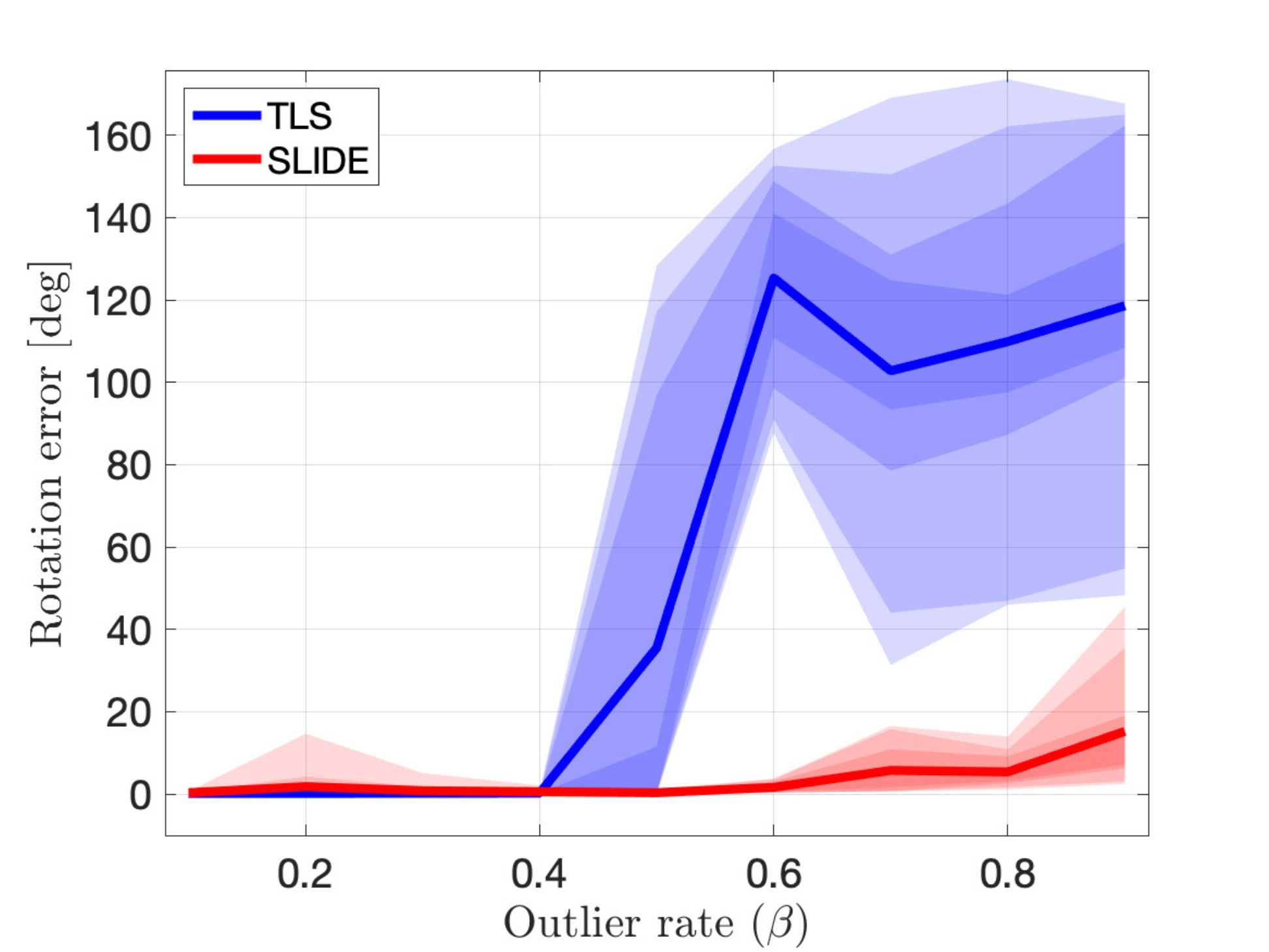}
\vspace{-5mm}\\
(a)
\end{minipage}
\begin{minipage}[b]{0.5\linewidth} 
\centering
\includegraphics[width=1\columnwidth, trim= 0mm 0mm 0mm 0mm, clip]{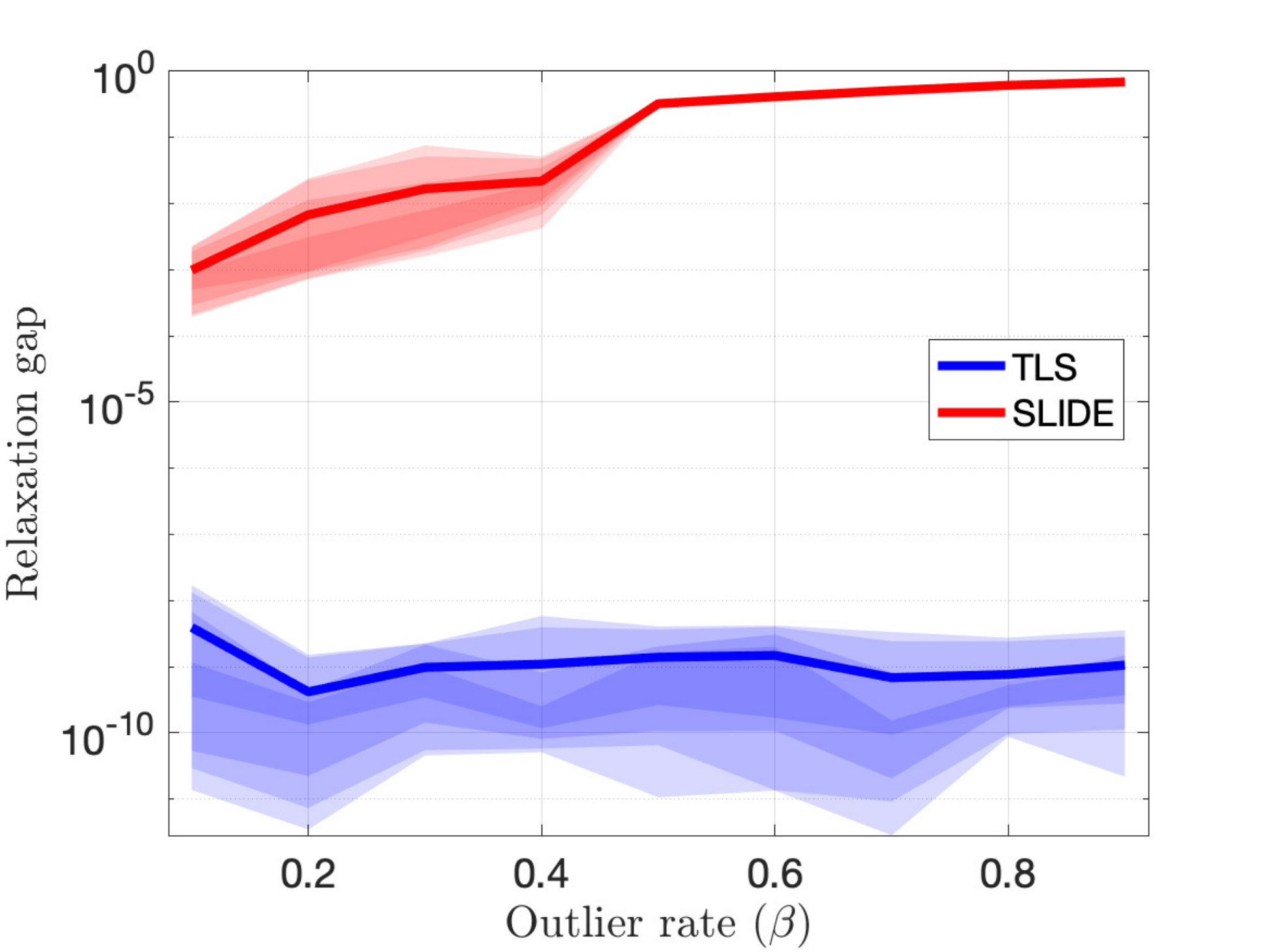}
\vspace{-5mm}\\
(b)
\end{minipage}
\vspace{-7mm}
\caption{Rotation search with mutually consistent outliers: 
(a) Rotation error in degrees for the TLS approach from~\cite{Yang19iccv-quasar} and \SLIDES for increasing outlier rates $\outlierRate$. 
(b) Relaxation gap for the same approaches ($y$-axis is on log scale). 
Shaded areas correspond to the 25\mth, 50\mth, 75\mth, and 90\mth percentiles.
Results are averaged over 10 runs.  \label{fig:stats_LDR_2objects}}
\end{figure}

\Cref{fig:stats_LDR_2objects} reports the same statistics, but now for a case with mutually consistent outliers. In this case, we generate the set of outliers according to the same generative model of the rotation search model~\eqref{eq:wahba}, but using a rotation $\MRout \neq \MR\gt$ that we choose at random. Essentially, the set of measurements now contains two different hypotheses for the parameter we want to estimate, and if the outliers are more than $50\%$ any single-hypothesis estimator is expected to fail (and return an estimate close to $\MRout$ instead of $\MR\gt$).
\Cref{fig:stats_LDR_2objects}(a) reports the rotation error for both TLS and \SLIDES.
 As expected, TLS experiences a phase transition, where for $\outlierRate < 0.5$ it is able to retrieve a good estimate for $\MR\gt$, while for $\outlierRate > 0.5$, the outliers form the most likely hypothesis, leading the estimator to perform arbitrarily poorly.
 However, one can observe how \SLIDES still performs similarly to the non-adversarial case, achieving much smaller errors (the mean error remains below $4^\circ$ for $\outlierRate = 0.8$).
\Cref{fig:stats_LDR_1object}(b) reports the relaxation gap for both approaches. 
\SLIDES's relaxation remains loose, while, quite interestingly, TLS remains tight even in the presence of 
mutually consistent outliers.

\begin{figure}[h!]
\begin{minipage}[b]{0.5\linewidth} 
\centering
\includegraphics[width=1\columnwidth, trim= 0mm 0mm 0mm 0mm, clip]{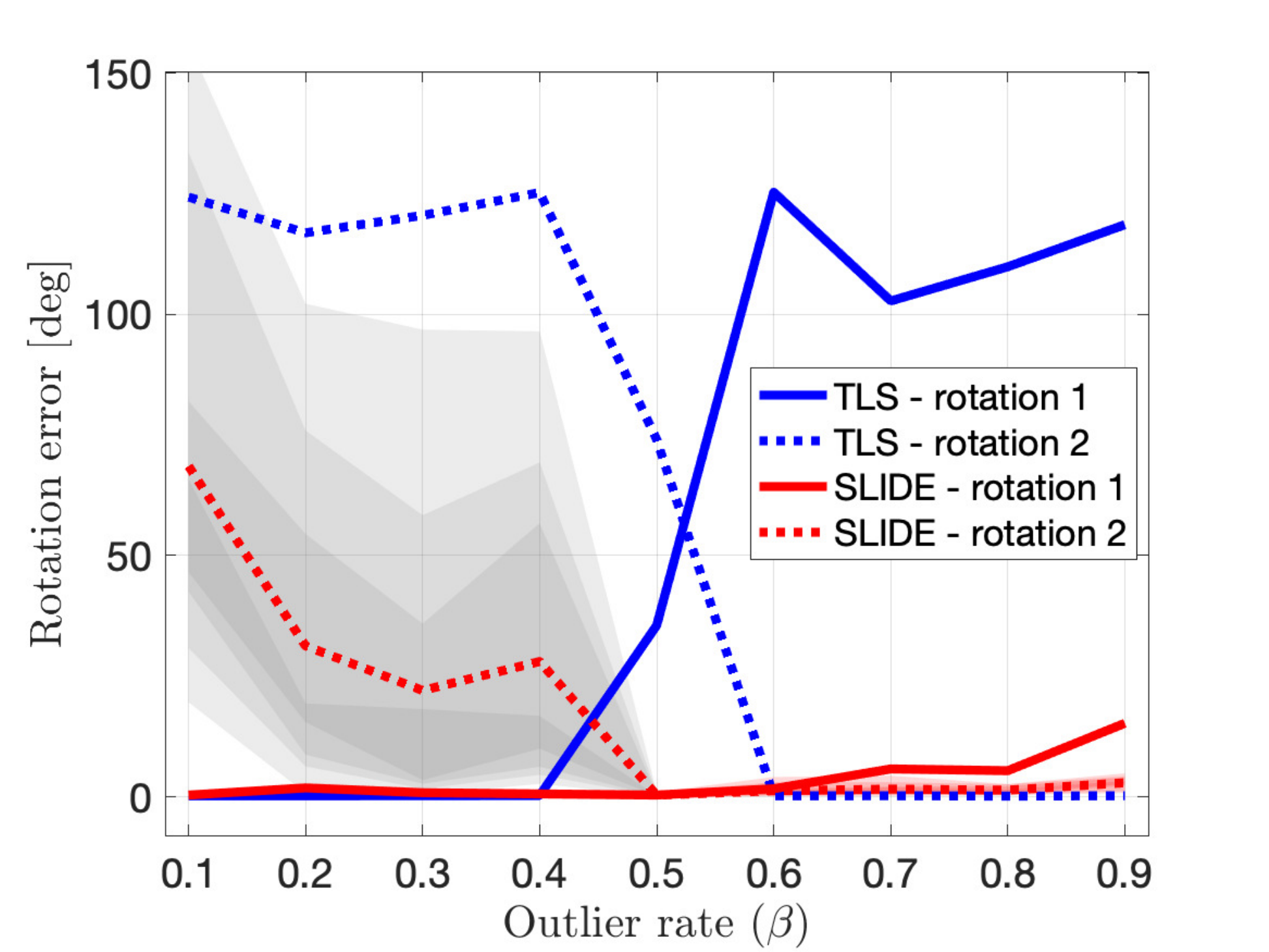}\\
(a)
\end{minipage}
\begin{minipage}[b]{0.5\linewidth} 
\centering
\includegraphics[width=1\columnwidth, trim= 0mm 0mm 0mm 0mm, clip]{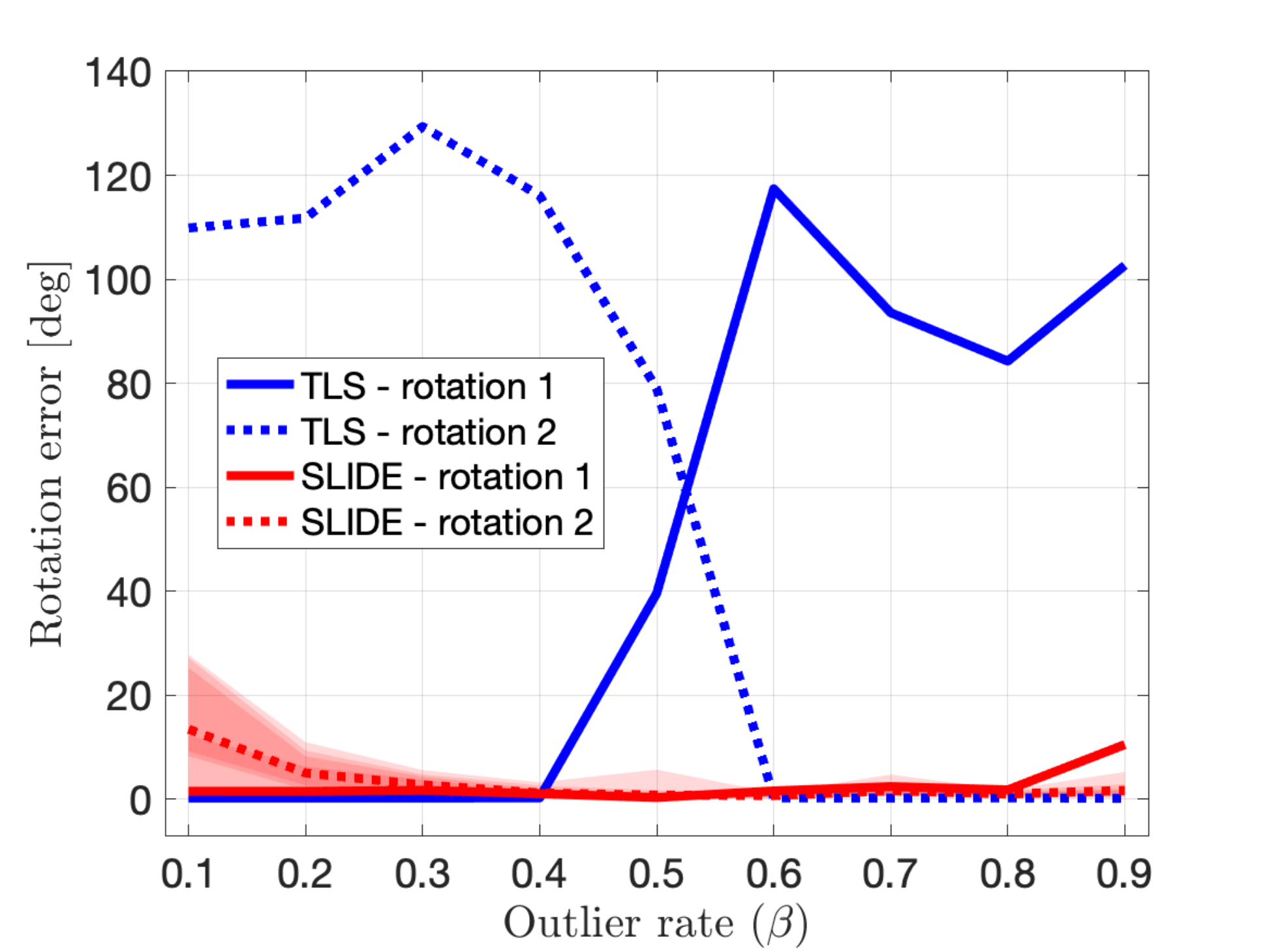}
\\
(b)
\end{minipage}
\vspace{-3mm}
\caption{Rotation search with mutually consistent outliers: 
rotation errors in degrees for the TLS approach from~\cite{Yang19iccv-quasar} and \SLIDES for increasing outlier rates $\outlierRate$. Errors are reported for both rotation hypotheses present in the data. Shaded areas are only shown for \SLIDES and for the second rotation (the shaded area for the first rotation can be observed in~\cref{fig:stats_LDR_1object}(a)) and correspond to the 25\mth, 50\mth, 75\mth, and 90\mth percentiles.
Results are averaged over 10 runs.  \label{fig:stats_LDR_2objects_bothErrors}}
\end{figure}

For the same set of experiments with mutually consistent outliers, \cref{fig:stats_LDR_2objects_bothErrors}(a) provides error statistics for both rotations ($\MRout$, $\MR\gt$) present in the sensor data. As expected, TLS, depending on the outlier rate, retrieves either one or the other rotation (\ie it only retrieves the dominant hypothesis in the data). 
Indeed,~\eqref{eq:TLS} ---similarly to~\eqref{eq:MC}--- looks for the estimate supported by the ``largest'' set of measurements, hence it is not able to find more complex patterns in the measurements. 
On the other hand, \SLIDES's behavior is much more intriguing. 
Let us start by considering the high-outlier regime ($\outlierRate>0.5$), which is easier to parse.
It is clear that when $\outlierRate>0.5$, \SLIDES is able to report both rotations producing the measurements, as long as the number of inliers for each is sufficiently large (similarly to what we observed in~\cref{fig:stats_LDR_1object,fig:stats_LDR_2objects}). 
 This is quite relevant for practical applications: for instance, in rotation search, the presence of different rotations generating the data may be due to the 3D vectors experiencing different motions over time (\eg a subset belongs to a static portion of the environment, while another subset belongs to a moving object), hence being able to retrieve both might be useful to simultaneously infer motion of multiple objects with respect to the sensor. 
In the low-outlier case $\outlierRate<0.5$, it might seem that \SLIDES is not able to recover the second rotation; however, this is only a byproduct of our experimental setup: 
for each $\outlierRate$, the polynomial optimization~\eqref{eq:LDR} that \SLIDES solves only searchers 
 for estimates ``supported'' by at least $\inlierRate = 1-\outlierRate$ fraction of the measurements.
 Therefore, for $\outlierRate=0.1$, we are asking \SLIDES to report any estimate that is supported by at least $90\%$ of the measurements, hence \SLIDES correctly reports a single estimate ($\MR\gt$), since the other estimate ($\MRout$) is only supported by $10\%$ of the measurements (\ie the outliers).
 If we instead pass the minimum $\min(\outlierRate,1-\outlierRate)$ between the inlier and the outlier rate as the fraction of the measurement set \SLIDES has to look for, we obtain the behavior in~\cref{fig:stats_LDR_2objects_bothErrors}(b), where \SLIDES is able to recover both rotations as long as they are supported by enough measurements.

\Cref{fig:stats_LDR_2objects_bothErrors_qualitative} provides qualitative evidence to further support the observation that
 \SLIDES is empirically able to report all the estimates supported by at least $\an$ inliers, as long as $\an$ is sufficiently large. \Cref{fig:stats_LDR_2objects_bothErrors_qualitative}(a) shows an example with mutually consistent outliers and $\outlierRate=0.5$: the figure shows the rotations $\MR\gt$ and $\MRout$ generating the data as 
coordinate frames with thick and short axes, and the estimates in the list returned by \SLIDES as coordinate frames with thin and long axes. The figure shows that all the estimates returned by \SLIDES  cluster around the  ground-truth rotations generating the data. 
\Cref{fig:stats_LDR_2objects_bothErrors_qualitative}(b) shows a more extreme case where the data is generated by 5 different rotations, each one producing $20\%$ of the measurements. The figure shows the 5 ground-truth rotations with thick and short axes, and the estimates from \SLIDES as thin and long axes. Also in this case, \SLIDES's estimates cluster around the ground-truth rotations, and indeed \SLIDES is able to recover all rotations within a maximum error of $2.6^\circ$. 
These results also seem to suggest other algorithmic variants of~\cref{algo:ldr}, where one, rather than sampling a smaller list, would get a list of size $\nrMeasurements$ and then reduce the size of the list by clustering it into fewer hypotheses. We discuss this and other extensions in the following section. 

As we mentioned in~\cref{sec:experiments-setup}, our implementation of \SLIDES is in Matlab and uses~\mosek~\cite{mosek} as an SDP solver. The average runtime of \SLIDES in problems with $\nrMeasurements = 50$ is~3 minutes.

\begin{figure}[h!]
\hspace{-0.5cm}
\begin{minipage}[b]{0.5\linewidth} 
\centering
\includegraphics[width=1.2\columnwidth, trim= 0mm 0mm 0mm 0mm, clip]{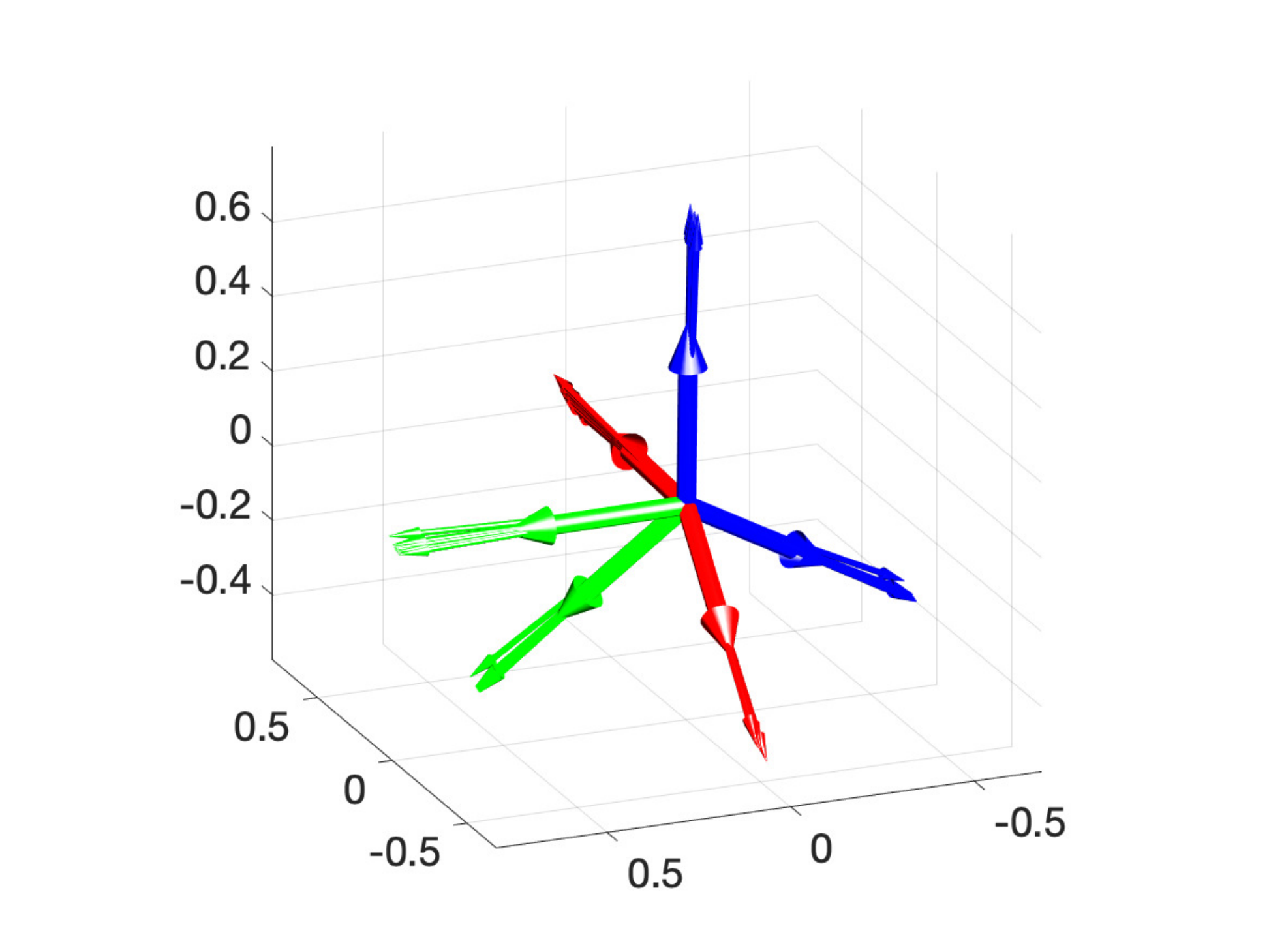}
\vspace{-8mm}\\
(a)
\end{minipage}\vspace{-5mm}
\begin{minipage}[b]{0.5\linewidth} 
\centering
\includegraphics[width=1.2\columnwidth, trim= 0mm 0mm 0mm 0mm, clip]{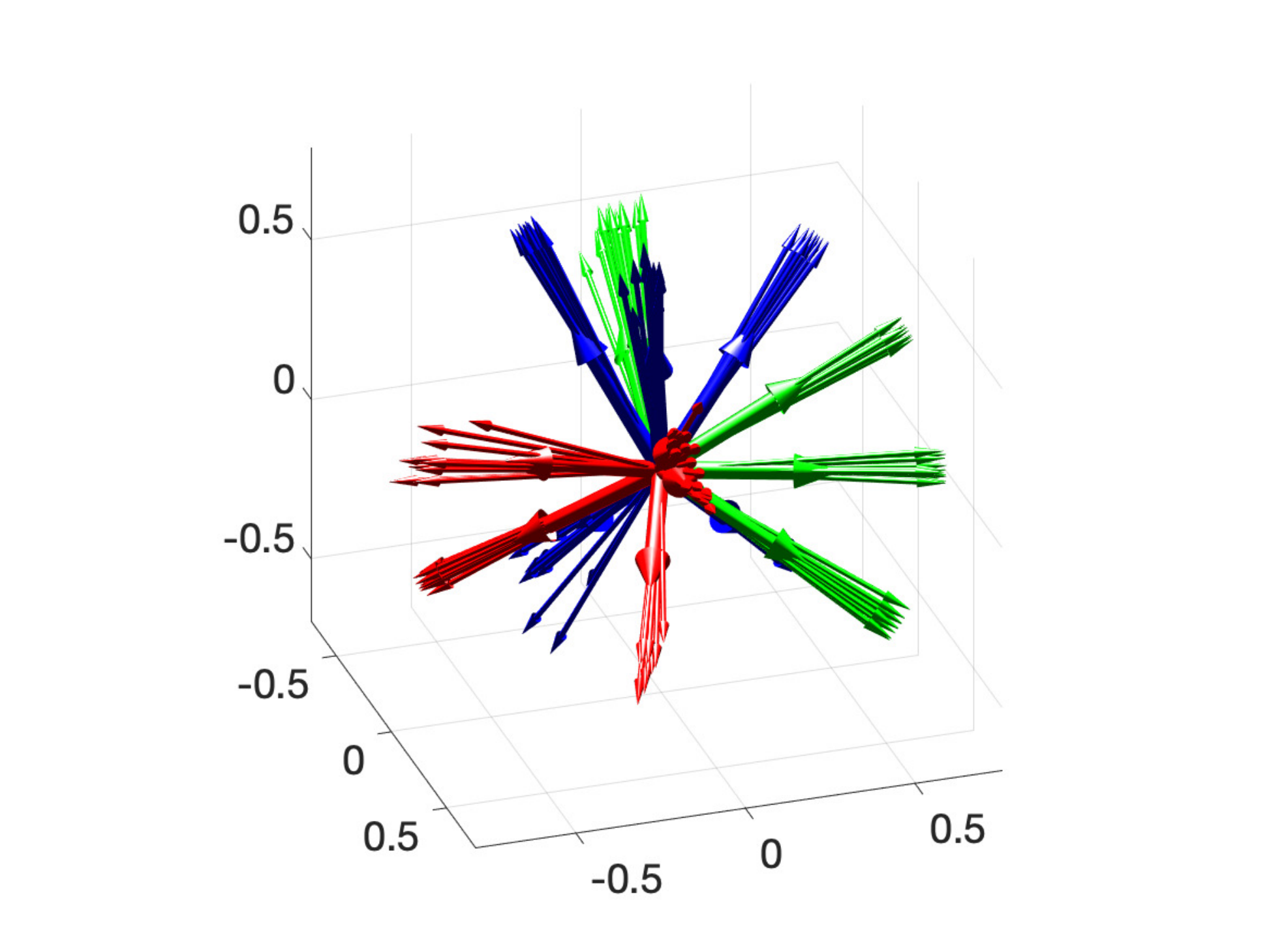}
\vspace{-8mm}\\
(b)
\end{minipage}\vspace{2mm}
\caption{Ground-truth rotations generating the data vs. estimates from \SLIDES.
The ground-truth rotations are visualized as coordinate frames with thick and short x (red), y (green), and z (blue) orthogonal axes, and the estimates in the list returned by \SLIDES as coordinate frames with thin and long axes. 
(a) Rotation search with 2 rotations, each generating $50\%$ of the measurements.
(b) Rotation search with 5 rotations, each generating $20\%$ of the measurements.   
 \label{fig:stats_LDR_2objects_bothErrors_qualitative}}
\end{figure}


\section{Extensions and Open Problems}
\label{sec:openProblems}

\myParagraph{Extensions to other geometric perception problems} 
In this \paper we focused on problems where inliers 
can be expressed via a linear model with additive noise: $\vy_i = \MA_i\tran \vxx\gt + \vepsilon$.
However, in many geometric perception problems, the measurements do not belong to a vector space and 
the noise is no longer additive. In the following, we review three examples of such problems, and 
 observe that ---in the outlier-free case--- they can still be cast as:
%
\begin{align}
\label{eq:LS-domain}
\textstyle \min_{\vxx \in \Domain} \; \normTwo{\vy_i - \MA_i\tran \vxx}^2, 
\end{align}
hence still being amenable for estimators such as~\eqref{eq:LTS},~\eqref{eq:MC}, and~\eqref{eq:TLS} 
with $f_i(\vxx) = \MA_i\tran \vxx$.

\begin{example}[Single Rotation Averaging]
\label{ex:singleRotationAveraging}
Estimate a rotation $\MR \in \SOthree$ given rotation measurements $\MR_i$, $i=1,\ldots,\nrMeasurements$. 
The measurement model for the (inlier) measurements is given by: 
\begin{align}
\textstyle \MR_i = \MR \cdot \MR_\epsilon, \qquad i \in [\nrMeasurements],
\end{align}
where $\MR_\epsilon$ is a random rotation describing the measurement noise.
When the rotation noise $\MR_\epsilon$ follows an isotropic Langevin distribution with mode
$\eye_3$ and concentration parameter $\kappa$ (inconsequential in this example), 
a maximum likelihood estimator for the outlier-free single rotation averaging problem 
is given by the following optimization problem (where $\normFrob{\cdot}$ denotes the Frobenius norm):
\begin{align}
 \min_{\MR \in \SOthree} \sumOverMeas \kappa \normFrob{\MR - \MR_i}^2,
\end{align}
which has the same form of~\eqref{eq:LS-domain} after vectorizing the matrix into a $9$-vector.
Rotation averaging finds application in camera
 calibration, motion capture, spacecraft
attitude determination, and crystallography, see~\cite{Lee20arxiv-robustRotationAveraging,Hartley13ijcv} and the references therein.
\end{example}

\begin{example}[Multiple Rotation Averaging] 
\label{ex:multirotation}
Estimate a set of rotations $\MR_k \in \SOthree$, $k=1,\ldots,N$, from relative rotation measurements $\bar{\MR}_{ij}$ between (a sufficiently large set of) pairs of rotations. The generative model for the (inlier) measurements is: 
\begin{align}
\label{eq:multipleRotationAveraging}
\textstyle \bar{\MR}_{ij} = \MR_i\tran \MR_j \MR_\epsilon, \qquad (i,j) \in \calE,
\quad 
\end{align}
where $\calE$ is the set of pairs $(i,j)$ such that a measurement $\bar{\MR}_{ij}$ is available.
The problem can be visualized as a graph, where each node is associated a to-be-estimated rotation, while edges correspond to pairwise rotation measurements.
When the rotation noise $\MR_\epsilon$ follows an isotropic Langevin distribution with mode
$\eye_3$ and concentration parameter $\kappa$ (inconsequential also in this example), a maximum likelihood estimator for the outlier-free multiple rotation averaging problem 
is given by the following optimization problem~\cite{Dellaert20eccv-shonan}:
\begin{align}
 \min_{\MR_k \in \SOthree, k=1,\ldots,N}  \sum_{(i,j)\in\calE} \kappa \normFrob{\MR_j - \MR_i \bar{\MR}_{ij}}^2,
\end{align}
which has the same form of~\eqref{eq:LS-domain} after vectorization.
Multiple rotation averaging arises in Structure from Motion and camera calibration~\cite{Hartley13ijcv} among other problems, and can be used to compute an initial guess for Simultaneous Localization and Mapping methods~\cite{Carlone15icra-initPGO3D}. It has also been studied in conjunctions with Shor's relaxation in
\cite{Eriksson18cvpr-strongDuality,Dellaert20eccv-shonan}.
\end{example}

\begin{example}[Pose Graph Optimization] 
Estimate a set of poses $(\vt_k,\MR_k)$ with $\vt_k \in \Real{3}$ and $\MR_k \in \SOthree$, $k=1,\ldots,N$, from relative pose measurements $(\bar{\vt}_{ij}, \bar{\MR}_{ij})$ between (a sufficiently large set of) pairs of poses. The generative model for the (inlier) measurements is: 
\begin{align}
\textstyle \bar{\MR}_{ij} = \MR_i\tran \MR_j \MR_\epsilon,  
\qquad \bar{\vt}_{ij} = \MR_i\tran (\vt_j - \vt_i) + \vt_\epsilon,  \qquad (i,j) \in \calE,
\quad 
\end{align}
where $\calE$ is the set of pairs $(i,j)$ such that a measurement $(\bar{\vt}_{ij}, \bar{\MR}_{ij})$ is available.
The problem can again be visualized as a graph, the \emph{pose graph}, where each node is associated a to-be-estimated pose, while edges correspond to pairwise pose measurements.
When the rotation noise $\MR_\epsilon$ follows an isotropic Langevin distribution with mode
$\eye_3$ and concentration parameter $\kappa$, and the translation error $\vt_\epsilon$ is a zero-mean Gaussian with 
covariance $\frac{1}{\tau}\eye_3$, a maximum likelihood estimator for the outlier-free pose graph optimization problem is given by the following optimization problem~\cite{Carlone15icra-verification,Carlone15iros-duality3D,Rosen18ijrr-sesync}:
\begin{align}
\label{eq:pgo}
 \min_{\vt_k \in \Real{3}, \MR_k \in \SOthree, k=1,\ldots,N} 
 \sum_{(i,j)\in\calE} \tau \normTwo{ \vt_j - \vt_i - \MR_i \bar{\vt}_{ij} }^2
 +
 \sum_{(i,j)\in\calE} \kappa \normFrob{\MR_j - \MR_i \bar{\MR}_{ij}}^2,
\end{align}
which, observing the quadratic nature of the cost in~\eqref{eq:pgo}, can be recast as in eq.~\eqref{eq:LS-domain}.
Pose graph optimization finds application in Simultaneous Localization and Mapping among other fields~\cite{Cadena16tro-SLAMsurvey} and has been investigated in conjunction with Shor's relaxation 
in~\cite{Carlone15icra-verification,Carlone15iros-duality3D,Rosen18ijrr-sesync,Briales17ral,Carlone16tro-duality2D,Fan19iros-2dpgo}.  
\end{example}

All the examples above still reduce to linear regression over a basic semi-algebraic set, hence 
we believe it is possible to extend the results presented in this \paper to these problems. 
Indeed, moment relaxations of a~\eqref{eq:TLS} formulation of
outlier-robust pose graph optimization have been proposed in~\cite{Lajoie19ral-DCGM}. 
At the same time, multiple rotation averaging and pose graph optimization 
pose further challenges, due to the very high-dimensional nature of the problem (for which 
even a sparse moment relaxation of order $2$ is out of reach for current SDP solvers and for realistic problem sizes~\cite{Lajoie19ral-DCGM}), 
and might benefit from different assumptions on the measurements that better
leverage the graph-theoretic nature of the problem (\eg problems~\eqref{eq:multipleRotationAveraging}  and~\eqref{eq:pgo} admit a unique solution 
only if the underlying graph is connected).

\myParagraph{Other extensions and open problems}
This \paper extends results from outlier-robust statistics to problems with vector-valued measurements, unknown outlier rates, and variables belonging to a basic semi-algebraic set~$\Domain$. 
However, all the algorithms presented in this \paper may return estimates outside $\Domain$.
While in principle, for the sets arising in geometric perception,
it is typically easy to project the outputs of these algorithms onto $\Domain$,
the corresponding estimation contracts currently do not account for such a rounding.
Leveraging Lemma 3 in~\cite{Doherty22arxiv-spectralInitialization}, we could account for the rounding by adding a constant factor of 2 in front of the error bounds in~\cref{thm:lowOut-apriori-LTS,thm:lowOut-apriori-MC,thm:lowOut-apriori-TLS,thm:highOut-apriori}.\footnote{We thank Kevin Doherty for pointing out the result in~\cite{Doherty22arxiv-spectralInitialization}.} 
However, it would be desirable to take advantage of the domain constraint to \emph{tighten} the error bounds in the proofs, rather than treating it as an afterthought (in this \paper, we mostly make sure that the domain constraint does not break the proofs).

A second limitation is that the estimation contracts presented in this \paper 
 require the corresponding algorithms to solve high-order moment relaxations.
 For instance, we have already observed that~\cref{thm:lowOut-apriori-LTS-objective} requires 
 a relaxation of order $\relaxOrder \geq \relaxLevel \geq 4$ (under $k$-certifiable hypercontractivity), 
 which would be impractical to solve with current SDP solvers; 
 in this case the bound $\relaxOrder \geq \relaxLevel \geq 4$ on the relaxation order is imposed by the degree of the objective in~\cref{algo:lts1} (which imposes $\relaxOrder \geq \relaxLevel$), and by the proof (which requires $\relaxLevel \geq 4$). 
 This requirement is further reinforced by the fact that typical distributions (\eg Gaussians) have been shown to be $4$-certifiably hypercontractive~\cite{Klivans18arxiv-robustRegression}, hence again requiring $\relaxLevel \geq 4$ for the performance guarantees to hold. 
The same issue arises for the other estimation contracts, due to the requirement that the relaxation order must satisfy $\relaxOrder \geq \relaxLevel/2$ for $\relaxLevel$-certifiably anti-concentrated inliers and that 
 $\relaxLevel \geq 6$, due to the degree of the polynomials in~\eqref{eq:certAntiConSet2}.
Therefore, it would be desirable to 
develop guarantees for more practical (and possibly sparse) order-2 relaxations,
with the goal of fully
explaining the empirical performance observed in~\cite{Yang22pami-certifiablePerception}.
Third, it would be interesting to explore the performance of list-decodable estimation in other perception problems.
In particular, it would be useful to see if the impressive performance observed in rotation search problems (\cref{sec:experiments-ldr}) can be also replicated  in other geometric perception problems.
Along these  lines, it would be desirable to develop more sophisticated  algorithmic variants of \SLIDES based on the framework in~\cite{Karmalkar19neurips-ListDecodableRegression}. For instance, rather than sampling potential estimates from the pseudo-moment matrix, one could collect all estimates (as done in \SLIDES), cluster them into a smaller number of hypotheses, and return the cluster centers, hoping that the cluster-wise averaging can further reduce the estimation errors.
Related to the discussion above, it would be desirable to extend the performance guarantees in~\cref{thm:highOut-apriori} to encompass order-2 relaxations and capture the fact that projecting the estimates to the domain $\Domain$ leads to much better performance in practice.

Fourth, it would be interesting to characterize the behavior of adversarial outliers in geometric perception.
How does an adversarial attack on a geometric perception problem look like? 
While in the experiments we observed that ``mutually consistent'' outliers are more challenging to reject, 
a strong adversary model would also leverage the set of inliers to create even more difficult problem instances.
Along these lines, designing algorithms to produce adversarial outliers in perception problems might provide further insights on outlier-robust estimation algorithms, somewhat drawing inspiration from the large and parallel body of research on adversarial machine learning, see, \eg~\cite{Goodfellow14arxiv-adversarialExample}.\footnote{An empirical study on adversarial attacks on SLAM has been recently proposed in~\cite{Ikram22ral-adversarialSLAM}.} 

Finally, a broader issue is that many of the estimation problems considered in this \paper 
must be solved on a stringent runtime budget. For instance, point cloud registration problems (\cref{ex:pointCloudRegistration}) are solved at frame-rate (\eg $>20$ Hz) in many RGB-D SLAM applications~\cite{Cadena16tro-SLAMsurvey}, hence 
requiring the estimation algorithm to run in a fraction of a second.
 SDP solvers applied to the moment relaxations considered in this \paper (even at order $2$) are far from meeting this runtime constraints. 
 Therefore, it would be interesting to develop specialized solvers that take advantage of the problem structure; 
 for instance, the work~\cite{Yang21arxiv-stride} leverages the fact that the SDP is a relaxation of a polynomial optimization problem to speed up computation, while~\cite{Rosen18ijrr-sesync} achieves real-time performance by solving large SDPs using the Riemannian staircase method~\cite{Boumal16arxiv}. 
 Finally, it is important to further extend the reach of sparse versions of Lasserre's hierarchy of
  moment relaxations, which can reduce the size of the matrices in the relaxation by leveraging the problem structure, see~\cite{Weisser18mpc-SBSOS,Wang21siopt-chordaltssos,Wang20arXiv-cs-tssos,Wang21SIOPT-tssos}.

\section{Conclusions}
\label{sec:conclusions}

We studied outlier-robust estimation in the context of the geometric perception problems arising in robotics and computer vision. Many of these problems can be reformulated as linear estimation problems with variables belonging to a basic semi-algebraic set, and the goal is to retrieve a good estimate of the variables in the presence of outliers. 
We provided a unified view of converging work on outlier-robust estimation across robust statistics, robotics, and computer vision and discussed technical tools underlying modern estimation approaches, including moment relaxations and sum-of-squares proofs.
Then, we reviewed existing algorithms and presented \emph{estimation contracts}, which establish conditions on the input measurements under which modern estimation algorithms  are guaranteed to recover an estimate close to the ground truth in the presence of outliers.  
Towards this goal, we adapted and extended recent results on outlier-robust \emph{linear regression}  (applicable to the low-outlier case with $\ll 50\%$ outliers) and 
\emph{list-decodable regression} (applicable to the high-outlier case with $\gg 50\%$ outliers) 
to the setup commonly found in robotics and vision, where 
(i)~variables (\eg rotations, poses) belong to non-convex sets, 
(ii)~measurements are vector-valued, and 
(iii)~the number of outliers is not known a priori.
Besides the technical results,
we hope this \paper can provide a unifying view of parallel research lines on outlier-robust estimation across fields.
Moreover, we hope that practitioners will benefit from our layman introduction to moment relaxations and sum-of-squares proofs and will use these tools to attack other outstanding problems in robotics and vision.
Finally, we hope that researchers in robust statistics will be intrigued by the formulations and 
the empirical performance observed in robotics and vision problems, and will contribute to 
bridging the current gap between theoretical results and practical algorithms.


\section*{Acknowledgments}

We thank Pablo Parrilo for pointing out relevant work in robust statistics and for 
suggesting potential connections that led to the development of this \paper.
We also thank Heng Yang for useful discussion about the relation between 
Putinar's and Schmudgen's \positiv and for documenting and sharing relevant code about rotation search and image stitching. 
Finally, we thank Sushrut Karmalkar for the useful discussion on 
certifiable anti-concentration and list-decodable regression.

\phantomsection
\addcontentsline{toc}{section}{References}
\bibliographystyle{IEEEtran}
\bibliography{../../references/myRefs.bib,../../references/refs.bib}

\renewcommand{\theequation}{A\arabic{equation}}
\renewcommand{\thetheorem}{A\arabic{theorem}}
\renewcommand{\thefigure}{A\arabic{figure}}
\renewcommand{\thetable}{A\arabic{table}}
\renewcommand{\thealgocf}{A\arabic{algocf}}

\appendix

\setcounter{equation}{0}
\setcounter{section}{0}
\setcounter{theorem}{0}
\setcounter{figure}{0}
\setcounter{algocf}{0}

\clearpage

\section{An Algorithmic View of Lasserre's Hierarchy of Moment Relaxations}
\label{app:momentRelaxation}

Here we provide an algorithmic (and somewhat unorthodox) view of Lasserre's hierarchy of moment relaxations~\cite{Lasserre01siopt-LasserreHierarchy}; we refer the reader to~\cite{Lasserre18icm-momentRelaxation,Lasserre10book-momentsOpt} for a more standard introduction. 

Lasserre's hierarchy provides a systematic way to relax a polynomial optimization problem~\eqref{eq:pop} into 
a semidefinite (convex) program. 
We start by restating~\eqref{eq:pop}:
\begin{equation}
\pstar  \triangleq \min_{\vxx \in \Real{\dimx}} \cbrace{p(\vxx) \ \middle\vert\ \substack{ \displaystyle h_i(\vxx) = 0, i=1,\dots,l_h \\ \displaystyle g_j(\vxx) \geq 0, j = 1,\dots,l_g } }, \tag{POP}
\end{equation} 
where $p(\vxx), h_i(\vxx), g_j(\vxx)$ are polynomials in the variable $\vxx \in \Real{\dimx}$.

The key idea behind Lasserre's hierarchy of moment relaxations
is to (i) rewrite the polynomial optimization problem~\eqref{eq:pop} using the moment matrix $\MX_{2\relaxOrder}$\footnote{Recall that the moment matrix is defined as $\MX_{2\relaxOrder} \triangleq [\vxx]_\relaxOrder [\vxx]_{\relaxOrder}\tran$, where $[\vxx]_\relaxOrder$ is the vector of monomials of degree up to $\relaxOrder$. For instance, for $\vxx = [x_1 \vcat x_2]$ and $\relaxOrder=2$, the matrix $\MX_{2\relaxOrder}$ takes the form in eq.~\eqref{eq:momentMatrix}.}, 
(ii) relax the (non-convex) rank-1 constraint on $\MX_{2\relaxOrder}$, and
(iii) add redundant constraints that are trivially satisfied in~\eqref{eq:pop} but might still improve the quality of the relaxation; 
as shown below, this leads to a semidefinite program.

\myParagraph{(i) Rewriting~\eqref{eq:pop} using $\MX_{2\relaxOrder}$} 
Recall that any polynomial of degree up to $2\relaxOrder$ can be written as a linear function of the moment matrix $\MX_{2\relaxOrder}$ (\cf~\cref{sec:pre-pop}). 
Therefore, we pick a positive integer $\relaxOrder$ (the \emph{order} of the relaxation) such that $2\relaxOrder \geq \max \{\deg{p}\!,\deg{h_1}\!, \ldots, \deg{h_{l_h}}\!,
\deg{g_1}\!, \ldots, \deg{g_{l_g}}\!\}$, such that 
 we can express both objective function and constraints as a linear function of $\MX_{2\relaxOrder}$. 
With this choice of $\relaxOrder$, we can rewrite the objective and the equality constraints in~\eqref{eq:pop} as: 
\beal\label{eq:objective}
\!\!\!\!\!\!\text{\grayout{objective}}: & \inprod{\MC_1}{ \MX_{2\relaxOrder} } \\
\eeal
\beal \label{eq:eqConstraints1}
\!\!\!\text{\grayout{equality constraints}}: & \!\!\! \inprod{\MA_{\meq,j}}{ \MX_{2\relaxOrder} } = 0, \; j=1,\ldots,l_h, \\
\eeal
for suitable matrices $\MC_1$ and $\MA_{\meq,j}$.

 \emph{(ii) Relaxing the (non-convex) rank-$1$ constraint on $\MX_{2\relaxOrder}$}. 
 At the previous point we noticed we can rewrite objective and constraints in~\eqref{eq:pop} as linear (hence convex) 
 functions of $\MX_{2\relaxOrder}$. However, $\MX_{2\relaxOrder}$ still belongs to the set of positive-semidefinite rank-1 matrices (since it is defined as $[\vxx]_\relaxOrder [\vxx]_{\relaxOrder}\tran$, where $[\vxx]_\relaxOrder$ is a vector of monomials), which is a non-convex set
 due to the rank constraint. Therefore, we simply relax the rank constraint and only enforce:
\beal \label{eq:eqMomentIsPSD}
\!\!\!\text{\grayout{pseudo-moment matrix}}: & \MX_{2\relaxOrder} \succeq 0. \\
\eeal

 \emph{(iii) Adding redundant constraints}. Since we have relaxed~\eqref{eq:pop} by re-parametrizing it using $\MX_{2\relaxOrder}$ 
 and dropping the rank constraint, the final step to obtain Lasserre's relaxation consists in adding extra constraints to make the relaxation tighter. 
 First of all, we observe that there are multiple repeated entries in the moment matrix (\eg in~\eqref{eq:momentMatrix}, the entry $x_1 x_2$ 
 appears 4 times in the matrix). Therefore, we can enforce these entries to be the same. In general, this leads to 
 $\mmom = \trinum(\binomial{d}{\relaxOrder}) - \binomial{d}{2\relaxOrder} + 1$ linear constraints,
where $\binomial{d}{2\relaxOrder} \triangleq \nchoosek{\dimx+2\relaxOrder}{2\relaxOrder}$ (the size of the monomial basis of degree up to $2\relaxOrder$, \ie $[\vxx]_{2\relaxOrder}$) and
$\trinum(n) \triangleq \frac{n(n+1)}{2}$ is the dimension of $\sym{n}$. These constraints are typically called \emph{moment constraints}:
\beal\label{eq:momentConstraints}
\text{\grayout{moment constraints}}: & \revise{\inprod{\MA_{\mathmom,0}}{ \MX_{2\relaxOrder} } = 1}, \\
& \inprod{\MA_{\mathmom,j}}{ \MX_{2\relaxOrder} } = 0, \\
&  j = 1, \ldots, \trinum(\binomial{d}{\relaxOrder}) - \binomial{d}{2\relaxOrder},
\eeal
\revise{where $\MA_{\mathmom,0}$ is all-zero except $[\MA_{\mathmom,0}]_{11} =1$, and it is used to define the constraint $[\MX_{2\relaxOrder}]_{11} = 1$, following from the 
definition of the moment matrix (see eq.~\eqref{eq:momentMatrix}).}

Second, we can also add \emph{redundant} equality constraints. Simply put, if $h_i = 0$, then also $h_i \cdot x_1 = 0$, $h_i \cdot x_2 = 0$, and so on, for any monomial we multiply by $h_i$. Since via $\MX_{2\relaxOrder}$ we can represent any polynomial of degree up to $2\relaxOrder$, we can write as linear constraints any polynomial equality in the form $h_i \cdot [\vxx]_{2\relaxOrder - \deg{h_i}} = \zero$ (the degree of the monomials is chosen such that the product does not exceed degree $2\relaxOrder$). These new equalities can again be written linearly as:
\beal\label{eq:redundantEqualityConstraints}  
\hspace{-3mm}\text{\grayout{(redundant) equality constraints}}: \inprod{\MA_{\mathreq,ij}}{ \MX_{2\relaxOrder} } = 0, \\
\quad\quad i = 1, \ldots, l_h, \ \ 
j = 1, \ldots, \binomial{d}{2\relaxOrder - \deg{h_i}},\!\!\!\!\!\!\!\!\!
\eeal
for suitable $\MA_{\mathreq,ij}$.
Since the first entry of $[\vxx]_{2\relaxOrder - \deg{h_i}}$ is always 1 (\ie the monomial of order zero),~eq.~\eqref{eq:redundantEqualityConstraints}  already includes the original 
equality constraints in~\eqref{eq:eqConstraints1}.

Finally, we observe that if $g_j \geq 0$, then for any positive semidefinite matrix $\MM$, it holds $g_j \cdot \MM \succeq 0$.
Since we can represent any polynomial of order up to $2\relaxOrder$ as a linear function of $\MX_{2\relaxOrder}$, 
we can add redundant constraints in the form $g_j \cdot \MX_{2(\relaxOrder - \ceil{\deg{g_j}/2})} \succeq 0$ 
(by construction $g_j \cdot \MX_{2(\relaxOrder - \ceil{\deg{g_j}/2})}$ only contains polynomials of degree up to $2\relaxOrder$).
To phrase the resulting 
relaxation in the standard form~\eqref{eq:primalSDP}, it is common to add extra matrix variables $\MX_{g_j} = g_j \cdot \MX_{2(\relaxOrder - \ceil{\deg{g_j}/2})}$ for $j=1,\ldots,l_g$ 
(the \emph{localizing matrices} \cite[\S 3.2.1]{Lasserre10book-momentsOpt})
and then force these matrices to be a linear function of  $\MX_{2\relaxOrder}$:
\beal\label{eq:locMatrices}
\text{\grayout{localizing matrices}}: & \MX_{g_j} \succeq 0, \;\; j=1,\ldots,l_g,
\eeal
\beal \label{eq:localizingConstraints}  
\hspace{-2mm} \text{\grayout{{localizing} constraints}}: \inprod{\MA_{\mathloc,jkh}}{ \MX_{2\relaxOrder} } = [\MX_{g_j}]_{hk}, \\
\quad \quad j = 1, \ldots, l_g,\ \  
1 \leq h\leq k \leq \binomial{d}{\relaxOrder - \ceil{\deg{g_j}/2}},
\eeal
where the linear constraints (for some matrix $\MA_{\mathloc,jkh}$) enforce each entry of $\MX_{g_j}$ to be a linear combination of entries of the matrix $\MX_{2\relaxOrder}$.

Following steps (i)-(iii) above, it is straightforward to obtain the following semidefinite program:
\begin{equation}\label{eq:lasserre}
\hspace{-4mm} \fstar_{2\relaxOrder} =\!\! \displaystyle\min_{\MX = (\MX_{2\relaxOrder} , \{\MX_{g_j}\}_{j \in [l_g]} )} \cbrace{\inprod{\MC_1}{\MX_{2\relaxOrder}} \mid \calA(\MX)\!=\!\vb,\MX\! \succeq\! 0}\!,\!\!\! 
\end{equation}
where the variable $\MX = (\MX_{2\relaxOrder} , \{\MX_{g_j}\}_{j \in [l_g]} )$ is a collection of positive-semidefinite matrices (\cf~\eqref{eq:eqMomentIsPSD} and~\eqref{eq:locMatrices}), 
the objective is the one given in~\eqref{eq:objective}, and the linear constraints $\calA(\MX)=\vb$
collect all the constraints in~\eqref{eq:momentConstraints},~\eqref{eq:redundantEqualityConstraints}, and~\eqref{eq:localizingConstraints}.
Problem \eqref{eq:lasserre} can be readily formulated as a multi-block SDP in the primal form~\eqref{eq:primalSDP}, which matches the data format used by common SDP solvers. 
The matrix $\MX_{2\relaxOrder}$ solving~\eqref{eq:lasserre} is typically referred to as the \emph{pseudo-moment matrix}.\footnote{The rationale behind this name will become apparent in~\cref{app:pseudo-distributions}.}
One can solve the relaxation for different choices of $\relaxOrder$, leading to a \emph{hierarchy} of convex relaxations.

While we presented Lasserre's hierarchy in a somewhat procedural way, the importance of the hierarchy lies in its stunning theoretical properties, that we review below.
\begin{theorem}[Lasserre's Hierarchy \cite{Lasserre01siopt-LasserreHierarchy,Lasserre10book-momentsOpt,Nie14mp-finiteConvergenceLassere}]
\label{thm:lasserre}
Let $-\infty < \pstar < \infty$ be the optimum of \eqref{eq:pop} and \revise{$\fstar_{2\relaxOrder}$ (resp. $\MXstar_{2\relaxOrder}$) be the optimum (resp. one optimizer) of \eqref{eq:lasserre},} and assume \eqref{eq:pop} is explicitly bounded 
(\ie it satisfies the Archimedeanness condition in \cite[Definition 3.137]{Blekherman12Book-sdpandConvexAlgebraicGeometry}), then 
\begin{enumerate}[label=(\roman*)]
\item \revise{(lower bound and convergence)} $\fstar_{2\relaxOrder}$ converges to $\pstar$ from below as $\relaxOrder \rightarrow \infty$, and convergence occurs at a finite $\relaxOrder$ under suitable technical conditions~\cite{Nie14mp-finiteConvergenceLassere};
\item \revise{(rank-one solutions)} if $\fstar_{2\relaxOrder} = \pstar$ at some finite $\relaxOrder$, then for every global minimizer $\vxxstar$ of~\eqref{eq:pop}, 
$\MXstar_{2\relaxOrder} \triangleq [\vxxstar]_{\relaxOrder} [\vxxstar]_{\relaxOrder}\tran$ is optimal for \eqref{eq:lasserre}, and every rank-one optimal solution $\MXstar_{2\relaxOrder}$ of \eqref{eq:lasserre} can be written as $[\vxxstar]_\relaxOrder [\vxxstar]_{\relaxOrder}\tran$ for some $\vxxstar$ that is optimal for \eqref{eq:pop};
\item \revise{(optimality certificate)} if $\rank{\MXstar_{2\relaxOrder}} = 1$ at some finite $\relaxOrder$, then $\fstar_{2\relaxOrder} = \pstar$.
\end{enumerate}
\end{theorem}
Theorem \ref{thm:lasserre} states that~\eqref{eq:lasserre} provides a hierarchy of lower bounds for~\eqref{eq:pop}. 
When the relaxation is exact ($\pstar\!=\!\fstar_{2\relaxOrder}$), global minimizers of~\eqref{eq:pop} correspond to rank-one solutions of~\eqref{eq:lasserre}. \revise{Moreover, after solving the convex SDP \eqref{eq:lasserre}, one can check the rank of the optimal solution $\MXstar_{2\relaxOrder}$ to obtain a \emph{certificate} of global optimality. 
}

\myParagraph{Further tightening the relaxation}
As we discussed above, in the standard presentation of Lasserre's hierarchy, one adds a localizing matrix for each inequality constraint to enforce constraints such as $g_j \cdot \MX_{2(\relaxOrder - \ceil{\deg{g_j}/2})} \succeq 0$. 
However, in principle, we could also add constraints enforcing $g_{j_1} \cdot g_{j_2} \cdot \MM \succeq 0$, for any pair of inequality constraints $g_{j_1} \!\geq\!0$ and $g_{j_2}\!\geq\!0$, for $j_1,j_2 \in [l_g]$. 
More generally, 
we can add constraints 
$\prod_{j\in\calS} g_j \cdot \MM \succeq 0$, for any subset $\calS \subseteq [l_g]$ as long as $\degree{\textstyle\prod_{j\in\calS} g_j}$ has degree no larger than $2\relaxOrder$.
After adding those extra constraints, we can still phrase the resulting 
relaxation in the standard form~\eqref{eq:primalSDP}, by adding extra matrix variables $\MX_{\calS} = \prod_{j\in\calS} g_j \cdot \MX_{2(\relaxOrder - \ceil{\sum_{j\in\calS} \deg{g_j}/2})}$, 
and then forcing these matrices to be a linear function of  $\MX_{2\relaxOrder}$:
\beal\label{eq:locMatrices}
\text{\grayout{localizing matrices}}: & \MX_{\calS} \succeq 0, \;\; \calS \subseteq [l_g],
\eeal
\beal \label{eq:localizingConstraints2}  
\hspace{-2mm} \text{\grayout{{localizing} constraints}}: \inprod{\MA_{\mathloc,\calS,kh}}{ \MX_{2\relaxOrder} } 
= [\MX_{\calS}]_{hk}, \\
\quad \quad \calS \subseteq [l_g],\ \  
1 \leq h\leq k \leq \binomial{d}{ \relaxOrder - \ceil{\sum_{j\in\calS} \deg{g_j}/2} },
\eeal
where, similarly to the standard Lasserre's relaxation, the linear constraints (for some suitable matrices $\MA_{\mathloc,\calS,kh}$) enforce each entry of $\MX_{\calS}$ to be a linear combination of the entries in $\MX_{2\relaxOrder}$.

The additional constraints in eq.~\eqref{eq:localizingConstraints2} make the relaxation tighter compared to the standard presentation of Lasserre's relaxation, 
  but are not necessary to obtain the convergence result in~\cref{thm:lasserre}, which holds regardless for explicitly bounded constraint sets. 
However, these constraints become necessary to obtain convergence results akin to~\cref{thm:lasserre} for the case where the set of constraints is \emph{not} explicitly bounded (see~\cite[Section 3.3]{Fleming19fnt-sosproofs} and~\cite[p. 115]{Blekherman12Book-sdpandConvexAlgebraicGeometry} for a more extensive discussion).
For this reason, in order to maintain generality, 
related work following the ``proofs to algorithms'' paradigm typically assumes those constraints to be present, see, \eg~\cite{Klivans18arxiv-robustRegression,Karmalkar19neurips-ListDecodableRegression}.
These terms will indeed appear in the definitions of \sos proofs and constrained pseudo-distribution, see~\cref{app:pseudo-distributions}.
 In order to keep the definitions in our \paper consistent with~\cite{Klivans18arxiv-robustRegression,Karmalkar19neurips-ListDecodableRegression}, we will also assume these terms to be present, even though they are not strictly necessary under~\cref{ass:explicitlyBoundedxs}.

\section{Pseudo-distributions and Moment Relaxations}
\label{app:pseudo-distributions}

The start of this section follows standard introductions about pseudo-distributions given in related work~\cite{Klivans18arxiv-robustRegression,Barak16notes-proofsBeliefsSos,Kothari18stoc-robustMomentEstimation,Karmalkar19neurips-ListDecodableRegression}, while later in the section we attempt to    
draw more explicit connections with the optimization machinery in~\cref{app:momentRelaxation}. 
Although such a connection is self-evident to the expert reader (indeed pseudo-distributions are the language traditionally used to justify the moment relaxation~\cite{Lasserre10book-momentsOpt}), such a connection is often less immediate for the practitioner, in particular when taking the algorithmic view of moment relaxations presented in~\cref{app:momentRelaxation}. 

\myParagraph{Pseudo-distributions}
Pseudo-distributions are a generalization of the concept of probability distribution.
A standard probability distribution $\dist$ with finite support in $\Real{\dimd}$ 
is simply a function $\dist : \Real{\dimd} \mapsto \Real{}$ 
such that $\sum_{\vxx \in \support(\dist)} \dist(\vxx)=1$ and 
$\dist(\vxx) \geq 0$ for all $\vxx$. In other words, if $\support(\dist)$ is a finite collection a  points in $\Real{\dimd}$, $\dist$ assigns a non-negative probability mass to each of these points, such that those probabilities sum up to 1. 
Similarly, a pseudo-distribution $\pdist$ is a finitely supported function such that 
 $\sum_{\vxx \in \support(\pdist)} \pdist(\vxx)=1$ but in this case the non-negativity condition is replaced by a milder condition (\ie a pseudo-distribution can assume negative values over its support). 

In order to formally introduce the notion of pseudo-distribution, we start by defining 
the \emph{pseudo-expectation} of a function $f:\Real{\dimd} \mapsto \Real{}$ under a finitely supported function $\pdist$:
\begin{align}
\pExpect{\pdist}{f(\vxx)} \doteq \sum_{\vxx \in \support(\pdist)} f(\vxx) \cdot \pdist(\vxx).
\end{align}

We are now ready to formally define a pseudo-distribution.
\begin{definition}[Pseudo-distribution]
A finitely supported function $\pdist: \Real{\dimd}\mapsto \Real{}$ is a \emph{level-$\levelpd$ pseudo-distribution} 
if $\pExpect{\pdist}{1} = 1$ and $\pExpect{\pdist}{f(\vxx)^2} \geq 0$ for all polynomials $f$ of degree $\degree{f}\leq \levelpd/2$.
\end{definition}

In words, $\pdist$ is a function that is allowed to become negative as long as its ``expectation'' (more precisely, pseudo-expectation) with respect to every squared polynomials $f(\vxx)^2$ of sufficiently low degree remains positive. 
It is possible to show that a level-$\infty$ 
 pseudo-distribution is an actual probability distribution, since the condition $\pExpect{\pdist}{f(\vxx)^2} \geq 0$  would enforce $\pdist$ to remain positive 
 (in this case the pseudo-expectation becomes the traditional expectation of the distribution).

Towards reconnecting pseudo-distributions with the optimization machinery in~\cref{app:momentRelaxation}, we start by observing the following link between pseudo-distributions and 
pseudo-moment matrices.

\begin{lemma}[Pseudo-moment matrix~\cite{Barak16notes-proofsBeliefsSos}]
\label{lem:pseudoMomentMatrix}
Let $\pdist\!:\!\Real{\dimd}\!\mapsto\!\Real{}$ be a finitely supported function with 
$\pExpect{\pdist}{1} \!\!=\!\! {1}$. Then, $\pdist$ is a level-$\levelpd$ pseudo-distribution iff the 
pseudo-moment matrix $\pExpect{\pdist}{[\vxx]_{\levelpd/2} [\vxx]_{\levelpd/2}\tran}$ is positive semidefinite, where $[\vxx]_{\levelpd/2}$ is the vector of monomials of degree up to $\levelpd/2$.
\end{lemma}

Now we define what it means for a pseudo-distribution to satisfy a set of polynomial constraints.

\begin{definition}[Constrained pseudo-distribution]
\label{def:constrainedPD}
Let $\calA \doteq \{f_1\geq 0, \ldots, f_m \geq 0\}$ be a set of polynomial constraints over $\Real{\dimd}$. Let $\pdist : \Real{\dimd}\!\mapsto\!\Real{}$ be a level-$\levelpd$ pseudo-distribution. 
We say that $\pdist$ \emph{satisfies} $\calA$ at degree $\satisfyOrder$, denoted as 
$\pdist \psatisfy{}{\satisfyOrder} \calA$, if every set $\calS \subset[m]$ and every sum-of-squares 
polynomial $h$ on $\Real{\dimd}$ with~$\degree{h} + \sum_{i\in \calS} \max\{\degree{f_i}, \satisfyOrder\}\leq \levelpd$ satisfies:
\begin{align}
\label{eq:constrainedPD}
\pExpect{\pdist}{h \cdot \prod_{i \in \calS} f_i } \geq 0.
\end{align}
Moreover, we say that $\pdist \psatisfy{}{\satisfyOrder} \calA$ holds \emph{approximately} 
if the above inequalities are satisfied up to an error of 
$2^{-\dimd^\levelpd} \cdot \|h\| \cdot \prod_{i \in \calS} \|f_i\|$, where $\|\cdot\|$ denotes the Euclidean norm of the coefficients of the polynomial. 
\end{definition}

The notion of pseudo-distributions \emph{approximately} satisfying a set of constraints is useful 
to account for the practical observation that numerical SDP solvers (which we are going to use to find pseudo-distributions, as discussed later in this section) will only satisfy the constraints up to some numerical tolerance, and we have to make sure that such numerical errors do not lead us to draw incorrect conclusions using the \sos proof system (see~\cref{app:sosProofs}). 
In this \paper, we make use of the following facts about pseudo-distributions.

\begin{fact}[Linearity~\cite{Barak14stoc-sos}] \label{lem:linearityOfPD}
Let $f,g$ be polynomials of degree at most $\ell$ in indeterminate $\vxx \in \Real{\dimd}$ and take
 $\alpha,\beta \in \Real{}$. 
Then, for any level-$\ell$ pseudo-distribution $\pdist$, 
\begin{align}
\pExpect{\pdist}{ \alpha \, f(\vxx) + \beta \, g(\vxx) } = \alpha \pExpect{\pdist}{ f(\vxx) }  + \beta \pExpect{\pdist}{ g(\vxx) } .
\end{align}
\end{fact}

\begin{fact}[\CS for pseudo-distributions~\cite{Karmalkar19neurips-ListDecodableRegression}]
\label{fact:CSforpdist}
Let $f,g$ be polynomials of degree at most $\ell$ in indeterminate $\vxx \in \Real{\dimd}$. 
Then, for any level-$\ell$ pseudo-distribution $\pdist$, 
\begin{align}
\label{eq:CSforpdist}
\pExpect{\pdist}{f\cdot g} \leq \sqrt{\pExpect{\pdist}{f^2}} \cdot \sqrt{\pExpect{\pdist}{g^2}},
\end{align}
and (specializing the result above to $g=1$):
\begin{align}
\label{eq:CSforpdist2}
\pExpect{\pdist}{f}^2 \leq {\pExpect{\pdist}{f^2}}.
\end{align}
\end{fact}


\begin{fact}[\HoldersInequality for pseudo-distributions~\cite{Klivans18arxiv-robustRegression}]
\label{fact:Holderforpdist2}
Let $f,g$ be \sos polynomials. Let $p,q$ be positive integers such that $1/p + 1/q = 1$. 
Then, for any pseudo-distribution $\pdist$ of level $\ell \geq pq \cdot \degree{f} \cdot \degree{g}$, 
we have: 
\begin{align}
\label{eq:Holderforpdist2a}
\left( \pExpect{\pdist}{f\cdot g} \right)^{pq} \leq 
\pExpect{\pdist}{f^p}^q \cdot \pExpect{\pdist}{g^q}^p .
\end{align}
In particular, for all even integers $k \geq 2$, and polynomial $f$ with $\degree{f} \cdot k \leq \ell$:
\begin{align}
\label{eq:Holderforpdist2b}
\left( \pExpect{\pdist}{f} \right)^{k} \leq 
\pExpect{\pdist}{f^k} .
\end{align}
\end{fact}

\begin{fact}[Norm inequality for pseudo-distributions]
\label{fact:normIneq-pdist}
Let $\vv$ be an $m$-vector with polynomial entries of degree at most $\ell/2$ 
in indeterminate $\vxx \in \Real{\dimd}$. 
Then, for any degree-$\ell$ pseudo-distribution $\pdist$, 
\begin{align}
\label{eq:normIneq-pdist}
\normTwo{ \pExpect{\pdist}{\vv}}^2 \leq  \pExpect{\pdist}{\normTwo{\vv}^2}.
\end{align}
\end{fact}

\begin{proof}
By definition, $\normTwo{\vv}^2 = \sum_{i=1}^m v_i^2$. Moreover, by~\eqref{eq:Holderforpdist2b},
$\left( \pExpect{\pdist}{v_i} \right)^{2} \leq 
\pExpect{\pdist}{v_i^2}$. Therefore:
\begin{align}
\normTwo{ \pExpect{\pdist}{\vv}}^2 = 
\sum_{i=1}^m \left( \pExpect{\pdist}{v_i} \right)^{2} 
 \leq 
\sum_{i=1}^m \pExpect{\pdist}{v_i^2} 
\overbrace{=}^{linearity} \pExpect{\pdist}{ \sum_{i=1}^m v_i^2}
= \pExpect{\pdist}{\normTwo{\vv}^2},
\end{align} 
proving the claim.
\end{proof}

\myParagraph{Making the connection with moment relaxations explicit}
The non-expert reader might still be confused about the relation between pseudo-distributions and 
moment relaxations. To shed some light, let us restate our~\eqref{eq:pop}:
\begin{align}\label{eq:pop_app}
\min_{\vxx \in \Real{\dimx}}\;&\; p(\vxx) \\
\subject \;&\; h_i(\vxx) = 0, \quad i=1,\dots,l_h \nonumber \\ 
\;&\; g_j(\vxx) \geq 0, \quad j = 1,\dots,l_g . \nonumber
\end{align} 

Now, we start by relaxing~\eqref{eq:pop_app} using pseudo-distributions, and show that this leads back to the relaxation presented in~\cref{app:momentRelaxation}.
In particular, we relax~\eqref{eq:pop_app} to:  
\begin{align} 
\min_{\pdist}\;&\; \pExpect{\pdist}{p(\vxx)} \label{eq:pdist_relax}
\\
\subject \;&\; \pdist \text{ is a level-$\levelpd$ pseudo-distribution} 
\label{eq:pdist_relax_pdist}
\\ 
\;&\; \pExpect{\pdist}{ h_i(\vxx) \cdot q(\vxx) } = 0,  
\label{eq:pdist_relax_eq}
\\
&  \text{for all $i=1,\dots,l_h$ and for all $q \in \polyring{\vxx}$, such that $\degree{h_i \cdot q} \leq \levelpd$} \nonumber\\ 
\;&\; \pExpect{\dist}{ \textstyle\prod_{j \in \calS }g_j(\vxx) \cdot s(\vxx)^2  } \geq 0,  
\label{eq:pdist_relax_ineq}
\\
&  \text{for all $\calS \subseteq [l_g]$ and for all $s \in \polyring{\vxx}$ such that $\degree{ \textstyle\prod_{j \in \calS}g_j \cdot s^2  } \leq \levelpd$}.  \nonumber
\end{align} 
Despite the complexity of~\eqref{eq:pdist_relax}, it is apparent that~\eqref{eq:pdist_relax} is a relaxation of~\eqref{eq:pop_app}: for any $\vxx$ that is feasible for~\eqref{eq:pop_app} (\ie that satisfies $h_i(\vxx) = 0$ and $g_j(\vxx) \geq 0$), we can define a \mbox{(pseudo-)distribution} $\mu_x$ supported on $\vxx$ (\ie $\mu_x(\vxx)=1$ and zero elsewhere) which is also feasible for~\eqref{eq:pdist_relax} (such pseudo-distribution $\mu_x$ is such that $\pExpect{\mu_x}{p(\vxx)} = p(\vxx)$ for any polynomial $p$, hence also preserving the objective of~\eqref{eq:pop_app}).
Indeed, it is possible to show that if we require $\pdist$ to be an actual distribution, and replace the pseudo-expectations with actual expectations, then~\eqref{eq:pdist_relax} becomes equivalent to~\eqref{eq:pop_app}, see~\cite{Lasserre18icm-momentRelaxation} for a more extensive discussion.
The advantage of the relaxation~\eqref{eq:pdist_relax} is its tractability: 
while~\eqref{eq:pop_app} is NP-hard~\cite{Lasserre18icm-momentRelaxation}, the relaxation~\eqref{eq:pdist_relax} can be written as a semidefinite program (SDP) and solved in polynomial time.
Indeed, in the following we show that rewriting~\eqref{eq:pdist_relax} as an SDP leads us back 
to the same moment relaxation we procedurally introduced in~\cref{app:momentRelaxation}. 
Towards this goal, we will need some extra notation. 

\emph{Preliminaries to connect problem~\eqref{eq:pdist_relax} with~\cref{app:momentRelaxation}:} Recall that $[\vxx]_{\levelpd/2}$ is 
the \revise{vector} of monomials of degree up to $\levelpd/2$ and therefore the 
moment matrix $\MX_{\levelpd} \triangleq [\vxx]_{\levelpd/2} [\vxx]_{\levelpd/2}\tran$ 
contains all monomials of degree up to $\levelpd$. 
It will be useful to define (and visualize) the pseudo-expectation of the moment matrix: 
$\pExpect{\pdist}{\MX_{\levelpd}}$. For instance, for the case with $\vxx = [x_1 \vcat x_2]$ and $\ell=4$:
\beq
\label{eq:pExpectmomentMatrix}
\pExpect{\pdist}{ \MX_{4} } \triangleq  \!=\!
\small{  
\left[
\begin{array}{cccccc}
	1 &  \pExpect{\pdist}{ x_1 }  &  \pExpect{\pdist}{ x_2 }  & \pExpect{\pdist}{ x_1^2 } & \pExpect{\pdist}{ x_1 x_2} & \pExpect{\pdist}{ x_2^2 } \\
	\pExpect{\pdist}{x_1} &  \pExpect{\pdist}{x_1^2}  &  \pExpect{\pdist}{x_1 x_2}  & \pExpect{\pdist}{x_1^3} & \pExpect{\pdist}{x_1^2 x_2} & \pExpect{\pdist}{x_1 x_2^2} \\
	\pExpect{\pdist}{x_2} &  \pExpect{\pdist}{x_1 x_2}  &  \pExpect{\pdist}{x_2^2}  & \pExpect{\pdist}{x_1^2 x_2} & \pExpect{\pdist}{x_1 x_2^2} & \pExpect{\pdist}{x_2^3} \\
	\pExpect{\pdist}{x_1^2} &  \pExpect{\pdist}{x_1^3}  &  \pExpect{\pdist}{x_1^2 x_2}  & \pExpect{\pdist}{x_1^4} & \pExpect{\pdist}{x_1^3 x_2} & \pExpect{\pdist}{x_1^2 x_2^2} \\
	 \pExpect{\pdist}{x_1 x_2} &   \pExpect{\pdist}{x_1^2 x_2}  &   \pExpect{\pdist}{x_1 x_2^2}  & \pExpect{\pdist}{x_1^3 x_2} & \pExpect{\pdist}{x_1^2 x_2^2} & \pExpect{\pdist}{x_1 x_2^3} \\
	 \pExpect{\pdist}{x_2^2} &  \pExpect{\pdist}{x_1 x_2^2} &  \pExpect{\pdist}{x_2^3}  & \pExpect{\pdist}{x_1^2 x_2^2} & \pExpect{\pdist}{x_1 x_2^3} & \pExpect{\pdist}{x_2^4} 
\end{array}
\right]
}.
\eeq

In the following, we will also need a more convenient way to index the monomials in $[\vxx]_{\levelpd}$ 
(and, as a consequence, the entries of $\MX_{\levelpd}$ and $\pExpect{\pdist}{\MX_{\levelpd}}$). Using standard notation, for a vector $\valpha \in \Natural{\dimd}$, 
we write $\vxx^{\valpha}$ to denote the monomial with 
exponents $\valpha$ (for instance, for $\valpha = [1 \vcat 3 \vcat 0 \vcat 5]$, $\vxx^{\valpha} = x_1 x_2^3 x_4^5$). 
We also denote $|\valpha| \triangleq \sum_{i=1}^{\dimd} \alpha_i$, which is the degree of the monomial.
Using this notation, we can index with $\valpha$ the monomials appearing in $[\vxx]_{\levelpd}$. 
For instance, for $\ell = 2$:
\begin{align}
[\vxx]_{2} \triangleq [
\overbrace{1}^{[0\vcat 0]}
\vcat 
\overbrace{x_1}^{[1\vcat 0]} 
\vcat 
\overbrace{x_2}^{[0\vcat 1]} 
\vcat 
\overbrace{x_1^2}^{[2\vcat 0]}  
\vcat 
\overbrace{x_1 x_2}^{[1\vcat 1]}  
\vcat 
\overbrace{x_2^2}^{[0\vcat 2]} 
],
\end{align}
where for each monomial, we reported the corresponding ``index'' $\valpha$. We can similarly index 
the rows and columns of $\MX_{\levelpd}$ and $\pExpect{\pdist}{\MX_{\levelpd}}$ using two indices $\valpha$ and $\vbeta$.
For instance, the entry of the matrix indexed by row $\valpha = [2 \vcat 0]$ and column $\valpha = [0 \vcat 1]$ in~\eqref{eq:pExpectmomentMatrix} will be $\pExpect{\pdist}{x_1^2 x_2}$. 
Note that the monomial appearing in row $\valpha$ and column $\vbeta$ of the moment matrix $\MX_{\levelpd}$ will always have exponent $\valpha + \vbeta$, since, due to the definition of the moment matrix ($\MX_{\levelpd} \triangleq [\vxx]_{\levelpd/2} [\vxx]_{\levelpd/2}\tran$) its entry $[\MX_{\levelpd}]_{\valpha\vbeta} = \vxx^{\valpha}\cdot \vxx^{\vbeta}$ 
and for two monomials $\vxx^{\valpha}$ and $\vxx^{\vbeta}$, it holds:\footnote{For instance, the product between the monomial $x_1 x_2^3 x_4^5$ (namely, $\vxx^{\valpha}$ with $\valpha = [1 \vcat 3 \vcat 0 \vcat 5]$) and the monomial
$x_1^2 x_2 x_3$ (namely, $\vxx^{\vbeta}$ with $\vbeta = [2 \vcat 1 \vcat 1 \vcat 0]$) is 
$(x_1 x_2^3 x_4^5) \cdot (x_1^2 x_2 x_3) = x_1^3 x_2^4 x_3 x_4^5$, which corresponds to the exponent vector $[3 \vcat 4 \vcat 1 \vcat 5]$.} 
\begin{align}
\label{eq:prodMonomials}
\vxx^{\valpha} \cdot \vxx^{\vbeta} = \vxx^{\valpha + \vbeta}.
\end{align}
Finally, we will conveniently use the following representation of a polynomial $f(\vxx)$ of degree $\ell$: 
\begin{align}
f(\vxx) &= \textstyle \sum_{\valpha: |\valpha| \leq \ell} \bar{f}_\valpha \; \vxx^{\valpha} , 
\label{eq:monomialRep}
\end{align}
where we simply observed that the polynomial is the sum of monomials $\vxx^{\valpha}$ of degree $|\valpha| \leq \ell$, 
and with suitable coefficients $\bar{f}_\valpha$, again indexed by $\valpha$. 

 
We are now ready to show that both objective and constraints~\eqref{eq:pdist_relax} can be rewritten in a way that leads
back to the moment relaxation in~\cref{app:momentRelaxation}. 

\emph{Rewriting the objective~\eqref{eq:pdist_relax}:} 
To simplify the objective $\pExpect{\pdist}{p(\vxx)}$, we note that the pseudo-expectation is 
a linear operator, hence:
\begin{align}
\pExpect{\pdist}{p(\vxx)} 
\overbrace{=}^{ \text{using~\eqref{eq:monomialRep}} } 
\pExpect{\pdist}{ \sum_{\valpha: |\valpha| \leq \ell} \bar{p}_\valpha \; \vxx^{\valpha} }
\overbrace{=}^{ \text{using~\cref{lem:linearityOfPD}} }  
 \!\!\sum_{\valpha: |\valpha| \leq \ell} \bar{p}_\valpha \;\pExpect{\pdist}{  \vxx^{\valpha} }
 \overbrace{=}^{ \text{for a suitable matrix $\MC_1$} }
 \inprod{ \MC_1 }{  \pExpect{\pdist}{\MX_{\levelpd} } }, 
 \label{eq:app-objective} 
\end{align}
which indeed produces the same structure as the objective of the moment relaxation in~\eqref{eq:objective} with $\levelpd = 2\relaxOrder$; 
as we will see in a while, $\pExpect{\pdist}{\MX_{\levelpd} }$ will become the main matrix variable in the
optimization.

\emph{Rewriting the equality constraints~\eqref{eq:pdist_relax_eq}:} 
To simplify the constraint $\pExpect{\pdist}{ h_i(\vxx) q(\vxx) } = 0$ (which has to hold for polynomials
$q$ of degree $\degree{h_i \cdot q} \leq \levelpd$) we note that 
it suffices to require $\pExpect{\pdist}{ h_i(\vxx) \vxx^\vbeta } = 0$ for $|\vbeta| \leq \levelpd - \degree{h_i}$; 
this follows from the fact that any polynomial is a sum of monomials and the pseudo-expectation is a linear function. 
Let us now manipulate $\pExpect{\pdist}{ h_i(\vxx) \vxx^\vbeta } = 0$ as follows:
\begin{align}
\pExpect{\pdist}{ h_i(\vxx) \vxx^\vbeta  } = 0 
\overbrace{\iff}^{\text{using~\eqref{eq:monomialRep} for $h_i(\vxx)$}}
\pExpect{\pdist}{ \sum_{\valpha} \bar{h}_{i,\valpha} \vxx^\valpha  \vxx^\vbeta} = 0 
\overbrace{\iff}^{\text{using~\eqref{eq:prodMonomials}}}
\pExpect{\pdist}{ \sum_{\valpha} \bar{h}_{i,\valpha} \vxx^{\valpha + \vbeta}} = 0 
\nonumber
\\
\label{eq:app-eq-constraints}
\overbrace{\iff}^{\text{using~\cref{lem:linearityOfPD}}}
\sum_{\valpha} \bar{h}_{i,\valpha} \pExpect{\pdist}{ \vxx^{\valpha + \vbeta}} = 0
\overbrace{\iff}^{\text{for a suitable matrix $\MA_{i,\vbeta}$}}
 \inprod{\MA_{i,\vbeta}}{ \pExpect{\pdist}{ \MX_{\levelpd}} } = 0,
\end{align}
which has to be imposed for each $\vbeta$ such that $|\vbeta| \leq \levelpd - \degree{h_i}$.
 Note that the constraints in~\eqref{eq:app-eq-constraints} capture both the equality
  constraints in~\eqref{eq:eqConstraints1} (for $|\vbeta|=0$) as well as the redundant constraints~\eqref{eq:redundantEqualityConstraints} (for $0 < |\vbeta| \leq \levelpd - \degree{h_i}$).

\emph{Rewriting the inequality constraints in~\eqref{eq:pdist_relax_ineq}:} 
We simplify the constraint $\pExpect{\dist}{ \textstyle\prod_{j\in \calS}g_j(\vxx) \cdot s(\vxx)^2  } \geq 0$,  
which has to hold for all $\calS \subseteq [l_g]$ and for all $s \in \polyring{\vxx}$ such that $\degree{ \textstyle\prod_{j\in \calS }g_j \cdot s^2  } \leq \levelpd$.
Towards this goal, we use the representation~\eqref{eq:monomialRep} for $s(\vxx)$ and write 
$s(\vxx) = \sum_{\vbeta: |\vbeta| \leq \newSize } \bar{s}_\vbeta \vxx^\vbeta$, where $\newSize \triangleq \floor{\frac{\levelpd - \degree{ \textstyle\prod_{j\in \calS}g_j} }{2}}$. Therefore, we obtain:
\begin{align}
\pExpect{\dist}{ g_j(\vxx) \cdot s(\vxx)^2  } \geq 0 
\overbrace{\iff}^{\text{expanding $s^2$}} 
\pExpect{\dist}{ g_j(\vxx) \cdot \sum_{\valpha: |\valpha| \leq \newSize } \bar{s}_\valpha \vxx^\valpha \sum_{\vbeta: |\vbeta| \leq \newSize } \bar{s}_\vbeta \vxx^\vbeta  } \geq 0
\\
\overbrace{\iff}^{\text{rearranging}} 
\pExpect{\dist}{ \sum_{\valpha, \vbeta: |\valpha|,|\vbeta| \leq \newSize } \bar{s}_\valpha \bar{s}_\vbeta 
 \vxx^{\valpha + \vbeta} g_j(\vxx) } \geq 0
 \\
\overbrace{\iff}^{\text{using~\eqref{eq:monomialRep} on $\prod_{\calS \subseteq [l_g] }g_j(\vxx)$}}
\pExpect{\dist}{ \sum_{\valpha, \vbeta: |\valpha|,|\vbeta| \leq \newSize } \bar{s}_\valpha \bar{s}_\vbeta 
 \vxx^{\valpha + \vbeta} \sum_{\vgamma: |\vgamma|\leq \degree{\prod_{j\in \calS}g_j(\vxx)}} \bar{g}_{\calS,\vgamma} \vxx^{\vgamma}  } \geq 0
\end{align}
\begin{align}
 \overbrace{\iff}^{\text{rearranging}} 
 \pExpect{\dist}{ \sum_{\valpha, \vbeta: |\valpha|,|\vbeta| \leq \newSize } \bar{s}_\valpha \bar{s}_\vbeta 
\sum_{\vgamma: |\vgamma|\leq \degree{\prod_{j\in \calS}g_j(\vxx)}} \bar{g}_{\calS,\vgamma}  \vxx^{\valpha + \vbeta + \vgamma}   } \geq 0
\\
\overbrace{\iff}^{\text{using~\cref{lem:linearityOfPD}}}
\sum_{\valpha, \vbeta: |\valpha|,|\vbeta| \leq \newSize } \bar{s}_\valpha \bar{s}_\vbeta 
\sum_{\vgamma: |\vgamma|\leq \degree{\prod_{j\in \calS}g_j(\vxx)}} \bar{g}_{\calS,\vgamma}   \pExpect{\dist}{  \vxx^{\valpha + \vbeta + \vgamma}   } \geq 0.
\label{eq:app-10-1}
\end{align}
Now note that $|\valpha + \vbeta + \vgamma| \leq \levelpd$ by construction, and hence we can write, for a given $\valpha$ and $\vbeta$, each $\sum_{\vgamma: |\vgamma|\leq \degree{\prod_{j\in \calS}g_j(\vxx) }} \bar{g}_{\calS,\vgamma}   \pExpect{\dist}{  \vxx^{\valpha + \vbeta + \vgamma}}$ as a linear function of $\pExpect{\pdist}{ \MX_{\levelpd}}$. 
Moreover, since $\valpha$ and $\vbeta$ are such that $|\valpha|,|\vbeta| \leq \newSize$, we can group these entries into an $\newSize \times \newSize$ matrix $\MX_{\calS}$, which is such that:
\begin{align}
\label{eq:app-recovered-loc-1}
[\MX_\calS]_{\valpha,\vbeta} = \inprod{ \MA_{\mathloc,\calS,\valpha\vbeta} }{ \pExpect{\pdist}{ \MX_{\levelpd}} },
\end{align}
for some suitable matrix $\MA_{\mathloc,\calS,\valpha\vbeta}$, such that $\inprod{ \MA_{\mathloc,\calS,\valpha\vbeta} }{ \pExpect{\pdist}{ \MX_{\levelpd}} } = \sum_{\vgamma: |\vgamma|\leq \degree{ \prod_{j\in \calS} g_j}} \bar{g}_{i,\vgamma}   \pExpect{\dist}{  \vxx^{\valpha + \vbeta + \vgamma}}$.
Using the matrix $\MX_\calS$ and defining a vector $\bar{\vs} \in \Real{\newSize}$ with entries $\bar{s}_\valpha$ for $|\valpha| \leq \newSize$, we rewrite~\eqref{eq:app-10-1} as:
 \begin{align}
 \sum_{\valpha, \vbeta: |\valpha|,|\vbeta| \leq \newSize } \bar{s}_\valpha \bar{s}_\vbeta [\MX_\calS]_{\valpha,\vbeta} \geq 0
\iff
\bar{\vs}\tran \MX_\calS \bar{\vs} \geq 0.
\end{align}
Since this has to hold for any $\bar{\vs}$ (\ie any polynomial $s(\vxx)$ of appropriate degree), we conclude the constraint above is equivalent to:
\beq
\label{eq:app-recovered-loc-2}
\MX_\calS\succeq 0.
\eeq
Now we can easily see that~\eqref{eq:app-recovered-loc-1} and~\eqref{eq:app-recovered-loc-2} match the localizing constraints we wrote in~\eqref{eq:localizingConstraints}.

\emph{Rewriting~\eqref{eq:pdist_relax_pdist}:} 
Finally, the constraint~\eqref{eq:pdist_relax_pdist} imposes that $\pdist$ must be  a
level~$\levelpd$ pseudo-distribution. However, we know from~\cref{lem:pseudoMomentMatrix} 
that $\pdist$ is a level-$\levelpd$ pseudo-distribution if and only if 
the pseudo-moment matrix $\pExpect{\pdist}{[\vxx]_{\levelpd/2} [\vxx]_{\levelpd/2}\tran}$ is positive semidefinite and $\pExpect{\pdist}{1} = 1$. 
Therefore, we can reparametrize the objective~\eqref{eq:app-objective} 
and constraints~\eqref{eq:app-eq-constraints},~\eqref{eq:app-recovered-loc-1},~\eqref{eq:app-recovered-loc-2} 
with a matrix variable (in place of $\pExpect{\pdist}{ \MX_{\levelpd}}$) that is constrained to be positive semidefinite and to have the top-left entry  equal to 1 (\cf~\eqref{eq:pExpectmomentMatrix}).
This yields back the relaxation described in~\cref{app:momentRelaxation} as expected. 

\myParagraph{Constrained pseudo-distributions: a practical view}
So far we have shown that taking suitable pseudo-expectations over the objective and constraints in a polynomial optimization problem leads to a convex relaxation, known as the moment relaxation.
Now we want to shed some light on~\cref{def:constrainedPD} by showing that the condition~\eqref{eq:constrainedPD} is indeed the same as the inequality constrains in~\eqref{eq:pdist_relax_ineq} and hence admits the same transcription as an SDP.

Towards this goal, let us consider the following \emph{feasibility} POP:  
\begin{align}
\label{eq:popFeasibility}
 \text{find} \;\;&\;\; \vxx \in \Real{\dimx} \\ 
 \subject \;&\;\; h_i(\vxx) = 0, i=1,\dots,l_h \\ 
 \;\;&\;\;  g_j(\vxx) \geq 0, j = 1,\dots,l_g . 
\end{align} 
This is similar to~\eqref{eq:pop}, with the exception that we are looking for a feasible solution rather than optimizing a cost function. Now note that we can write a polynomial equality 
$h_i(\vxx) = 0$ as two inequality constraints $h_i(\vxx) \leq 0$ and $-h_i(\vxx) \leq 0$. 
Hence, without loss of generality we rewrite~\eqref{eq:popFeasibility}~as:
\begin{align}
\label{eq:popFeasibility2}
 \text{find} \;\;&\;\; \vxx \in \Real{\dimx} \\
 \subject \;&\;\; f_j(\vxx) \geq 0, \quad j=1,\ldots,m, 
\end{align} 
for suitable polynomials $f_i$, $i=1,\ldots,m$. 
Similarly to what we did earlier in this section, we relax~\eqref{eq:popFeasibility2} by using pseudo-expectations:
\begin{align}
\label{eq:constrainedPD-SDP}
\text{find} \;\;&\;\; \pdist \\
\subject \;&\; \pdist \text{ is a level-$\levelpd$ pseudo-distribution} 
\\
\;&\; \textstyle\pExpect{\dist}{ s(\vxx)^2 \cdot \prod_{i \in \calS} f_i(\vxx)   } \geq 0,  
\label{eq:pdist_relax_ineq}
\\
&  \text{for every set $\calS \in [m]$ and every $s \in \polyring{\vxx}$ such that 
$\textstyle\degree{ s^2  \cdot \prod_{i \in \calS} f_i  } \leq \levelpd$}.  \nonumber
\end{align} 
First of all, we note that~\eqref{eq:pdist_relax_ineq} matches the definition of constrained pseudo-distribution in~\cref{def:constrainedPD} for $\satisfyOrder=0$.
Moreover, following the same derivation as above, we can easily show that (i)~\eqref{eq:constrainedPD-SDP} 
can be transcribed as a standard SDP, and (ii) every pseudo-distribution solving Lasserre's relaxation of a~\eqref{eq:pop} 
satisfies the set of constraints in the~\eqref{eq:pop} in the sense of~\cref{def:constrainedPD}.

Now we note that~\cref{def:constrainedPD} allows some extra slack through the parameter $\satisfyOrder$, \ie 
$\pdist$ {satisfies} $\calA$ at degree $\satisfyOrder$, if every set $\calS \subset[m]$ and every sum-of-squares 
polynomial $h$ on $\Real{\dimd}$ with~$\degree{s^2} + \sum_{i\in \calS} \max\{\degree{f_i}, \satisfyOrder\}\leq \levelpd$ satisfies
$\pExpect{\pdist}{s^2 \cdot \prod_{i \in \calS} f_i } \geq 0$. 
This essentially means that the inequality $\pExpect{\pdist}{s^2 \cdot \prod_{i \in \calS} f_i } \geq 0$ is enforced for a smaller number of subsets $\calS$.

\section{Sum-of-Squares Proofs}
\label{app:sosProofs}

Sum-of-squares proofs provide an advanced way to reason about polynomial constraints and 
to infer properties of pseudo-distributions, or, equivalently, properties of 
the moment relaxation in~\cref{app:momentRelaxation}.
The presentation in this section builds on~\cite{Barak16notes-proofsBeliefsSos}, but also collects inference rules from other papers, which we cite as we present the results. 

Let us denote with $f(\vxx)$ a polynomial in variables $\vxx = [x_1;x_2;\ldots;x_\dimd]$
and let $\calA = \{f_1(\vxx) \geq 0,\ldots,f_m(\vxx) \geq 0\}$ be a system of polynomial constraints over $\Real{\dimd}$.
In the following, we omit the argument when clear from the context and write $f$ instead of $f(\vxx)$.
A polynomial $p$ is sum-of-squares (\sos) if there exist polynomials $q_1,\ldots,q_t$ 
such that $p = q^2_1 + \ldots + q_t^2$.

The key idea is to relate two sets of polynomial constraints using a ``sum-of-squares proof'' (the definition below is the same as~\cref{def:sosProof} in the main manuscript).
\begin{definition}[Sum-of-squares proof]
\sosProofDef{eq:sosProof}
\end{definition}

From eq.~\eqref{eq:sosProof}, it is clear why  the polynomials $p_\calS$ are a ``proof'' of $g \geq 0$ for any $\vxx$ satisfying $\calA$: for any $\vxx \in \calA$, $\prod_{i\in\calS}f_i \geq 0$ by definition, hence if we can write $g$ as the product of a sum-of-squares (hence non-negative) polynomial and $\prod_{i\in\calS}f_i$, we automatically prove that $g \geq 0$ whenever $\vxx \in \calA$. 

Sum-of-squares proofs allow us to deduce properties of
pseudo-distributions: in particular, if we have an \sos proof relating two sets of constraints, we
can conclude that any pseudo-distribution satisfying a set of constraints, must also satisfy the other. 
This is formalized below.

\begin{fact}[Soundness~\cite{Karmalkar19neurips-ListDecodableRegression}]
\label{fact:soundness}
Consider a level-$\levelpd$ pseudo-distribution $\pdist$ such that $\pdist \psatisfy{}{\satisfyOrder} \calA$.
 If there exists a sum-of-squares proof that $\calA \sosimply{}{\satisfyOrder'} \calB$,
 then $\pdist \psatisfy{}{\satisfyOrder\cdot \satisfyOrder' + \satisfyOrder'} \calB$.
\end{fact}

If the pseudo-distribution $\pdist$ satisfies $\calA$ only approximately, soundness continues to hold but we require an upper bound on the bit-complexity of the sum-of-squares proof 
$\calA \sosimply{}{\satisfyOrder'} \calB$ (\ie the number of bits to write down the proof). 
In this \paper, we mostly disregard bit-complexity issues and refer the reader to~\cite[\S 3]{Fleming19fnt-sosproofs} for a 
more formal discussion. 
In other words, similarly to~\cite{Klivans18arxiv-robustRegression,Karmalkar19neurips-ListDecodableRegression}, we assume that 
all numbers appearing in the input have bit complexity $\dimd^{O(1)}$ and all \sos proofs 
will have bit complexity $\dimd^{O(\levelpd)}$, which is enough to claim soundness for the \sos proof system.

Not only \sos proofs allow us to infer properties of pseudo-distributions, but also the reverse is true.
The following fact states that every property of a pseudo-distribution can be derived via a sum-of-squares proof.

\begin{fact}[Completeness~\cite{Karmalkar19neurips-ListDecodableRegression}]
Suppose $\levelpd \geq \satisfyOrder' \geq \satisfyOrder$ and $\calA$ is a system of explicitly bounded polynomial constraints with degree at most $\satisfyOrder$ (\ie $\calA \sosimply{}{} \{\normTwo{\vxx}^2 \leq M_x^2\}$ for some finite $M_x$).
Let $\{g\geq0\}$ be a polynomial constraint. If every level-$\levelpd$ pseudo-distribution that satisfies
$\pdist \psatisfy{}{\satisfyOrder} \calA$ also satisfies $\pdist \psatisfy{}{\satisfyOrder'} \calB$, then 
for every $\epsilon > 0$ there is a sum-of-squares proof $\calA \sosimply{}{\levelpd} \{g \geq -\epsilon\}$.
\end{fact}

\myParagraph{Sos rules}
Sum-of-squares provide a proof system to reasons about polynomial
constraints. For instance, if we have an \sos proof that $\calA$ implies $g \geq 0$, we may want to use 
such a proof system to infer if another implication also holds true, say $\calA$ implies $g' \geq 0$ (for some other polynomial $g'$). Reasoning in this proof system is not immediate. 
For instance, the fact that $p(\vxx) \geq 0$ for some degree-$\satisfyOrder$ polynomial does not necessarily imply that there is a sum-of-squares proof $\sosimply{\vxx}{\satisfyOrder} \{ p(\vxx) \geq 0 \}$.
Similarly, for some polynomial constraints $\calA$ and polynomials $g(\vxx)$ and $g(\vxx)'$
 with $g'(\vxx) \geq g(\vxx)$ for every $\vxx \in \Real{\dimd}$, 
the fact that $\calA \sosimply{\vxx}{\satisfyOrder} \{ g(\vxx) \geq 0 \}$ 
does not necessarily imply that $\calA \sosimply{\vxx}{\satisfyOrder} \{ g'(\vxx) \geq 0 \}$, since the latter fact might not admit a sum-of-squares proof. In this sense, the \sos proof system is more restrictive than the typical algebraic manipulation we are used to. 
Fortunately, previous work provides a toolkit of inference rules that can be used to correctly reason over \sos proofs. We collect key facts below, mostly drawing from~\cite{Karmalkar19neurips-ListDecodableRegression,Klivans18arxiv-robustRegression,
Hopkins18stoc-mixtureModelAndSoS,Ma16focs-tensorDecompositionViaSoS,Diakonikolas22pmlr-robustMeanViaSoS}.  

\begin{fact}[Inference Rules~\cite{Karmalkar19neurips-ListDecodableRegression}]
The following inference rules hold for systems of polynomial constraints $\calA, \calB, \calC$ 
and polynomials $f,g: {\Real{\dimd}}\mapsto{\Real{}}$:
\begin{align}
\text{addition:} &\quad \frac{ \calA \sosimply{}{\satisfyOrder} \{f\geq 0, g\geq 0\}  }{  \calA \sosimply{}{\satisfyOrder} \{f+g\geq 0\} } 
\\
\text{multiplication:} &\quad \frac{ \calA \sosimply{}{\satisfyOrder} \{f\geq 0\} \;,\; \calA \sosimply{}{\satisfyOrder'} \{g\geq 0\}   }{  \calA \sosimply{}{\satisfyOrder+\satisfyOrder'} \{f \cdot g\geq 0\} } 
\\
\text{transitivity:} &\quad \frac{ \calA \sosimply{}{\satisfyOrder} \calB \;,\; \calB \sosimply{}{\satisfyOrder'} \calC }{  \calA \sosimply{}{\satisfyOrder\cdot\satisfyOrder'} \calC } 
\end{align}
where, for two logical statements $\scenario{A}$ and $\scenario{B}$, we use the standard inference notation $\frac{\;\;\scenario{A}\;\;}{\;\;\scenario{B}\;\;}$
to denote that if $\scenario{A}$ is true, then $\scenario{B}$ must be true.
\end{fact}

\begin{fact}[Basics, p. 59 in~\cite{Blekherman12Book-sdpandConvexAlgebraicGeometry} and p. 70 in~\cite{Fleming19fnt-sosproofs}]
\label{fact:basics}
Let $p(\vxx)$ be a degree-$\satisfyOrder$ polynomial such that $p(\vxx) \geq 0$ for all $\vxx \in \Real{\dimd}$.
Then:
\begin{align}
\label{eq:sosbasics}
\sosimply{\vxx}{\satisfyOrder} 
\left\{ 
p(\vxx) \geq 0
 \right\}
 \end{align}
(\ie $p(\vxx)$ is sos) if: 
 \begin{itemize}
 	\item $\dimd=1$ (univariate case),
 	\item $\satisfyOrder=2$ (quadratic polynomials), or
 	\item $\dimd=2$ and $\satisfyOrder=4$ (bivariate, quartic polynomials).
 \end{itemize}
Moreover,~\eqref{eq:sosbasics} holds whenever $p$ is a function over the Boolean hypercube $p : \{0,1\}^\dimd \mapsto \Real{}$.
\end{fact}

\begin{fact}[Univariate polynomials over interval, Fact 3.7 in~\cite{Karmalkar19neurips-ListDecodableRegression}]
For any univariate degree $\satisfyOrder$ polynomial $p(x)\geq0$ for $x\in [a,b]$, 
\begin{align}
\{ x \geq a, x\leq b \}
\sosimply{x}{\satisfyOrder} 
\left\{ 
p(\vxx) \geq 0
 \right\}.
 \end{align}
\end{fact}

\begin{fact}[\Sos generalized triangle inequality, Fact 4.8 in~\cite{Klivans18arxiv-robustRegression}]
\label{fact:sosTriangleGeneralized}
For any $a_1,a_2,\ldots,a_m$
\begin{align}
\label{eq:sosTriangleGeneralized}
\sosimply{a_1,a_2,\ldots,a_m}{\satisfyOrder} \left\{ \left( \sum_{i=1}^{m} a_i\right)^\satisfyOrder \leq m^\satisfyOrder  \left( \sum_{i=1}^{m} a_i^\satisfyOrder\right)\right\}.
\end{align}
\end{fact}

\begin{fact}[\Sos triangle inequality (same as~\cref{fact:sosTriangleGeneralized} with $m=2$ and $\satisfyOrder = 2$)]
For any $a_1,a_2$
\begin{align}
\label{eq:sosTriangle}
\sosimply{a,b}{2} \left\{ \left( a + b\right)^2 \leq 2^2  a^2 + 2^2 b^2 \right\}.
\end{align}
\end{fact}

\begin{fact}[\Sos triangle inequality 2.0, p. 18 in~\cite{Klivans18arxiv-robustRegression}]
\label{fact:sosTriangleGeneralized2}
For any indeterminates $a$, $b$, scalar $\delta$, and even integer $\satisfyOrder$:
\begin{align}
\label{eq:sosTriangleGeneralized2}
\sosimply{a,b}{\satisfyOrder} \delta^\satisfyOrder a^\satisfyOrder \leq (2 \delta)^\satisfyOrder (a-b)^\satisfyOrder + (2\delta)^\satisfyOrder b^\satisfyOrder.
\end{align}
\end{fact}


\begin{fact}[\Sos squaring]
\label{fact:sossquaring}
Let $f,g$ be \sos polynomials of degree at most $\satisfyOrder$ and $\calA = \{f_1(\vxx) \geq 0,\ldots,f_m(\vxx) \geq 0\}$ be a system of polynomial inequalities. 
If $\calA \sosimply{\vxx}{\satisfyOrderTwo} \{f \geq g\}$, then $\calA \sosimply{\vxx}{\satisfyOrderTwo+\satisfyOrder} \{f^2 \geq g^2\}$.
\end{fact}

\begin{proof}
The assumption $\calA \sosimply{\vxx}{k} f = 0$ implies that:
\begin{align}
f - g = \sum_{\calS\subseteq [m]} p_\calS \cdot \prod_{i\in\calS}f_i.
\end{align}
Now note that $f^2 - g^2 = (f-g)(f+g)$ hence:
\begin{align}
\label{eq:diffSquares}
f^2 - g^2 = (f-g)(f+g) = (f+g) \sum_{\calS\subseteq [m]} p_\calS \cdot \prod_{i\in\calS}f_i.
\end{align}
Since $f,g$ are \sos polynomial, the previous relation proves $\calA \sosimply{\vxx}{\satisfyOrderTwo+\satisfyOrder} \{f^2 \geq g^2\}$ with 
\sos proof $(f+g) \cdot p_\calS$ and by noting that the maximum degree appearing in~\eqref{eq:diffSquares} is $\satisfyOrderTwo+\satisfyOrder$.
\end{proof}



\begin{fact}[\Sos triangle inequality with norms, Fact A.2 in~\cite{Hopkins18stoc-mixtureModelAndSoS}]
\label{fact:sosTriangleNorm}
Let $\vxx_1$ and $\vxx_2$ be $n$-length vectors of indeterminates. Then:
\begin{align}
\label{eq:sosTriangleNorm}
\sosimply{\vxx_1,\vxx_2}{2} \left\{ \normTwo{\vxx_1 + \vxx_2}^2 \leq 2 \normTwo{\vxx_1}^2 + 2 \normTwo{\vxx_2}^2 \right\}.
\end{align}
\end{fact}

\begin{fact}[\Sos generalized triangle inequality with norms]
\label{fact:sosTriangleNormGeneralized}
Let $\vxx_1$ and $\vxx_2$ be $n$-length vectors of indeterminates and $\satisfyOrder\in\Natural{}$ be even. Then:
\begin{align}
\label{eq:sosTriangleNormGeneralized}
\sosimply{\vxx_1,\vxx_2}{\satisfyOrder} \left\{ \normTwo{\vxx_1 + \vxx_2}^\satisfyOrder \leq 2^\satisfyOrder \normTwo{\vxx_1}^\satisfyOrder + 2^\satisfyOrder \normTwo{\vxx_2}^\satisfyOrder \right\}.
\end{align}
\end{fact}
 
\begin{proof}
\begin{align}
\sosimply{\vxx_1,\vxx_2}{\satisfyOrder}
\normTwo{\vxx_1 + \vxx_2}^\satisfyOrder = \left( \normTwo{\vxx_1 + \vxx_2}^2 \right)^{\frac{\satisfyOrder}{2}} 
 \overbrace{\leq}^{using~\eqref{eq:sosTriangleNorm}}   
 \left( 2 \normTwo{ \vxx_1 }^2 + 2 \normTwo{ \vxx_2 }^2  \right)^{\frac{\satisfyOrder}{2}} \\
 \overbrace{\leq}^{using~\eqref{eq:sosTriangleGeneralized}}  2^{\frac{\satisfyOrder}{2}}  (2 \normTwo{ \vxx_1 }^2)^{\frac{\satisfyOrder}{2}} + 2^{\frac{\satisfyOrder}{2}} (2\normTwo{ \vxx_2 })^{\frac{\satisfyOrder}{2}}
= 2^\satisfyOrder  \normTwo{ \vxx_1 }^\satisfyOrder + 2^\satisfyOrder \normTwo{ \vxx_2 }^\satisfyOrder .
\end{align}
\end{proof}

\begin{fact}[\Sos \CS, Fact A.1 in~\cite{Hopkins18stoc-mixtureModelAndSoS}] 
Let $x_1,x_2,\ldots,x_n$ and $y_1,y_2,\ldots,y_n$
be polynomials in some indeterminates. Then:
\begin{align}
\sosimply{x_1,x_2,\ldots,x_n, y_1,y_2,\ldots,y_n}{4} 
\left\{ \left( \sum_{i=1}^n x_i y_i\right)^2 
\leq  
\left( \sum_{i=1}^n x_i^2 \right)
\left( \sum_{i=1}^n y_i^2\right)
 \right\},
 \end{align}
 or, written in vector form, for two $n$-length vectors $\vxx$ and $\vy$:
 \begin{align}
\sosimply{\vxx,\vy}{4} 
\left\{ \left( \vxx\tran \vy \right)^2 
\leq  
\left( \normTwo{\vxx}^2 \right)
\left( \normTwo{\vy}^2\right)
 \right\}.
 \end{align}
\end{fact}


\begin{fact}[\Sos \HoldersInequality, Fact 4.4 in~\cite{Klivans18arxiv-robustRegression}]
\label{fact:sosHolder1}
Let $f_i,g_i$ for $1 \leq i \leq n$ be \sos polynomials. Let $p$, $q$ be integers such that $\frac{1}{p} + \frac{1}{q} = 1$. Then:
\begin{align}
\sosimply{}{pq}
\left\{
\left( \aveOverMeas f_i g_i \right)^{pq} 
\leq 
\left( \aveOverMeas f_i^{p} \right)^q 
\left( \aveOverMeas g_i^{q} \right)^p 
\right\}.
\end{align}
\end{fact}

\begin{fact}[\Sos \HoldersInequality 2.0, Fact A.6 in~\cite{Hopkins18stoc-mixtureModelAndSoS}]
\label{fact:holder2}
Let $\omega_1,\ldots,\omega_n$ and $x_1,\ldots,x_n$ be indeterminates. 
Let $q \in \Natural{}$ be a power of 2. Then:
\begin{align}
\label{eq:holder2_1}
\left\{ \omega_i^2 = \omega_i, \forall i \in [n] \right\} 
\sosimply{ \omega_1,\ldots,\omega_n, x_1,\ldots,x_n }{ O(q) } 
\left\{ \left( \sum_{i=1}^n \omega_i x_i\right)^q 
\leq  
\left( \sum_{i=1}^n \omega_i^2 \right)^{q-1}
\left( \sum_{i=1}^n x_i^q\right)
 \right\},
\end{align}
and
\begin{align}
\left\{ \omega_i^2 = \omega_i, \forall i \in [n] \right\}
\sosimply{ \omega_1,\ldots,\omega_n, x_1,\ldots,x_n }{ O(q) } 
\left\{ \left( \sum_{i=1}^n \omega_i x_i\right)^q 
\leq  
\left( \sum_{i=1}^n \omega_i^2 \right)^{q-1}
\left( \sum_{i=1}^n \omega_i x_i^q\right)
 \right\}.
\end{align}
\end{fact}

\begin{fact}[\Sos \HoldersInequality 3.0, Fact A.3 in~\cite{Diakonikolas22pmlr-robustMeanViaSoS}]
Let $f_i,g_i$ for $1 \leq i \leq n$ be indeterminates. Then:
\begin{align}
\sosimply{}{2}
\left\{
\left( \aveOverMeas f_i g_i \right)^2 
\leq 
\left( \aveOverMeas f_i^{2} \right) 
\left( \aveOverMeas g_i^{2} \right) 
\right\}.
\end{align}
\end{fact}

\begin{fact}[Lemma A.3 in~\cite{Ma16focs-tensorDecompositionViaSoS}]
\label{fact:squareBound}
Let $x$ be indeterminate and $a$ be a positive real number. Then:
\begin{align}
\{x^2 \leq a^2\} \sosimply{}{} 
\left\{ 
x\leq a, x\geq-a
 \right\}.
 \end{align}
\end{fact}

\begin{fact}[Lemma A.2 in~\cite{Ma16focs-tensorDecompositionViaSoS}]
Let $\vxx$ be indeterminate and $\va$ be a unit vector. Let $\calA \{ \|\vxx\|^2=1, (\vxx\tran \va)^2 \leq \tau \}$.
Then, for any $\vb$ such that $\normTwo{\va-\vb}\leq 2\delta$, we have:
\begin{align}
\calA \sosimply{}{} 
\left\{ 
(\vxx\tran \vb)^2\leq (\sqrt{\tau} + \sqrt{\delta})^2
 \right\}.
 \end{align}
\end{fact}

We conclude with a self-evident fact that reassures us that certain manipulations of polynomials 
are easy to reason over, even in the \sos proof system.
\begin{fact}[Equalities]
Let $f,g$ be polynomials and $\calA$ be a system of polynomial inequalities. If $f = g$ 
and $\calA \sosimply{\vxx}{\satisfyOrder} f = 0$, then $\calA \sosimply{\vxx}{\satisfyOrder}
g = 0$.
\end{fact}


\section{Proof of Proposition~\ref{thm:lowOut-aposterioriNoisyTLS}: 
A Posteriori Contract for~\eqref{eq:TLS}}
\label{app:proof-aposteriori-TLS-noisy}

We start by restating the theorem for the reader's convenience.

\noindent\fbox{
\parbox{\textwidth}{
\begin{proposition}[Restatement of Proposition~\ref{thm:lowOut-aposterioriNoisyTLS}]
Consider~\cref{prob:outlierRobustEstimation} with 
measurements ($\vy_i, \MA_i$), $i \in [\nrMeasurements]$, and
denote with $\gamma\gt$ the squared residual error of the ground truth $\vxx\gt$ over the set of inliers $\InlierSet$, 
\ie $\gamma\gt \triangleq\sum_{ i \in \InlierSet } \normTwo{ \vy_i - \MA_i \tran \vxx\gt }^2$.
Moreover, assume the measurement set contains at least $\frac{\nrMeasurements + \minDim}{2} + \frac{\gamma\gt}{\barcsq}$ 
inliers, where $\minDim$ is the 
size of a minimal set, and that every subset of $\minDim$ inliers is \nondegenerate. 
Then, for any integer $\dimJ$ such that $\minDim \leq \dimJ \leq (2\inlierRate-1)n - \frac{\gamma\gt}{\barcsq}$,
an optimal solution $\vxx\tls$ of~\eqref{eq:TLS} satisfies
\beq
\label{eq:noisyBoundTLS-statement-app}
\normTwo{\vxx\tls - \vxx\gt} \leq \naiveBoundTLS \mathcom
\eeq
where $\InlierSet\tls$ is the set of inliers selected by~\eqref{eq:TLS}, $\MA_\calJ$ is the matrix obtained by horizontally stacking all submatrices $\MA_i$ for all $i \in \calJ$, and $\sigma_\min(\cdot)$ denotes the smallest singular value of a matrix. 
Moreover, if the inliers are noiseless, \ie $\vepsilon = \zero$ 
in eq.~\eqref{eq:linModelProbStatement}, and for a sufficiently small $\barc > 0$, 
 $\vxx\tls = \vxx\gt$.
\end{proposition}
}}

\begin{proof}
Call 
$\InlierSet(\vxx\tls)$ the inlier set corresponding to an estimate 
$\vxx\tls$ (\ie $\InlierSet(\vxx\tls) \triangleq  \setdef{i\in[n]}{ \normTwo{\vy_i - \MA\tran \vxx\tls} \leq \barc }$). Moreover, define the TLS cost at $\vxx\tls$ as:
\begin{align}
f(\vxx\tls) = \sum_{ i \in \InlierSet(\vxx\tls) } \normTwo{ \vy_i - \MA_i \tran \vxx\tls }^2 + \barcsq (\nrMeasurements- |\InlierSet(\vxx\tls)|).
\end{align} 

We first prove that $|\InlierSet(\vxx\tls) | \geq \an - \frac{\gamma\gt}{\barcsq}$. 
Towards this goal, we observe that the cost evaluated at the ground truth $\vxx\gt$ is:
\begin{align} 
f(\vxx\gt) = \gamma\gt +\barcsq (\nrMeasurements- |\InlierSet|)  = \gamma\gt +\barcsq (\nrMeasurements- \an) 
\end{align}
which follows from the assumption that the inliers have squared residual error $\gamma\gt$  and there are $\an$ of them.
Now assume by contradiction that there exists an $\vxx\tls$ that solves~\eqref{eq:TLS} and is such that $|\InlierSet(\vxx\tls) | < \an - \frac{\gamma\gt}{\barcsq}$. 
Such an estimate would achieve a cost:
\begin{align}
f(\vxx\tls) = \sum_{ i \in \InlierSet(\vxx\tls) } \normTwo{ \vy_i - \MA_i \tran \vxx\tls }^2 + \barcsq (\nrMeasurements- |\InlierSet(\vxx\tls)|) \\
 > \sum_{ i \in \InlierSet(\vxx\tls) } \normTwo{ \vy_i - \MA_i \tran \vxx\tls }^2 + \barcsq 
 \left(\nrMeasurements- \an + \frac{\gamma\gt}{\barcsq} \right) 
 \\
 \geq \barcsq 
 \left(\nrMeasurements- \an + \frac{\gamma\gt}{\barcsq} \right)  = \gamma\gt + \barcsq 
 \left(\nrMeasurements- \an \right),
\end{align} 
which is larger than $f(\vxx\gt)$, hence contradicting optimality of $\vxx\tls$, and implying $|\InlierSet(\vxx\tls) | \geq \an - \frac{\gamma\gt}{\barcsq}$.

Since $|\InlierSet(\vxx\gt)| =\an$ and $|\InlierSet(\vxx\tls)| \geq \an - \frac{\gamma\gt}{\barcsq}$ then:
\beq 
|\InlierSet(\vxx\gt) \cap \InlierSet(\vxx\tls)| \overbrace{\geq}^{\text{sets overlap in $[\nrMeasurements]$}} 
2\an - n - \frac{\gamma\gt}{\barcsq}
 \overbrace{\geq}^{\text{using } \an \geq \frac{\nrMeasurements + \minDim}{2} + \frac{\gamma\gt}{\barcsq} }  \!\minDim,
\eeq

The subset of measurements $|\InlierSet(\vxx\gt) \cap \InlierSet(\vxx\tls)|$ are simultaneously solved by $\vxx\tls$ and $\vxx\gt$ (\ie are such that $\normTwo{\vy_i - \MA\tran \vxx} \leq \barc$ for both $\vxx = \vxx\gt$ and $\vxx = \vxx\tls$). 
Therefore, we can follow the same line of thoughts as in the proof of~\cref{thm:lowOut-aposterioriNoisyMC}, and prove the first claim.

In the case of noiseless inliers, $\gamma\gt = 0$ (or, equivalently, $\vy_i - \MA\tran \vxx\gt = \zero$, for all $i\in\InlierSet$) and
we can always choose $\barc > 0$ small enough such that the corresponding estimate $\vxx\tls$ 
satisfies the selected measurements exactly, \ie $\vy_i - \MA\tran \vxx\tls = \zero$. 
Therefore, we can follow the same line of the proof of~\cref{thm:lowOut-aposterioriNoisyMC} (for the case of noiseless inliers) to conclude 
$\vxx\tls = \vxx\gt$.
\end{proof}

\section{Proof of Theorem~\ref{thm:lowOut-apriori-LTS-objective}: Contract for Relaxation of~\eqref{eq:LTS-k}}
\label{app:proof-lowOut-apriori-LTS-objective}

We start by restating the theorem for the reader's convenience.

\noindent\fbox{
\parbox{\textwidth}{
\begin{theorem}[Restatement of Theorem~\ref{thm:lowOut-apriori-LTS-objective}]
\ltsObjectiveThm{eq:app-lowOut-apriori-LTS-objective}
\end{theorem}
}}



The proof is an adaptation of Lemma 5.3 in~\cite{Klivans18arxiv-robustRegression} to the 
case of vector-valued measurements. 
Let us start by clarifying all relevant notation:
\begin{align}
\axiomsLTS \doteq \left\{
	\begin{array}{cl}
	\omega_i^2 = \omega_i, \;\; i\in[\nrMeasurements] \\ 
	\sumOverMeas{\omega_i} = \alpha \nrMeasurements \\
	\omega_i \cdot(\bar{\vy}_i - {\vy}_i) = \zero \;\; i\in[\nrMeasurements]\\
	\omega_i \cdot(\bar{\MA}_i - {\MA}_i) = \zero \;\; i\in[\nrMeasurements]\\
	\vxx \in \Domain
	\end{array}
	\right\} 
	\qquad \text{(constraints in~\eqref{eq:LTS-k})}
\end{align}

\begin{align}
\{ \vy_i, \MA_i\}_{i\in[\nrMeasurements]} 
	& \qquad \text{(given measurements)}
	\\
\{\vy_i\uncorr \MA_i\uncorr\}_{i\in[\nrMeasurements]},
	& \qquad \text{(uncorrupted measurements with outliers replaced by inliers)}
     \\
\MV \triangleq \{\bar{\vy}_i \bar{\MA}_i\}_{i\in[\nrMeasurements]}, 
	& \qquad \text{(auxiliary variables in~\eqref{eq:LTS-k})}
\end{align}
\begin{align}
\err_\DistSampleInliers(\vxx) = \aveOverMeas \residualUncorri{\vxx}^2 
	& \;\; \text{(error of $\vxx$ w.r.t. uncorrupted measurements)} \label{eq:sampleErrorAtx}
\\
\errFeasible{\vomega}{\vxx}{\MV} \doteq \aveOverMeas \normTwo{ \bar{\vy}_i - \bar{\MA}_i\tran \vxx }^2 
& \;\; \text{(cost in~\eqref{eq:LTS-k} without exponent $\nicefrac{\relaxLevel}{2}$)} \mathper \label{eq:err_w_x}
\end{align}

\begin{proof}
The proof is quite involved and proceeds in two steps. First, we derive an \sos proof that states that 
any set of variables that are feasible for~\eqref{eq:LTS-k}, must also satisfy 
 a desired error bound. Then, we move to pseudo-expectations and conclude that the result of the moment relaxation must satisfy the same bound, which can be manipulated into eq.~\eqref{eq:app-lowOut-apriori-LTS-objective}.

 \myParagraph{Sos proof of robust certifiability (adapted from Lemma 5.6 in~\cite{Klivans18arxiv-robustRegression})}
 We now show that any set of variables that are feasible for~\eqref{eq:LTS-k} (\ie that satisfy the constraint set $\axiomsLTS$), must also satisfy 
 the following bound  
 \begin{align}
 \label{eq:sosRobustCertifiability}
\axiomsLTS \sosimply{ \vxx}{\relaxLevel} 
(\err_\DistSampleInliers(\vxx) - \errFeasible{\vomega}{\vxx}{\MV} )^{\frac{\relaxLevel}{2}}
\leq 
C_1(\relaxLevel,\outlierRate)
( \errFeasible{\vomega}{\vxx}{\MV} )^{\frac{\relaxLevel}{2} } 
+
C_2(\relaxLevel,\outlierRate)
\aveOverMeas \normTwo{\vy_i\uncorr - (\MA_i\uncorr)\tran \vxx\opt }^{\relaxLevel} 
\mathper
\end{align}

For a given $\vomega$ that satisfies $\axiomsLTS$, let $\vomega'$ be such that $\omega_i' = \omega_i$ iff 
$i$ is an inlier and $\omega_i' = 0$ otherwise (intuitively, $\vomega'$ is the indicator for subset of the selected measurements $\vomega$ that are inliers). 
Note that $\sumOverMeas \omega_i' \geq (1-2\outlierRate) \nrMeasurements$: this follows from the fact that the two sets, 
the selected measurements $\{i: \omega_i = 1\}$ and the set of true inliers, 
have each size $(1-\outlierRate) \nrMeasurements$, hence their intersection must contain 
at least $(1-2\outlierRate) \nrMeasurements$ measurements.
Therefore:
\begin{align}
\label{eq:overlapBound}
\explanation{ \aveOverMeas (1-\omega_i')^2 \overbrace{=}^{\text{binary variable}} \aveOverMeas (1-\omega_i') = 1 - \aveOverMeas \omega_i' \leq 1 - (1-2 \outlierRate)  \text{ hence: }} 
\axiomsLTS\overbrace{\sosimply{\vomega'}{2}}^{\cref{fact:basics}} 
\left\{ \aveOverMeas (1-\omega_i')^2 \leq 2\outlierRate \right\} \mathper
\end{align}
Moreover, by definition~\eqref{eq:sampleErrorAtx}, $\err_{\DistSampleInliers}(\vxx)$ is \sos, hence we have:
\begin{align}
\sosimply{ \vomega,\vxx}{4} 
\err_{\DistSampleInliers}(\vxx) \overbrace{=}^{\text{by definition}~\eqref{eq:sampleErrorAtx}}
\aveOverMeas \residualUncorri{\vxx}^2 \\
= \overbrace{
\aveOverMeas \omega_i' \residualUncorri{\vxx}^2
}^{\text{error for selected inliers}} +
\overbrace{
\aveOverMeas (1-\omega_i') \residualUncorri{\vxx}^2
}^{\text{error for all other uncorrupted measurements }}\mathper
\label{eq:proof-zaa1}
\end{align}

On the other hand, since $\vomega'$ only picks a subset of the selected measurements which are also inliers (\ie for which $\bar{\vy}_i = \vy_i\uncorr$ and $\bar{\MA}_i = \MA_i\uncorr$):
\begin{align}
\axiomsLTS
\sosimply{ \vomega,\vxx, \tocheck{\MV} }{4} 
\aveOverMeas \omega_i' \residualUncorri{\vxx}^2 
\overbrace{=}^{\text{by definition of $\vomega'$}} 
\aveOverMeas \omega_i'\normTwo{ \bar{\vy}_i - (\bar{\MA}_i)\tran \vxx }^2 
\\
\overbrace{\leq}^{\omega_i' \leq 1}
\aveOverMeas \normTwo{ \bar{\vy}_i - (\bar{\MA}_i)\tran \vxx }^2 
\overbrace{=}^{\text{by definition~\eqref{eq:err_w_x}}} \errFeasible{\vomega}{\vxx}{\MV}
\mathper
\label{eq:proof-aa1}
\end{align}
Combining~\eqref{eq:proof-zaa1} and~\eqref{eq:proof-aa1}, elevating to $\relaxLevel/2$, and then using the \sos version of \HoldersInequality (\cref{fact:sosHolder1}), we get: 
\begin{align}
\axiomsLTS
\sosimply{ {\vomega},\vxx, \tocheck{\MV} }{ \tocheck{ \relaxLevel } } 
(\err_\DistSampleInliers(\vxx) - \errFeasible{\vomega}{\vxx}{\MV} )^{\frac{\relaxLevel}{2}}
\overbrace{\leq}^{using~\eqref{eq:proof-aa1}} 
\left( 
\err_\DistSampleInliers(\vxx) - \aveOverMeas \omega_i' \residualUncorri{\vxx}^2
\right)^{\frac{\relaxLevel}{2}}
\\
\overbrace{=}^{using~\eqref{eq:proof-zaa1}} 
\overbrace{\left(
\aveOverMeas (1-\omega_i') \residualUncorri{\vxx}^2
\right)^{\frac{\relaxLevel}{2}}}^{ \text{error for uncorrupted measurements not selected by $\vomega'$} }
\\
\hspace{-4mm}
\leq 
\overbrace{
\tocheck{
\left( \aveOverMeas (1-\omega_i')^2 \right)^{\frac{\relaxLevel}{2}-1}
\left( \aveOverMeas \residualUncorri{\vxx}^\relaxLevel\right)
}
}^{\text{using~\eqref{eq:holder2_1} in~\cref{fact:holder2}}}
\leq 
\overbrace{
(2\outlierRate)^{\frac{\relaxLevel}{2}-1}
}^{\text{using~\eqref{eq:overlapBound}}}
\left( \aveOverMeas \residualUncorri{\vxx}^\relaxLevel \right) \mathper
\label{eq:proof13}
\end{align}
Note that~\cref{fact:holder2} requires the exponent to be $\geq 1$ and  a power of $2$, which in turns 
implies $\frac{\relaxLevel}{2} \geq 2$ or $\relaxLevel \geq 4$, as required by the statement of the theorem. 
Now we observe that:
\begin{align}
\sosimply{ \vxx }{\relaxLevel}
\left( \aveOverMeas \residualUncorri{\vxx}^{\relaxLevel} \right)
\\ 
\overbrace{=}^{\text{adding/subtracting $(\MA_i\uncorr)\tran \vxx\opt$}}
\left( \aveOverMeas \normTwo{ \vy_i\uncorr - (\MA_i\uncorr)\tran \vxx + (\MA_i\uncorr)\tran \vxx\opt - (\MA_i\uncorr)\tran \vxx\opt }^{\relaxLevel} \right) =
\\
\overbrace{\leq}^{\text{using~\cref{fact:sosTriangleNormGeneralized}}}
2^{\relaxLevel}  \aveOverMeas \normTwo{ \vy\uncorr_i - (\MA_i\uncorr)\tran \vxx\opt }^{\relaxLevel}
+
2^{\relaxLevel} \aveOverMeas \normTwo{ (\MA_i\uncorr)\tran (\vxx - \vxx\opt)}^{\relaxLevel} 
\mathper \label{eq:proof11}
\end{align}

By certifiable hypercontractivity of $\MA_i\uncorr$, $i\in[\nrMeasurements]$:
\begin{align}
\axiomsLTS
\sosimply{ \vxx }{\relaxLevel} \aveOverMeas \normTwo{ (\MA_i\uncorr)\tran (\vxx - \vxx\opt)}^{\relaxLevel} 
 \leq \Crelax^{\frac{\relaxLevel}{2}} \left( \aveOverMeas \normTwo{ (\MA_i\uncorr)\tran (\vxx - \vxx\opt)}^2 \right)^{\frac{\relaxLevel}{2} } \mathper
 \label{eq:proof-aa2}
\end{align}
We can further bound the term above as follows:
\begin{align}
\axiomsLTS
\sosimply{ \vxx }{\relaxLevel}
\left( \aveOverMeas \normTwo{ (\MA_i\uncorr)\tran (\vxx - \vxx\opt)}^2 \right)^{\frac{\relaxLevel}{2} }
\\
\overbrace{=}^{\text{adding/subtracting $\vy_i\uncorr$}}
\left( \aveOverMeas \normTwo{ -\vy_i\uncorr + (\MA_i\uncorr)\tran \vxx + \vy_i\uncorr - (\MA_i\uncorr)\tran \vxx\opt }^2 \right)^{\frac{\relaxLevel}{2} } 
\\
\overbrace{\leq}^{\text{using~\cref{fact:sosTriangleNorm}}}  
\left( \aveOverMeas \left(2 \normTwo{ \vy_i\uncorr - (\MA_i\uncorr)\tran \vxx}^2 + 2 \normTwo{\vy_i\uncorr - (\MA_i\uncorr)\tran \vxx\opt }^2 \right) \right)^{\frac{\relaxLevel}{2} }
\end{align}
\begin{align}
\overbrace{=}^{\text{rearranging}} 
\left( 2 \aveOverMeas \normTwo{ \vy_i\uncorr - (\MA_i\uncorr)\tran \vxx}^2 + 2 \aveOverMeas  \normTwo{\vy_i\uncorr - (\MA_i\uncorr)\tran \vxx\opt }^2 \right)^{\frac{\relaxLevel}{2} }
\\
 \overbrace{\leq}^{\text{using~\cref{fact:sosTriangleGeneralized}}} 
 2^{\frac{\relaxLevel}{2}} \left( 2\aveOverMeas \normTwo{ \vy_i\uncorr - (\MA_i\uncorr)\tran \vxx}^2 \right)^{\frac{\relaxLevel}{2} }
+
2^{\frac{\relaxLevel}{2}} \left( 2 \aveOverMeas \normTwo{\vy_i\uncorr - (\MA_i\uncorr)\tran \vxx\opt }^2 \right)^{\frac{\relaxLevel}{2} } 
\\
 \overbrace{=}^{\text{rearranging}} 
 2^{\relaxLevel} \left( \aveOverMeas \normTwo{ \vy_i\uncorr - (\MA_i\uncorr)\tran \vxx}^2 \right)^{\frac{\relaxLevel}{2} }
+
 2^{\relaxLevel} \left( \aveOverMeas \normTwo{\vy_i\uncorr - (\MA_i\uncorr)\tran \vxx\opt }^2 \right)^{\frac{\relaxLevel}{2} } \mathper \label{eq:proof12}
\end{align}


Finally, using again the \sos version of \HoldersInequality
\begin{align}
\axiomsLTS
\sosimply{ \vxx }{\relaxLevel}
\left( \aveOverMeas \normTwo{\vy_i\uncorr - (\MA_i\uncorr)\tran \vxx\opt }^2 \right)^{ \frac{\relaxLevel}{2} } 
\leq 
\tocheck{
 \aveOverMeas \normTwo{\vy_i\uncorr - (\MA_i\uncorr)\tran \vxx\opt }^{\relaxLevel} \mathper
 }\label{eq:proof-aa5}
\end{align}

Combining the above: 
\begin{align}
\axiomsLTS
\sosimply{ \vxx }{\relaxLevel}
\left( \aveOverMeas \residualUncorri{\vxx}^{\relaxLevel} \right) 
\\
\overbrace{\leq}^{\eqref{eq:proof11}} 
2^{\relaxLevel}  \aveOverMeas \normTwo{ \vy\uncorr_i - (\MA_i\uncorr)\tran \vxx\opt }^{\relaxLevel}
+
2^{\relaxLevel} \aveOverMeas \normTwo{ (\MA_i\uncorr)\tran (\vxx - \vxx\opt)}^{\relaxLevel}
\\
\overbrace{\leq}^{\eqref{eq:proof-aa2}} 
2^{\relaxLevel}  \aveOverMeas \normTwo{ \vy\uncorr_i - (\MA_i\uncorr)\tran \vxx\opt }^{\relaxLevel}
+
2^{\relaxLevel} \Crelax^{\frac{\relaxLevel}{2}} 
\left( \aveOverMeas \normTwo{ (\MA_i\uncorr)\tran (\vxx - \vxx\opt)}^2 
\right)^{\frac{\relaxLevel}{2} }
\\
\overbrace{\leq}^{\eqref{eq:proof12}} 
2^{\relaxLevel}  \aveOverMeas \normTwo{ \vy\uncorr_i - (\MA_i\uncorr)\tran \vxx\opt }^{\relaxLevel}
\\
+
2^{\relaxLevel} \Crelax^{\frac{\relaxLevel}{2}} 
\left(
2^{\relaxLevel} \left( \aveOverMeas \normTwo{ \vy_i\uncorr - (\MA_i\uncorr)\tran \vxx}^2 \right)^{\frac{\relaxLevel}{2} }
+
2^{\relaxLevel} \left( \aveOverMeas \normTwo{\vy_i\uncorr - (\MA_i\uncorr)\tran \vxx\opt }^2 \right)^{\frac{\relaxLevel}{2} }
\right)
\\
\overbrace{\leq}^{\eqref{eq:proof-aa5}} 
2^{\relaxLevel}  \aveOverMeas \normTwo{ \vy\uncorr_i - (\MA_i\uncorr)\tran \vxx\opt }^{\relaxLevel}
\\
+
2^{\relaxLevel} \Crelax^{\frac{\relaxLevel}{2}} 
\left(
2^{\relaxLevel} \left( \aveOverMeas \normTwo{ \vy_i\uncorr - (\MA_i\uncorr)\tran \vxx}^2 \right)^{\frac{\relaxLevel}{2} }
+
2^{\relaxLevel} 
\aveOverMeas \normTwo{\vy_i\uncorr - (\MA_i\uncorr)\tran \vxx\opt }^{\relaxLevel}
\right)
\\
\overbrace{=}^{\text{rearranging}} 
 \Crelax^{\frac{\relaxLevel}{2}} 
2^{2\relaxLevel} \bigg( 
\overbrace{\aveOverMeas \normTwo{ \vy_i\uncorr - (\MA_i\uncorr)\tran \vxx}^2}^{\err_\DistSampleInliers(\vxx)}
 \bigg)^{\frac{\relaxLevel}{2} }
\\
+
\left(2^{\relaxLevel} + \Crelax^{\frac{\relaxLevel}{2}} 
2^{2\relaxLevel} \right)
\aveOverMeas \normTwo{\vy_i\uncorr - (\MA_i\uncorr)\tran \vxx\opt }^{\relaxLevel} \mathper
\end{align}
Hence, together with~\eqref{eq:proof13}:
\begin{align}
\axiomsLTS
\sosimply{ \vxx}{\relaxLevel} 
(\err_\DistSampleInliers(\vxx) - \errFeasible{\vomega}{\vxx}{\MV} )^{\frac{\relaxLevel}{2}}
\leq 
(2\outlierRate)^{\frac{\relaxLevel}{2}-1} 
 \Crelax^{\frac{\relaxLevel}{2}} 
2^{2\relaxLevel} 
\err_\DistSampleInliers(\vxx)^{\frac{\relaxLevel}{2} }
\\
+
(2\outlierRate)^{\frac{\relaxLevel}{2}-1} 
\left(2^{\relaxLevel} + \Crelax^{\frac{\relaxLevel}{2}} 
2^{2\relaxLevel}  \right)
\aveOverMeas \normTwo{\vy_i\uncorr - (\MA_i\uncorr)\tran \vxx\opt }^{\relaxLevel} \mathper
\end{align}

Applying~\cref{fact:sosTriangleGeneralized2} 
 to the right-hand-side with $a = \err_\DistSampleInliers(\vxx)$, 
$b = \errFeasible{\vomega}{\vxx}{\MV}$, \\
$\delta = (2\outlierRate)^{\frac{\relaxLevel}{2}-1}
 \Crelax^{\frac{\relaxLevel}{2}} 
2^{2\relaxLevel}$, and exponent $\relaxLevel/2$: 
\begin{align}
\axiomsLTS
\sosimply{ \vxx}{\relaxLevel} 
(\err_\DistSampleInliers(\vxx) - \errFeasible{\vomega}{\vxx}{\MV} )^{\frac{\relaxLevel}{2}}
\leq 
\overbrace{2^{\frac{\relaxLevel}{2}}
(2\outlierRate)^{\frac{\relaxLevel}{2}-1} 
 \Crelax^{\frac{\relaxLevel}{2}} 
2^{2\relaxLevel} 
}^{ = \coeffCk  }
(\err_\DistSampleInliers(\vxx) - \errFeasible{\vomega}{\vxx}{\MV} )^{\frac{\relaxLevel}{2} } 
\\
+ 
2^{\frac{\relaxLevel}{2}}
(2\outlierRate)^{\frac{\relaxLevel}{2}-1} 
 \Crelax^{\frac{\relaxLevel}{2}} 
2^{2\relaxLevel} 
( \errFeasible{\vomega}{\vxx}{\MV} )^{\frac{\relaxLevel}{2} } 
\\
+
(2\outlierRate)^{\frac{\relaxLevel}{2}-1} 
\left(2^{\relaxLevel} + \Crelax^{\frac{\relaxLevel}{2}} 
2^{2\relaxLevel}  \right)
\aveOverMeas \normTwo{\vy_i\uncorr - (\MA_i\uncorr)\tran \vxx\opt }^{\relaxLevel} \mathper
\end{align}
Rearranging the terms:
\begin{align}
\axiomsLTS
\sosimply{ \vxx}{\relaxLevel} 
({1-\coeffCk})(\err_\DistSampleInliers(\vxx) - \errFeasible{\vomega}{\vxx}{\MV} )^{\frac{\relaxLevel}{2}}
\leq 
\coeffCk
( \errFeasible{\vomega}{\vxx}{\MV} )^{\frac{\relaxLevel}{2} } 
\\
+
(2\outlierRate)^{\frac{\relaxLevel}{2}-1} 
\left(2^{\relaxLevel} + \Crelax^{\frac{\relaxLevel}{2}} 
2^{2\relaxLevel}  \right)
\aveOverMeas \normTwo{\vy_i\uncorr - (\MA_i\uncorr)\tran \vxx\opt }^{\relaxLevel} \label{eq:ineq1111}
\mathper
\end{align}
Noting that choosing $\outlierRate < \sqrt[\frac{k}{2}-1]{\frac{1}{ \Crelax^{\frac{\relaxLevel}{2}} 2^{ 3\relaxLevel-1} }}$ makes the constant $1-\coeffCk$ positive, we can divive both members of the inequality~\eqref{eq:ineq1111} by such constant and obtain:
\begin{align}
\axiomsLTS
\sosimply{ \vxx}{\relaxLevel} 
(\err_\DistSampleInliers(\vxx) - \errFeasible{\vomega}{\vxx}{\MV} )^{\frac{\relaxLevel}{2}}
\leq 
\frac{\coeffCk}{1-\coeffCk}
( \errFeasible{\vomega}{\vxx}{\MV} )^{\frac{\relaxLevel}{2} } 
\\
+
\frac{
(2\outlierRate)^{\frac{\relaxLevel}{2}-1} 
\left(2^{\relaxLevel} + \Crelax^{\frac{\relaxLevel}{2}} 
2^{2\relaxLevel}  \right)
}{1-\coeffCk}
\aveOverMeas \normTwo{\vy_i\uncorr - (\MA_i\uncorr)\tran \vxx\opt }^{\relaxLevel} \mathcom
\end{align}
which matches our claim in~\eqref{eq:sosRobustCertifiability}
 for $C_1(\relaxLevel,\outlierRate)\triangleq \frac{\coeffCk}{1-\coeffCk}$ 
 and $C_2(\relaxLevel,\outlierRate)\triangleq \frac{
(2\outlierRate)^{\frac{\relaxLevel}{2}-1} 
\left(2^{\relaxLevel} + \Crelax^{\frac{\relaxLevel}{2}} 
2^{2\relaxLevel}  \right)
}{1-\coeffCk}$.


 \myParagraph{Completing the proof by moving to pseudo-distributions}
Consider a pseudo-distribution $\pdist$ that satisfies $\axiomsLTS$. 
Using the \sos proof in~\eqref{eq:sosRobustCertifiability} and thanks to~\cref{fact:soundness}, we conclude that 
if $\pdist$ satisfies $\axiomsLTS$ then it must also satisfy:
 \begin{align}
\pExpect{\pdist}{ ( \err_{\DistSampleInliers(\vxx)} - \errFeasible{\vomega}{\vxx}{\MV} )^{\frac{\relaxLevel}{2}} }
\leq 
C_1(\relaxLevel,\outlierRate) \hspace{-3mm}
\overbrace{ \pExpect{\pdist}{\errFeasible{\vomega}{\vxx}{\MV}^{\frac{\relaxLevel}{2}}} }^{\text{by definition this is }\opttsos^{\frac{\relaxLevel}{2}}}
\hspace{-5mm}
+
C_2(\relaxLevel,\outlierRate)
\left(\frac{1}{\nrMeasurements} \sumOverMeas 
\normTwo{\vy_i\uncorr - (\MA_i\uncorr)\tran \vxx\opt}^{\relaxLevel} \right) \mathper
\end{align}
Elevating to the power $\frac{2}{\relaxLevel}$ both sides and recalling that 
{$(a+b)^{q} \leq {a}^{q} + {b}^q$ for any $0 < q < 1$}:
 \begin{align}
\left(
 \pExpect{\pdist}{ ( \err_{\DistSampleInliers(\vxx)} - \errFeasible{\vomega}{\vxx}{\MV} )^{\frac{\relaxLevel}{2} }  } 
 \right)^{\frac{2}{\relaxLevel}}
\leq 
C_1(\relaxLevel,\outlierRate)^{\frac{2}{\relaxLevel}}
\; \opttsos+
C_2(\relaxLevel,\outlierRate)^{\frac{2}{\relaxLevel}}
\left(\frac{1}{\nrMeasurements} \sumOverMeas 
\normTwo{\vy_i\uncorr - (\MA_i\uncorr)\tran \vxx\opt}^{\relaxLevel} \right)^{\frac{2}{\relaxLevel}} \mathper
\label{eq:proof-aa1a}
\end{align}
Now using the \sos version of \HoldersInequality for pseudo-expectations (\cref{fact:Holderforpdist2}, eq.~\eqref{eq:Holderforpdist2b}):
\begin{align} 
\pExpect{\pdist}{ \err_{\DistSampleInliers(\vxx)} - \errFeasible{\vomega}{\vxx}{\MV} }^{\frac{\relaxLevel}{2}} \leq
\pExpect{\pdist}{ ( \err_{\DistSampleInliers(\vxx)} - \errFeasible{\vomega}{\vxx}{\MV} )^{\frac{\relaxLevel}{2}}} \mathcom
\end{align}
and therefore~\eqref{eq:proof-aa1a} becomes:
\begin{align}
\pExpect{\pdist}{ ( \err_{\DistSampleInliers(\vxx)} - \errFeasible{\vomega}{\vxx}{\MV} ) }
\leq 
C_1(\relaxLevel,\outlierRate)^{\frac{2}{\relaxLevel}}
\; \opttsos+
C_2(\relaxLevel,\outlierRate)^{\frac{2}{\relaxLevel}}
\left(\frac{1}{\nrMeasurements} \sumOverMeas 
\normTwo{\vy_i\uncorr - (\MA_i\uncorr)\tran \vxx\opt}^{\relaxLevel} \right)^{\frac{2}{\relaxLevel}} \mathper
\end{align}
By linearity of the pseudo-expectation and rearranging:
\begin{align}
\pExpect{\pdist}{ \err_{\DistSampleInliers(\vxx)} }
\leq 
\pExpect{\pdist}{ \errFeasible{\vomega}{\vxx}{\MV} }
+
C_1(\relaxLevel,\outlierRate)^{\frac{2}{\relaxLevel}}
\; \opttsos+
C_2(\relaxLevel,\outlierRate)^{\frac{2}{\relaxLevel}}
\left(\frac{1}{\nrMeasurements} \sumOverMeas 
\normTwo{\vy_i\uncorr - (\MA_i\uncorr)\tran \vxx\opt}^{\relaxLevel} \right)^{\frac{2}{\relaxLevel}}
\\
(\text{using: } 
\pExpect{\pdist}{ \errFeasible{\vomega}{\vxx}{\MV} } = 
\left( \pExpect{\pdist}{ \errFeasible{\vomega}{\vxx}{\MV} }^{\frac{\relaxLevel}{2}} \right)^{\frac{2}{\relaxLevel}} 
\overbrace{\leq}^{\text{\cref{fact:Holderforpdist2}, eq.~\eqref{eq:Holderforpdist2b}}}  
\left( \pExpect{\pdist}{ \errFeasible{\vomega}{\vxx}{\MV}^{\frac{\relaxLevel}{2}} } \right)^{\frac{2}{\relaxLevel}}  
= \opttsos 
\text{)}
\\
\pExpect{\pdist}{ \err_{\DistSampleInliers(\vxx)} }
\leq
(1 + C_1(\relaxLevel,\outlierRate)^{\frac{2}{\relaxLevel}} )
\;  \opttsos+
C_2(\relaxLevel,\outlierRate)^{\frac{2}{\relaxLevel}}
\left(\frac{1}{\nrMeasurements} \sumOverMeas 
\normTwo{\vy_i\uncorr - (\MA_i\uncorr)\tran \vxx\opt}^{\relaxLevel} \right)^{\frac{2}{\relaxLevel}} 
\end{align}
Applying \sos \HoldersInequality one last time $  \err_{\DistSampleInliers} (\pExpect{\pdist}{ \vxx} )  \leq \pExpect{\pdist}{ \err_{\DistSampleInliers}(\vxx) } $ leads to:
\begin{align}
\err_{\DistSampleInliers} (\pExpect{\pdist}{ \vxx} )
\leq
(1 + C_1(\relaxLevel,\outlierRate)^{\frac{2}{\relaxLevel}} )
\;  \opttsos+
C_2(\relaxLevel,\outlierRate)^{\frac{2}{\relaxLevel}}
\left(\frac{1}{\nrMeasurements} \sumOverMeas 
\normTwo{\vy_i\uncorr - (\MA_i\uncorr)\tran \vxx\opt}^{\relaxLevel} \right)^{\frac{2}{\relaxLevel}} \mathper
\label{eq:proof-aa-2b}
\end{align}
Finally, we need to prove that {$\opttsos \leq \optt_{\DistSampleInliers}$.
Towards this goal, we observe that the \mbox{(pseudo-)distribution} 
supported on the point $(\vomega\opt,\vxx\opt, \MV)$ where $\omega_i\opt=1$ for the true inliers and zero otherwise is feasible for $\axiomsLTS$, hence by optimality $\opttsos^{\frac{\relaxLevel}{2}} \leq \left( \aveOverMeas \normTwo{\bar{\vy}_i - \bar{\MA}_i\tran \vxx\opt}^2\right)^{\frac{\relaxLevel}{2}}$, from which it follows:
\begin{align}
\opttsos \leq  \aveOverMeas \normTwo{\bar{\vy}_i - \bar{\MA}_i\tran \vxx\opt}^2
\overbrace{\leq}^{\text{$\bar{\vy}_i,\bar{\MA}_i$ are inliers or zero}} 
\aveOverMeas \residualUncorri{\vxx\opt} = \optt_{\DistSampleInliers} \mathper
\label{eq:proof-aa-2a}
\end{align}
Substituting~\eqref{eq:proof-aa-2a} back into~\eqref{eq:proof-aa-2b}: 
}
\begin{align}
\err_{\DistSampleInliers} (\pExpect{\pdist}{ \vxx} )
\leq
(1 + C_1(\relaxLevel,\outlierRate)^{\frac{2}{\relaxLevel}} )
\;  \optt_{\DistSampleInliers}+
C_2(\relaxLevel,\outlierRate)^{\frac{2}{\relaxLevel}}
\left(\frac{1}{\nrMeasurements} \sumOverMeas 
\normTwo{\vy_i\uncorr - (\MA_i\uncorr)\tran \vxx\opt}^{\relaxLevel} \right)^{\frac{2}{\relaxLevel}} \mathcom
\end{align}
which proves the claim of~\cref{thm:lowOut-apriori-LTS-objective}.

\end{proof}

\section{Proof of Proposition~\ref{thm:lowOut-apriori-LTS}: Contract for Relaxation of~\eqref{eq:LTS-T}}
\label{app:proof-lowOut-apriori-LTS}

We start by restating the proposition for the reader's convenience.

\noindent\fbox{
\parbox{\textwidth}{
\begin{proposition}[Restatement of Proposition~\ref{thm:lowOut-apriori-LTS}]
\ltsThm{eq:proof-lowOut-apriori-LTS}
\end{proposition}
}}

Towards proving the proposition we need to prove two technical lemmas (\cref{eq:boundxxgt,lem:exists_vi} below). These lemmas extend results~\cite{Karmalkar19neurips-ListDecodableRegression} 
 to vector-valued and noisy measurements, and while in~\cite{Karmalkar19neurips-ListDecodableRegression}  they have been proposed to attack the high-outlier case (\ie for list-decodable regression), we show they are also useful to 
 prove estimation contracts for the low-outlier case.
 
 Note that the two lemmas below use a subset of constraints compared to the one in the constraint set of~\eqref{eq:LTS-T} (\ie the set $\axiomsMC$ below does not contain the constraint $\sumOverMeas{\omega_i} = \alpha \nrMeasurements$): this will allow us to use them also to discuss the performance of~\eqref{eq:MC} and~\eqref{eq:TLS} later on.

\begin{lemma}[Adapted from Lemma 4.1 in~\cite{Karmalkar19neurips-ListDecodableRegression}]
\label{eq:boundxxgt}
Consider the following constraint set, for given measurements $(\vy_i,\MA_i)$, $i \in [\nrMeasurements]$, 
a constant $\barc \geq 0$, and where $\Domain$ is an explicitly bounded basic semi-algebraic set (\cf~\cref{ass:explicitlyBoundedxs}): 
\begin{align}
\axiomsMC \doteq \left\{
    \begin{array}{cl}
    \omega_i^2 = \omega_i, \;\; i\in[\nrMeasurements] \\ 
    \omega_i \cdot \residuali{\vxx}^2 \leq \barc^2 \;\; i\in[\nrMeasurements]\\
    \vxx \in \Domain
    \end{array}
    \right\} \mathper
\end{align}

For any $t \geq {k}$ and set of $\nrMeasurements$ measurements with at least $\an$ inliers, such that for the set of inliers $\setInliers$, the set of matrices $\MA_i$, $i\in\setInliers$, is
$k$-certifiably $\antiConReq$-anti-concentrated,\footnote{The constant $``32''$ in the anti-concentration requirement is arbitrary (\ie it just amounts to a re-scaling of the parameter $\eta$) and has been chosen to keep the result consistent with the original statement in~\cite{Karmalkar19neurips-ListDecodableRegression}. }
\begin{align}
\axiomsMC \sosimply{\vomega,\vxx}{t}
\left\{
\aveOverInliers \omega_i \normTwo{\vxx - \vxx\gt}^2 \leq \frac{\inlierRate^2 \eta^2 \boundx^2}{4}
\right\}.
\end{align}
\end{lemma}

\begin{proof} 
We follow the same logic as the proof of Lemma 4.1 in~\cite{Karmalkar19neurips-ListDecodableRegression}, but provide a slightly simpler derivation, based on our definition of certifiable anti-concentration. 
We first observe that for the inliers (\ie $i \in \InlierSet$) it holds:\footnote{Observe the analogy with the proof of~\cref{thm:lowOut-aposterioriNoisyMC}.}
\begin{align}
\axiomsMC \sosimply{\vxx}{t} 
\omega_i \cdot\normTwo{\MA_i\tran (\vxx-\vxx\gt)}^2
\overbrace{=}^{\text{adding/subtracting $\vy_i$}} 
\omega_i \cdot\normTwo{(\vy_i - \MA_i\tran \vxx\gt) - (\vy_i - \MA_i\tran \vxx)}^2
\\
\overbrace{\leq}^{\cref{fact:sosTriangleNorm}}
\omega_i \cdot \left(2 \normTwo{\vy_i - \MA_i\tran \vxx\gt}^2
+ 2 \normTwo{\vy_i - \MA_i\tran \vxx}^2 \right) 
\\
\overbrace{\leq}^{\omega_i \leq 1 \text{ and~\cref{fact:basics} }}
2 \normTwo{\vy_i - \MA_i\tran \vxx\gt}^2
+ 2 \cdot \omega_i \cdot  \normTwo{\vy_i - \MA_i\tran \vxx}^2 
\\
\explanation{\text{(since $(\vomega,\vxx)$ satisfy $\axiomsMC$, and inliers by definition satisfy
$\normTwo{\vy_i - \MA_i\tran \vxx\gt}^2 \leq \barcsq$)}}
\leq 
4 \barcsq \mathper
\label{eq:boundAtxx}
\end{align}

Since the set of matrices $\MA_i$, $i\in\setInliers$, is $k$-certifiably
$\antiConReq$-anti-concentrated, then there exists
 a univariate polynomial $p$ such that for every $i\in\setInliers$ and for every $t \geq k$:
 \begin{align}
\overbrace{
\left\{ 
\omega_i \cdot \normTwo{\MA_i\tran (\vxx-\vxx\gt)}^2 \leq 4\barcsq 
\right\} 
}^{ \text{from~\eqref{eq:boundAtxx}} }
\overbrace{
\sosimply{\vxx}{t} 
 p^2\left(\omega_i \cdot \normTwo{\MA_i\tran (\vxx-\vxx\gt)} \right) \geq (1-2\barc)^2
}^{ \text{using \eqref{eq:certAntiConSet1} with $\delta=2\barc$} }
\label{eq:proof21}
 \end{align}
 and 
 \begin{align}
\normTwo{\vxx}^2 \leq \boundx
 \overbrace{
\sosimply{\vxx}{t} \normTwo{\vxx - \vxx\gt}^2 \leq 4 \boundx^2
}^{\cref{fact:sosTriangleNorm} }
\\
 \overbrace{
\sosimply{\vxx}{t} 
\left\{ \normTwo{\vxx - \vxx\gt}^2 \cdot  \aveOverInliers p^2\left( \normTwo{\MA_i\tran (\vxx-\vxx\gt)} \right) \leq \frac{\inlierRate^2\eta^2 (1-2\barc)^2 \boundx^2}{4} \right\} \mathper
}^{\text{using certifiable anti-concentration in eq.~\eqref{eq:certAntiConSet2} }}
\label{eq:proof22}
\end{align}

%
Combining the conclusions in~\eqref{eq:proof21} and~\eqref{eq:proof22}, we obtain:
\begin{align}
\axiomsMC
\sosimply{\vomega,\vxx}{t}
\aveOverInliers \omega_i \normTwo{\vxx-\vxx\gt}^2 
\overbrace{
\leq \aveOverInliers \omega_i \normTwo{\vxx-\vxx\gt}^2 \cdot \frac{1}{ (1-2\barc)^2 } p^2\left(\omega_i  \cdot \normTwo{\MA_i\tran (\vxx-\vxx\gt)} \right)
}^{\text{ since $\frac{1}{ (1-2\barc)^2 }p^2\left(\omega_i  \cdot \normTwo{\MA_i\tran (\vxx-\vxx\gt)} \right) \geq 1$ per~\eqref{eq:proof21} }}
\\ 
\overbrace{
\leq  \frac{1}{ (1-2\barc)^2 }\normTwo{\vxx-\vxx\gt}^2 \cdot \aveOverInliers  p^2\left( \normTwo{\MA_i\tran (\vxx-\vxx\gt)} \right)
}^{\text{$\omega_i$ is binary, hence $ \omega_i  \cdot  p^2\left(\omega_i  \cdot \normTwo{\MA_i\tran (\vxx-\vxx\gt)} \right) \leq p^2\left(\normTwo{\MA_i\tran (\vxx-\vxx\gt)} \right) $ }}
\overbrace{
 \leq \frac{\inlierRate^2 \eta^2 \boundx^2}{4}
 }^{\text{using~\eqref{eq:proof22}}} \mathcom
\end{align}
which concludes the proof.
\end{proof} 
 

\begin{lemma}[Adapted from Lemma 4.2 in~\cite{Karmalkar19neurips-ListDecodableRegression}] \label{lem:exists_vi}
Under the same assumptions of~\cref{eq:boundxxgt}, 
for any pseudo-distribution $\pdist$ of level at least ${k}$ satisfying $\axiomsMC$,
\begin{align} 
\aveOverInliers \pExpect{\pdist}{\omega_i} \normTwo{\vv_i-\vxx\gt} \leq \frac{\inlierRate \eta \boundx}{2},
\end{align}
where the vectors $\vv_i$ are extracted from the pseudo-moment matrix by setting $\vv_i = \frac{\pExpect{\pdist}{\omega_i\vxx}}{\pExpect{\pdist}{\omega_i}}$ if $\pExpect{\pdist}{\omega_i} > 0$, or $\vv_i = \zero$ otherwise, for $i\in[\nrMeasurements]$.
\end{lemma}
\begin{proof}
By~\cref{eq:boundxxgt}, we have $\axiomsMC\sosimply{\vomega,\vxx}{k}
\left\{
\aveOverInliers \omega_i \normTwo{\vxx - \vxx\gt}^2 \leq \frac{\inlierRate^2 \eta^2 \boundx^2}{4}
\right\}$.
We also have: $\axiomsMC
\sosimply{\vomega}{2} \{\omega_i^2 = \omega_i\}$ for any $i$. Therefore:
\begin{align}
\axiomsMC\sosimply{\vomega,\vxx}{k}
\left\{
\aveOverInliers \normTwo{ \omega_i\vxx -  \omega_i \vxx\gt}^2 \leq \frac{\inlierRate^2 \eta^2 \boundx^2}{4}
\right\} \mathper
\end{align}
Since $\pdist$ satisfies $\axiomsMC$, then it also satisfies:
\begin{align}
\aveOverInliers \pExpect{\pdist}{  
\normTwo{ \omega_i\vxx -  \omega_i \vxx\gt}^2} 
\leq \frac{\inlierRate^2 \eta^2 \boundx^2}{4} \mathper
\label{eq:proof51}
\end{align}
Using the norm inequality for pseudo-distributions in~\cref{fact:normIneq-pdist}, we get $\normTwo{ \pExpect{\pdist}{ \omega_i\vxx 
 -   \omega_i \vxx\gt } }^2 \leq \pExpect{\pdist}{  
\normTwo{ \omega_i\vxx -  \omega_i \vxx\gt}^2} $; then 
observing that for any $m$-vector $\vz$, $\normOne{\vz} \leq \sqrt{m} \normTwo{\vz}$ or, equivalently, $\normOne{\vz}^2 \leq m \normTwo{\vz}^2$ (below we will apply this inequality to the vector of size $|\setInliers|$ with entries $z_i = \frac{1}{|\setInliers|} \normTwo{ \pExpect{\pdist}{ \omega_i\vxx } -  \pExpect{\pdist}{ \omega_i } \vxx\gt }$),
and chaining the inequalities back to~\eqref{eq:proof51}: 
\begin{align}
\overbrace{
\left( 
\aveOverInliers   
\normTwo{ \pExpect{\pdist}{ \omega_i\vxx } -  \pExpect{\pdist}{ \omega_i } \vxx\gt }
\right)^2
}^{\normOne{\vz}^2}
\hspace{3mm}
\overbrace{\leq}^{\normOne{\vz}^2 \leq m \normTwo{\vz}^2}
\hspace{3mm} 
\overbrace{
\aveOverInliers 
\normTwo{ \pExpect{\pdist}{ \omega_i\vxx } -  \pExpect{\pdist}{ \omega_i } \vxx\gt }^2
}^{m \normTwo{\vz}^2 \text{ with } m=|\setInliers|}
\\
\overbrace{\leq}^{\cref{fact:normIneq-pdist}}
\aveOverInliers \pExpect{\pdist}{  
\normTwo{ \omega_i\vxx -  \omega_i \vxx\gt}^2} 
 \overbrace{\leq}^{\eqref{eq:proof51}} \frac{\inlierRate^2 \eta^2 \boundx^2}{4} \mathper
 \label{eq:proof-ab-1}
\end{align}
Remembering that  $\vv_i = \frac{\pExpect{\pdist}{\omega_i\vxx}}{\pExpect{\pdist}{\omega_i}}$ if $\pExpect{\pdist}{\omega_i} > 0$, or $\vv_i = \zero$ otherwise, and taking the square root of both members 
in~\eqref{eq:proof-ab-1}: 
\begin{align}
\frac{1}{|\setInliers|} \sum_{i \in \setInliers, \pExpect{\pdist}{\omega_i} > 0} 
\pExpect{\pdist}{\omega_i} \normTwo{\vv_i - \vxx\gt} 
\overbrace{=}^{\text{by def. of $\vv_i$}} 
\aveOverInliers   
\normTwo{ \pExpect{\pdist}{ \omega_i\vxx } -  \pExpect{\pdist}{ \omega_i } \vxx\gt }
 \overbrace{\leq}^{\eqref{eq:proof-ab-1}}
 \frac{\inlierRate \eta \boundx}{2} \mathcom
\end{align}
concluding the proof of~\cref{lem:exists_vi}.
\end{proof}

\begin{proof}[{\bf Proof of Proposition~\ref{thm:lowOut-apriori-LTS}:}]
First of all, we note that since $\axiomsLTST$ in~\cref{algo:lts2} contains a superset of the constraints in $\axiomsMC$ defined in~\cref{eq:boundxxgt}, the conclusions of~\cref{lem:exists_vi} and~\cref{eq:boundxxgt} still hold if we replace $\axiomsMC$ with $\axiomsLTST$.
 Therefore we have that any pseudo-distribution of level at least $k$ (hence produced by a relaxation of order at least $k/2$) satisfying $\axiomsLTST$ also satisfies:
\begin{align}
\label{eq:boundInliers}
\aveOverInliers \pExpect{\pdist}{\omega_i} \normTwo{\vv_i-\vxx\gt} \leq \frac{\inlierRate \eta \boundx}{2}
\overbrace{\iff}^{|\InlierSet|=\an} 
\sumOverInliers \pExpect{\pdist}{\omega_i} \normTwo{\vv_i-\vxx\gt} \leq \frac{\inlierRate^2 \, \eta \, \boundx \, \nrMeasurements}{2} \mathper
\end{align}

Let us define the set of outliers $\setOutliers \triangleq [\nrMeasurements] \setminus \setInliers$.
We observe that since $\pExpect{\pdist}{\omega_i} \leq 1$, then 
$\sumOverOutliers \pExpect{\pdist}{\omega_i} \leq (1-\inlierRate)\nrMeasurements$.
Moreover, {using the triangle inequality $\normTwo{\vv_i-\vxx\gt} \leq 2 \boundx$,} hence:
\begin{align}
\label{eq:boundOutliers}
\sumOverOutliers \pExpect{\pdist}{\omega_i} \normTwo{\vv_i-\vxx\gt}  \leq 2 \nrMeasurements \boundx (1-\inlierRate) \mathper
\end{align}

Using~\cref{eq:boundInliers} and~\cref{eq:boundOutliers}:
\begin{align}
\sumOverMeas \pExpect{\pdist}{\omega_i} \normTwo{\vv_i-\vxx\gt} =  
\sumOverInliers \pExpect{\pdist}{\omega_i} \normTwo{\vv_i-\vxx\gt} 
+ \sumOverOutliers \pExpect{\pdist}{\omega_i} \normTwo{\vv_i-\vxx\gt}
\\
 \leq \frac{\inlierRate^2 \; \eta \; \boundx \; \nrMeasurements}{2} + 2 \nrMeasurements \boundx (1-\inlierRate) \mathper
 \label{eq:proof-ab-22}
\end{align}
Now note that any pseudo-distribution $\pdist$ satisfying $\axiomsLTST$ is such that 
$\pExpect{\pdist}{ \sumOverMeas \omega_i } = \an$ (due to the constraint $\sumOverMeas \omega_i = \an$ in $\axiomsLTST$),
hence by linearity $\sumOverMeas \pExpect{\pdist}{ \omega_i } = \an$.
Dividing both members of~\eqref{eq:proof-ab-22} by $\sumOverMeasj \pExpect{\pdist}{\omega_j}$ (where we switched to using $j$ as an index to avoid confusion):
\begin{align}
\sumOverMeas \frac{ \pExpect{\pdist}{\omega_i} }{ \sumOverMeasj \pExpect{\pdist}{\omega_j} } \normTwo{\vv_i-\vxx\gt} 
 \leq \frac{1}{\sumOverMeasj \pExpect{\pdist}{\omega_j} }
 \left( 
 \frac{\inlierRate^2 \; \eta \; \boundx \; \nrMeasurements}{2} + 2 \nrMeasurements   \boundx (1-\inlierRate)
 \right)
 \\
=
 \frac{1}{\an } \left( 
\frac{\inlierRate^2 \; \eta \; \boundx \; \nrMeasurements}{2} + 2 \nrMeasurements  \boundx (1-\inlierRate)
 \right) \mathper
 \label{eq:proof61T}
\end{align}
Using Jensen's inequality, we observe $\normTwo{ \sumOverMeas \frac{ \pExpect{\pdist}{\omega_i} }{ \sumOverMeasj \pExpect{\pdist}{\omega_j} } \vv_i  - \vxx\gt} \leq \sumOverMeas \frac{ \pExpect{\pdist}{\omega_i} }{ \sumOverMeasj \pExpect{\pdist}{\omega_j} } \normTwo{\vv_i-\vxx\gt}$, hence~\eqref{eq:proof61T} becomes: 
\begin{align}
\normTwo{ \sumOverMeas \frac{ \pExpect{\pdist}{\omega_i} }{ \sumOverMeasj \pExpect{\pdist}{\omega_j} } \vv_i  - \vxx\gt} 
\leq
 \frac{\inlierRate \; \eta \; \boundx }{2} +  2 \boundx \frac{1-\inlierRate}{\inlierRate} \mathcom
\end{align}
which, recalling that $\vxx\ltssdpT = \sumOverMeas \frac{ \pExpect{\pdist}{\omega_i} }{ \sumOverMeasj \pExpect{\pdist}{\omega_j} } \vv_i$, concludes the proof.

\end{proof}

\section{Proof of Proposition~\ref{thm:lowOut-apriori-MC}: Contract for Relaxation of~\eqref{eq:MC-lin}}
\label{app:proof-lowOut-apriori-MC}

We start by restating the proposition for the reader's convenience.

\noindent\fbox{
\parbox{\textwidth}{
\begin{proposition}[Restatement of Proposition~\ref{thm:lowOut-apriori-MC}]
\mcThm{eq:proof-lowOut-apriori-MC}
\end{proposition}
}}

\begin{proof}
First of all, we note that the constraint set $\axiomsMC$ in~\eqref{eq:MC-lin} is the same as~\cref{eq:boundxxgt} and~\cref{lem:exists_vi}.
 Therefore we have that any pseudo-distribution $\pdist$  of level at least $k$ (hence produced by a relaxation of order at least $k/2$) 
  satisfying $\axiomsMC$ also satisfies:
\begin{align}
\label{eq:boundInliersMC}
\aveOverInliers \pExpect{\pdist}{\omega_i} \normTwo{\vv_i-\vxx\gt} \leq \frac{\inlierRate \eta \boundx}{2}
\overbrace{\iff}^{|\InlierSet|=\an} 
\sumOverInliers \pExpect{\pdist}{\omega_i} \normTwo{\vv_i-\vxx\gt} \leq \frac{\inlierRate^2 \, \eta \, \boundx \, \nrMeasurements}{2}\mathper
\end{align}

Let us define the set of outliers $\setOutliers \triangleq [\nrMeasurements] \setminus \setInliers$.
We observe that since $\pExpect{\pdist}{\omega_i} \leq 1$, then 
$\sumOverOutliers \pExpect{\pdist}{\omega_i} \leq (1-\inlierRate)\nrMeasurements$.
Moreover, using the triangle inequality $\normTwo{\vv_i-\vxx\gt} \leq 2 \boundx$, hence:
\begin{align}
\label{eq:boundOutliersMC}
\sumOverOutliers \pExpect{\pdist}{\omega_i} \normTwo{\vv_i-\vxx\gt}  \leq 2 \nrMeasurements \boundx (1-\inlierRate)\mathper
\end{align}

Using~\cref{eq:boundInliersMC} and~\cref{eq:boundOutliersMC}:
\begin{align}
\sumOverMeas \pExpect{\pdist}{\omega_i} \normTwo{\vv_i-\vxx\gt} =  
\sumOverInliers \pExpect{\pdist}{\omega_i} \normTwo{\vv_i-\vxx\gt} 
+ \sumOverOutliers \pExpect{\pdist}{\omega_i} \normTwo{\vv_i-\vxx\gt}
\\
 \leq \frac{\inlierRate^2 \; \eta \; \boundx \; \nrMeasurements}{2} + 2 \nrMeasurements \boundx (1-\inlierRate)\mathper
 \label{eq:proof-abc-22}
\end{align}

Let us call $\pdist$ the pseudo-distribution that achieves the optimal solution in~\eqref{eq:MC-lin}, 
and observe that the corresponding optimal objective $\sumOverMeas \pExpect{\pdist}{\omega_i} \geq \an$: 
this follows from optimality of $\pdist$ and from the fact that the pseudo-distribution supported on the single point $(\vxx\gt,\vomega\gt)$, where $\omega_i\gt=1$ if $i\in\setInliers$ or zero otherwise, is feasible for~\eqref{eq:MC-lin} and achieves an objective $\an$.

Now dividing both members of~\eqref{eq:proof-abc-22} by $\sumOverMeasj \pExpect{\pdist}{\omega_j}$:
\begin{align}
\sumOverMeas \frac{ \pExpect{\pdist}{\omega_i} }{ \sumOverMeasj \pExpect{\pdist}{\omega_j} } \normTwo{\vv_i-\vxx\gt} 
 \leq \frac{1}{\sumOverMeasj \pExpect{\pdist}{\omega_j} }
 \left( 
 \frac{\inlierRate^2 \; \eta \; \boundx \; \nrMeasurements}{2} + 2 \nrMeasurements   \boundx (1-\inlierRate)
 \right)
 \\
\leq
 \frac{1}{\an } \left( 
\frac{\inlierRate^2 \; \eta \; \boundx \; \nrMeasurements}{2} + 2 \nrMeasurements  \boundx (1-\inlierRate)
 \right)\mathper
 \label{eq:proof61M}
\end{align}
Using Jensen's inequality $\normTwo{ \sumOverMeas \frac{ \pExpect{\pdist}{\omega_i} }{ \sumOverMeasj \pExpect{\pdist}{\omega_j} } \vv_i  - \vxx\gt} \leq \sumOverMeas \frac{ \pExpect{\pdist}{\omega_i} }{ \sumOverMeasj \pExpect{\pdist}{\omega_j} } \normTwo{\vv_i-\vxx\gt}$ hence~\eqref{eq:proof61M} becomes: 
\begin{align}
\normTwo{ \sumOverMeas \frac{ \pExpect{\pdist}{\omega_i} }{ \sumOverMeasj \pExpect{\pdist}{\omega_j} } \vv_i  - \vxx\gt} 
\leq
 \frac{\inlierRate \; \eta \; \boundx }{2} +  2 \boundx \frac{1-\inlierRate}{\inlierRate}\mathcom
\end{align}
which, recalling that $\vxx\mcsdp = \sumOverMeas \frac{ \pExpect{\pdist}{\omega_i} }{ \sumOverMeasj \pExpect{\pdist}{\omega_j} } \vv_i$, concludes the proof.
\end{proof}

\section{Proof of Proposition~\ref{thm:lowOut-apriori-TLS}: Contract for Relaxation of~\eqref{eq:TLS-lin}}
\label{app:proof-lowOut-apriori-TLS}

We start by restating the proposition for the reader's convenience.

\noindent\fbox{
\parbox{\textwidth}{
\begin{proposition}[Restatement of Proposition~\ref{thm:lowOut-apriori-TLS}]
\tlsThm{eq:proof-lowOut-apriori-TLS}
\end{proposition}
}}

\begin{proof}
First of all, we note that the constraint set $\axiomsTLS$ in~\eqref{eq:TLS-lin} is the same as~\cref{eq:boundxxgt} and~\cref{lem:exists_vi}. Therefore we have that any pseudo-distribution $\pdist$ 
of level at least $k$ (hence produced by a relaxation of order at least $k/2$) 
satisfying $\axiomsTLS$ also satisfies:
\begin{align}
\label{eq:boundInliersTLS}
\aveOverInliers \pExpect{\pdist}{\omega_i} \normTwo{\vv_i-\vxx\gt} \leq \frac{\inlierRate \eta \boundx}{2}
\overbrace{\iff}^{|\InlierSet|=\an} 
\sumOverInliers \pExpect{\pdist}{\omega_i} \normTwo{\vv_i-\vxx\gt} \leq \frac{\inlierRate^2 \, \eta \, \boundx \, \nrMeasurements}{2} \mathper
\end{align}

Let us define the set of outliers $\setOutliers \triangleq [\nrMeasurements] \setminus \setInliers$.
We observe that since $\pExpect{\pdist}{\omega_i} \leq 1$, then 
$\sumOverOutliers \pExpect{\pdist}{\omega_i} \leq (1-\inlierRate)\nrMeasurements$.
Moreover, using the triangle inequality $\normTwo{\vv_i-\vxx\gt} \leq 2 \boundx$, hence:
\begin{align}
\label{eq:boundOutliersTLS}
\sumOverOutliers \pExpect{\pdist}{\omega_i} \normTwo{\vv_i-\vxx\gt}  \leq 2 \nrMeasurements \boundx (1-\inlierRate) \mathper
\end{align}

Using~\cref{eq:boundInliersTLS} and~\cref{eq:boundOutliersTLS}:
\begin{align}
\sumOverMeas \pExpect{\pdist}{\omega_i} \normTwo{\vv_i-\vxx\gt} =  
\sumOverInliers \pExpect{\pdist}{\omega_i} \normTwo{\vv_i-\vxx\gt} 
+ \sumOverOutliers \pExpect{\pdist}{\omega_i} \normTwo{\vv_i-\vxx\gt}
\\
 \leq \frac{\inlierRate^2 \; \eta \; \boundx \; \nrMeasurements}{2} + 2 \nrMeasurements \boundx (1-\inlierRate) \mathper
 \label{eq:proof-abc-22}
\end{align}

Let us call $\pdist$ the pseudo-distribution that achieves the optimal solution in~\eqref{eq:TLS-lin}, and observe that $\pdist$ achieves a cost:
\begin{align}
 \pExpect{\pdist}{ \sumOverMeas \omega_i \cdot\residuali{\vxx}^2 + (1-\omega_i) \cdot \barcsq}
 =
 \sumOverMeas \pExpect{\pdist}{ \omega_i \cdot\residuali{\vxx}^2 }
 + \sumOverMeas (1-\pExpect{\pdist}{\omega_i}) \cdot \barcsq \mathper
 \label{eq:proof-aa-tls1}
\end{align}
Now observe that the pseudo-distribution supported on the single point $(\vxx\gt,\vomega\gt)$, where $\omega_i\gt=1$ if $i\in\setInliers$ or zero otherwise, is feasible for~\eqref{eq:TLS-lin} and achieves an objective $\sumOverInliers \residuali{\vxx\gt}^2 + (1-\inlierRate) \nrMeasurements \barcsq$. 
Therefore, by using~\eqref{eq:proof-aa-tls1} and by optimality of $\pdist$:
\begin{align}
\sumOverMeas \pExpect{\pdist}{ \omega_i \cdot\residuali{\vxx}^2 }
 + \overbrace{ 
 \sumOverMeas (1-\pExpect{\pdist}{\omega_i}) \cdot \barcsq 
 }^{=\nrMeasurements \barcsq - \sumOverMeas \pExpect{\pdist}{\omega_i} \barcsq} 
 \leq \sumOverInliers \residuali{\vxx\gt}^2 
 + 
 \overbrace{(1-\inlierRate) \nrMeasurements \barcsq}^{ =\nrMeasurements \barcsq - \an \; \barcsq } \mathper
\end{align}
Rearranging the terms in the previous inequality:
\begin{align}
\sumOverMeas \pExpect{\pdist}{\omega_i}  \geq 
\frac{1}{\barcsq} 
\left(
\sumOverMeas \pExpect{\pdist}{ \omega_i \cdot\residuali{\vxx}^2 }
- \sumOverInliers \residuali{\vxx\gt}^2 + \an \; \barcsq 
 \right)
 \\
  \geq 
\frac{1}{\barcsq} 
\left(
\an \; \barcsq
- \sumOverInliers \residuali{\vxx\gt}^2  
 \right)
 = \an  - \frac{1}{\barcsq} \sumOverInliers \residuali{\vxx\gt}^2  \mathper
 \label{eq:proof61T0} 
\end{align}

Now dividing both members of~\eqref{eq:proof-abc-22} by $\sumOverMeasj \pExpect{\pdist}{\omega_j}$
 and defining $\gamma\gt \triangleq  \sumOverInliers \residuali{\vxx\gt}^2$:
\begin{align}
\sumOverMeas \frac{ \pExpect{\pdist}{\omega_i} }{ \sumOverMeasj \pExpect{\pdist}{\omega_j} } \normTwo{\vv_i-\vxx\gt} 
 \leq \frac{1}{\sumOverMeasj \pExpect{\pdist}{\omega_j} }
 \left( 
 \frac{\inlierRate^2 \; \eta \; \boundx \; \nrMeasurements}{2} + 2 \nrMeasurements   \boundx (1-\inlierRate)
 \right)
 \\
\overbrace{\leq}^{\text{using~\eqref{eq:proof61T0}}}
 \frac{1}{\an - \frac{\gamma\gt}{\barcsq} } \left( 
\frac{\inlierRate^2 \; \eta \; \boundx \; \nrMeasurements}{2} + 2 \nrMeasurements  \boundx (1-\inlierRate)
 \right) \mathper
 \label{eq:proof61T}
\end{align}
Using Jensen's inequality $\normTwo{ \sumOverMeas \frac{ \pExpect{\pdist}{\omega_i} }{ \sumOverMeasj \pExpect{\pdist}{\omega_j} } \vv_i  - \vxx\gt} \leq \sumOverMeas \frac{ \pExpect{\pdist}{\omega_i} }{ \sumOverMeasj \pExpect{\pdist}{\omega_j} } \normTwo{\vv_i-\vxx\gt}$ hence~\eqref{eq:proof61T} becomes: 
\begin{align}
\normTwo{ \sumOverMeas \frac{ \pExpect{\pdist}{\omega_i} }{ \sumOverMeasj \pExpect{\pdist}{\omega_j} } \vv_i  - \vxx\gt} 
\leq
 \frac{1}{\an - \frac{\gamma\gt}{\barcsq}  } \left( 
\frac{\inlierRate^2 \; \eta \; \boundx \; \nrMeasurements}{2} + 2 \nrMeasurements  \boundx (1-\inlierRate)
 \right) \mathcom
\end{align}
which, recalling that $\vxx\tlssdp = \sumOverMeas \frac{ \pExpect{\pdist}{\omega_i} }{ \sumOverMeasj \pExpect{\pdist}{\omega_j} } \vv_i$, concludes the proof.
\end{proof}

\section{Proof of Theorem~\ref{thm:highOut-apriori}: Contract for Relaxation of~\eqref{eq:LDR}}
\label{app:proof-highOut-apriori}

We start by restating the theorem for the reader's convenience.

\noindent\fbox{
\parbox{\textwidth}{
\begin{theorem}[Restatement of Proposition~\ref{thm:highOut-apriori}]
\ldrThm{eq:proof-highOut-apriori}
\end{theorem}
}}

\begin{proof}
 Note that~\cref{eq:boundxxgt} and~\cref{lem:exists_vi} use a subset of constraints compared to the set $\axiomsLDR$ in~\eqref{eq:LDR} (\ie the set $\axiomsMC$ in the lemmas does not contain the constraint $\sumOverMeas{\omega_i} = \alpha \nrMeasurements$, while $\axiomsLDR$ does). Therefore, their conclusions will still hold in the context of~\eqref{eq:LDR}. 
We start by proving the following lemma, which 
 shows that the pseudo-distribution $\pdist$ built by optimizing the moment relaxation of~\eqref{eq:LDR} 
 ``spreads'' (\ie has enough support) across the inliers. 
 The proof is an extension of Lemma 4.3 in~\cite{Karmalkar19neurips-ListDecodableRegression} to the case
 of vector-valued measurements.

\begin{lemma}[Adapted from Lemma 4.3 in~\cite{Karmalkar19neurips-ListDecodableRegression}]  \label{lem:sum_wi}
For any pseudo-distribution $\pdist$ 
satisfying $\axiomsLDR$ that minimizes 
$\normTwo{\pExpect{\pdist}{\vomega}}^2$, $\sumOverInliers \pExpect{\pdist}{\omega_i} \geq \inlierRate^2 \nrMeasurements$.
\end{lemma}
\begin{proof}
Let $\vu = \frac{1}{\inlierRate \nrMeasurements} \pExpect{\pdist}{\vomega}$.
 Then, $\vu$ is a non-negative vector satisfying $\sumOverMeas u_i = 1$.
  Let $\wt(\setInliers) = \sumOverInliers u_i$ 
  and let $\wt(\setOutliers) = \sum_{i \in \setOutliers} u_i$, where $\setOutliers \triangleq [\nrMeasurements] \setminus \setInliers$ is the set of outliers. Then, 
  $\wt(\setInliers) + \wt(\setOutliers) = 1$. 

  By contradiction, we show that if $\wt(\setInliers) < \inlierRate$, then there exists 
  a pseudo-distribution satisfying $\axiomsLDR$ that achieves a lower value of 
  $\normTwo{\pExpect{\pdist}{\vomega}}^2$, hence contradicting optimality of $\pdist$. 
  Towards this goal, we define a pseudo-distribution $\pdist\opt$ which is supported on a single $(\vomega,\vxx)$, the indicator vector $\ones_{\setInliers}$ and $\vxx\gt$. Therefore, $\pExpect{\pdist\opt}{\omega_i}=1$ iff $i\in\setInliers$ and zero otherwise. 
  Clearly, $\pdist\opt$ satisfies $\axiomsLDR$. Therefore, any convex combination $\pdist_\lambda = (1-\lambda) \pdist + \lambda \pdist\opt$ also satisfies $\axiomsLDR$. 
  We now show that whenever $\wt(\setInliers)< \inlierRate$, then  $\normTwo{\pExpect{\pdist_\lambda}{\vomega}}^2 < \normTwo{\pExpect{\pdist}{\vomega}}^2$ for some $\lambda >0$, thus contradicting optimality of $\pdist$. We observe that:
  \begin{align}
  \vu_\lambda = \frac{1}{\an} \pExpect{\pdist_\lambda}{\vomega} = 
\frac{1}{\an} (1-\lambda)\pExpect{\pdist}{\vomega} + \frac{1}{\an} (\lambda)\pExpect{\pdist\opt}{\vomega} = (1-\lambda)\vu + \frac{\lambda}{\an} \ones_{\setInliers} \mathper
\label{eq:up}
  \end{align}

  First, we compute the squared norm of $\vu_\lambda$ using~\eqref{eq:up}:
  \begin{align}
  \normTwo{\vu_\lambda}^2 
  = 
  \overbrace{
  (1-\lambda)^2 \normTwo{\vu}^2 + 2\lambda(1-\lambda) \frac{\wt(\setInliers)}{\an} + \frac{\lambda^2}{\an} 
  }^{\text{observing $\ones_\setInliers\tran \ones_\setInliers = \an$ and $\ones_\setInliers\tran \vu = \wt(\setInliers)$}} \mathper
  \label{eq:norm_u}
  \end{align}

  Next, we lower bound $\normTwo{\vu}^2$ in terms of $\wt(\setInliers)$ and $\wt(\setOutliers)$.
  Observe that for any fixed values of  $\wt(\setInliers)$ and $\wt(\setOutliers)$, {the minimum of
  $\normTwo{\vu}^2$ is attained by the vector
   $\vu$ such that $u_i =\frac{1}{\an} \wt(\setInliers)$ for each $i\in\setInliers$ and 
   $u_i =\frac{1}{(1-\inlierRate)\nrMeasurements} \wt(\setOutliers)$ otherwise.} 
   This gives:
   \begin{align}
    \normTwo{\vu}^2 \geq 
    \overbrace{
    \left( \frac{\wt(\setInliers)}{\an} \right)^2 \an
    }^{\text{sum of $u_i^2$ for $i\in\setInliers$}}
     + 
     \overbrace{
    \left( \frac{1-\wt(\setInliers)}{(1-\inlierRate)\nrMeasurements} \right)^2 (1-\inlierRate)\nrMeasurements 
    }^{\text{sum of $u_i^2$ for $i\in\setOutliers$}}
    \\ 
    \explanation{ 
    =\frac{\wt(\setInliers)^2}{\an}  +
     \frac{( 1-\wt(\setInliers) )^2}{(1-\inlierRate)\nrMeasurements} 
     }
    = 
    \frac{1}{\an} \cdot \left( \wt(\setInliers)^2 + (1-\wt(\setInliers))^2 \left( \frac{\inlierRate}{1-\inlierRate} \right) \right) \mathper
    \label{eq:norm_up}
    \end{align} 
	Combining~\eqref{eq:norm_u} and~\eqref{eq:norm_up}: 
	\begin{align}
	\normTwo{\vu_\lambda}^2 - \normTwo{\vu}^2 = \overbrace{(-2\lambda + \lambda^2) \normTwo{\vu}^2 }^{ = (1-\lambda)^2 \normTwo{\vu}^2  - \normTwo{\vu}^2  } 
    + 2\lambda(1-\lambda) \frac{\wt(\setInliers)}{\an}
	+   \frac{\lambda^2}{\an} 
  	\\
	\overbrace{\leq}^{\text{since $(-2\lambda + \lambda^2)\leq 0$}}
	\frac{-2\lambda + \lambda^2}{\an} \left( 
	\wt(\setInliers)^2 + (1-\wt(\setInliers))^2 \frac{\inlierRate}{1-\inlierRate} 
 \right)
 + 2\lambda(1-\lambda) \frac{\wt(\setInliers)}{\an} + \frac{\lambda^2}{\an} \mathper
	\end{align}
 	%
  Rearranging (note that this part slightly differs from~\cite{Karmalkar19neurips-ListDecodableRegression}, but with the same conclusion):
  \begin{align}
  \normTwo{\vu}^2 - \normTwo{\vu_\lambda}^2 \geq 
  \frac{\lambda}{\an} 
  \left( 
  (2-\lambda) 
    \left( 
    \wt(\setInliers)^2 + (1-\wt(\setInliers))^2 \frac{\inlierRate}{1-\inlierRate} 
    \right)
   - 2(1-\lambda) \wt(\setInliers) - \lambda
  \right)
  \\
  = 
  \frac{\lambda (2-\lambda)}{\an} 
  \left(  
    \left( 
    \wt(\setInliers)^2 + (1-\wt(\setInliers))^2 \frac{\inlierRate}{1-\inlierRate} 
    \right)
   - \frac{2(1-\lambda)}{(2-\lambda)} \wt(\setInliers) - \frac{\lambda}{(2-\lambda)}
  \right)
  \\
  \explanation{
  \text{observing } 
  \frac{2(1-\lambda)}{(2-\lambda)} = \frac{2(1-\lambda)}{2(1-\lambda)+ \lambda} < 1 
  \text{ (for $0 < \lambda \leq 1$)}
  }
  \explanation{
   \text{ and } 
   \frac{1}{(2-\lambda)} \leq 1 < \frac{1-\wt(\setInliers)}{1-\inlierRate}
   \text{ (for $0 \leq \wt(\setInliers) < \inlierRate \leq 1$)}
  }
  >
  \frac{\lambda(2-\lambda)}{\an} 
    \left( 
    \wt(\setInliers)^2 + (1-\wt(\setInliers))^2 \frac{\inlierRate}{1-\inlierRate}  - \wt(\setInliers) - \frac{1-\wt(\setInliers)}{1-\inlierRate} \lambda
    \right)
    \\
    =
    \frac{\lambda(2-\lambda)}{\an} 
    \left( 
    -\wt(\setInliers) (1-\wt(\setInliers)) + (1-\wt(\setInliers))^2 \frac{\inlierRate}{1-\inlierRate} - \frac{1-\wt(\setInliers)}{1-\inlierRate} \lambda
    \right)
  \end{align}
  \begin{align}
  =
  \frac{\lambda(2-\lambda) (1-\wt(\setInliers))}{\an(1-\inlierRate)} 
    \left( 
    -\wt(\setInliers) (1-\inlierRate) + (1-\wt(\setInliers)) \inlierRate - \lambda
    \right)
    \\
  = \overbrace{
  \frac{\lambda(2-\lambda) (1-\wt(\setInliers))}{\an(1-\inlierRate)} 
  }^{\geq 0}
    \left( 
    \inlierRate -\wt(\setInliers)  - \lambda
    \right)\mathper
  \end{align}
 	Now whenever $\wt(\setInliers) < \inlierRate$, $\left(\inlierRate -\wt(\setInliers)  - \lambda\right) > 0$ for a sufficiently small $\lambda$. Thus we can choose a small enough $\lambda > 0$ such that $\normTwo{\vu}^2 - \normTwo{\vu_\lambda}^2 > 0$, which contradicts optimality of $\pdist$.
\end{proof}

Using~\cref{lem:exists_vi} and~\cref{lem:sum_wi} we can finally prove the correctness of~\cref{thm:highOut-apriori}.
Let $\pdist$ be a pseudo-distribution 
satisfying $\axiomsLDR$ that minimizes 
$\normTwo{\pExpect{\pdist}{\vomega}}^2$. 
Such a pseudo-distribution exists since the set contains at least the distribution with $\omega_i = 1$ iff $i\in\setInliers$ and $\vxx = \vxx\gt$. 

From~\cref{lem:exists_vi}, we have $\aveOverInliers \pExpect{\pdist}{\omega_i} \normTwo{\vv_i - \vxx\gt} \leq 
\frac{\inlierRate \eta \boundx}{2}$. Let 
$Z \doteq \sumOverInliers \frac{\pExpect{\pdist}{\omega_i}}{ |\setInliers| }$ (this is a normalization factor, such that $\frac{\pExpect{\pdist}{\omega_i}}{ Z |\setInliers| }$ is a valid pdf over the inliers, \ie sums up to 1). By a rescaling, we obtain:
\begin{align}
 \sumOverInliers \frac{\pExpect{\pdist}{\omega_i}}{ Z |\setInliers| } 
 \normTwo{\vv_i - \vxx\gt} \leq 
\frac{1}{Z} \frac{\inlierRate \eta \boundx}{2}\mathper
\end{align}

Using~\cref{lem:sum_wi}, $Z \geq \inlierRate$. Therefore,
\begin{align}
\label{eq:proof31}
 \sumOverInliers \frac{\pExpect{\pdist}{\omega_i}}{Z |\setInliers|} \normTwo{\vv_i - \vxx\gt} \leq 
\frac{\eta \boundx}{2} \mathper
\end{align}

Let $i \in [\nrMeasurements]$ be chosen with probability $\frac{\pExpect{\pdist}{\omega_i}}{\an}$.
 Then, we sample $i \in\setInliers$ with probability $Z \geq \inlierRate$. 

 By Markov's inequality:
\begin{align}
 \prob{\normTwo{\vv_i - \vxx\gt} \leq \eta \boundx }  = 
 \prob{\normTwo{\vv_i - \vxx\gt} \leq \eta \boundx | i \in \setInliers } \cdot 
 \overbrace{\prob{i \in \setInliers}}^{\geq \alpha}
 \\
 \geq \alpha \cdot \prob{\normTwo{\vv_i - \vxx\gt} \leq \eta \boundx | i \in \setInliers } 
 \\
 \explanation{
  \text{Markov's inequality: } \prob{X \geq a} \leq \frac{\Expect{}{X}}{a} \iff 
  \prob{X \leq a} \geq 1 - \frac{\Expect{}{X}}{a} 
 }
 \geq \alpha \left( 1 - \frac{1}{\eta \boundx} \Expect{i\in\setInliers}{ \normTwo{\vv_i - \vxx\gt} } \right) 
 = \alpha \left( 1 - \frac{1}{\eta \boundx} \sumOverInliers  \frac{\pExpect{\pdist}{\omega_i}}{Z |\setInliers|} \normTwo{\vv_i - \vxx\gt} \right)
 \\ 
 \overbrace{\geq}^{\text{using~\eqref{eq:proof31}}}
 \alpha \left( 1 - \frac{1}{\eta \boundx} \frac{\eta \boundx}{2} \right) = \frac{\alpha}{2} \mathper
\end{align}

So we concluded that   
$\prob{\normTwo{\vv_i - \vxx\gt} \leq \eta \boundx } \geq \frac{\alpha}{2}$ (this is the probability that a single draw satisfies $\normTwo{\vv_i - \vxx\gt} \leq \eta \boundx$).
Calling $S$ (as in ``success'') the event that $\normTwo{\vv_i - \vxx\gt} \leq \eta \boundx$, we get that the probability of $S$ after $m$ draws is:
\begin{align}
\prob{S_m} = 1 - \overbrace{ (1 - \prob{S})^m }^{\text{failing $m$ times}}
 \geq 1 - \left(1 - \frac{\inlierRate}{2} \right)^m
\end{align}

Finally, choosing the number of draws $m \geq \frac{N}{\inlierRate}$, we obtain
\begin{align}
\label{eq:finalIneq2}
\prob{S_m} \geq  1 - \left(1 - \frac{\inlierRate}{2} \right)^{ \frac{N}{\inlierRate} }
\end{align}
which matches the first claim in~\cref{thm:highOut-apriori}.

Now the final claim (\ie the claim that~\eqref{eq:proof-highOut-apriori} is satisfied with probability at
least $\probBoundReal$ for $\inlierRate \geq 0.01$ and $N=10$) is just a particularization of~\eqref{eq:finalIneq2} to 
the given choice of $N$. In particular, we first observe that the probability of success $1 - \left(1 - \frac{\inlierRate}{2} \right)^{ \frac{N}{\inlierRate} }$ is a non-decreasing function of $\inlierRate$.
Then we note that the function $f(\inlierRate,N) \triangleq 1 - \left(1 - \frac{\inlierRate}{2} \right)^{ \frac{N}{\inlierRate} }$ evaluated at $\inlierRate = 0.01$ and $N=10$ is such that $f(0.01,10) \geq \probBoundReal$, which concludes the proof.


\end{proof} 





\section{\SLIDESlong (\SLIDES)}
\label{app:sparse-relaxations-LDR}

In this appendix, we present a variant of~\cref{algo:ldr} that empirically returns an accurate list of estimates, as shown in~\cref{sec:experiments-ldr}. 
We start by stating the algorithm, whose pseudocode is given in~\cref{algo:slides}.

\begin{algorithm}[h!]
  \caption{\SLIDESlong (\SLIDES).\label{algo:slides}}
  \SetAlgoLined
  \KwIn{ input data $(\vy_i,\MA_i)$, $i \in [\nrMeasurements]$, inlier rate $\inlierRate$.}
  \KwOut{ list of estimates of $\vxx\gt$. }

  \tcc{Algorithm solves a relaxation of the following problem:\\
  \vspace{-5mm}
  \begin{align}
  \tag{LDR}
  \min_{\vomega,\vxx} & \normTwo{\vomega}^2, 
  \;\;
  \text{s.t.} \;\; \axiomsLDR \triangleq \left\{
  \begin{array}{cl}
  \omega_i^2 = \omega_i, \;\; i=[\nrMeasurements] \\ 
  \sumOverMeas{\omega_i} = \inlierRate \nrMeasurements  \\
  \omega_i \cdot \residuali{\vxx}^2 \leq \barcsq, \;\; i=[\nrMeasurements]\\
  \vxx \in \Domain 
  \end{array}
  \right\}
  \end{align}
  \vspace{-5mm}
  }

  \tcc{Compute matrix $\MX\opt$ by solving SDP from sparse moment relaxation}
  $\MX\opt = \scenario{solve\_sparse\_moment\_relaxation\_at\_order\_}2\,\eqref{eq:LDR}$ 
  \label{line:sparseRelax}

  \tcc{Compute list of estimates}
  
  create empty list $\List = \emptyset$

	\For{$i \in [\nrMeasurements]$}{\label{line:forAllHypotheses}
	$\vv_i = 
	  \left\{
	  \begin{array}{ll}
	  \frac{ \getEntries{\MX\opt}{\omega_i\vxx} }{ \getEntries{\MX\opt}{\omega_i} } & \text{ if $\getEntries{\MX\opt}{\omega_i}>0$}
	  \\
	  \zero & \text{ otherwise}
	  \end{array}\right.$

	     $\vxx_i = \scenario{project\_to\_}\Domain(\vv_i)$ \label{line:roundToX}

	     add $\vxx_i$ to $\List$
	}

  \Return{$\List$.}
\end{algorithm}

The proposed algorithm is very close to~\cref{algo:ldr} (and is still based on the key insights from~\cite{Karmalkar19neurips-ListDecodableRegression}), but includes three small but important changes.
First of all, instead of solving a moment relaxation of order $\relaxOrder=2$, which is still expensive for large $\nrMeasurements$,\footnote{An order-2 moment relaxation of~\eqref{eq:LDR} entails solving an SDP with a matrix of size $\nchoosek{\nrMeasurements+\dimx+2}{2}$, which is already as large as {$1830$} for $\dimx = 9$ and $\nrMeasurements = 50$, which is the typical setup considered in our experiments. In our tests, \mosek~\cite{mosek} runs out of memory when fed an SDP of size larger than 1000. } we develop a sparse relaxation (line~\ref{line:sparseRelax}). 
The sparse relaxation uses the following sparse monomial basis
\beq
\label{eq:monomialBasisRotSearch}
\vm(\vomega,\vxx) \triangleq
[1 \vcat \omega_1 \vcat \ldots \vcat \omega_\nrMeasurements \vcat \vxx \vcat \omega_1 \vxx \vcat  \ldots \vcat \omega_\nrMeasurements \vxx ], 
\eeq
which neglects other degree-2 monomials (\eg $\omega_i \cdot \omega_j$) that do not appear in problem~\eqref{eq:LDR} while still giving access to the pseudo-expectations used in~\cref{algo:ldr}.
Note that the sparse relaxation leads to SDPs of more manageable size $(\nrMeasurements+1)(\dimx+1)$.\footnote{For instance, when $\dimx = 9$ and $\nrMeasurements = 50$, the sparse relaxation leads to a more compact SDP with a moment matrix of size $510 \times 510$.} 
Note that the idea of using a sparse moment relaxation is not new (see Remark~10 in~\cite{Yang22pami-certifiablePerception}), but our relaxation is slightly different from~\cite{Yang22pami-certifiablePerception,Yang19iccv-quasar} and tailored to list-decodable estimation. 
Later in this section, we provide a derivation of the sparse moment relaxation for the rotation search problem.

The second modification is to round the estimates to the domain $\Domain$ (line~\ref{line:roundToX}).
The latter is a consequential change: we empirically noticed that the original approach in~\cref{algo:ldr} (with our sparse relaxation) produces estimates with norm close to zero, hence leading to large estimation errors. Projecting the estimates to $\Domain$ has the effect of re-normalizing the result and correcting scaling problems, enabling the compelling results in~\cref{sec:experiments-ldr}. At the end of this section we show that projecting to the domain $\Domain$ is straightforward in the rotation search problem.

Finally, the third modification with respect to~\cref{algo:ldr}, is that~\cref{algo:slides} always returns $\nrMeasurements$ hypotheses (line~\ref{line:forAllHypotheses}), rather than sampling; this makes the result deterministic and independent on the choice of number of hypotheses (which can no longer be guided by the guarantees in~\cref{thm:highOut-apriori}).\footnote{Note that we can safely discard the hypotheses corresponding to $\getEntries{\MX\opt}{\omega_i}=0$ since those are uninformative (\ie they always correspond to $\vv_i = \zero$). We only keep them in~\cref{algo:slides} for the sake of simplicity, such that the output list $\List$ has always size $\nrMeasurements$.}

\myParagraph{Sparse moment relaxation for rotation search}
Here we provide an example of sparse relaxation arising when applying \SLIDES to the rotation search problem. Let us start by tailoring the polynomial optimization problem~\eqref{eq:LDR} to rotation search:
\begin{align}
	\label{eq:ldr-rotSearch}
  \min_{\vomega,\MR} & \normTwo{\vomega}^2, 
  \;\;
  \text{s.t.} \;\; 
  \left\{
  \begin{array}{l}
  \omega_i^2 = \omega_i, \;\; i=[\nrMeasurements] \\ 
  \sumOverMeas{\omega_i} = \inlierRate \nrMeasurements  \\
  \omega_i \cdot \|\vb_i - \MR \va_i \|^2 \leq \barcsq, \;\; i=[\nrMeasurements]\\
  \MR \in \SOthree 
  \end{array}
  \right\}
\end{align}
where we substituted the residual errors ($\residuali{\vxx}^2$) with their expression in the rotation search problem ($\|\vb_i - \MR \va_i \|^2$), and where we made explicit that the domain is $\SOthree$.
Before presenting the relaxation, we reparametrize~\eqref{eq:ldr-rotSearch} using unit quaternions: 
while we could directly relax~\eqref{eq:ldr-rotSearch} following the approach we describe below, using quaternions has the benefit of (i) leading to an even smaller relaxation (since the quaternion is parametrized by $\dimx=4$ variables instead of $9$ variables needed to write a rotation matrix) and (ii) admitting a straightforward projection to the domain $\Domain$.

\begin{proposition}[Quaternion-based reformulation of~\eqref{eq:ldr-rotSearch}]
\label{prop:quaternion_formulation_rotSearch}
The polynomial optimization problem~\eqref{eq:ldr-rotSearch} can be equivalently written as: 
 \begin{align}
	\label{eq:ldr-rotSearch-q}
  \min_{\vomega,\vq} & \normTwo{\vomega}^2, 
  \;\;
  \text{s.t.} \;\; 
  \left\{
  \begin{array}{l}
  \omega_i^2 = \omega_i, \;\; i=[\nrMeasurements] \\ 
  \sumOverMeas{\omega_i} = \inlierRate \nrMeasurements  \\
  \omega_i \cdot \left( \|\vb_i\|^2 + \|\va_i\|^2 - 2 
  \trace{\MM_{ij}\tran \vq \vq\tran} \right) \leq \barcsq, \;\; i=[\nrMeasurements]
  \\
   \| \vq \|^2 = 1 
  \end{array}
  \right\}
\end{align} 
such that any optimal solution of~\eqref{eq:ldr-rotSearch-q} can be mapped back to an optimal solution of~\eqref{eq:ldr-rotSearch} and vice-versa. In~\eqref{eq:ldr-rotSearch-q}, $\MM_{ij}$ is a constant matrix whose expression  depends on $\va_i$ and $\vb_i$.  
\end{proposition} 

\begin{proof}
The proof proceeds by inspection, by reparametrizing the rotation $\MR$ with the corresponding unit quaternion. Each unit quaternion corresponds to a unique rotation, hence we replace the domain $\MR \in \SOthree$ with the constraint that the quaternion must have unit norm (\ie $\| \vq \|^2 = 1 $). 
Comparing~\eqref{eq:ldr-rotSearch} and~\eqref{eq:ldr-rotSearch-q}, we realize we only have to rewrite  the 
maximum-residual inequality constraint in~\eqref{eq:ldr-rotSearch} in the quaternion-based form in~\eqref{eq:ldr-rotSearch-q}. This derivation is largely inspired by~\cite{Yang19iccv-quasar} (which presents a similar reformulation applied to a different polynomial optimization problem), but here we present a simpler proof. 
We start by observing that the rotation matrix associated to the quaternion $\vq = [q_1 \vcat q_2 \vcat q_3 \vcat q_4]$ (in our notation, $q_4$ is the scalar part of the quaternion) is:
\begin{align}
\MR =& 
\label{eq:quatToRot1}
\matThree{
2(q_1^2 + q_4^2) - 1 & 2(q_1 q_2 - q_3 q_4) & 2(q_1 q_3 + q_2 q_4)\\
2(q_1 q_2 + q_3 q_4) & 2(q_2^2 + q_4^2) - 1 & 2(q_2 q_3 - q_1 q_4)\\
2(q_1 q_3 - q_2 q_4) & 2(q_2 q_3 + q_1 q_4) & 2(q_3^2 + q_4^2) - 1
} = \\
\label{eq:quatToRot2}
& \matThree{
q_1^2 + q_4^2 - q_2^2 - q_3^2 & 2(q_1 q_2 - q_3 q_4) & 2(q_1 q_3 + q_2 q_4)\\
2(q_1 q_2 + q_3 q_4) & q_2^2 + q_4^2 - q_1^2 - q_3^2 & 2(q_2 q_3 - q_1 q_4)\\
2(q_1 q_3 - q_2 q_4) & 2(q_2 q_3 + q_1 q_4) & q_3^2 + q_4^2 - q_1^2 - q_2^2 
}
\end{align}
where the expression in~\eqref{eq:quatToRot2} is obtained by substituting $\|\vq\|^2 = q_1^2 + q_2^2+q_3^2+q_4^2=1$ (instead of $1$) in the diagonal entries of the expression in~\eqref{eq:quatToRot1}.
Now, by inspection from~\eqref{eq:quatToRot2}, we note that:
\begin{align}
\label{eq:quatToRot_vec}
\vectorize{\MR} = \MP \cdot \vectorize{\vq \vq\tran}
\end{align} 
where:
\vspace{-10mm}

\begin{align}
\MP \triangleq 
{
\left[
\begin{array}{cccc | cccc | cccc | cccc}
1& 0& 0& 0& 0& -1& 0& 0& 0& 0& -1& 0& 0& 0& 0& 1 \\
0& 1& 0& 0& 1& 0& 0& 0& 0& 0& 0& 1& 0& 0& 1& 0 \\
0& 0& 1& 0& 0& 0& 0& -1& 1& 0& 0& 0& 0& -1& 0& 0 \\
\hline
0& 1& 0& 0& 1& 0& 0& 0& 0& 0& 0& -1& 0& 0& -1& 0 \\
-1& 0& 0& 0& 0& 1& 0& 0& 0& 0& -1& 0& 0& 0& 0& 1 \\
0& 0& 0& 1& 0& 0& 1& 0& 0& 1& 0& 0& 1& 0& 0& 0 \\
\hline
0& 0& 1& 0& 0& 0& 0& 1& 1& 0& 0& 0& 0& 1& 0& 0 \\
0& 0& 0& -1& 0& 0& 1& 0& 0& 1& 0& 0& -1& 0& 0& 0 \\
-1& 0& 0& 0& 0& -1& 0& 0& 0& 0& 1& 0& 0& 0& 0& 1 \\
\end{array}
\right]},
\quad
\vectorize{\vq \vq\tran} = 
{\scriptsize
\left[
\begin{array}{c}
q_1^2 \\ q_2 q_1 \\ q_3 q_1 \\ q_4 q_1 \\
\hline
q_2 q_1 \\ q_2^2 \\ q_3 q_2 \\ q_4 q_2 \\
\hline
q_3 q_1 \\ q_3 q_2 \\ q_3^2 \\ q_4 q_3 \\
\hline
q_4 q_1 \\ q_4 q_2 \\ q_4 q_3 \\ q_4^2 \\
\end{array}
\right]
}
\nonumber
\end{align}

Equipped with these relations,
we are now ready to rewrite the inequality constraint in~\eqref{eq:ldr-rotSearch} as in~\eqref{eq:ldr-rotSearch-q}.
We develop the squared residual $\|\vb_i - \MR \va_i \|^2$ in~\eqref{eq:ldr-rotSearch} as follows:
\begin{align}
\|\vb_i - \MR \va_i \|^2 = 
\quad & \quad 
\textExplanation{ developing the squares }
\\
\|\vb_i\|^2 + \|\va_i \|^2 - 2\vb_i\tran \MR \va_i = 
\quad & \quad
 \textExplanation{ recalling that for a scalar $a = \vectorize{a}$ }
\\
\|\vb_i\|^2 + \|\va_i \|^2 - 2 \vectorize{\vb_i\tran \MR \va_i} = 
\quad & \quad
 \textExplanation{using $\vectorize{\MA \MB \MC} =(\MC\tran \otimes \MA) \vectorize{\MB}$}
\\
\|\vb_i\|^2 + \|\va_i \|^2 - 2 (\va_i\tran \otimes \vb_i\tran) \vectorize{\MR} =
\quad & \quad
 \textExplanation{using~\eqref{eq:quatToRot_vec}}
 \\
\|\vb_i\|^2 + \|\va_i \|^2 - 2 (\va_i\tran \otimes \vb_i\tran) \MP \;  \vectorize{\vq\vq\tran} =
\quad & \quad
 \textExplanation{using $\trace{\MA\tran \MB} = 
\vectorize{\MA}\tran \vectorize{\MB}$}
 \\
\|\vb_i\|^2 + \|\va_i \|^2 - 2 \trace{\MM_{ij}\tran \vq\vq\tran }
\label{eq:result_proof_quat}
\end{align}
where $\MM_{ij}$ is a $4\times4$ matrix such that $\vectorize{\MM_{ij}} = ((\va_i\tran \otimes \vb_i\tran) \MP)\tran = \MP\tran (\va_i \otimes \vb_i)$ (in other words, $\MM_{ij}$ simply rearranges the 16 entries of the vector $\MP\tran (\va_i \otimes \vb_i)$ into a $4\times4$ matrix).
Replacing $\|\vb_i - \MR \va_i \|^2$ in~\eqref{eq:ldr-rotSearch} with~\eqref{eq:result_proof_quat} yields the inequality in~\eqref{eq:ldr-rotSearch-q}, hence proving the claim. 
\end{proof}

\clearpage

\begin{proposition}[Sparse moment relaxation of~\eqref{eq:ldr-rotSearch-q}]
The following SDP is a convex relaxation of the non-convex optimization problem~\eqref{eq:ldr-rotSearch-q}: 
\begin{align}
\label{eq:ldr-rotSearch-q-relax}
  \min_{\MX \in \sym{5(\nrMeasurements+1)}} 
  \;\;&\;\; \sumOverMeas \MX\blku{\omega_i}^2
  \\
  \text{s.t.} 
  \;\;&\;\; 
\MX\blke{\omega_i}{\omega_i} = \MX\blku{\omega_i}, \;\; i=[\nrMeasurements] \label{eq:omega2_eq_omega} \\ 
  &
  \sumOverMeas \MX\blku{\omega_i} = \inlierRate \nrMeasurements  \\
  &
  \MX\blku{\omega_i} \cdot \left( \|\vb_i\|^2 + \|\va_i\|^2 \right)
   - 2  \trace{\MM_{ij}\tran \MX\blke{\vq}{\omega_i \vq\tran} }  \leq \barcsq, \;\; i=[\nrMeasurements]
  \\
  &
  \trace{\MX\blke{\vq}{\vq\tran} }= 1 \label{eq:finalOriginalConstraint}
  \\
  \nonumber
  \\
  &
  \MX \succeq 0\label{eq:psdRotSearch}
  \\
  &
  \MX\blku{1} = 1  \label{eq:1is1}
  \\
  &
  \MX\blke{\omega_i \vq}{\omega_i \vq\tran} = \MX\blke{\vq}{ \omega_i \vq\tran} , 
  \;\; i=[\nrMeasurements] 
  \\
  &
  \MX\blku{\omega_i \vq\tran} = \MX\blke{\omega_i}{\vq\tran}, \;\; i=[\nrMeasurements] 
  \\
  &
  \MX\blku{\omega_i \vq\tran} = \MX\blke{\omega_i}{ \omega_i \vq\tran}, \;\; i=[\nrMeasurements] 
  \\
  &
  \MX\blke{\omega_i \vq}{\omega_j \vq\tran} = \MX\blke{\omega_i \vq}{\omega_j \vq\tran}\tran, \;\; i,j=[\nrMeasurements] 
  \\
  &
  \trace{ \MX\blke{\omega_i \vq}{\omega_j \vq\tran} } = \MX\blke{\omega_i}{\omega_j}, \;\; i,j=[\nrMeasurements] 
  \label{eq:wiqwjq}
\end{align} 
where we index the rows of the matrix $\MX$ according to the monomials $\vm(\vomega,\vq)$, index the columns of $\MX$ according to the monomials $\vm(\vomega,\vq)\tran \triangleq [1 \hcat \omega_1 \hcat \ldots \hcat \omega_\nrMeasurements \hcat \vq\tran \hcat \omega_1 \vq\tran \hcat  \ldots \hcat \omega_\nrMeasurements \vq\tran ]$, 
and use the notation $\MX_{[i,j]}$ to access entries of the matrix with row indexed by monomial $i$ and column indexed by monomial $j$; we also overload the notation and 
write as $\MX_{[i]}$ to denote $\MX_{[i,1]}$. 
\end{proposition} 

\begin{proof}
While the SDP appears to be quite complicated, its constraints should become apparent from the structure of the moment matrix built on the sparse monomial basis $\vm(\vomega,\vq) \triangleq [1 \vcat \omega_1 \vcat \ldots \vcat \omega_\nrMeasurements \vcat \vq \vcat \omega_1 \vq \vcat  \ldots \vcat \omega_\nrMeasurements \vq ]$:
\begin{align}
\nonumber
& \;\;\;
\grayout{
\begin{array}{cccccccccc}
1 \myspa \omega_1 \myspa \ldots \myspa \omega_\nrMeasurements \myspa\;\; \vq\tran \myspa\;\;\;\; \omega_1 \vq\tran \myspa\;\;\; \ldots \myspa\;\;\; \omega_\nrMeasurements \vq\tran
\end{array}
}
\\
\MX = \vm(\vomega,\vq) \vm(\vomega,\vq)\tran =\quad
 \grayout{
\begin{array}{c}
1 \\ \omega_1 \\ \vdots \\ \omega_\nrMeasurements \\ \vq \\ \omega_1 \vq \\ \vdots \\ \omega_\nrMeasurements \vq
\end{array}
}
&
\left[
\begin{array}{c ccc |c| cccc}
1 & \omega_1 & \ldots & \omega_\nrMeasurements & \vq\tran & \omega_1 \vq\tran & \ldots & \omega_\nrMeasurements \vq\tran\\
* &  \omega_1^2 & \ldots & \vdots & \omega_1 \vq\tran & \omega_1^2 \vq\tran & \ldots & \omega_1 \omega_\nrMeasurements \vq\tran\\
* &  *  & \ddots & \vdots & \vdots & \vdots & \vdots & \vdots\\
* &  *& \ldots & \omega_\nrMeasurements^2 & \omega_\nrMeasurements \vq\tran & \omega_1 \omega_\nrMeasurements \vq\tran & \ldots & \omega_\nrMeasurements^2  \vq\tran \\
\hline
* &  *& \ldots & * & \vq \vq\tran &  \omega_1 \vq  \vq\tran & \ldots & \omega_\nrMeasurements \vq  \vq\tran\\
\hline
* &  *& \ldots & * & * &  \omega_1^2 \vq  \vq\tran & \ldots & \omega_1 \omega_\nrMeasurements \vq  \vq\tran\\
* &  *& \ldots & * & * &  * & \ddots & \vdots\\
* &  *& \ldots & * & * &  * & \ldots & \omega_\nrMeasurements^2 \vq  \vq\tran\\
\end{array}
\right]
\label{eq:momentMatrixRotSearch}
\end{align}
where we also reported in gray the row and column indices described in the statement of the proposition.
We prove the proposition in two steps. 
First, we show how to rewrite~\eqref{eq:ldr-rotSearch-q} using the moment matrix $\MX$ in~\eqref{eq:momentMatrixRotSearch}, which leads to the objective and constraints in~\eqref{eq:ldr-rotSearch-q-relax}-\eqref{eq:finalOriginalConstraint}. Second, we show that any moment matrix with the structure in~\eqref{eq:momentMatrixRotSearch} satisfies the constraints~\eqref{eq:psdRotSearch}-\eqref{eq:wiqwjq}, hence the feasible set of~\eqref{eq:ldr-rotSearch-q-relax} contains the feasible set of~\eqref{eq:ldr-rotSearch-q}. 
Let us start by rewriting~\eqref{eq:ldr-rotSearch-q} using the moment matrix $\MX$:
  \begin{align}
  \label{eq:ldr-rotSearch-q-with-X}
  \min_{\vomega,\vq,\MX} 
  \;\;&\;\; \sumOverMeas \MX\blku{\omega_i}^2
  \\
  \text{s.t.} 
  \;\;&\;\; 
\MX\blke{\omega_i}{\omega_i} = \MX\blku{\omega_i}, \;\; i=[\nrMeasurements] \label{eq:omega2_eq_omega} \\ 
  &
  \sumOverMeas \MX\blku{\omega_i} = \inlierRate \nrMeasurements  \\
  &
  \MX\blku{\omega_i} \cdot \left( \|\vb_i\|^2 + \|\va_i\|^2 \right)
   - 2  \trace{\MM_{ij}\tran \MX\blke{\vq}{\omega_i \vq\tran} }  \leq \barcsq, \;\; i=[\nrMeasurements]
  \\
  &
  \trace{\MX\blke{\vq}{\vq\tran} }= 1
  \\
  &
  \MX = \vm(\vomega,\vq) \cdot \vm(\vomega,\vq)\tran
\end{align} 
where $\vm(\vomega,\vq) \triangleq 
[1 \vcat \omega_1 \vcat \ldots \vcat \omega_\nrMeasurements \vcat \vq \vcat \omega_1 \vq \vcat  \ldots \vcat \omega_\nrMeasurements \vq ]$, 
and we simply noticed (from inspection of~\eqref{eq:momentMatrixRotSearch}) that $\MX\blke{\omega_i}{\omega_i} = \omega_i^2$,  $\MX\blku{\omega_i} = \omega_i$, $\MX\blke{\vq}{\omega_i \vq\tran} = \omega_i \vq \vq\tran$, and $\trace{\MX\blke{\vq}{\vq\tran} } = \trace{\vq \vq\tran} = \vq\tran\vq = \|\vq\|^2$, hence~\eqref{eq:ldr-rotSearch-q-with-X} just rewrites objective and constraints in~\eqref{eq:ldr-rotSearch-q} using the entries of the moment matrix $\MX$ in~\eqref{eq:monomialBasisRotSearch}. 
Problem~\eqref{eq:ldr-rotSearch-q-with-X} is equivalent to~\eqref{eq:ldr-rotSearch-q} and is still non-convex due to the non-convexity of the constraint $\MX = \vm(\vomega,\vq) \cdot \vm(\vomega,\vq)\tran$.

Now we are only left to prove that the feasible set of~\eqref{eq:ldr-rotSearch-q-relax} contains the feasible set of~\eqref{eq:ldr-rotSearch-q-with-X}. More precisely, we prove that any matrix that satisfies $\MX = \vm(\vomega,\vq) \cdot \vm(\vomega,\vq)\tran$ also satisfies constraints~\eqref{eq:psdRotSearch}-\eqref{eq:wiqwjq} in~\eqref{eq:ldr-rotSearch-q-relax}. 
Clearly, any $\MX = \vm(\vomega,\vq) \cdot \vm(\vomega,\vq)\tran$ is such that $\MX\succeq 0$.
The rest of the constraints can be also seen to hold by simple inspection of the entries of the moment matrix~\eqref{eq:monomialBasisRotSearch} and recalling that our constraint set also imposes $\omega_i^2 = \omega_i$ (for all $i\in[\nrMeasurements]$) and $\trace{\vq\vq\tran} = \|\vq\|^2 = 1$.
Therefore, since~\eqref{eq:ldr-rotSearch-q-relax} has the same objective of~\eqref{eq:ldr-rotSearch-q-with-X}, but its feasible set includes the feasible set of~\eqref{eq:ldr-rotSearch-q-with-X}, problem~\eqref{eq:ldr-rotSearch-q-relax} is a relaxation of~\eqref{eq:ldr-rotSearch-q-with-X}. Finally, we observe that~\eqref{eq:ldr-rotSearch-q-relax} is a convex program, since it minimizes a convex cost function over the cone of positive-semidefinite matrices and subject to a set of linear constraints.
\end{proof}

\myParagraph{Rounding for rotation search}
According to~\cref{algo:slides}, after solving the sparse moment relaxation and obtaining the matrix $\MX^\star$, we build the vectors $\vv_i$ from the entries of the matrix $\MX^\star$, and then project those vectors to the domain $\Domain$. In our quaternion-based formulation of the rotation search problem (\cref{prop:quaternion_formulation_rotSearch}), $\vv_i$ are 4-dimensional vectors, while $\Domain$ is the set of unit quaternions. Hence projecting onto the domain $\Domain$ (line~\ref{line:roundToX} in~\cref{algo:slides}) only requires normalizing the vectors $\vv_i$ to have unit norm, \ie $\vxx_i = \vv_i / \|\vv_i\|$. In particular, we add $\vxx_i = \vv_i / \|\vv_i\|$ whenever $\|\vv_i\| > 0$, while we mark an estimate as invalid when $\|\vv_i\| = 0$ and disregard it from the evaluation.

\end{document}